\documentclass[Afour,sageh,times]{sagej}

\usepackage{moreverb,url}
\usepackage{url}

\usepackage{stfloats}
\usepackage{subcaption}
\usepackage{graphicx}
\usepackage{algorithm}
\usepackage{algpseudocode}

\usepackage[colorlinks,bookmarksopen,bookmarksnumbered,citecolor=blue,urlcolor=blue]{hyperref}
\usepackage{makecell}
\usepackage{multirow}
\usepackage{tabularx}

\newtheorem{problem}{\bf Problem}
\newtheorem{definition}{\bf Definition}
\newtheorem{assumption}{\bf Assumption}
\newtheorem{remark}{\bf Remark}
\newtheorem{property}{\bf Property}
\newtheorem{lemma}{\bf Lemma}
\newtheorem{theorem}{\bf Theorem}
\newtheorem{corollary}{Corollary}

\newcommand\BibTeX{{\rmfamily B\kern-.05em \textsc{i\kern-.025em b}\kern-.08em
T\kern-.1667em\lower.7ex\hbox{E}\kern-.125emX}}

\def\volumeyear{2016}
\begin{document}

\runninghead{Liu et al.}

\title{A Fast Method for Planning All Optimal Homotopic Configurations for Tethered Robots and Its Extended Applications}



\author{Jinyuan Liu\affilnum{1}, Minglei Fu\affilnum{1}, Ling Shi\affilnum{2}, Chenguang Yang\affilnum{3}, and Wenan Zhang\affilnum{1}}

\affiliation{\affilnum{1}College of Information Engineering, Zhejiang University of Technology, Hangzhou, Zhejiang, China\\
\affilnum{2}Department of Electronic and Computer Engineering, Hong Kong University of Science and Technology, Hong Kong, China\\
\affilnum{3}Department of Computer Science, University of Liverpool, Liverpool, UK
}

\corrauth{Wenan Zhang, College of Information Engineering, Zhejiang University of Technology, 18 Chaowang Road, Hangzhou, Zhejiang, China, 310014.}
\email{wazhang@zjut.edu.cn}

\begin{abstract}
	Tethered robots play a pivotal role in specialized environments such as disaster response and underground exploration, where their stable power supply and reliable communication offer unparalleled advantages. However, their motion planning is severely constrained by tether length limitations and entanglement risks, posing significant challenges to achieving optimal path planning. To address these challenges, this study introduces CDT-TCS (Convex Dissection Topology-based Tethered Configuration Search), a novel algorithm that leverages CDT Encoding as a homotopy invariant to represent topological states of paths. By integrating algebraic topology with geometric optimization, CDT-TCS efficiently computes the complete set of optimal feasible configurations for tethered robots at all positions in 2D environments through a single computation. Building on this foundation, we further propose three application-specific algorithms: i) CDT-TPP for optimal tethered path planning, ii) CDT-TMV for multi-goal visiting with tether constraints, iii) CDT-UTPP for distance-optimal path planning of untethered robots. All theoretical results and propositions underlying these algorithms are rigorously proven and thoroughly discussed in this paper. Extensive simulations demonstrate that the proposed algorithms significantly outperform state-of-the-art methods in their respective problem domains. Furthermore, real-world experiments on robotic platforms validate the practicality and engineering value of the proposed framework.
\end{abstract}

\keywords{Tethered Robots, Path Planning, Homotopy Invariant, Optimal Distance, Mobile Robots}

\maketitle

\makeatletter
\renewcommand{\thefootnote}{\arabic{footnote}}
\setcounter{footnote}{0}
\makeatother

\section{Introduction}
In recent years, autonomous robotic systems have become indispensable in a wide range of applications, from industrial automation \citep{friedrich2017efficient} and logistics to environmental monitoring \citep{du2023ai,hitz2017adaptive} and disaster response \citep{wan2022accurate}. Among these systems, tethered robots—robots connected to a fixed base or mobile platform via a physical cable—have garnered increasing attention due to their unique advantages in specific operational scenarios \citep{mcgarey2018developing}. Tethered robots are particularly valuable in high-energy-consumption and hazardous environments such as mining operations, biochemical or radiation-contaminated area monitoring, underwater inspections of known structures (e.g., offshore drilling platforms) \citep{filliung2024augmented}, and confined space navigation \citep{polzin2024into}. The presence of a physical tether enables continuous power supply and reliable data transmission \citep{mcgarey2017tslam}, which significantly enhances the robot's endurance and communication stability—critical factors in missions where failure is not an option.

However, the very feature that provides these benefits also introduces significant challenges. The limited length and potential entanglement of the tether impose strict topological constraints on the robot's motion planning \citep{chen2013history}. Unlike free-moving agents, the path of a tethered robot must not only avoid obstacles but also ensure that the tether remains unobstructed and within its maximum allowable length throughout the entire trajectory. These constraints dramatically reduce the feasible solution space and require specialized planning techniques.

Path planning, as a fundamental problem in robotics and automatic control, aims to find a collision-free and dynamically feasible continuous trajectory connecting the initial and goal states within the configuration space. Although many mature algorithms have been developed for general-purpose mobile robots \citep{liu2023path}, including sampling-based methods \citep{gammell2021asymptotically} and graph-search techniques, they often fail to perform effectively when directly applied to tethered robots. The inherent topological constraints introduced by the tether render many traditional planners inefficient or even inapplicable, necessitating the development of novel theoretical frameworks and algorithmic solutions tailored specifically to this class of constrained robotic systems.

In this study, we focus on the path planning problem for tethered robots operating in a known 2D workspace with static obstacles, under the constraint of limited tether length. Given the intrinsic topological relationship between the robot's motion path and its evolving tether configuration \citep{munkres2018elements}, we argue that efficiently identifying all potential feasible configurations at key positions is critical to solving a wide range of planning tasks involving tethered systems. To this end, this work addresses a fundamental and challenging problem: how to rapidly compute the complete set of optimal and feasible configurations for a tethered robot at every location in the environment, given a fixed anchor point and maximum tether length. Once this configuration space is precomputed, it can be leveraged to enable fast and efficient solutions to various downstream planning problems involving tethered robots.

Homotopy invariants, as a fundamental concept in topology, have become central to the study of path planning for tethered robots \citep{bhattacharya2018path}. Their primary role lies in distinguishing the homotopic relationships among different curves (e.g., paths and tether configurations) in the workspace.  Existing approaches often index robot configurations by combining the robot's position with homotopy invariants of the tether curve, such as $h$-signatures \citep{bhattacharya2012topological,bhattacharya2015topological} and winding numbers \citep{pokorny2016high,chipade2024withy}. Prior to executing classical optimal graph search algorithms, these methods typically require precomputing the workspace into a so-called Homotopy-Augmented Graph (HAG) \citep{kim2014path,mccammon2017planning}, where all nodes represent robot configurations that are feasible under tether length constraints. Although such topological representations of configuration space are mathematically elegant, they often suffer from significant computational inefficiencies. The exhaustive preprocessing of the entire workspace can be prohibitively expensive, as the number of valid configurations grows dramatically—far exceeding the number of collision-free positions in the 2D environment. Moreover, a large portion of these configurations may never be used by the planner, leading to unnecessary resource consumption \citep{yang2022efficient}. Additionally, these methods heavily rely on the discretization level of the workspace: excessive discretization leads to combinatorial explosion during preprocessing, while insufficient discretization compromises spatial resolution and may fail to capture critical configurations accurately.

To address these challenges, this work builds upon our previously proposed homotopy invariant framework based on Convex Dissection Topology (CDT) \citep{liu2023homotopy, liu2023homotopyarXiv}, integrating concepts from algebraic topology and geometric graph theory. We introduce a novel algorithm, termed CDT-TCS (Convex Dissection Topology-based Tethered Configuration Search), which enables the rapid identification of the complete set of optimal and feasible configurations for a tethered robot at every location in the environment. On this basis, we further explore several extended applications and propose corresponding algorithms: CDT-TPP, CDT-TMV, and CDT-UTPP. These extensions efficiently address key problems in optimal tethered path planning, multi-goal visiting, and general distance-optimal path planning, respectively. The main contributions of this study are as follows:
\begin{enumerate}
    \item {We further refine and extend our previously proposed homotopy invariant theoretical framework based on Convex Dissection Topology, applying it to the field of path planning for tethered robots. This overcomes the limitations of traditional topological representations.}
    \item {We propose the CDT-TCS algorithm—a novel optimal configuration search method for tethered robots—capable of efficiently determining the complete set of optimal feasible configurations at any position in the space under given tether length constraints.}
    \item {Building upon CDT-TCS, we extend our research to address several key application problems, including optimal tethered path planning, tethered multi-goal visiting, and general distance-optimal path planning. For each of these problems, we introduce efficient solutions: CDT-TPP, CDT-TMV, and CDT-UTPP, respectively.}
    \item {All theoretical results and propositions presented in this paper are rigorously proven and thoroughly discussed, providing a new perspectives for future research in this area.}
    \item {Extensive simulations and real-world experiments are conducted to validate the effectiveness and advantages of the proposed theoretical framework and algorithms.}
\end{enumerate}

This paper is structured as follows. Section 2 provides a comprehensive review of related literature in the field. Section 3 introduces the necessary mathematical definitions, problem formulations, and our previously proposed homotopy invariant framework that serves as the theoretical foundation for this study. Section 4 presents the theoretical analysis and proofs related to optimal tethered configuration search, and introduces the CDT-TCS algorithm. Building upon the results of Section 4, Section 5 explores several extended applications and proposes the corresponding algorithms: CDT-TPP, CDT-TMV, and CDT-UTPP. Section 6 evaluates the proposed algorithms through extensive simulations, comparing their performance against state-of-the-art planners to demonstrate their efficiency. Subsequently, Section 7 deploys the proposed algorithms on real robotic platforms and conduct physical experiments to further validate their practicality and engineering value. Finally, Section 8 concludes the paper and discusses potential directions for future research.

\section{Related Work}

\subsection{Optimal Path Planning}
Graph-based search and sampling-based methods are two of the most widely adopted paradigms for path planning in robotics. Graph-search algorithms, such as Dijkstra's algorithm \citep{dijkstra1959note} and A* \citep{hart1968formal}, address path planning problems by applying graph theory to discretized state spaces. These methods rely heavily on the resolution of the discretization and suffer from exponential increases in computational cost as the problem scale grows, making them less suitable for high-dimensional or continuous environments. Sampling-based algorithms, such as Rapidly-exploring Random Trees (RRT) \citep{lavalle2001randomized} and Probabilistic Roadmaps (PRM) \citep{kavraki1996probabilistic}, circumvent the limitations of grid-based discretization by randomly sampling the continuous configuration space. These approaches are particularly effective in high-dimensional spaces and can efficiently find feasible paths in complex environments. Moreover, asymptotically optimal variants, such as RRT* and PRM* \citep{karaman2011sampling}, are probabilistically complete and asymptotically optimal. This means that as the number of samples approaches infinity, the probability of finding the optimal solution converges to one. However, a well-known drawback of these methods is their slow convergence to the optimal solution, often requiring a large number of iterations before producing near-optimal paths.

In response to the aforementioned challenges, several mainstream improvements at the planner level have emerged in recent years: (i) Leveraging information from currently known paths to narrow the sampling region, thereby focusing exploration on areas that are more likely to contain improved solutions \citep{gammell2018informed,choudhury2016regionally,chintam2024informed}. (ii) Employing batch sampling instead of incremental expansion, combined with heuristic ranking or lazy evaluation strategies to enhance the overall search efficiency \citep{janson2015fast,gammell2020batch,strub2022adaptively}. (iii) Applying post-processing optimization techniques to newly generated paths (or subpaths) in order to rapidly improve path quality and accelerate convergence toward near-optimal solutions \citep{ratliff2009chomp,jeong2019quick,miao2021path,tan2025manipulability}. Additionally, with the rapid advancement of artificial intelligence, (iv) learning-based planning methods, such as imitation learning \citep{yang2025motor,yang2025constrained,tan2024metaheuristic} and reinforcement learning \citep{chen2022deep,vashisth2024deep}, have demonstrated powerful performance in handling complex and dynamic environments. However, these methods typically require substantial computational resources for both training and execution, and they often lack theoretical guarantees of optimality. This limitation is particularly critical for safety-critical applications, such as tethered robot planning, where reliability and efficiency are paramount.

In this study, we primarily adopt the third paradigm (iii) to efficiently obtain optimal paths (or configurations). Specifically, our approach first identifies the homotopy class that contains the optimal path, along with any representative path within that class. Subsequently, a geometry-driven optimization process is applied to this path, enabling it to converge rapidly to the optimal path within the given homotopy class.

\subsection{Hierarchical Path Planning Based on Spatial Abstraction}

Hierarchical path planning based on spatial abstraction is a widely adopted strategy to enhance planning efficiency in complex environments. This approach improves scalability by abstracting the original high-dimensional or fine-grained state space into a coarser, lower-dimensional representation, enabling planning to be performed at multiple levels of abstraction and ultimately synthesizing a complete path \citep{sacerdoti1974planning}. Depending on the abstraction methodology employed, spatial abstraction-based planning can be broadly categorized into the following three classes: (i) Grid-based Spatial Abstraction methods \citep{han2022grid,sun2024accelerated,voros2001low} discretize the continuous environment into grids (e.g., 2D occupancy grids or octrees), where each cell represents an abstract state. High-level planners operate on coarse grids to generate rough paths, while low-level planners refine these paths using finer-resolution grids. (ii) Topological Abstraction methods \citep{ge2011simultaneous,an2024etpnav,chen2021topological,sun2022multi,yang2024rampage} extract key topological features of the environment—such as corridors, rooms, and passages in real-world settings, or subgraphs and subtrees in graph representations—and use them as high-level nodes to form a topological map for global planning. (iii) Geometric Feature-based Abstraction methods \citep{lozano1979algorithm,chi2021generalized,wen2024g,liu2025cdrt} construct hierarchical representations based on geometric properties such as obstacle distribution and free-space shape. These approaches are particularly effective at identifying narrow passages and improving planning success rates. Moreover, they allow planners to exploit geometric information during path generation, significantly enhancing both planning efficiency and path quality.

In this study, we adopt an environment representation based on convex decomposition and build upon it to construct a high-level topological abstraction (Subsection 3.3). This enables efficient hierarchical path planning that leverages both geometric and topological insights.

\subsection{Path Planning Considering Topology}

Various topological tools have been introduced into robotic motion planning to address challenges such as computing distinct non-homotopic paths, enforcing constraints within specific homotopy classes, and guiding the motion of tethered robots. These approaches leverage topological invariants to enhance planning efficiency and ensure path feasibility under complex constraints.

In \citep{pokorny2016high}, a topology-driven motion planning framework was proposed that combines 2D projection of high-dimensional configuration spaces with persistent homology analysis. By projecting the high-dimensional space onto a 2D plane and identifying key topological features (winding centers) the method enables efficient exploration of diverse homotopy classes without explicitly modeling the full-dimensional space. \citep{bhattacharya2012topological} explored topological constraints in search-based planning by introducing the concept of $h$-signatures, which are homology-based descriptors for characterizing path topology. In 2D and 3D environments, these signatures were derived using principles from complex analysis and electromagnetic theory, respectively. Recently, the authors of \citep{bhattacharya2012topological} further proposed a neighborhood-incremental graph structure in \citep{sahin2024topo}, allowing the computation of $k$ geodesic paths that differ not only in cost but also in their topological properties. This approach avoids the need for complex geometric constructions and is not limited to predefined homotopy classes. Another notable contribution is presented in \citep{yang2024tree}, where a novel tree-like structure was introduced to represent homotopy-enriched graphs in 2D environments. This structure facilitates topological reasoning while leveraging geometric information, enabling the planner to incrementally explore all possible topological path variations from the start point until $k$ optimal non-homotopic paths are identified.

In our previous work \citep{liu2023homotopy}, we proposed a topology-aware graph construction method based on convex dissection, offering a compact representation of environmental connectivity. Based on this structure, we introduced a new homotopy invariant termed CDT Encoding. Unlike $h$-signatures and winding numbers methods, CDT Encoding not only captures topological relationships between paths but also actively guides the planning process through its underlying convex dissection graph. For instance, in CDT-RRT* \citep{liu2023homotopy}, we designed an efficient sampling strategy based on CDT Encoding that encourages active exploration of unvisited homotopy classes, combined with an elastic band optimization to extract the shortest path within each class. In CDT-Dijkstra \citep{liu2023cdt}, we further demonstrated how CDT Encoding can be used to label and direct the expansion of homotopy-equivalent paths, enabling fast planning of globally optimal paths for all points in 2D continuous space.

\subsection{Tethered Robot Path Planning}

Early research on path planning tasks for tethered robots primarily focused on multi-robot planning in environments without independent obstacles \citep{sinden1990tethered,hert1995moving,hert2002motion,hert1996ties}. In these studies, each robot was tethered to a point on the boundary of the environment. The tether was assumed to be non-crossable, meaning that as robots moved toward their target positions, their tethers would bend upon encountering other robots. From a topological perspective, this problem can be formulated as a classical Planar Graph Embedding problem \citep{gross2001topological}. Specifically, by treating the anchor points and target positions of the robots as vertices, and the potential paths between them along with the environmental boundaries as edges, the problem reduces to determining whether the resulting graph can be embedded in a plane such that no edges intersect. The solution methods for such problems are well-established in topology and have been implemented in various commercial applications. For instance, several popular puzzle games, such as `Flow Free' and `Same Color: Connect the Dots', are based on the principles of Planar Graph Embedding. These games require players to connect pairs of points in the environment without crossing lines, effectively simulating the constraints imposed by tethered robots in simplified environments.

The development of planning algorithms for single tethered robots typically focuses on guiding the robot to navigate around obstacles and reach the target while satisfying tether length constraints \citep{cao2023neptune}. In general planar environments, determining the optimal motion of a tethered robot hinges on exploring its valid configuration space. In \citep{kim2014path}, a homotopy-augmented graph was constructed using a fixed-radius wavefront propagation algorithm in a grid-based environment to represent the configuration space of the tethered robot. Optimal tethered paths were then searched within this graph. In subsequent work, the authors eliminated the pre-construction step by leveraging homotopy-based heuristics combined with a multi-heuristic A* algorithm \citep{kim2015path}. However, the multi-heuristic A* approach achieved near-optimal solutions but could not guarantee global optimality. In addition to discretizing the environment into grids, sampling-based representations of the workspace have also been explored \citep{wang2018topological}. Building on these foundational works, recent years have seen the emergence of a series of novel planning problems and intriguing research directions related to tethered robots \citep{mccammon2017planning,mechsy2017novel,shapovalov2020exploration}.

\section{Previews}
\subsection{Definitions}
\begin{definition}[Path and Path Space]
	\label{defPath}
	Let $X$ be a continuous space, and the path in $X$ is a continuous function $\sigma:I\to X$ from the interval $I=[0,1]$ to $X$. For the subpaths we use $\sigma_{[t_1,t_2]}$, that is:
	\begin{equation}
		\label{eqSubPath}
		\begin{aligned}
			 & \sigma_{[t_1,t_2]}(t) = \sigma(t_2t - t_1t + t_1), \\
			 & \begin{array}{r@{\quad}r@{}l@{\quad}l}
				   \text{s.t.  }t_1, t_2 \in I.
			   \end{array}
		\end{aligned}
	\end{equation}
	The set of all paths in $X$ are denoted as path space $P(X)$. The set of all paths starting at $x_0$ and ending at $x_1$ are denoted as $P(X;x_0,x_1)$. Furthermore, for a subset $X' \subset X$ the set of all paths starting at $x$ and ending in $X'$ is defined as:
	\begin{equation}
		\label{eqPathSpase2}
		P(X;x, X') = \bigcup_{x'\in X'} P(X;x, x').
	\end{equation}
\end{definition}

\begin{definition}[Straight-line Path]
	\label{defLine}
	The symbol $l$ is used to denote the straight-line path between two points. For a straight-line path with start point $x_1$ and end point $x_2$, it is denoted as $l_{x_1}^{x_2}$, defined as:
	\begin{equation}
		\label{eqLine}
		l_{x_1}^{x_2}(t) = x_1(1-t) + x_2t,
	\end{equation}
	where $t \in I$.
\end{definition}

\begin{definition}[Product of Paths]
	\label{defProPath}
	For $\sigma_1 \in P(X;x_0,x_1)$, $\sigma_2 \in P(X;x_1,x_2)$, let $\sigma_1 * \sigma_2 \in P(X;x_0,x_2)$ denote their product (concatenation),
	\begin{equation}
		\label{eqCoP}
		{\sigma_1 * \sigma_2(t)} =
		\begin{cases}
			\sigma_1(2 t),   & t \in[0,0.5),  \\
			\sigma_2(2 t-1), & t \in[0.5,1].
		\end{cases}
	\end{equation}
	\textup{For convenience, in this study, we extend the definition of the equivalence relation `$=$' in $(P(X), *)$ as follows:}

	Let $\sigma_1, \sigma_2 \in P(X)$. If there exists a continuous mapping $\varphi : I \to I$ with $\dot{\varphi} \geq 0$ such that for all $t \in I$, $\sigma_1 (t) = \sigma_2(\varphi(t))$, then $\sigma_1$ and $\sigma_2$ are said to be equivalent path, denoted as $\sigma_1 = \sigma_2$. $P(X)$, $*$ and $=$ have the following properties:

	\begin{enumerate}
		\item[1)] {\textbf{Associativity:} If $\sigma_1 * \sigma_2$ and $\sigma_2 * \sigma_3$ are meaningful, then
		      \begin{equation}
			      (\sigma_1 * \sigma_2) * \sigma_3 = \sigma_1 * (\sigma_2 * \sigma_3).
		      \end{equation}
		      }
		\item[2)] {\textbf{Identity:} Given $x\in X$, let $e_x: I\to x$. If $\sigma\in P(X;x_0,x_1)$, then
		      \begin{equation}
			      e_{x_0} * \sigma = \sigma\quad \text{and} \quad \sigma * e_{x_1} = \sigma.
		      \end{equation}
		      }
	\end{enumerate}
\end{definition}

\begin{property}
	\label{th_eqRelation} 
	The `$=$' introduced by \textbf{Definition~\ref{defProPath}} is an equivalence relation, i.e., `$=$' satisfies the following conditions:
	\begin{enumerate}
		\item[1)] {\textbf{Reflexivity:} $\forall \sigma \in P(X)$, $\sigma=\sigma$.}
		\item[2)] {\textbf{Symmetry:} $\forall \sigma_1,\sigma_2 \in P(X)$, $\sigma_1=\sigma_2$ if and only if $\sigma_2=\sigma_1$.}
		\item[3)] {\textbf{Transitivity:} $\forall \sigma_1,\sigma_2,\sigma_3 \in P(X)$, if $\sigma_1=\sigma_2$ and $\sigma_2=\sigma_3$ then $\sigma_1=\sigma_3$.}
	\end{enumerate}
\end{property}
\begin{proof}
	See Appendix 3.1.\qed
\end{proof}

\begin{definition}[Cost of Paths]
	\label{defCostPath}
	The cost function $c:P(X)\to [0,\infty)$ is defined as follows:
	\begin{equation}
		\label{eqPCost}
		c(\sigma)=\lim _{n \to \infty}\sum_{\tau=1}^{n}\left\lVert  \sigma\left(\frac{\tau}{n}\right)-\sigma\left(\frac{\tau-1}{n}\right)\right\rVert,
	\end{equation}
	where $\|\cdot\|$ is the 2-Norm in Euclidean space. Additionally, the cost function $c$ satisfies the following property:
	\begin{equation}
		\label{eqPPCost}
		c(\sigma_1*\sigma_2) = c(\sigma_1)+c(\sigma_2).
	\end{equation}
\end{definition}

\begin{figure*}[htbp]
	\centering
    \begin{minipage}[c]{5.4in}
        \subfloat[]{\includegraphics[width=\linewidth]{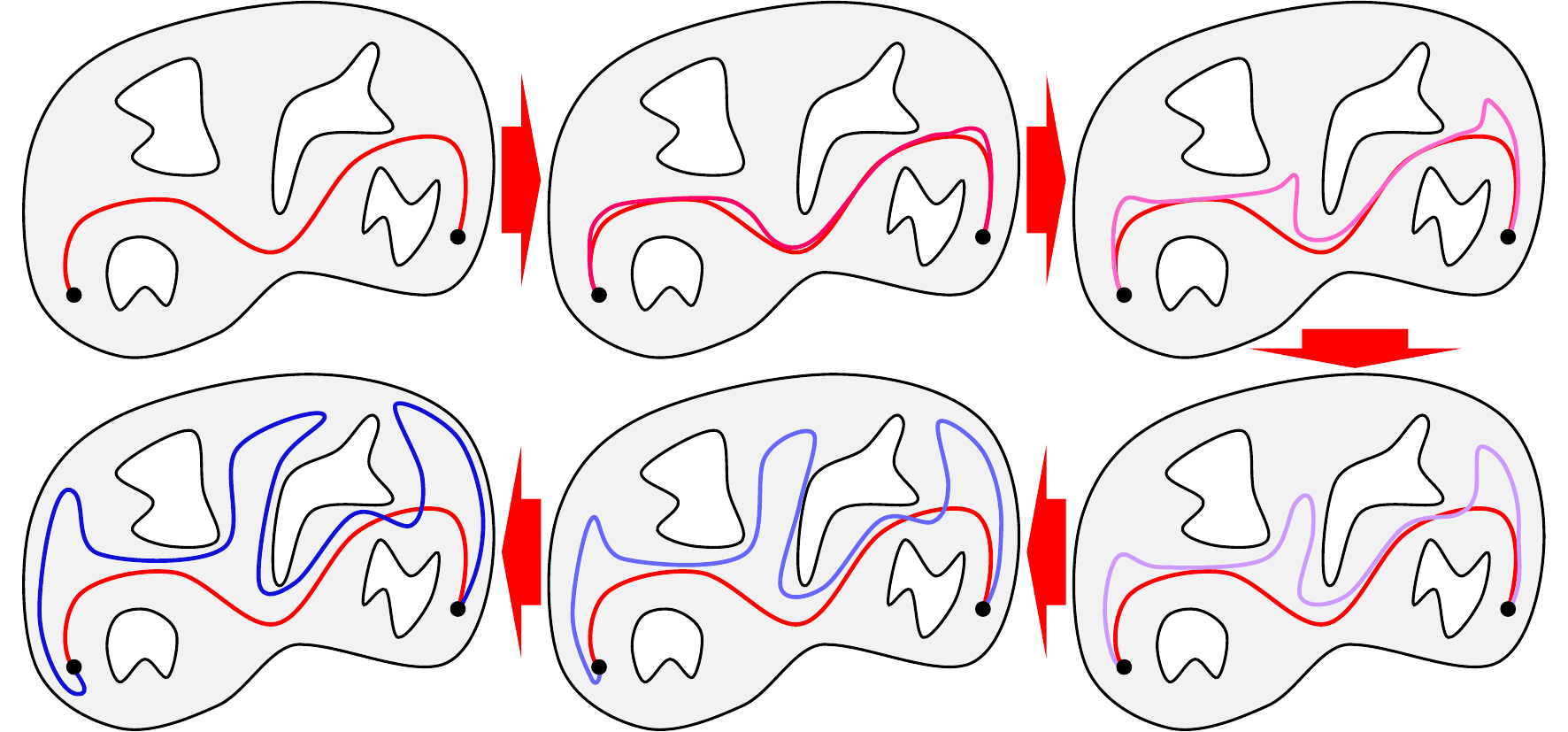}}
    \end{minipage}
    \hfil
    \begin{minipage}[c]{1.25in}
        \subfloat[]{\includegraphics[width=\linewidth]{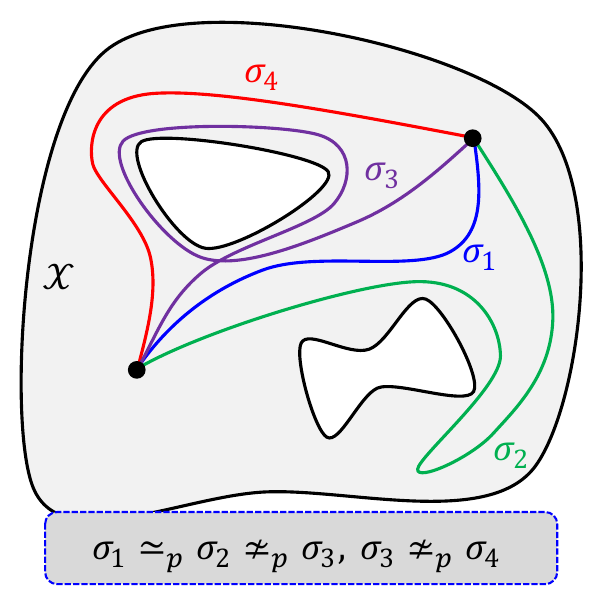}} \\
        \subfloat[]{\includegraphics[width=\linewidth]{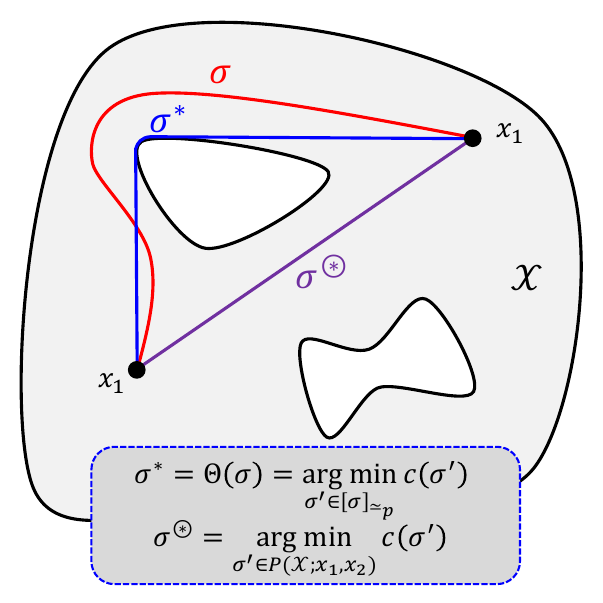}}
    \end{minipage}
	\caption{Illustration of path homotopy. (a) The red path can be continuously deformed into the blue path, indicating that the red and blue paths are homotopic. (b) Examples of homotopic and non-homotopic paths. (c) Illustration of the optimal homotopic path and the globally optimal path.}
	\label{fig_HALO}
\end{figure*}

\begin{definition}[Path Homotopy]
	\label{defPathHomotopy}
	Let $\sigma_1, \sigma_2 \in P(X;x_0,x_1)$. If there is a continuous map $F:I^2 \to X$ such that $F(t,0)=\sigma_1(t)$, $F(t,1)=\sigma_2(t)$, $F(0,\tau)=x_0$, $F(1,\tau)=x_1$, then we say $\sigma_1$ and $\sigma_2$ are path homotopic, denoted as $\sigma_1 \simeq_p \sigma_2$, as shown Fig.~\ref{fig_HALO}(a)(b). The equivalence class of a path $\sigma$ under this relation is called the homotopy class of $\sigma$, often denoted $[\sigma]_{\simeq_p}$.
\end{definition}

The groupoid $(P(X), *, \simeq_p)$ formed by this relation has the following properties:
\begin{enumerate}
	\item[1)] {\textbf{Associativity:} If $\sigma_1 * \sigma_2$ and $\sigma_2 * \sigma_3$ are meaningful, then
			\begin{equation}
				(\sigma_1 * \sigma_2) * \sigma_3 \simeq_p \sigma_1 * (\sigma_2 * \sigma_3).
			\end{equation}
			}
	\item[2)] {\textbf{Identity:} Given $x\in X$, let $e_x: I\to x$. If $\sigma\in P(X;x_0,x_1)$, then
			\begin{equation}
				e_{x_0} * \sigma \simeq_p \sigma\quad \text{and} \quad \sigma * e_{x_1} \simeq_p \sigma.
			\end{equation}
			}
	\item[3)] {\textbf{Inverse:} Given $\sigma \in P(X;x_0,x_1)$, we define the inverse of $\sigma$ as $\overline{\sigma}(t) = \sigma(1-t)$, then
			\begin{equation}
				\sigma * \overline{\sigma} \simeq_p e_{x_0} \quad \text{and} \quad \overline{\sigma} * \sigma \simeq_p e_{x_1}.
			\end{equation}
			}
\end{enumerate}

\begin{definition}[Optimal Homotopic Path\footnote{In conventional path planning research, such paths are often referred to as locally optimal paths.}]
	\label{defOptHomotopyPath}
	For a path $\sigma \in P(X)$, the optimal homotopic path $\sigma^*$ is defined as the path with the minimum cost within its homotopy class $[\sigma]_{\simeq_p}$, i.e.,
	\begin{equation}
		\sigma^* =  \underset{\sigma' \in [\sigma]_{\simeq_p}}{\mathrm{arg\,min}}\ c(\sigma').
	\end{equation}
	For convenience, we define a mapping $\Theta: P(X) \to P(X)$ that maps each path to its optimal homotopic path, i.e.,
	\begin{equation}
		\Theta(\sigma) =  \underset{\sigma' \in [\sigma]_{\simeq_p}}{\mathrm{arg\,min}}\ c(\sigma').
	\end{equation}
\end{definition}

\begin{definition}[Homotopy Invariant of Path]
	\label{defHomotopyInv}
	$H$ is a homotopy invariant of paths in $P(X)$ if and only if for any $\sigma_1$ and $\sigma_2 \in P(X)$, $H(\sigma_1) = H(\sigma_2) \Leftrightarrow \sigma_1 \simeq_p \sigma_2$ and $H(\sigma_1) \neq H (\sigma_2) \Leftrightarrow \sigma_1 \not\simeq_p \sigma_2$.
\end{definition}

\subsection{Problem Formulation}
This study aims to investigate the path planning problem for tethered agents operating within a bounded two-dimensional space. The problem can be formulated as follows:

Let a simply connected $X \subset \mathbb{R}^2$ represent the environmental space, where the obstacle and free spaces are denoted as $X_{obs}$ and $X_{free}=X\backslash X_{obs}$, respectively. The maximum length of the tether connecting the agent to its anchor point is denoted by $\zeta$. Considering that the tether can be viewed as a continuous function connecting the tether anchor point $x_\star$ to the agent's position $x_{agent}$, the feasible configuration space of the tethered agent can be represented in a form similar to a path space, as follows:
\begin{equation}
	\label{eqDefConfigure}
	\begin{aligned}
		\mathcal{C}_{x_\star ,\zeta} = \left\{\sigma \in P(X_{free};x_\star,X_{free}) \middle| c(\sigma)\leq \zeta \right\}.
	\end{aligned}
\end{equation}
To distinguish between the agent's movement path and its configuration, we use the symbol $\varsigma \in \mathcal{C}_{x_\star ,\zeta}$ to represent the feasible configuration of the tethered agent. Here, $\varsigma(0) = x_\star$ and $\varsigma(1) = x_{agent}$.
\begin{remark}
	\label{re_PeqC}
	In this study, based on the definition of the tethered robot configuration, its mathematical essence is also a special type of path. Therefore, any theorem or statement related to paths in this paper is equally applicable to configurations.
\end{remark}

\begin{assumption}
	\label{ass1}
	The tether can be dragged or tightened freely within the free space $X_{free}$ without any resistance.
\end{assumption}
It is important to note that the deformation or tightening process of the tether\footnote{During this process, the tether must not break or splice.} can be viewed as a homotopy transformation as defined in \textbf{Definition~\ref{defPathHomotopy}}. This implies that when the agent is stationary, different configurations of the tether resulting from its deformation belong to the same homotopy class. Based on this, we can define the concept of an optimal homotopic configuration.
\begin{definition}[Optimal Homotopic Configuration]
	\label{defOHConfiguration}
	For any configuration $\varsigma$ of a tethered agent, its optimal homotopic configuration, denoted as $\varsigma^*$, is defined as:
	\begin{equation}
		\label{eqOHConfiguration}
		\varsigma^* = \Theta(\varsigma).
	\end{equation}
\end{definition}

The primary problem considered in this study is as described in \textbf{Problem~\ref{proTCS}}. Additionally, based on this research, we further explore extended applications of \textbf{Problem~\ref{proTCS}} (\textbf{Problem~2-4}). Below, we will introduce \textbf{Problem~1-4} with reference to Fig.~\ref{fig_Pro}.

\begin{figure*}[!t]
	\centering
	\subfloat[]{\includegraphics[width=1.65in]{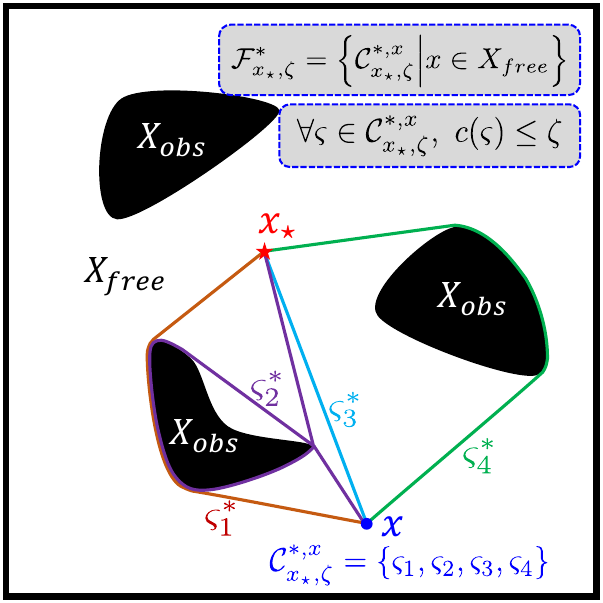}}
	\hfil
	\subfloat[]{\includegraphics[width=1.65in]{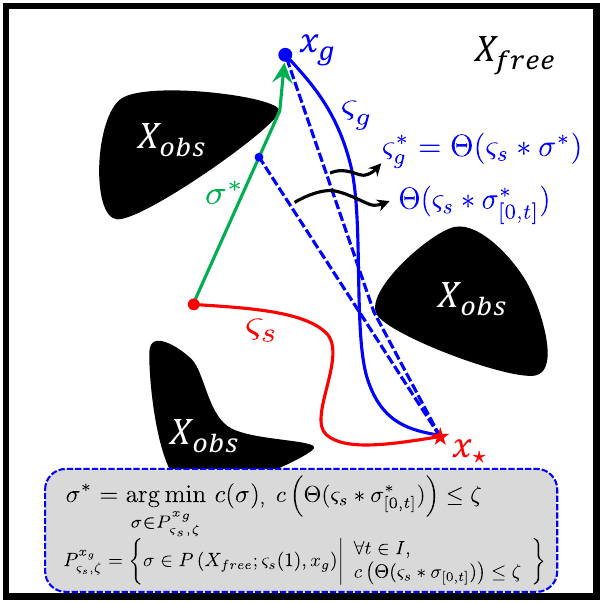}}
	\hfil
	\subfloat[]{\includegraphics[width=1.65in]{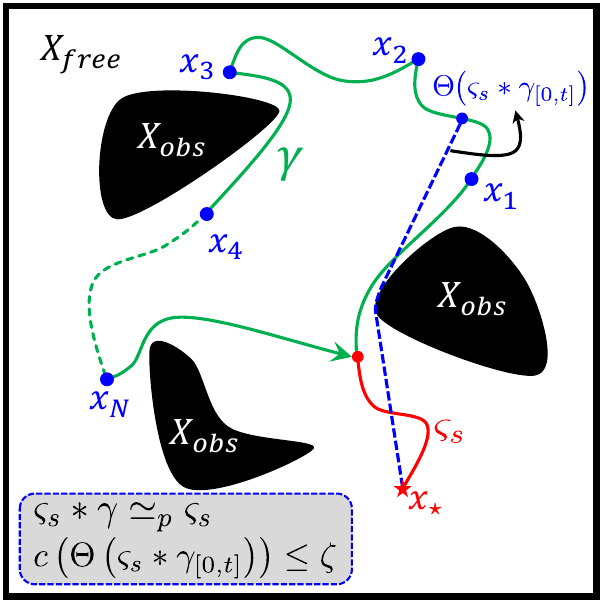}}
	\hfil
	\subfloat[]{\includegraphics[width=1.65in]{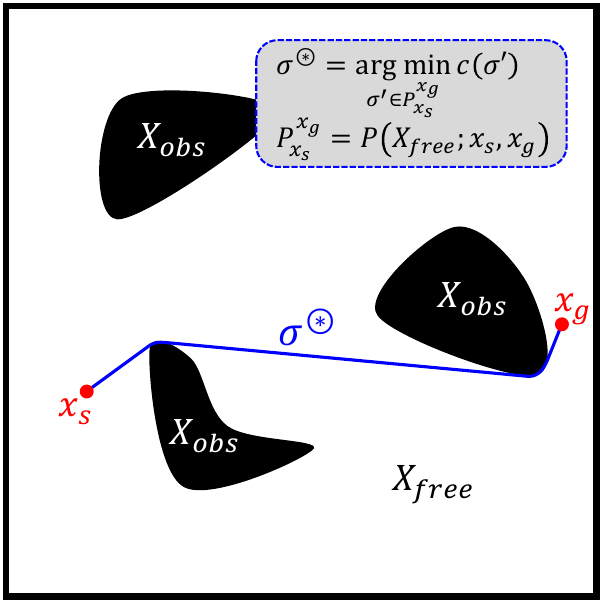}}
	\caption{Illustration of the four problems addressed in this study. (a) Optimal Tethered Configuration Search. (b) Optimal Tethered Path Planning. (c) Tethered Multi-Goal Visiting. (d) Fast Optimal Untethered Path Planning.}
	\label{fig_Pro}
\end{figure*}

\begin{problem}[Optimal Tethered Configuration Search, TCS]
\label{proTCS}
For a tethered robot with a tether length $\zeta$, given an anchor point $x_\star$ and a position $x$,  the objective is to quickly determine the set of all feasible optimal configurations $\mathcal{C}^{*,x}_{x_\star ,\zeta}$ of the tethered robot at the position, defined as:
\begin{equation}
	\label{eqProb1}
	\mathcal{C}^{*,x}_{x_\star ,\zeta} = \left\{ \Theta(\varsigma) \middle| \varsigma\in \mathcal{C}_{x_\star ,\zeta}, \varsigma(1)=x\right\}.
\end{equation}
While the above focuses on determining the optimal configuration set for a specific position $x$, the emphasis of this study is on how to efficiently find the optimal configuration set for the tethered agent for all points in the space at once, given a fixed anchor point $x_\star$. Specifically, for a given $x_\star$ and $\zeta$, determine a family of sets containing the optimal configurations for all possible target positions, denoted as $\mathcal{F}^*_{x_\star ,\zeta}$:
\begin{equation}
	\label{eqProb1Family}
	\mathcal{F}^*_{x_\star ,\zeta} = \left\{ \mathcal{C}^{*,x}_{x_\star ,\zeta} \middle| x\in X_{free}\right\}.
\end{equation}
\end{problem}

\begin{remark}
	\label{re_proTCS}
	The description of \textbf{Problem~\ref{proTCS}} is not strictly feasible, as there are infinitely many points in a continuous space. This means that the family $\mathcal{F}^*_{x_\star ,\zeta}$ would contain an infinite number of elements, making it impossible for any device to store the complete $\mathcal{F}^*_{x_\star ,\zeta}$. Therefore, an algorithm is considered to satisfy the problem's requirements if, after processing the anchor point $x_\star$, it can return the optimal configuration set $\mathcal{C}^{*,x}_{x_\star ,\zeta}$ for any point $x$ extremely quickly.
\end{remark}

\begin{problem}[Optimal Tethered Path Planning, TPP]
\label{proTPP}
Given an starting configuration $\varsigma_s$ and a goal position $x_g$, the TPP problem aims to determine the shortest motion path $\sigma^*$ such that:
\begin{equation}
	\label{eqProb2}
	\sigma^*  = \underset{\sigma \in P^{x_g}_{\varsigma_s,\zeta}}{\mathrm{arg\,min}}\ c(\sigma),
\end{equation}
where $P^{x_g}_{\varsigma_s,\zeta}$ denotes the set of all feasible paths from $\varsigma_s(1)$ to $x_g$, defined as:
\begin{align}
	\label{eqProbPSet}
	P^{x_g}_{\varsigma_s,\zeta} = 
  \left\{ \scalebox{0.9}{$\sigma \in P\left(X_{free};\varsigma_s(1),x_g\right)$} \middle| \begin{array}{l}
		\forall t\in I, \\
		\scalebox{0.9}{$c\left(\Theta(\varsigma_s * \sigma_{[0,t]})\right) \leq \zeta$}
	\end{array} \right\}.
\end{align}
where the constraint on the right-hand side in equation ensures that the tether length constraint is satisfied throughout the motion of the agent along the path $\sigma$.
\end{problem}

\begin{problem}[Tethered Multi-Goal Visiting, TMV]
\label{proTMV}
Given the starting configuration $\varsigma_s$ as the home configuration and a sequence of goal positions $x_1,x_2,\dots,x_N$, the TMV problem aims to determine a motion loop $\gamma$ that satisfies the following conditions:
\begin{enumerate}
	\item[1)] {
	      The agent must sequentially visit each goal position and return to the starting position $\varsigma_s(1)$. The loop $\gamma$ can be expressed as:
	      \begin{equation}
		      \label{eqProbTMV1}
		      \gamma=\sigma_{s,1}*\sigma_{1,2}*\sigma_{2,3}*\dots*\sigma_{N,s},
	      \end{equation}
	      where $\sigma_{i,i+1} \in P(X_{free};x_i,x_{i+1})$, and $\sigma_{s,1}(0) = \sigma_{N,s}(1) = \varsigma_s(1)$.
	      }
	\item[2)] {
          The loop $\gamma$ must satisfy the tether length constraint throughout the motion:
          \begin{equation}
              \label{eqProbTMV2}
              \forall t\in I,\ c\left(\Theta\left(\varsigma_s * \gamma_{[0,t]}\right)   \right)  \leq \zeta.
          \end{equation}
          }
	\item[3)] {
	      After completing the motion along $\gamma$, the configuration of the agent must remain homotopic to the home configuration:
	      \begin{equation}
		      \label{eqProbTMV3}
		      \varsigma_s * \gamma \simeq_p \varsigma_s.
	      \end{equation}
	      }
\end{enumerate}
Let $P_{\gamma}$ denote the set of all feasible loops for the TMV problem. Then, optimal TMV problem aims to determine the shortest path within $P_{\gamma}$, i.e.,
\begin{equation}
	\label{eqProbOptTMV}
	\gamma^* = \underset{\gamma \in P_{\gamma}}{\mathrm{arg\,min}}\ c(\gamma).
\end{equation}
\end{problem}

\begin{problem}[Fast Optimal Untethered Path Planning, UTPP]
\label{proUTPP}
Given an starting position $x_s$ and a goal position $x_g$, the UTPP problem aims to determine the shortest motion path\footnote{In this study, $\sigma^*$ generally denotes the optimal homotopic path (locally optimal path) within a certain homotopy class, while $\sigma^\circledast$ represents the globally optimal path between two points without considering tether constraints.} $\sigma^\circledast$ for a classical mobile robot traveling from $x_s$ to $x_g$. This can be expressed as:
\begin{equation}
	\label{eqProb4}
	\sigma^\circledast = \underset{\sigma \in P^{x_g}_{x_s}}{\mathrm{arg\,min}}\ c(\sigma),
\end{equation}
where,
\begin{equation}
	\label{eqProb4P}
	P^{x_g}_{x_s} = P(X_{free};x_s,x_g).
\end{equation}
\end{problem}

\subsection{Homotopy Invariant Based on Convex Dissection}

This study builds upon our previous work, CDT Encoding \citep{liu2023homotopy}. We provide a brief introduction to this method in this subsection.

\begin{figure}[!t]
	\centering
	\subfloat[]{\includegraphics[height=1.65in]{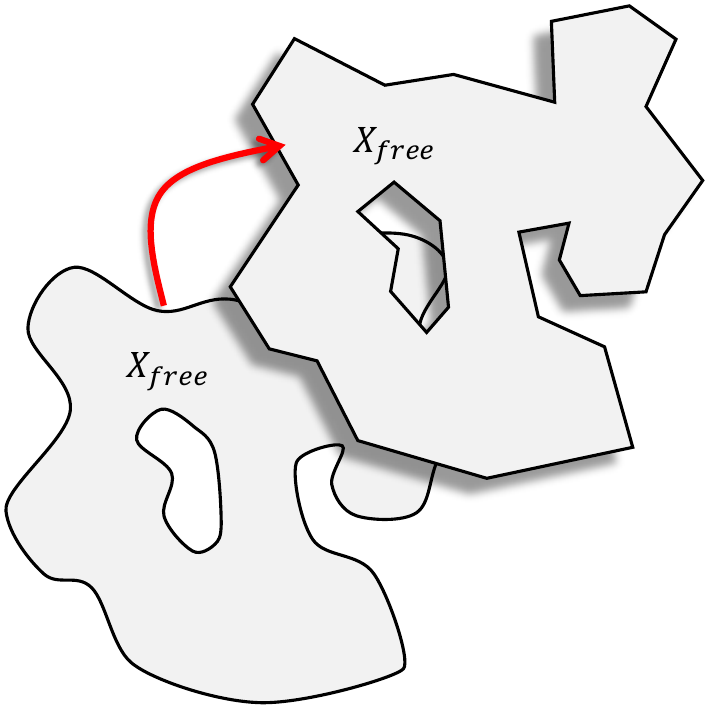}}
	\hfil
	\subfloat[]{\includegraphics[height=1.65in]{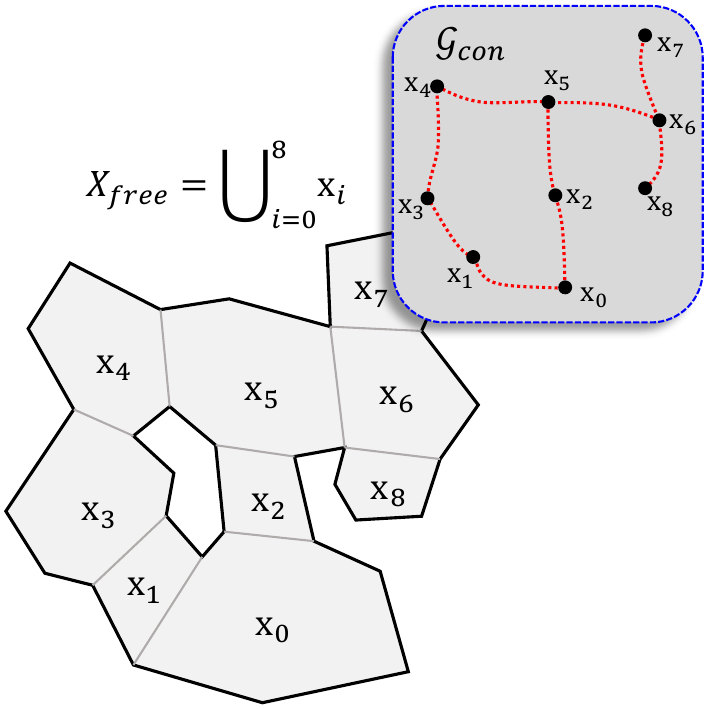}}
	\\
	\subfloat[]{\includegraphics[height=1.65in]{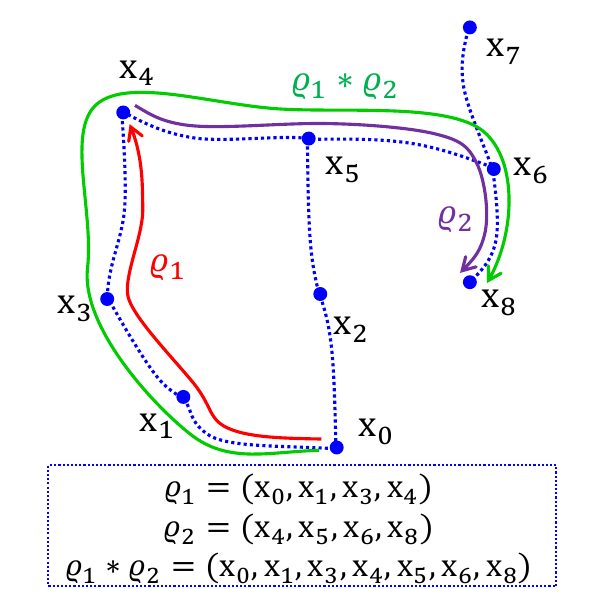}}
	\hfil
	\subfloat[]{\includegraphics[height=1.65in]{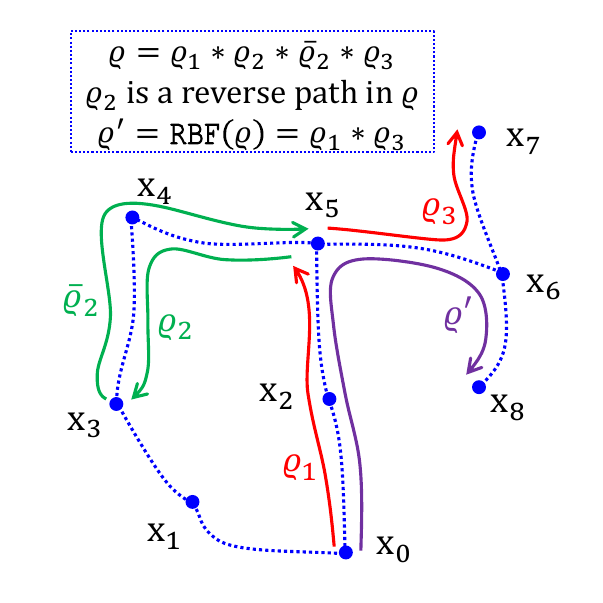}}
	\caption{(a) Illustration of fitting the free space with a simple polygon. (b) Performing convex polygon decomposition and constructing the topological graph $\mathcal{G}_{con}$. (c) Paths on $\mathcal{G}_{con}$ and the concatenation of paths. (d) Illustration of the Rollback Path in $\mathcal{G}_{con}$.}
	\label{fig_CD}
\end{figure}

As illustrated in Fig.~\ref{fig_CD}, any connected 2D free space $X_{free}$ can be approximated by fitting its boundary with a simple polygon. Furthermore, the polygon can be decomposed into a combination of convex polygons. Each convex polygon can then be treated as a node $\mathbf{x}$, while the edges of the decomposition (referred to as `cutlines') are treated as connections between these nodes $\{\mathbf{x}_i,\mathbf{x}_j\}$. The construction of the topological graph is completed by connecting these nodes and edges. This undirected topological graph is denoted as $\mathcal{G}_{con}=\{V_{con}, E_{con}\}$. Based on this, we can define paths on the graph $\mathcal{G}_{con}$ and the mapping of paths from $X_{free}$ to $\mathcal{G}_{con}$.
\begin{definition}[Path in $\mathcal{G}_{con}$]
	\label{defTPPath}
	The path $\varrho$ in $\mathcal{G}_{con}$ is a finite sequence. $\varrho \subset \mathcal{G}_{con}$, and $\forall t\in \mathbb{N}^{T_\varrho}_2$, $\varrho(t)$, $\varrho(t-1)$ are adjacent. The $\varrho(t)$ represents the $t-\mathrm{th}$ element of the sequence $\varrho$, the $T_\varrho$ represents the length of the sequence $\varrho$. For the inverse paths of $\varrho$ we use $\overline{\varrho}$, i.e.,  $\overline{\varrho}(t) = \varrho(T_\varrho - t + 1)$.
\end{definition}

\begin{definition}[Product of Paths in $\mathcal{G}_{con}$]
	\label{defProTPPath}
	\ \newline \indent
	Let $\varrho_1 \in P(\mathcal{G}_{con};\mathbf{x}_0,\mathbf{x}_1)$ and $\varrho_2 \in P(\mathcal{G}_{con};\mathbf{x}_1,\mathbf{x}_2)$. Let $\varrho_1*\varrho_2 \in P(\mathcal{G}_{con};\mathbf{x}_0,\mathbf{x}_2)$ denote their product,
	\begin{equation}
		\label{eq12}
		{\varrho_1*\varrho_2(t)} =
		\begin{cases}
			\varrho_1(t),                          & t \in \mathbb{N}_{1}^{T_{\varrho_1}},                               \\
			\varrho_2\left(t-T_{\varrho_1}\right), & t \in \mathbb{N}_{T_{\varrho_1}+1}^{T_{\varrho_2}+T_{\varrho_1}-1}.
		\end{cases}
	\end{equation}
\end{definition}

\begin{figure*}[!t]
	\centering
	\subfloat[]{\includegraphics[height=1.12in]{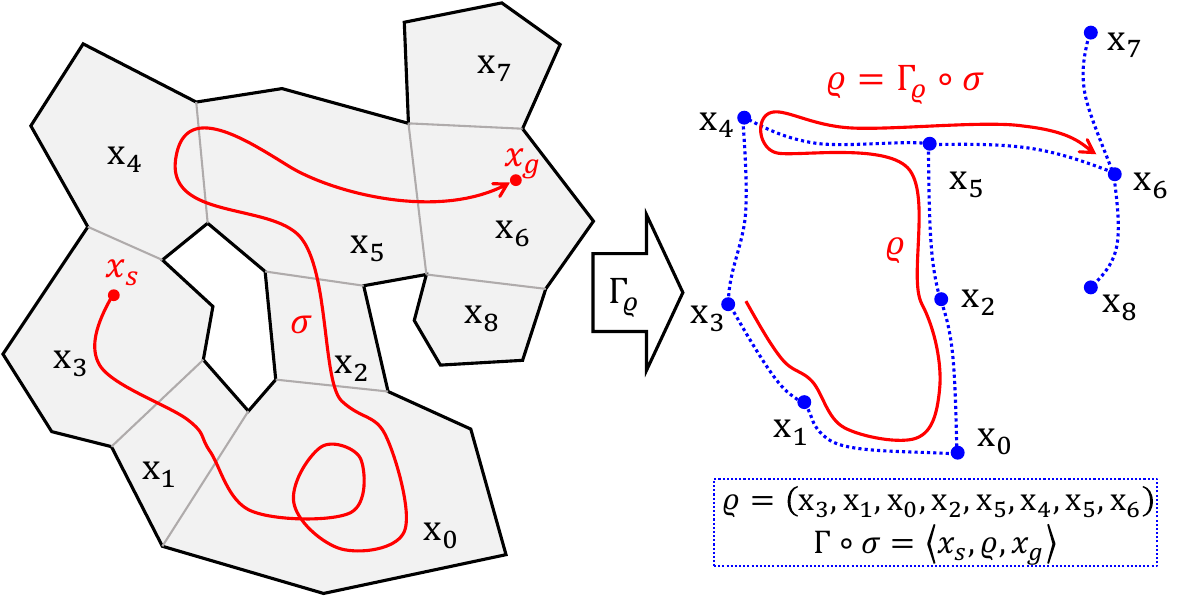}}
	\hfil
	\subfloat[]{\includegraphics[height=1.12in]{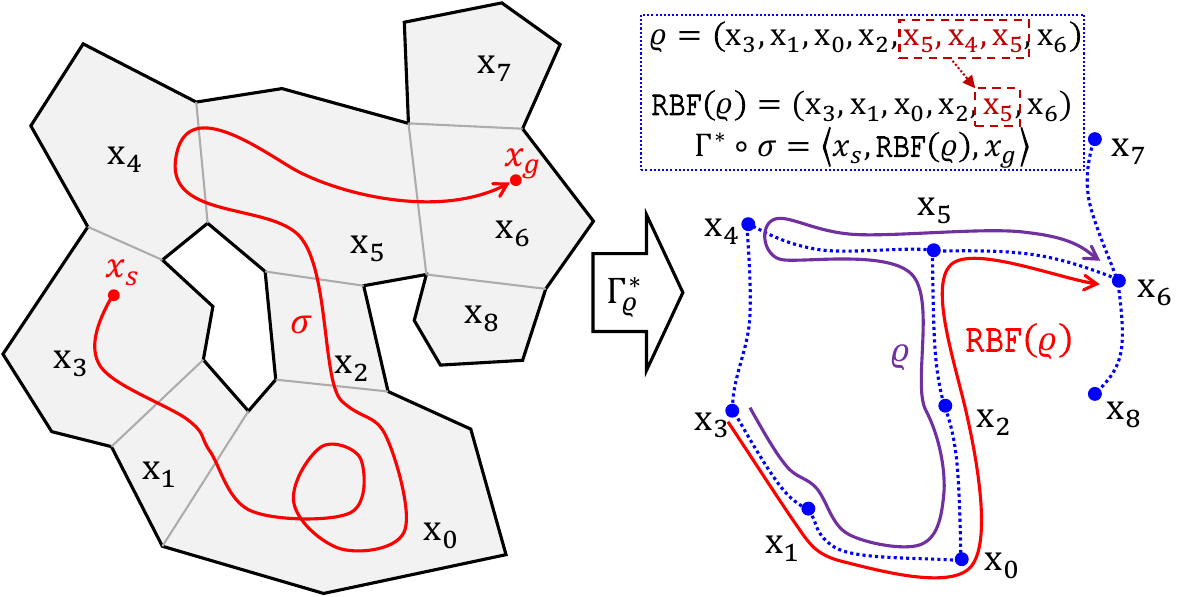}}
	\hfil
	\subfloat[]{\includegraphics[height=1.12in]{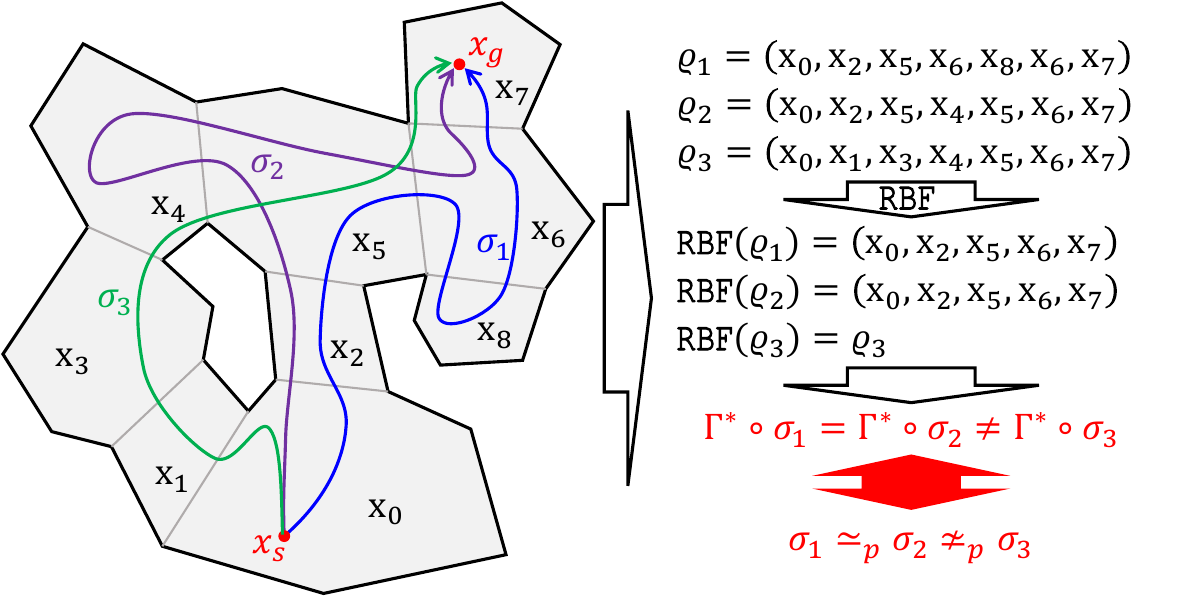}}
	\caption{(a) Illustration of the mappings $\Gamma$ and $\Gamma_\varrho$. (b) Illustration of the mappings $\Gamma^*$ (CDT-Encoder) and $\Gamma^*_\varrho$, where $\Gamma^*_\varrho \circ \sigma = {\tt{RBF}}(\Gamma_\varrho \circ \sigma)$. (c) $\Gamma^*$ is a homotopy invariant; if the CDT encodings of two paths are equal, the paths are homotopic; otherwise, the paths are not homotopic.}
	\label{fig_CDE}
\end{figure*}

\begin{definition}[Mapping $\Gamma$]
	\label{defMAPT}
	As illustrated in Fig.~\ref{fig_CDE}(a). For $\sigma \in P(X_{free})$, $\Gamma \circ \sigma = \langle \sigma(0),\varrho,\sigma(1)\rangle $ is a list consisting of three elements, where $\varrho$ is the sequence of nodes in $P(\mathcal{G}_{con})$ corresponding to the convex polygons passed by $\sigma(t)$ with $t$ in the order from $0$ to $1$. For convenience, we denote the $\Gamma_\varrho \circ \sigma = \varrho$.
\end{definition}

\begin{definition}[Rollback Path in $\mathcal{G}_{con}$]
	\label{defTPRP}
	If the path $\varrho\in P(\mathcal{G}_{con})$ can be written as $\varrho = \varrho_1*\varrho_2*\overline {\varrho}_2*\varrho_3$, where the length of $\varrho_2$ is greater than $1$, $\overline{\varrho}_2$ is the reverse sequence of $\varrho_2$, then $\varrho_2 * \overline{\varrho}_2$ is a rollback path in $\varrho$. And if there is no rollback path in $\varrho_1*\varrho_3$, we call $\varrho_1*\varrho_3$ a rollback-free path of $\varrho$, denoted as ${\tt{RBF}}(\varrho)=\varrho_1*\varrho_3$.
\end{definition}

\begin{definition}[Mapping $\Gamma^*$]
	\label{defMAPTstar}
	Similar to $\Gamma \circ \sigma$, $\Gamma^* \circ \sigma$ is also a list consisting of three elements. The difference is that $\Gamma^*_\varrho \circ \sigma = {\tt{RBF}}(\Gamma_\varrho \circ \sigma)$.

	In addition, we define the product operation for $\Gamma^* \circ \sigma$ as follows:

	Let $\sigma_1 \in P(X_{free};x_0,x_1)$ and $\sigma_2 \in P(X_{free};x_1,x_2)$. The product $(\Gamma^* \circ \sigma_1)*(\Gamma^* \circ \sigma_2)$ is defined as
	\begin{equation}
		\label{eqMAPTstarProd}
		\begin{aligned}
			(\Gamma^* \circ \sigma_1)*(\Gamma^* \circ \sigma_2) & = \langle x_0,\varrho_1,x_1\rangle * \langle x_1,\varrho_2,x_2\rangle \\
			                                                    & = \langle x_0,{\tt{RBF}}(\varrho_1*\varrho_2),x_2\rangle.
		\end{aligned}
	\end{equation}
\end{definition}
It is easy to observe that $\Gamma^*$ is a homomorphism, i.e., $\Gamma^*$ satisfies the following property (referring to Fig.~\ref{fig_CDEhom}(b)):
\begin{equation}
	\label{eqMAPTstarProd2}
	\Gamma^* \circ (\sigma_1*\sigma_2) = (\Gamma^* \circ \sigma_1)*(\Gamma^* \circ \sigma_2).
\end{equation}

\begin{remark}
	Compared to the previous work \citep{liu2023homotopy}, the definitions of $\Gamma$ and $\Gamma^*$ in this paper have been slightly revised to ensure greater mathematical rigor in the subsequent formulations.
\end{remark}

We refer to the map $\Gamma^*$ as the Convex Dissection Topological Encoder (CDT-Encoder). Similar to the $h$-signature \citep{bhattacharya2012topological}, the CDT-Encoder is a homotopy invariant for 2D paths. Specifically, for any homotopy class of paths $[\sigma]_{\simeq_p}$ in $P(X_{free})$, $\Gamma^* \circ [\sigma]_{\simeq_p}$ has a unique encoding. In other words, the following relation is equivalent:
\begin{equation}
	\label{eqCDTencoding}
	\Gamma^* \circ \sigma_1 = \Gamma^* \circ \sigma_2 \Longleftrightarrow \sigma_1 \simeq_p \sigma_2.
\end{equation}

Additionally, to facilitate the subsequent discussions, we introduce the generalized inverse mapping $\overline{\Gamma^*}$ of $\Gamma^*$.
\begin{definition}[Generalized Inverse Mapping $\overline{\Gamma^*}$]
	\label{defInvMAPTstar}
	As illustrated in Fig.~\ref{fig_CDEhom}(a). For any CDT encoding $\langle x_0,\varrho,x_1\rangle$, $\overline{\Gamma^*} \circ \langle x_0,\varrho,x_1\rangle$ represents the optimal path within the homotopy class corresponding to the given CDT encoding. Specifically, we define it as:
	\begin{equation}
		\label{eqInvMAPTstar}
		\overline{\Gamma^*} \circ \langle x_0,\varrho,x_1\rangle = \underset{\sigma \in P_{\langle x_0,\varrho,x_1\rangle}}{\mathrm{arg\,min}}\ c(\sigma),
	\end{equation}
	where $P_{\langle x_0,\varrho,x_1\rangle}$ is defined as:
	\begin{equation}
		P_{\langle x_0,\varrho,x_1\rangle} = \left\{\sigma \in P(X_{free};x_0,x_1) \middle| \Gamma^* \circ \sigma = \langle x_0,\varrho,x_1\rangle \right\}.
	\end{equation}
\end{definition}
Similar to $\Gamma^*$, $\overline{\Gamma^*}$ is also homomorphism, i.e., 
\begin{align}
	&\overline{\Gamma^*} \circ (\langle x_0,\varrho_1,x_1\rangle*\langle x_1,\varrho_2,x_2\rangle) = \nonumber \\
	&\qquad\qquad(\overline{\Gamma^*} \circ \langle x_0,\varrho_1,x_1\rangle)*(\overline{\Gamma^*} \circ \langle x_1,\varrho_2,x_2\rangle).
\end{align}
The implementation of $\overline{\Gamma^*}$ will be provided in \textbf{Algorithm~\ref{alg_ShortestPath}} in the subsequent Subsection 4.2. Furthermore, for $\Gamma^*$ and $\overline{\Gamma^*}$ there is the following simple corollary.

\begin{corollary}
	\label{Cor}
	$\Theta$ is equivalent to the composite of $\overline{\Gamma^*}$ and $\Gamma^*$, i.e.,
	\begin{equation}
		\Theta = \overline{\Gamma^*} \circ \Gamma^*.
	\end{equation}
\end{corollary}

\begin{figure}[!t]
	\centering
	\subfloat[]{\includegraphics[height=1.65in]{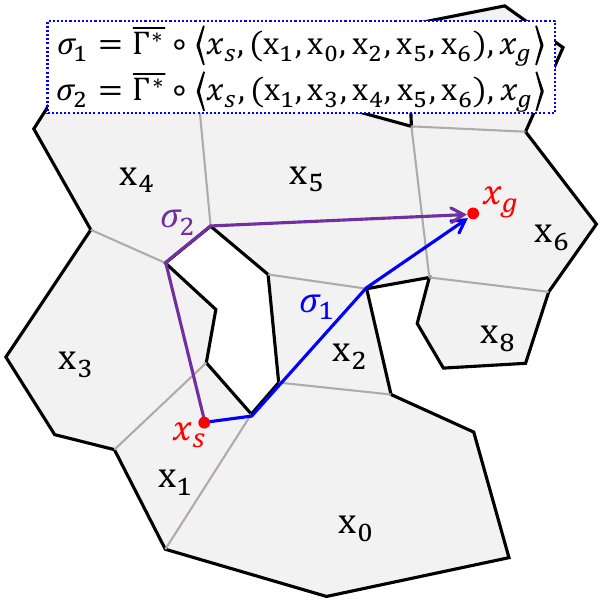}}
	\hfil
	\subfloat[]{\includegraphics[height=1.65in]{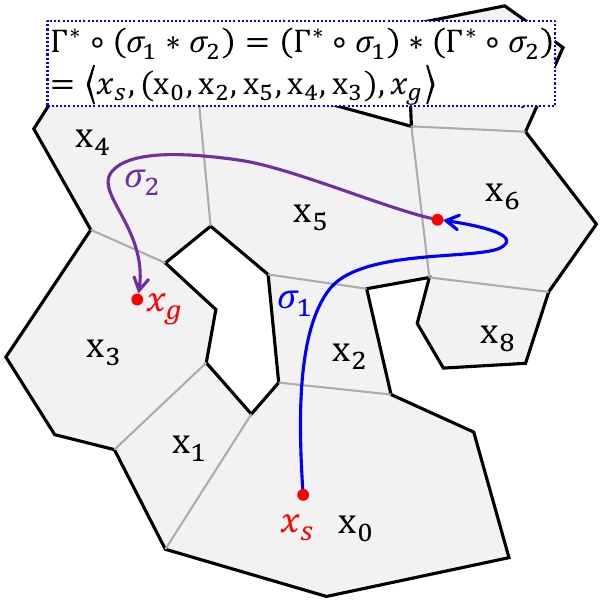}}
	\caption{(a) Illustration of the Generalized Inverse Mapping $\overline{\Gamma^*}$ of $\Gamma^*$. (b) $\Gamma^*$ is a homomorphism.}
	\label{fig_CDEhom}
\end{figure}

\section{Optimal Tethered Configuration Search}
In this section, we will focus on the theoretical derivation and solution methods for the TCS (Optimal Tethered Configuration Search) problem, which forms the foundation for solving other problems such as TTP, TMV, and  UTPP.

\subsection{Theory and Problem Simplification}
In this subsection, we simplify the TCS problem addressed in this study. First, we will prove that it is sufficient to consider only the points on the cutlines to represent the optimal path in any homotopy path classes.
\begin{lemma}
	\label{NSCLOP} 
	A necessary and sufficient condition for $\sigma^* \in P(X_{free})$ to be the shortest path in $[\sigma^*]_{\simeq_p}$ is that, for any $t_1,t_2 \in I$, $\sigma^*_{[t_1,t_2]}$ is the shortest path in $\left[\sigma^*_{[t_1,t_2]}\right]_{\simeq_p}$.
\end{lemma}
\begin{proof}
	See Appendix 3.2.\qed
\end{proof}

\begin{theorem}
	\label{SSDR} 
	For any optimal path $\sigma^* \in [\sigma^*]_{\simeq_p}$ where $x_s = \sigma^*(0)$, $x_e = \sigma^*(1)$, $\sigma^*$ can definitely be  expressed in the form of multiple line splices as follows.
	When the length of $\Gamma^*_\varrho \circ \sigma^*$ is $1$,
	\begin{equation}
		\label{eq_SSDR1}
		\sigma^* = l^{x_e}_{x_s}.
	\end{equation}
	When the length of $\Gamma^*_\varrho \circ \sigma^*$ is not $1$,
	\begin{equation}
		\label{eq_SSDR2}
		\sigma^* = l_{x_0}^{x_1} * l_{x_1}^{x_2} * \dots * l_{x_{n-1}}^{x_n},
	\end{equation}
	where $x_0=x_s$, $x_n=x_e$, $x_1,\dotsc,x_{n-1}$ is the intersection of $\sigma^*$ in turn with the cutlines.
\end{theorem}
\begin{proof}
	See Appendix 3.3.\qed
\end{proof}

Therefore, according to \textbf{Theorem~\ref{SSDR}}, for the TCS problem, we only need to focus on $x_\star$, $x_g$, and the points on the cutlines. Additionally, guided by \textbf{Lemma~\ref{NSCLOP}} and \textbf{Theorem~\ref{SSDR}}, we can easily derive the following corollary:
\begin{corollary}
	\label{Cor1}
	For any path class $[\sigma]_{\simeq_p}\subset P(X_{free};x_s,x_g) $, if the length of $\Gamma^*_\varrho \circ [\sigma]_{\simeq_p}$ is not $1$, then there exists a point $x_c$ on the cutline between the last two convex polygons of $\Gamma^*_\varrho \circ [\sigma]_{\simeq_p}$, such that the optimal path $\sigma^* \in [\sigma]_{\simeq_p}$ can be expressed as:
	\begin{equation}
		\label{eq_Cor1}
		\sigma^* = \sigma^*_c * l_{x_c}^{x_g},
	\end{equation}
	where $\sigma^*_c$ is the optimal path from $x_s$ to $x_c$ in $[\sigma^*_c]_{\simeq_p}$.
\end{corollary}
Building upon \textbf{Corollary~\ref{Cor1}}, consider the sequence $\varrho = \Gamma^*_\varrho \circ [\sigma]_{\simeq_p}$, with $T_\varrho$ representing its length. Since $x_c$ is a point on the cutline between $\varrho(T_\varrho-1)$ and $\mathbf{x}_g$ (where $\mathbf{x}_g = \varrho(T_\varrho)$, i.e. the convex polygon where $x_g$ is located), we may assume that $x_c\in\varrho(T_\varrho-1)$. Combining equations (\ref{eqMAPTstarProd}), (\ref{eqMAPTstarProd2}), and (\ref{eq_Cor1}), we obtain:
\begin{align}
	\label{eq_BFS_O}
	\Gamma^* \circ \sigma^*           & = (\Gamma^* \circ \sigma^*_c) * (\Gamma^* \circ l_{x_c}^{x_g}), \nonumber                                \\
	                                  & \Downarrow \nonumber                                                                                     \\
	\langle x_s, \varrho, x_g \rangle & = \langle x_s, \varrho_c, x_c \rangle * \langle x_c, (\varrho(T_\varrho -1), \mathbf{x}_g), x_g \rangle,
\end{align}
where $\varrho_c$ is the subsequence of $\varrho$ from $\varrho(0)$ to $\varrho(T_\varrho -1)$. As shown in equation (\ref{eq_BFS_O}), the second element $\varrho$ of the CDT Encoding for $\sigma^*$ is simply the sequence $\varrho_c$ with an additional convex polygon $\mathbf{x}_g$, which contains the goal point $x_g$, appended at the end. Meanwhile, $\varrho_c$ belongs to the CDT Encoding of paths whose endpoints lie within convex polygons adjacent to $\mathbf{x}_g$. It is important to note that the above discussion applies to all optimal homotopic paths terminating in $\mathbf{x}_g$.

Let $P^*_\zeta(X_{free};x_s,\mathbf{x})$ represent the set of all optimal homotopic paths starting at $x_s$, ending in the convex polygon $\mathbf{x}$, and with length less than $\zeta$, defined as:
\begin{align}
	\label{eqProb1_plus}
	P^*_\zeta(X_{free};x_s,\mathbf{x}) = \left\{ \Theta(\sigma) \middle| \begin{array}{l}
		\sigma \in P(X_{free};x_s,\mathbf{x}) \\
		\land c\left(\Theta(\sigma)\right) \leq \zeta
	    \end{array}\right\}.
\end{align}
Let $P^*_\zeta(\mathcal{G}_{con};\mathbf{x}_s,\mathbf{x})$ denote the set of all rollback-free paths considering $\zeta$ constraints on the topological graph $\mathcal{G}_{con}$, starting is the convex polygon $\mathbf{x}_s$ (with $x_s \in \mathbf{x}_s$) and ending is $\mathbf{x}$. According to the definition of $\Gamma^*_\varrho$, it follows that:
\begin{equation}
	\Gamma^*_\varrho \circ P^*_\zeta(X_{free};x_s,\mathbf{x}) = P^*_\zeta(\mathcal{G}_{con};\mathbf{x}_s,\mathbf{x}).
\end{equation}
Further, based on the previous discussion related to (\ref{eq_BFS_O}), the following iterative relationship can be derived:
\begin{align}
	\label{eq_CodingIteration}
	 & P^*_\zeta(\mathcal{G}_{con};\mathbf{x}_s,\mathbf{x}) = \nonumber                                      \\
	 &\bigcup_{\mathbf{x}' \in {\tt{near}}(\mathbf{x})}  \scalebox{0.9}{$\left\{ \varrho = \varrho' + (\mathbf{x}) \middle|
	\begin{array}{l}
		\varrho' \in P^*_\zeta(\mathcal{G}_{con};\mathbf{x}_s,\mathbf{x}') \\
		\land \varrho'(T_{\varrho'}-1) \neq \mathbf{x}               \\
		\land \underset{x \in \mathbf{x}}{\mathrm{min}}\ c\left(\overline{\Gamma^*} \circ \langle x_s, \varrho, x\rangle \right) \leq \zeta
	\end{array} \right\}$},
\end{align}
where ${\tt{near}}(\mathbf{x})$ denotes the convex polygons adjacent to $\mathbf{x}$, $\varrho' + (\mathbf{x})$ means append element $\mathbf{x}$ to sequence $\varrho'$. The second constraint on the right-hand side of (\ref{eq_CodingIteration}), $\varrho'(T_{\varrho'}-1) \neq \mathbf{x}$, ensures that the second-to-last element of $\varrho'$ is not $\mathbf{x}$. This prevents the tail of $\varrho$ from forming a rollback path, which would violate the definition of $\Gamma^*_\varrho$.

It is worth noting that the discussion from \textbf{Corollary~\ref{Cor1}} to this point is based on the case where the length of $\Gamma^*_\varrho \circ \sigma$ is not $1$. For the case where the length is $1$, it occurs only when both the starting and ending points of the path are within the same convex polygon (i.e., $\mathbf{x}_s = \mathbf{x}$). In this case, $\Gamma^*_\varrho \circ \sigma$ is simply the sequence $(\mathbf{x}_s)$. At this point, we should revise (\ref{eq_CodingIteration}) into the following form:
\begin{align}
	\label{eq_CodingIteration2}
	 & P^*_\zeta(\mathcal{G}_{con};\mathbf{x}_s,\mathbf{x}_s) = \left\{(\mathbf{x}_s)\right\}  \cup \nonumber    \\
	 & \bigcup_{\mathbf{x}' \in {\tt{near}}(\mathbf{x_s})}  \scalebox{0.9}{$\left\{ \varrho = \varrho' + (\mathbf{x}_s) \middle|
	\begin{array}{l}
		\varrho' \in P^*_\zeta(\mathcal{G}_{con};\mathbf{x}_s,\mathbf{x}') \\
		\land \varrho'(T_{\varrho'}-1) \neq \mathbf{x}_s             \\
		\land \underset{x \in \mathbf{x}_s}{\mathrm{min}}\ c\left(\overline{\Gamma^*} \circ \langle x_s, \varrho, x\rangle \right) \leq \zeta
	\end{array} \right\}$}.
\end{align}
Thus we call $P^*_\zeta(\mathcal{G}_{con};\mathbf{x}_s,\mathbf{x})$ a \textbf{Tethered Configuration CDT Encoding Set} of convex polygon $\mathbf{x}$.

\begin{remark}
	\label{re_TCS}
	Based on equations (\ref{eq_CodingIteration}) and (\ref{eq_CodingIteration2}), we can easily construct a simple iterative process to find $P^*_\zeta(\mathcal{G}_{con};\mathbf{x}_s,\mathbf{x})$ for any convex polygon $\mathbf{x} \in V_{con}$, Specifically, the iteration begins by adding $(\mathbf{x}_s)$ to $P^*_\zeta(\mathcal{G}_{con};\mathbf{x}_s,\mathbf{x}_s)$; Subsequently, whenever a new element $\varrho$ is added to $P^*_\zeta(\mathcal{G}_{con};\mathbf{x}_s,\mathbf{x}')$ for a neighboring convex polygon $\mathbf{x}'$ of $\mathbf{x}$, (\ref{eq_CodingIteration}) is used to add new elements to $P^*_\zeta(\mathcal{G}_{con};\mathbf{x}_s,\mathbf{x})$. Changes in $P^*_\zeta(\mathcal{G}_{con};\mathbf{x}_s,\mathbf{x})$ will, in turn, trigger updates to other surrounding convex polygons.
\end{remark}

Revisiting the definition of the TCS problem (\textbf{Problem~\ref{proTCS}}), we can observe the connection between equations (\ref{eqProb1}) and (\ref{eqProb1_plus}). Since a configuration is essentially a type of path (as noted in \textbf{Remark~\ref{re_PeqC}}), equation (\ref{eqProb1_plus}) can be interpreted as representing the union of all feasible optimal configurations for target positions within the convex polygon $\mathbf{x}$. Specifically:
\begin{equation}
	P^*_\zeta(X_{free};x_\star,\mathbf{x}) = \bigcup_{x \in \mathbf{x}} \mathcal{C}^{*,x}_{x_\star ,\zeta}.
\end{equation}
Further:
\begin{align}
	\label{eq_GconPathSet}
	\Gamma^*_\varrho \circ \left(\bigcup_{x \in \mathbf{x}} \mathcal{C}^{*,x}_{x_\star ,\zeta}\right) & = \Gamma^*_\varrho \circ P^*_\zeta(X_{free};x_\star,\mathbf{x}) \nonumber \\
	  & = P^*_\zeta(\mathcal{G}_{con};\mathbf{x}_\star,\mathbf{x}).
\end{align}
It is worth noting that although the union $\bigcup_{x \in \mathbf{x}} \mathcal{C}^{*,x}_{x_\star ,\zeta}$ contains an infinite number of elements, the set of elements in $\Gamma^*_\varrho \circ \left(\bigcup_{x \in \mathbf{x}} \mathcal{C}^{*,x}_{x_\star ,\zeta}\right)$ is finite. This is because, for most optimal configurations goal points within the same convex polygon, the second element of their CDT encodings is identical. Thus, knowing $P^*_\zeta(\mathcal{G}_{con};\mathbf{x}_\star,\mathbf{x})$ for any $\mathbf{x}\in V_{con}$ allows us to derive $\mathcal{C}^{*,x_g}_{x_\star ,\zeta}$ for any specific goal point $x_g$ in a straightforward manner. First, determine the convex polygon $\mathbf{x}_g$ containing $x_g$. Then, traverse the sequences $\varrho$ in $P^*_\zeta(\mathcal{G}_{con};\mathbf{x}_\star,\mathbf{x}_g)$, computing and filtering all valid optimal configurations. This process can be mathematically expressed as:
\begin{equation}
	\label{eq_GetTCS}
	\mathcal{C}^{*,x_g}_{x_\star ,\zeta} = \left\{ \overline{\Gamma^*} \circ \langle x_\star, \varrho, x_g\rangle \middle|
	\begin{array}{l}
		\varrho \in P^*_\zeta(\mathcal{G}_{con};\mathbf{x}_\star,\mathbf{x}_g) \\
		\land c\left(\overline{\Gamma^*} \circ \langle x_\star, \varrho, x_g\rangle\right) \leq \zeta
	\end{array}\right\}.
\end{equation}

For any $\mathbf{x}\in V_{con}$ we can use the process described in \textbf{Remark~\ref{re_TCS}} to obtain $P^*_\zeta(\mathcal{G}_{con};\mathbf{x}_\star,\mathbf{x})$. However, one significant challenge remains: observing the three constraints on the right-hand side of (\ref{eq_CodingIteration}), we find that constraints 1 and 2 involve simple traversal and logical checks. In contrast, constraint 3 ($\underset{x \in \mathbf{x}}{\mathrm{min}}\ c\left(\overline{\Gamma^*} \circ \langle x_s, \varrho, x\rangle \right) \leq \zeta$) entails a highly complex and computationally intensive process, as it requires solving the optimal paths for an infinite number of target points. Therefore, the remainder of this subsection will focus on developing an efficient approach to solve constraint 3.

For constraint 3 in (\ref{eq_CodingIteration}), consider an arbitrary point $x_{in}$ inside a convex polygon $\mathbf{x}$. According to \textbf{Corollary~\ref{Cor1}}, the optimal homotopic path $\overline{\Gamma^*} \circ \langle x_s, \varrho, x_{in}\rangle$ can be expressed as:
\begin{equation}
	\overline{\Gamma^*} \circ \langle x_s, \varrho, x_{in}\rangle = \left(\overline{\Gamma^*} \circ \langle x_s, \varrho, x_c\rangle\right) * l_{x_c}^{x_{in}},
\end{equation}
where $x_c$ is a point on the cutline between the convex polygons $\varrho(T_{\varrho})$ and $\varrho(T_{\varrho }-1)$. According to equation (\ref{eqPPCost}), we have:
\begin{align}
	c\left(\overline{\Gamma^*} \circ \langle x_s, \varrho, x_{in}\rangle\right) & = c\left(\overline{\Gamma^*} \circ \langle x_s, \varrho, x_c\rangle\right) + c\left(l_{x_c}^{x_{in}}\right) \nonumber \\
	                                                                            & \geq c\left(\overline{\Gamma^*} \circ \langle x_s, \varrho, x_c\rangle\right).
\end{align}
Thus, to determine whether constraint 3 in (\ref{eq_CodingIteration}) is satisfied, it suffices to consider only the points on the cutline between $\varrho(T_{\varrho})$ and $\varrho(T_{\varrho }-1)$, i.e.,
\begin{align}
	\label{eq_constraintLine}
	                    & \underset{x \in \mathbf{x}}{\mathrm{min}}\ c\left(\overline{\Gamma^*} \circ \langle x_s, \varrho, x\rangle \right) \leq \zeta \nonumber \\
	\Longleftrightarrow & \underset{t \in I}{\mathrm{min}}\ c\left(\overline{\Gamma^*} \circ \langle x_s, \varrho, l_c(t)\rangle \right) \leq \zeta,
\end{align}
where $l_c$ denotes the cutline between $\varrho(T_{\varrho})$ and $\varrho(T_{\varrho }-1)$.

To develop a method for quickly determining whether the points on the cutline satisfy Constraint 3, we will construct two special functions and prove their relevant properties.

\begin{figure}[!t]
	\centering
	\subfloat[]{\includegraphics[height=1.65in]{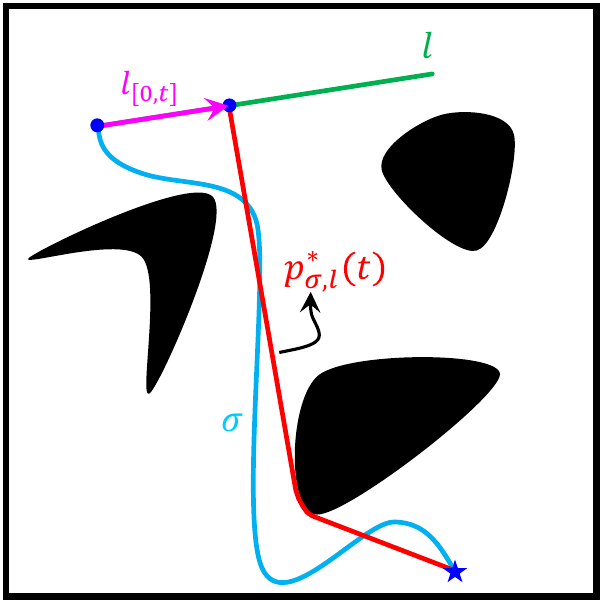}}
	\hfil
	\subfloat[]{\includegraphics[height=1.65in]{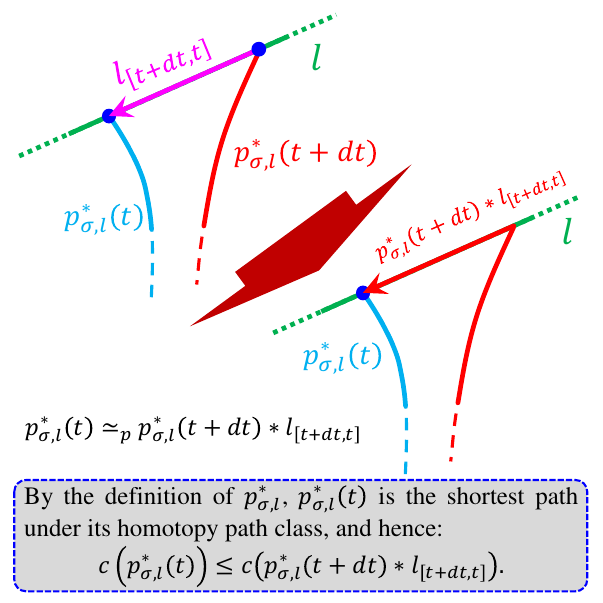}}
	\caption{Illustration of the function $p^*_{\sigma, l}$. The function $p^*_{\sigma, l}(t)$ represents the optimal homotopic path within the homotopy class $\left[\sigma * l_{[0,t]}\right]_{\simeq_p}$.}
	\label{fig_FPfun}
\end{figure}

Let $\sigma$ denote an arbitrary path in $P(X_{free})$, and let $l$ be a straight line connected to the end of $\sigma$, such that $l(0) = \sigma(1)$. Using this setup, we define a special path function $p^*_{\sigma, l}:I\to P(X_{free};\sigma(0),l)$ as follows:
\begin{equation}
	\label{eq_pfun}
	p^*_{\sigma, l}(t) = \Theta\left(\sigma * l_{[0,t]}\right),
\end{equation}
as illustrated in Fig.~\ref{fig_FPfun}(a). The function $p^*_{\sigma, l}$ represents the optimal homotopic path corresponding to the movement of the endpoint of $\sigma$ along the line $l$ as it continuously extends. Further, we define the following composite function $f:I \to [0,\infty)$:
\begin{equation}
	\label{eq_ffun}
	f(t) = \left(c \circ p^*_{\sigma, l}\right) (t).
\end{equation}
The functions $p^*_{\sigma, l}$ and $f$ have the following properties.

\begin{lemma}
	\label{th_phomotopy}
	For any $t_1,t_2 \in I$, 
	\begin{equation}
		p^*_{\sigma, l}(t_2) \simeq_p p^*_{\sigma, l}(t_1)*l_{[t_1,t_2]}.
	\end{equation}
\end{lemma}
\begin{proof}
	See Appendix 3.4.\qed
\end{proof}

\begin{theorem}
	\label{th_finequality}
	Let $t \in I$ and $t+dt \in I$. Then the function $f$ satisfies the following inequality:
	\begin{equation}
		\label{eq_finequality0P}
		f(t+dt) - c^l_{dt} \leq f(t) \leq f(t+dt) + c^l_{dt},
	\end{equation}
	\begin{equation}
		\label{eq_finequality0M}
		f(t-dt) - c^l_{dt} \leq f(t) \leq f(t-dt) + c^l_{dt},
	\end{equation}
	where
	\begin{equation}
		\label{eq_finequalityC}
		c^l_{dt} = c\left(l_{[t,t+dt]}\right)  = \left\lvert dt\right\rvert \cdot c(l).
	\end{equation}
\end{theorem}
\begin{proof}
	See Appendix 3.5.\qed
\end{proof}

\begin{lemma}
	\label{th_fcontinuity}
	The function $f$ is continuous on $I$.
\end{lemma}
\begin{proof}
	See Appendix 3.6.\qed
\end{proof}

\begin{theorem}
	\label{th_fconvex}
	The function $f$ is convex.
\end{theorem}
\begin{proof}
	See Appendix 3.7.\qed
\end{proof}

Returning to the discussion in (\ref{eq_constraintLine}) on quickly determining whether the points on the dividing line satisfy the constraint, based on \textbf{Theorem~\ref{th_finequality}}, we set $t = 0.5$ in (\ref{eq_finequality0P}), yielding:
\begin{equation}
	f\left(\frac{1}{2}\right) - c^l_{dt} \leq f\left(\frac{1}{2} + dt\right) \leq f\left(\frac{1}{2}\right) + c^l_{dt},
\end{equation}
for $dt\in[-0.5,0.5]$, we can determine the range of values for the minimum cost:
\begin{equation}
	\underset{t \in I}{\mathrm{min}}\ c\left(\overline{\Gamma^*} \circ \langle x_s, \varrho, l_c(t)\rangle \right) \in \left[f\left(\frac{1}{2}\right) - \frac{c(l_c)}{2} , f\left(\frac{1}{2}\right)\right].
\end{equation}
From this, we design a method to conservatively determine whether the cutline satisfies (\ref{eq_constraintLine}), as follows:
\begin{algorithm}[H]
	\begin{algorithmic}
		\State $c_{min} \gets$ Calculate $c\left(\overline{\Gamma^*} \circ \langle x_s, \varrho, l_c(0.5)\rangle\right)$
		\State \textbf{if} $c_{min} - \frac{c(l_c)}{2} > \zeta$ \textbf{then}
		\State \hspace{0.25cm}(\ref{eq_constraintLine}) is not satisfied
		\State \textbf{elif} $c_{min} \leq \zeta$ \textbf{then}
		\State \hspace{0.25cm}(\ref{eq_constraintLine}) is satisfied
		\State \textbf{else}
		\State \hspace{0.25cm}Unable to conservatively determine
	\end{algorithmic}
\end{algorithm}

In cases where the constraint (\ref{eq_constraintLine}) cannot be conservatively determined, \textbf{Theorem~\ref{th_fconvex}} suggests that convex optimization techniques, such as the ternary search method, can be employed iteratively to quickly determine whether (\ref{eq_constraintLine}) is satisfied.

\subsection{Optimal Tethered Configuration Search Algorithm}
In this subsection, we present and explain the details of the CDT-TCS algorithm (Optimal Tethered Configuration Search algorithm based on CDT encoding), developed to solve \textbf{Problem~\ref{proTCS}} based on the preceding discussion. The characteristic of this algorithm is as follows: (i) given an anchor point $x_\star$ and tether length $\zeta$, it can quickly generate the \textbf{Tethered Configuration CDT Encoding Set} corresponding to all convex polygons. (ii) for any goal point $x_g$ in the free space $X_{free}$ the algorithm can return the feasible optimal configuration set $\mathcal{C}^{*,x_{g}}_{x_\star ,\zeta}$ with very high efficiency. The CDT-TCS corresponding to its characteristics can be divided into two independent stages, which are namely \textbf{Preprocessing Stage} and \textbf{GetConfigurations Stage}.

The pseudocode of the CDT-TCS algorithm is shown in \textbf{Algorithm~\ref{alg_TCS}} through \textbf{Algorithm~\ref{alg_ShortestPath}}. Specifically, \textbf{Algorithm~\ref{alg_TCS}} and \textbf{Algorithm~\ref{alg_TCSAFOC}} illustrate the \textbf{Preprocessing Stage} and the \textbf{GetConfigurations Stage} of the CDT-TCS, respectively. \textbf{Algorithm~\ref{alg_EncodingValidity}} presents the fast decision-making method for the third constraint in (\ref{eq_CodingIteration}). \textbf{Algorithm~\ref{alg_ShortestPath}} demonstrates an implementation of $\overline{\Gamma^*}$. The following sections provide a detailed description of each of these algorithms.

In the \textbf{Preprocessing Stage} (\textbf{Algorithm~\ref{alg_TCS}}), the algorithm begins by initializing the input free space $X_{free}$. Specifically, it performs polygonal fitting and convex decomposition on $X_{free}$ to construct the graph $\mathcal{G}_{con}$ (line 2). Next, the convex polygon $\mathbf{x}_\star$ containing the anchor point $x_\star$ is identified (line 3). The algorithm then initializes the queue $\mathcal{Q}_\varrho$ and the Tethered Configuration CDT Encoding Set for all convex polygons with the empty set $\varnothing$, where the queue $\mathcal{Q}_\varrho$ stores newly added Tethered Configuration CDT Encodings.

Lines 5-14 of \textbf{Algorithm~\ref{alg_TCS}} illustrate the iterative process described in \textbf{Remark~\ref{re_TCS}}. In lines 5 and 6, the algorithm initially adds $\mathbf{x}_\star$ to $P^*_\zeta(\mathcal{G}_{con};\mathbf{x}_\star,\mathbf{x}_\star)$ according to (\ref{eq_CodingIteration2}), thereby triggering the entire iteration. In lines 7-14, the algorithm iteratively pops the first element $\varrho$ from the head of the queue $Q_\varrho$ and traverses the neighbors of the last convex polygon $\varrho(T_\varrho)$ in the sequence. For each neighboring convex polygon $\mathbf{x}_{near}$, the algorithm first excludes those that do not satisfy Constraint 2 in (\ref{eq_CodingIteration}), and then uses \textbf{Algorithm~\ref{alg_EncodingValidity}} to check if $\mathbf{x}_{near}$ satisfies Constraint 3 in equation (\ref{eq_CodingIteration}) (i.e., Constraint (\ref{eq_constraintLine})). For valid $\varrho_{new}$, it is added to both $P^*_\zeta(\mathcal{G}_{con};\mathbf{x}_\star,\mathbf{x}_{near})$ and the queue $\mathcal{Q}_\varrho$.

\begin{algorithm}[H]
	\caption{CDT-TCS.}
	\begin{algorithmic}[1]
		\State {\textbf{Input: }}$X_{free}$, $x_\star$, $\zeta$
		\State $\mathcal{G}_{con}\gets{\tt{InitializeFreeSpace}}(X_{free})$
		\State $\mathbf{x}_\star \gets {\tt{GetConvexPolygon}}(x_\star)$
		\State Use $\varnothing$ to initialize $\mathcal{Q}_\varrho$ and $P^*_\zeta(\mathcal{G}_{con};\mathbf{x}_\star,\mathbf{x})$, $\mathbf{x}\in V_{con}$
		\State $P^*_\zeta(\mathcal{G}_{con};\mathbf{x}_\star,\mathbf{x}_\star) \overset{+}{\gets} \{(\mathbf{x}_\star)\}$
		\State $\mathcal{Q}_\varrho \overset{+}{\gets} \{(\mathbf{x}_\star)\}$
		\State \textbf{while} $\mathcal{Q}_\varrho \not\equiv \varnothing$ \textbf{do}
		\State \hspace{0.25cm}$\varrho \gets {\tt{PopFirst}}(\mathcal{Q}_\varrho)$
		\State \hspace{0.25cm}\textbf{for} $\mathbf{x}_{near} \in {\tt{Neighbor}}\left(\varrho(T_\varrho)\right) $ \textbf{do}
		\State \hspace{0.50cm}\textbf{if} $\mathbf{x}_{near} \equiv \varrho(T_\varrho-1)$ \textbf{then Continue}
		\State \hspace{0.50cm}$\varrho_{new} = \varrho + (\mathbf{x}_{near})$
		\State \hspace{0.50cm}\textbf{if} ${\tt{EncodingValidity}}(\varrho_{new},x_\star,\zeta)$ \textbf{then}
		\State \hspace{0.75cm}$P^*_\zeta(\mathcal{G}_{con};\mathbf{x}_\star,\mathbf{x}_{near}) \overset{+}{\gets} \{\varrho_{new}\}$
		\State \hspace{0.75cm}$\mathcal{Q}_\varrho \overset{+}{\gets} \{\varrho_{new}\}$
	\end{algorithmic}
	\label{alg_TCS}
\end{algorithm}

In the \textbf{GetConfigurations Stage} (\textbf{Algorithm~\ref{alg_TCSAFOC}}), the algorithm first initializes $\mathcal{C}^{*,x_{g}}_{x_\star ,\zeta}$ as the empty set $\varnothing$ and identifies the convex polygon $\mathbf{x}_g$ containing the goal point $x_g$. Subsequently, the algorithm retrieves all feasible optimal configurations for $x_g$ through an iterative process based on (\ref{eq_GetTCS}), storing the results in $\mathcal{C}^{*,x_{g}}_{x_\star,\zeta}$.

\begin{algorithm}[H]
	\caption{GetAllFOC: GetAllFeasibleOptConfigs $\mathcal{C}^{*,x_{g}}_{x_\star ,\zeta}$.}
	\begin{algorithmic}[1]
		\State {\textbf{Input: }}$x_{g}$
		\State Use $\varnothing$ to initialize $\mathcal{C}^{*,x_{g}}_{x_\star ,\zeta}$
		\State $\mathbf{x}_g \gets {\tt{GetConvexPolygon}}(x_g)$
		\State \textbf{for} $\varrho_g \in P^*_\zeta(\mathcal{G}_{con};\mathbf{x}_\star,\mathbf{x}_g) $ \textbf{do}
		\State \hspace{0.25cm}$\varsigma^*_g \gets \overline{\Gamma^*} \circ \langle x_\star, \varrho_g, x_{g}\rangle$
		\State \hspace{0.25cm}\textbf{if} $c(\varsigma^*_g) \leq \zeta$ \textbf{then}
		\State \hspace{0.50cm}$\mathcal{C}^{*,x_{g}}_{x_\star ,\zeta} \overset{+}{\gets} \{\varsigma^*_g\}$
		\State \textbf{return} $\mathcal{C}^{*,x_{g}}_{x_\star ,\zeta}$
	\end{algorithmic}
	\label{alg_TCSAFOC}
\end{algorithm}

Based on the discussion in Subsection 4.1, \textbf{Algorithm~\ref{alg_EncodingValidity}} presents an implementation of a rapid method to determine whether a Tethered Configuration CDT Encoding $\varrho$ is valid. The algorithm first computes the cutline $l_c$ between the last two convex polygons of $\varrho$. It then conservatively evaluates the validity of $\varrho$ using \textbf{Theorem~\ref{th_finequality}} (lines 3-7). For cases where the validity cannot be conservatively determined (lines 9-28), the algorithm proceeds according to \textbf{Theorem~\ref{th_fconvex}}. It first assesses the monotonicity of the cost function $c\left(\overline{\Gamma^*} \circ \langle x_\star, \varrho, l_c(t)\rangle\right)$ (i.e., $f(t)$ in (\ref{eq_ffun})) with respect to $t$. If the function is monotonic, the validity of $\varrho$ is determined based on the costs at the endpoints (lines 9-14). If the function is non-monotonic, the algorithm applies ternary search to quickly locate the extremum of $c\left(\overline{\Gamma^*} \circ \langle x_\star, \varrho, l_c(t)\rangle\right)$, checking the validity of $\varrho$ during the process (lines 15-28).

\begin{algorithm}[H]
	\caption{Encoding Validity.}
	\begin{algorithmic}[1]
		\State {\textbf{Input: }}$\varrho$, $x_\star$, $\zeta$
		\State $l_c \gets {\tt{GetCutline}}(\varrho(T_{\varrho}-1),\varrho(T_{\varrho}))$
		\State $c_{mid} \gets c\left(\overline{\Gamma^*} \circ \langle x_\star, \varrho, l_c(0.5)\rangle\right) $
		\State \textbf{if} $c_{mid} - \frac{c(l_c)}{2} > \zeta$ \textbf{then}
		\State \hspace{0.25cm}\textbf{return} ${\tt{False}}$
		\State \textbf{elif} $c_{mid} \leq \zeta$ \textbf{then}
		\State \hspace{0.25cm}\textbf{return} ${\tt{True}}$
		\State \textbf{else}
		\State \hspace{0.25cm}$c_0 \gets c\left(\overline{\Gamma^*} \circ \langle x_\star, \varrho, l_c(0)\rangle\right)$
		\State \hspace{0.25cm}$c_1 \gets c\left(\overline{\Gamma^*} \circ \langle x_\star, \varrho, l_c(1)\rangle\right)$
		\State \hspace{0.25cm}\textbf{if} $c_0 \leq \zeta$ \textbf{or} $c_1 \leq \zeta$ \textbf{then}
		\State \hspace{0.50cm}\textbf{return} ${\tt{True}}$
		\State \hspace{0.25cm}\textbf{if} $\zeta < c_0 \leq c_{mid} \leq c_1$ \textbf{or} $\zeta < c_1 \leq c_{mid} \leq c_0$ \textbf{then}
		\State \hspace{0.50cm}\textbf{return} ${\tt{False}}$
		\State \hspace{0.25cm}$t_{low} \gets 0$, $t_{high} \gets 1$
		\State \hspace{0.25cm}\textbf{loop}
		\State \hspace{0.50cm}$t_{mid1} \gets t_{low} + (t_{high} - t_{low})/3$
		\State \hspace{0.50cm}$t_{mid2} \gets t_{high} - (t_{high} - t_{low})/3$
		\State \hspace{0.50cm}$c_{mid1} \gets c\left(\overline{\Gamma^*} \circ \langle x_\star, \varrho, l_c(t_{mid1})\rangle\right)$
		\State \hspace{0.50cm}$c_{mid2} \gets c\left(\overline{\Gamma^*} \circ \langle x_\star, \varrho, l_c(t_{mid2})\rangle\right)$
		\State \hspace{0.50cm}\textbf{if} $c_{mid1} \leq \zeta$ \textbf{or} $c_{mid2} \leq \zeta$ \textbf{then}
		\State \hspace{0.75cm}\textbf{return} ${\tt{True}}$
		\State \hspace{0.50cm}\textbf{if} $c_{mid1} \leq c_{mid2}$ \textbf{then}
		\State \hspace{0.75cm}$t_{high} \gets t_{mid2}$
		\State \hspace{0.50cm}\textbf{else}
		\State \hspace{0.75cm}$t_{low} \gets t_{mid1}$
		\State \hspace{0.25cm}\textbf{Until} Condition (e.g. $\left\lvert c_{mid1}-c_{mid2}\right\rvert <$ threshold)
		\State \hspace{0.25cm}\textbf{return} ${\tt{False}}$
	\end{algorithmic}
	\label{alg_EncodingValidity}
\end{algorithm}

For the implementation of $\overline{\Gamma^*}$ in \textbf{Algorithm~\ref{alg_ShortestPath}}, we adopt the approach used in our previous work \citep{liu2023cdt} for obtaining the optimal homotopic path. The core of this method involves constructing a compressed mapping and iteratively minimizing the cost of the current path, eventually converging to the optimal path within its homotopy class. Detailed discussions, proofs, and programming optimizations for \textbf{Algorithm~\ref{alg_ShortestPath}} can be found in Subsection 3.4 of \citep{liu2023cdt}, and will not be reiterated here.

\begin{algorithm}[H]
	\caption{$\overline{\Gamma^*}$: GetOptimaltHomotopyPath.}
	\begin{algorithmic}[1]
		\State {\textbf{Input: }}$\left\langle x_s, \varrho, x_e\right\rangle $
		\State Use $\varrho$ to get the cutline sequence $\{l_{c1},l_{c2},\dots ,l_{cn}\}$
		\State $\{x_0,x_1,\dots ,x_{n+1}\} \gets \{x_s,l_{c1}(\frac{1}{2}),\dots ,l_{cn}(\frac{1}{2}),x_e\}$
		\State $\sigma^* \gets l_{x_0}^{x_1} * l_{x_1}^{x_2} * \dots * l_{x_n}^{x_{n+1}}$
		\State \textbf{repeat}
		\State \hspace{0.5cm} $c_{min} \gets c(\sigma^*) $
		\State \hspace{0.5cm} \textbf{for $k \in \{1,2,\dots,n\}$ do}
		\State \hspace{1.0cm} $x_k = \mathop {\arg \min }\limits_{x \in l_{ck}}\ c(l_{x_{k-1}}^x * l_x^{x_{k+1}})$
		\State \hspace{0.5cm} $\sigma^* \gets l_{x_0}^{x_1} * l_{x_1}^{x_2} * \dots * l_{x_n}^{x_{n+1}}$
		\State \textbf{until} $c(\sigma^*) - c_{min} <$ threshold
		\State \textbf{return} $\sigma^*$
	\end{algorithmic}
	\label{alg_ShortestPath}
\end{algorithm}

\section{Applications and Extensions}
Section 4 presented the foundational theory related to the TCS problem and introduced the solution method, CDT-TCS. Building upon this foundation, the current section explores several applications based on CDT-TCS and discusses improved versions that can be adapted to various application scenarios.

\subsection{Fast Optimal Path Planning for Tethered Robots}
The problem of fast optimal path planning under tethering constraints (\textbf{Problem~\ref{proTPP}}) is one of the most fundamental issues in tethered robot motion planning. Its objective is to plan an optimal trajectory $\sigma^*$ for a tethered robot from its current configuration $\varsigma_s$ to the goal position $x_g$.

According to \textbf{Assumption~\ref{ass1}}, the configuration changes caused by the motion of the robot in free space are related as follows:
\begin{equation}
	\varsigma_g \simeq_p \varsigma_s * \sigma,
\end{equation}
where $\sigma \in P(X_{free};\varsigma_s(1),x_g)$ represents the robot's path, and $\varsigma_{g}$ denotes the new configuration after moving along the $\sigma$ to $x_g$. Thus, considering the configuration change resulting from the robot's motion along the optimal path $\sigma^*$, we have:
\begin{equation}
	\label{eq_FOPP1}
	\varsigma_g \simeq_p \varsigma_s * \sigma^* \Longleftrightarrow \sigma^* \simeq_p \overline{\varsigma_s} * \varsigma_g.
\end{equation}
Due to the constraint on the right-hand side of (\ref{eqProbPSet}) (when $t = 1$) that limits the range of $\sigma^*$, the new configuration $\varsigma_g$ satisfies the following condition:
\begin{equation}
	c\left(\Theta(\varsigma_g)\right) \leq \zeta.
\end{equation}
Let $\varsigma^*_g$ represent the optimal configuration $\Theta(\varsigma_g)$. (\ref{eq_FOPP1}) can then be rewritten as:
\begin{equation}
	\label{eq_FOPP2}
	\sigma^* \simeq_p \overline{\varsigma_s} * \varsigma_g \simeq_p \overline{\varsigma_s} * \varsigma^*_g.
\end{equation}
Since $\sigma^*$ is the optimal path in its homotopy class, we have the following minimization:
\begin{equation}
	\sigma^* = \Theta\left(\overline{\varsigma_s} * \varsigma^*_g\right).
\end{equation}
Based on the definition of $\mathcal{C}^{*,x_{g}}_{x_\star ,\zeta}$, it follows that $\varsigma^*_g \in \mathcal{C}^{*,x_{g}}_{x_\star ,\zeta}$, which leads to the determination of the range of $\sigma^*$ as:
\begin{equation}
	\label{eq_FOPPset}
	\sigma^* \in \left\{ \Theta\left(\overline{\varsigma_s} * \varsigma^*_g\right) \middle| \varsigma^*_g \in \mathcal{C}^{*,x_{g}}_{x_\star ,\zeta}\right\}.
\end{equation}

In summary, a preliminary approach for efficiently solving the TPP problem has been established. Specifically, the method involves first utilizing CDT-TCS to obtain $\mathcal{C}^{*,x_{g}}_{x_\star ,\zeta}$, and then determining the optimal motion path $\sigma^*$ within the determined range, as defined in (\ref{eq_FOPPset}). It is important to note that for the function `$\Theta$' in (\ref{eq_FOPPset}), although algorithms such as the Elastic Band method \citep{quinlan1993elastic} can be employed to quickly obtain the optimal path from the path class $\left[\overline{\varsigma_s} * \varsigma^*_g\right]_{\simeq_p}$, this step can actually be omitted. In the following, we will further simplify the solution process described by (\ref{eq_FOPPset}) for greater efficiency.

Based on \textbf{Definition~\ref{defInvMAPTstar}}, (\ref{eqMAPTstarProd2}) and (\ref{eqMAPTstarProd}), the extremum function in (\ref{eq_FOPPset}) can be expressed as:
\begin{align}
	\Theta\left(\overline{\varsigma_s} * \varsigma^*_g\right) & = \overline{\Gamma^*} \circ \left(\Gamma^* \circ (\overline{\varsigma_s} * \varsigma^*_g)\right)  \nonumber     \\
	                                                                                                                    & = \overline{\Gamma^*} \circ \left\langle x_\star, {\tt{RBF}}(\overline{\varrho_s}*\varrho_g), x_g\right\rangle,
\end{align}
where $\varrho_s=\Gamma^*_{\varrho}\circ \varsigma_s$ and $\varrho_g=\Gamma^*_{\varrho} \circ \varsigma^*_g$. Therefore, the value of the function `$\Theta$' in (\ref{eq_FOPPset}) is determined by the values of $\varrho_g$, which leads to the following:
\begin{align}
	\label{eq_FOPPset1}
	 & \left\{ \Theta\left(\overline{\varsigma_s} * \varsigma^*_g\right) \middle| \varsigma^*_g \in \mathcal{C}^{*,x_{g}}_{x_\star ,\zeta}\right\} = \nonumber \\
	 & \qquad\ \left\{\overline{\Gamma^*} \circ \langle x_\star, {\tt{RBF}}(\overline{\varrho_s}*\varrho_g), x_g\rangle \middle| \varrho_g \in  \Gamma^*_{\varrho} \circ \mathcal{C}^{*,x_{g}}_{x_\star ,\zeta}\right\}.
\end{align}
Furthermore, according to (\ref{eq_GetTCS}), we have:
\begin{align}
	\label{eq_FOPPset2}
	 & \Gamma^*_{\varrho} \circ \mathcal{C}^{*,x_{g}}_{x_\star ,\zeta} =  \nonumber                                                                                                                         \\
	 & \quad \left\{ \varrho_g \in P^*_\zeta(\mathcal{G}_{con};\mathbf{x}_\star,\mathbf{x}_g) \middle| c\left(\overline{\Gamma^*} \circ \langle x_\star, \varrho_g, x_g \rangle\right) \leq \zeta\right\}.
\end{align}
Based on equations (\ref{eq_FOPPset}), (\ref{eq_FOPPset1}), and (\ref{eq_FOPPset2}), the feasible range of $\sigma^*$ can be updated as follows:
\begin{align}
	\label{eq_FOPPsetE}
	\sigma^* & \in \left\{ \Theta\left(\overline{\varsigma_s} * \varsigma^*_g\right) \middle| \varsigma^*_g \in \mathcal{C}^{*,x_{g}}_{x_\star ,\zeta}\right\}\nonumber  \\
	         & =\left\{\overline{\Gamma^*} \circ \langle x_\star, {\tt{RBF}}(\overline{\varrho_s}*\varrho_g), x_g\rangle \middle| \varrho_g \in  \Gamma^*_{\varrho} \circ \mathcal{C}^{*,x_{g}}_{x_\star ,\zeta}\right\}\nonumber  \\
	         & = \scalebox{0.86}{$\left\{\overline{\Gamma^*} \circ \langle x_\star, {\tt{RBF}}(\overline{\varrho_s}*\varrho_g), x_g\rangle \middle| \begin{array}{l}
		                                                                                                                               \varrho_g \in P^*_\zeta(\mathcal{G}_{con};\mathbf{x}_\star,\mathbf{x}_g) \\
		                                                                                                                               \land c\left(\overline{\Gamma^*} \circ \langle x_\star, \varrho_g, x_g \rangle\right) \leq \zeta
	                                                                                                                               \end{array}\right\}$}.
\end{align}
Since \textbf{Algorithm~\ref{alg_TCS}} can efficiently generate the set $P^*_\zeta(\mathcal{G}_{con};\mathbf{x}_\star,\mathbf{x}_g)$ for any $\mathbf{x}_g \in V_{con}$, the solution to the optimal tethered path planning problem becomes a straightforward and rapid process. Specifically, it involves traversing $P^*_\zeta(\mathcal{G}_{con}; \mathbf{x}_\star,\mathbf{x}_g)$ to find the configuration that minimizes $\overline{\Gamma^*} \circ \langle x_\star, {\tt{RBF}}(\varrho_s*\varrho_g), x_g\rangle$.\footnote{The implementation of the ${\tt{RBF}}$ (clear rollback path) is relatively simple and can be likened to string processing in programming.}

The pseudocode for the aforementioned solution process is presented in \textbf{Algorithm~\ref{alg_TPP}} (CDT-TPP). Prior to running CDT-TPP, \textbf{Algorithm~\ref{alg_TCS}} needs to be executed once for preprocessing. The CDT-TPP algorithm begins by initializing $\varrho_s$, $\mathbf{x}_g$, and $\sigma^*$ based on the given planning task (lines 1-4). Subsequently, the algorithm iterates through $P^*_\zeta(\mathcal{G}_{con};\mathbf{x}_\star,\mathbf{x}_g)$ to identify the optimal motion path (lines 5-10). It is important to note that during this process, $\varrho_g$ that do not satisfy the tether length constraint are excluded (lines 6, 7).

\begin{algorithm}[H]
	\caption{CDT-TPP.\Comment{Pre-run Algorithm 1}}
	\begin{algorithmic}[1]
		\State {\textbf{Input: }}$\varsigma_s$, $x_g$
		\State $\varrho_s \gets \Gamma^*_\varrho \circ \varsigma_s$
		\State $\mathbf{x}_g \gets {\tt{GetConvexPolygon}}(x_g)$
		\State $\sigma^* \gets {\tt{null}}$ \Comment{$c({\tt{null}}) = \infty$.}
		\State \textbf{for} $\varrho_g \in P^*_\zeta(\mathcal{G}_{con};\mathbf{x}_\star,\mathbf{x}_g) $ \textbf{do}
		\State \hspace{0.25cm}$\varsigma^*_g \gets \overline{\Gamma^*} \circ \langle x_\star, \varrho_g, x_{g}\rangle$
		\State \hspace{0.25cm}\textbf{if} $c(\varsigma^*_g) \leq \zeta$ \textbf{then}
		\State \hspace{0.50cm}$\sigma \gets \overline{\Gamma^*} \circ \langle x_\star, {\tt{RBF}}(\overline{\varrho_s}  * \varrho_g), x_{g}\rangle$
		\State \hspace{0.50cm}\textbf{if} $c(\sigma) < c(\sigma^*)$ \textbf{then}
		\State \hspace{0.75cm}$\sigma^* \gets \sigma$
		\State \textbf{return} $\sigma^*$
	\end{algorithmic}
	\label{alg_TPP}
\end{algorithm}

In summary, we have designed an efficient algorithm (CDT-TPP) to solve \textbf{Problem~\ref{proTPP}}. However, there appears to be a slight lack of rigor in the preceding discussion. Specifically, we only considered the case where the constraint on the right-hand side of (\ref{eqProbPSet}) holds for $t = 1$. As a result, the solution $\sigma^*$ determined for \textbf{Problem~\ref{proTPP}} does not seem to guarantee that $c\left(\Theta(\varsigma_s * \sigma_{[0,t]})\right) \leq \zeta$ holds for all $t \in I$ (although, in fact, the solution $\sigma^*$ does satisfy this condition). To address potential confusion among readers regarding this issue, we will introduce an additional discussion below. We aim to prove that if $\sigma^*$ satisfies the constraint for $t = 1$ on the right-hand side of (\ref{eqProbPSet}), then it satisfies the constraint for all $t \in I$.\footnote{If you are a practitioner in this field and are primarily concerned with the basic idea of CDT-TPP, you may choose to skip the subsequent discussion in this subsection. This discussion will not lead to any changes in the CDT-TPP algorithm itself.}

\begin{figure}[!t]
	\centering
	\includegraphics[height=3.0in]{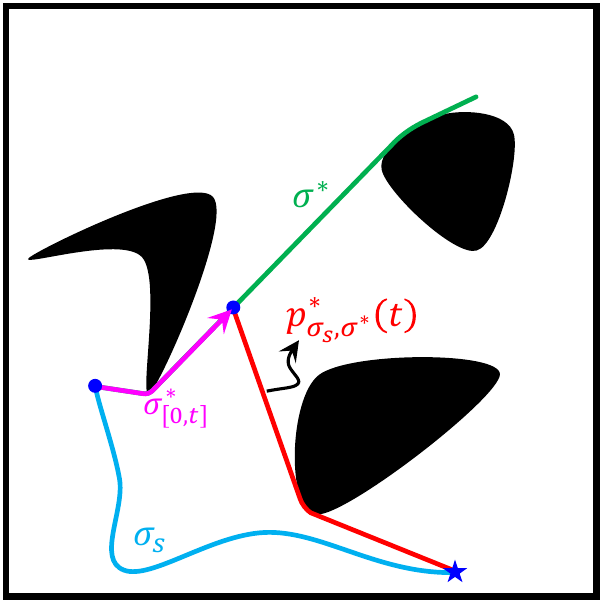}
	\caption{Illustration of the function $p^*_{\sigma_s, \sigma^*}$. The function $p^*_{\sigma_s, \sigma^*}(t)$ represents the optimal homotopic path within the homotopy class $\left[\sigma_s * \sigma^*_{[0,t]}\right]_{\simeq_p}$.}
	\label{fig_genfmap}
\end{figure}

Referring to Fig.~\ref{fig_genfmap}, we can construct two functions analogous to $p^*_{\sigma_s, l}$ and $f$:
\begin{equation}
\label{eq_genpfun}
p^*_{\sigma_s, \sigma^*}(t) = \Theta\left(\sigma_s * \sigma^*_{[0,t]}\right),
\end{equation}
where $\sigma^*$ is an arbitrary optimal homotopic path in $P(X_{free})$, and $\sigma^*(0) = \sigma_s(1)$;
\begin{equation}
	\label{eq_genffun}
	g(t) = \left(c \circ p^*_{\sigma_s, \sigma^*}\right) (t).
\end{equation}
In fact, $p^*_{\sigma_s, l}$ and $f$ can be regarded as the special cases of $p^*_{\sigma_s, \sigma^*}$ and $g$, respectively, when $\sigma^*$ is a straight line. Similar to the function $f$, we can prove the convexity of $g$.
\begin{theorem}
	\label{th_genfconvex}
	The function $g$ is convex.
\end{theorem}
\begin{proof}
	See Appendix 3.8.\qed
\end{proof}

Based on \textbf{Theorem~\ref{th_genfconvex}}, the following corollary can be derived to ensure that the solution $\sigma^*$ of \textbf{Problem~\ref{proTPP}} satisfies $c\left(\Theta(\varsigma_s * \sigma_{[0,t]})\right) \leq \zeta$ for any $t \in I$.
\begin{corollary}
	\label{Cor_genf}
	For any two configurations $\varsigma_s$ and $\varsigma_g$, if they satisfy $c(\varsigma_s) \leq \zeta$ and $c(\varsigma_g) \leq \zeta$, then for any $t \in I$, $\Theta(\varsigma_s * \sigma_{[0,t]})$ satisfies:
	\begin{equation}
		c\left(\Theta(\varsigma_s * \sigma^*_{[0,t]})\right) \leq \zeta,
	\end{equation}
	where $\sigma^*=\Theta(\overline{\varsigma_s} * \varsigma_g)$.
\end{corollary}

\subsection{Optimal Multi-Goal Visit for Tethered Robots}

The optimal multi-goal visiting problem is an important issue in path planning, with wide applications in logistics, inspection, tourism planning, and other fields. For traditional optimal multi-goal visiting problems, they are often decomposed into multiple independent optimal path planning sub-problems for resolution. However, as shown in Fig.~\ref{fig_TMVdep}, for the optimal multi-goal visiting problem of tethered robots (\textbf{Problem~\ref{proTMV}}, TMV), the topological constraints introduced by the tether cause the subpath planning tasks to be interdependent and exhibit aftereffects. Consequently, these subtasks cannot be considered independently or resolved simply through Dynamic Programming (DP) methods to determine the solution to the optimal TMV problem. In this subsection, we will discuss how to transform the TMV problem into a generalized path planning problem in a directed graph space with positive weights, and propose the CDT-TMV algorithm based on Dijkstra's algorithm to solve such problems. Before formally introducing the CDT-TMV algorithm, we will first analyze the properties of the optimal solution $\gamma^*$ for \textbf{Problem~\ref{proTMV}}.

\begin{figure*}[!t]
	\centering
	\subfloat[]{\includegraphics[width=1.65in]{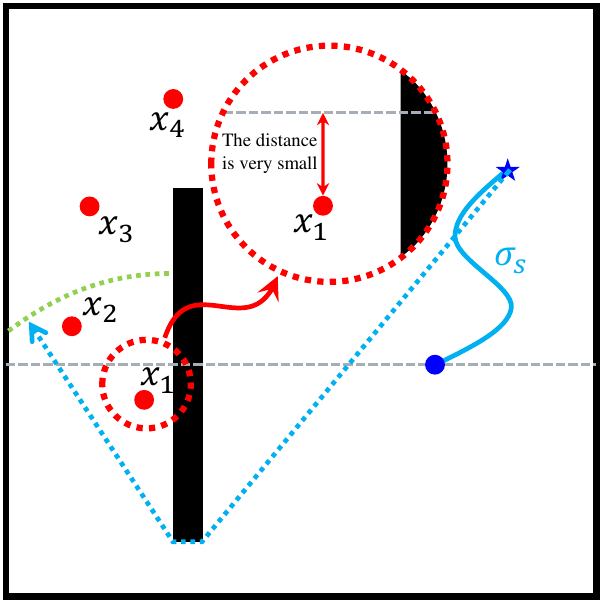}}
	\hfil
	\subfloat[]{\includegraphics[width=1.65in]{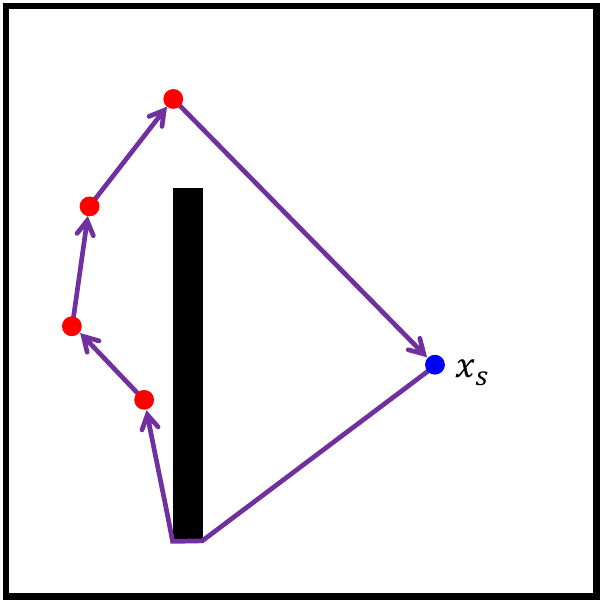}}
	\hfil
	\subfloat[]{\includegraphics[width=1.65in]{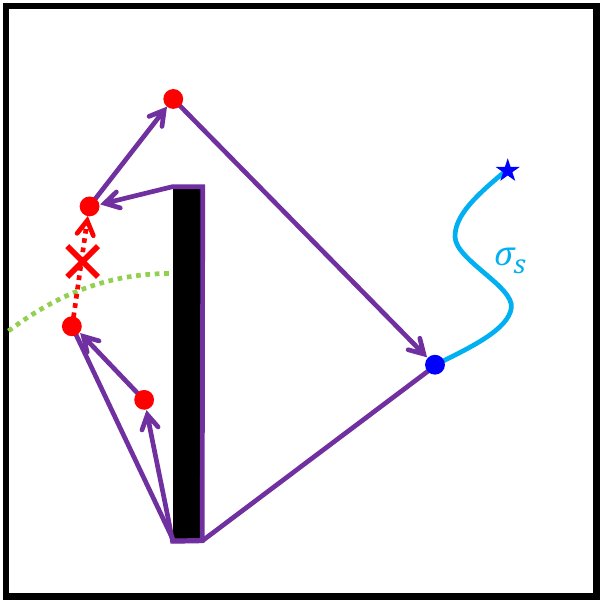}}
	\hfil
	\subfloat[]{\includegraphics[width=1.65in]{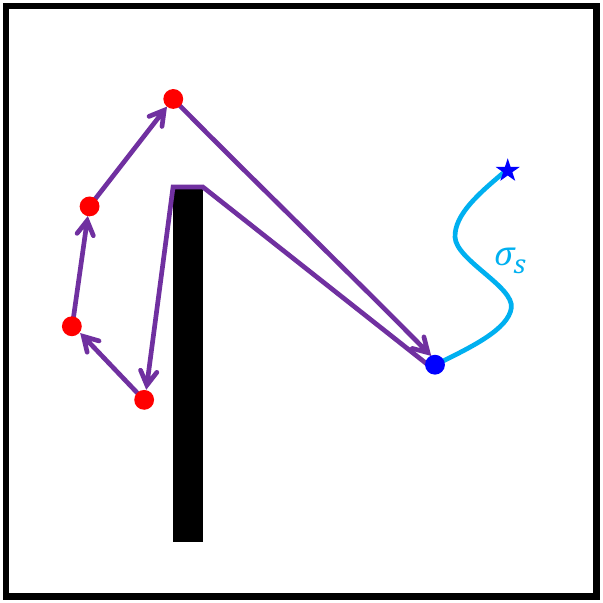}}
	\caption{Illustration of a simple TMV problem. (a) Task overview of the TMV problem, where $\sigma_s$ represents the initial configuration, the gray dashed lines indicate the axes of symmetry for rectangular obstacles, the blue dashed arrows depict the state when the tether is stretched to its maximum length $\zeta$, and the green dashed line shows the farthest position the tethered robot can reach when passing below the obstacle. $x_1$ to $x_4$ represent the points that need to be visited sequentially. (b) For the multi-goal visiting problem of non-tethered robots, it can be decomposed into multiple independent optimal path planning subproblems. (c) Path planned using the Dynamic Programming (DP) method. (d) The optimal path for this TMV problem. Clearly, due to the presence of aftereffects in the TMV problem, the DP method cannot be used to find the optimal path.}
	\label{fig_TMVdep}
\end{figure*}

\begin{theorem}
	\label{th_TMVPath}
	For any set of configurations $\varsigma^*_1 \in \mathcal{C}^{*,x_1}_{x_\star ,\zeta}$, $\varsigma^*_2 \in \mathcal{C}^{*,x_2}_{x_\star ,\zeta}$, $\dots$, $\varsigma^*_N \in \mathcal{C}^{*,x_N}_{x_\star ,\zeta}$, define the path $\gamma$ as:
	\begin{align}
		\label{eq_TMVPath}
		\gamma = \Theta\left(\overline{\varsigma^*_s} * \varsigma^*_1\right) * \Theta\left(\overline{\varsigma^*_1} * \varsigma^*_2\right) * \dots * \Theta\left(\overline{\varsigma^*_N} * \varsigma^*_s\right).
	\end{align}
	Then, the path $\gamma$ is a solution to the TMV problem that sequentially visits $x_1, x_2, \dots, x_N$.
\end{theorem}
\begin{proof}
	See Appendix 3.9.\qed
\end{proof}

\begin{lemma}
	\label{th_optTMVsubPath}
	Let $\gamma^*$ denote the optimal solution to the TMV problem, which can be divided into $N+1$ locally optimal subpaths. That is, for any $\sigma_{i,j} \in \{\sigma_{s,1},\sigma_{1,2},\dots,\sigma_{N,s}\}$ in (\ref{eqProbTMV1}), the following holds:
	\begin{align}
		\sigma_{i,j} = \Theta\left(\sigma_{i,j}\right).
	\end{align}
\end{lemma}
\begin{proof}
	See Appendix 3.10.\qed
\end{proof}

\begin{theorem}
	\label{th_optTMVPath}
	Let $\gamma^*$ denote the optimal solution to the TMV problem. Then, there exists a set of configurations $\varsigma^*_1 \in \mathcal{C}^{*,x_1}_{x_\star ,\zeta}$, $\varsigma^*_2 \in \mathcal{C}^{*,x_2}_{x_\star ,\zeta}$, $\dots$, $\varsigma^*_N\in \mathcal{C}^{*,x_N}_{x_\star ,\zeta}$, such that:
	\begin{align}
		\label{eq_optTMVPath}
		\gamma^* = \Theta\left(\overline{\varsigma^*_s} * \varsigma^*_1\right) * \Theta\left(\overline{\varsigma^*_1} * \varsigma^*_2\right) * \dots * \Theta\left(\overline{\varsigma^*_N} * \varsigma^*_s\right).
	\end{align}
\end{theorem}
\begin{proof}
	See Appendix 3.11.\qed
\end{proof}

According to \textbf{Theorem~\ref{th_TMVPath}} and \textbf{Theorem~\ref{th_optTMVPath}}, the simple way to obtain the optimal solution $\gamma^*$ for the TMV problem is to traverse all combinations of configurations in $\mathcal{C}^{*,x_1}_{x_\star ,\zeta}, \mathcal{C}^{*,x_2}_{x_\star ,\zeta}, \dots, \mathcal{C}^{*,x_N}_{x_\star ,\zeta}$ and use (\ref{eq_optTMVPath}) to find the optimal $\gamma^*$. However, this approach requires a total of $(N+1)\bigcup_{i\in N} \left\lvert\mathcal{C}^{*,x_i}_{x_\star ,\zeta}\right\rvert$ calculations of locally optimal paths $\Theta(\overline{\varsigma^*_i} * \varsigma^*_{x_j})$. In terms of computational complexity, this is clearly impractical. To address this, we will now introduce how to use the idea of Dijkstra algorithm to solve the TMV problem.

Let $\gamma$ denote an arbitrary solution to the TMV problem (visiting $N$ targets) as expressed in (\ref{eq_TMVPath}), and let $(\varsigma^*_1, \varsigma^*_2, \dots, \varsigma^*_N)$ represent the sequence of configurations used to determine $\gamma$. Define $\gamma'$ as the path determined by its subsequence (the first $N-1$ elements), i.e.,
\begin{align}
	\gamma' = \Theta\left(\overline{\varsigma^*_s} * \varsigma^*_1\right) * \Theta\left(\overline{\varsigma^*_1} * \varsigma^*_2\right) * \dots * \Theta\left(\overline{\varsigma^*_{N-1}} * \varsigma^*_s\right).
\end{align}
According to \textbf{Theorem~\ref{th_TMVPath}}, it is not difficult to see that $\gamma'$ is a solution to the TMV problem for sequentially visiting $N-1$ targets (i.e., $x_1, x_2, \dots, x_{N-1}$). In other words, for any $\gamma$ that visits $N$ targets, it can be regarded as being derived from a $\gamma'$ that visits $N-1$ targets, i.e., by appending a configuration from $\varsigma^*_N$ to the configuration sequence of $\gamma'$. It is worth noting that the index $N$ here can be replaced by any number from $1$ to $N$.

\begin{figure*}[!t]
	\centering
	\includegraphics[width=6.7in]{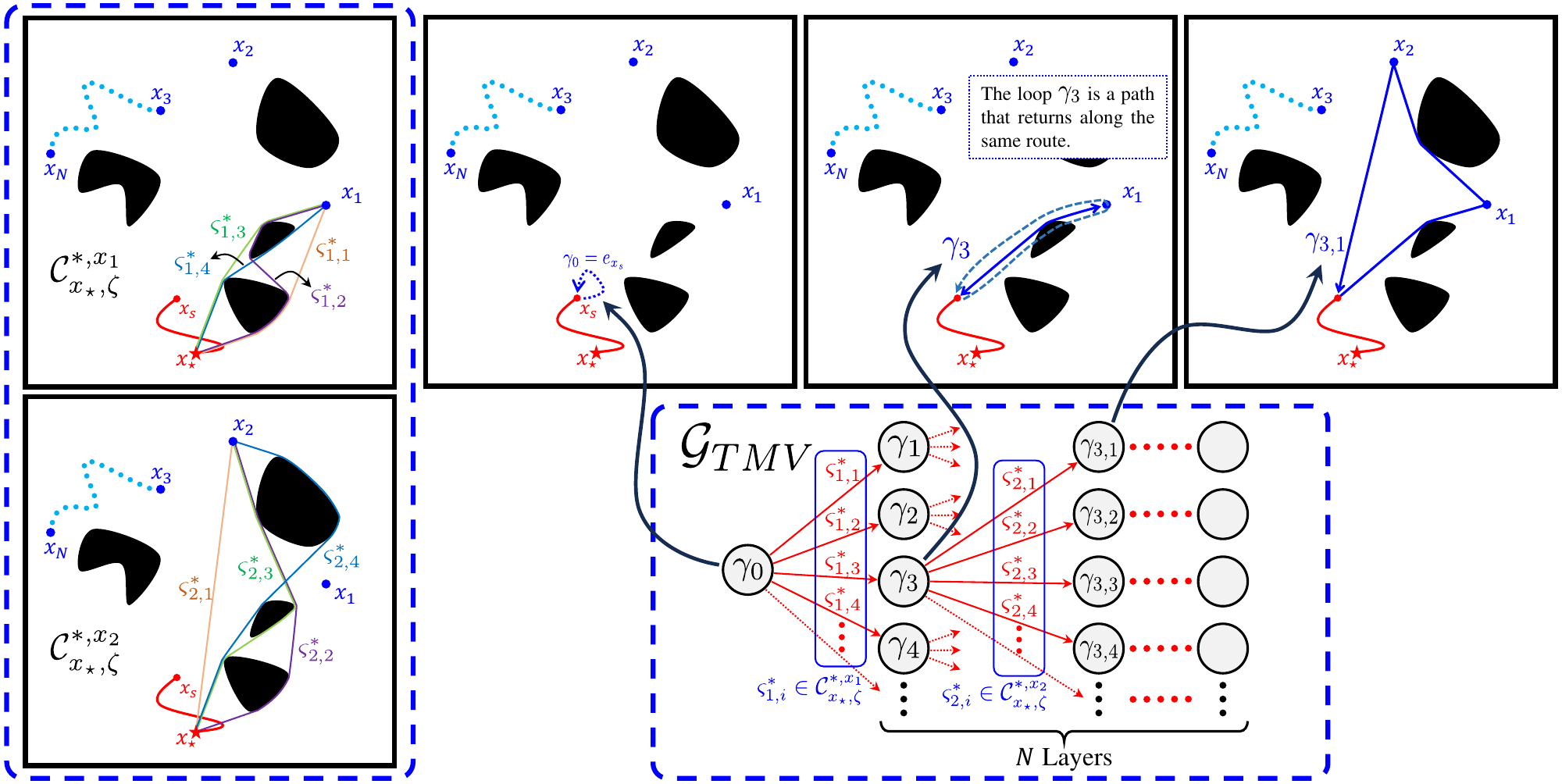}
	\caption{Illustration of the structure of the directed graph $\mathcal{G}_{TMV}$. The node in the graph corresponds to a solution to the TMV problem involving the sequential visit of $k$ goals. The edges in the graph $\mathcal{G}_{TMV}$ correspond to configurations in the sets $\mathcal{C}^{*,x_1}_{x_\star ,\zeta}, \mathcal{C}^{*,x_2}_{x_\star ,\zeta}, \dots, \mathcal{C}^{*,x_N}_{x_\star ,\zeta}$.}
	\label{fig_GTMV}
\end{figure*}

Based on the above discussion, a special directed graph $\mathcal{G}_{TMV}$ can be constructed. The nodes in this graph represent solutions $\gamma_k$ to the TMV problem that sequentially visit $k$ targets, as expressed in (\ref{eq_TMVPath}):
\begin{align}
	\gamma_k = \Theta\left(\overline{\varsigma^*_s} * \varsigma^*_1\right) * \Theta\left(\overline{\varsigma^*_1} * \varsigma^*_2\right) * \dots * \Theta\left(\overline{\varsigma^*_k} * \varsigma^*_s\right).
\end{align}
where $k \in \{1, 2, \dots, N\}$, $\varsigma^*_1 \in \mathcal{C}^{*,x_1}_{x_\star ,\zeta}$, $\varsigma^*_2 \in \mathcal{C}^{*,x_2}_{x_\star ,\zeta}$, $\dots$, $\varsigma^*_k \in \mathcal{C}^{*,x_k}_{x_\star ,\zeta}$. Additionally, we define $\gamma_0 = e_{x_s}$, meaning that for the TMV problem with zero targets to visit, the tethered robot does not move. The edges in the graph $\mathcal{G}_{TMV}$ correspond to configurations in the sets $\mathcal{C}^{*,x_1}_{x_\star ,\zeta}, \mathcal{C}^{*,x_2}_{x_\star ,\zeta}, \dots, \mathcal{C}^{*,x_N}_{x_\star ,\zeta}$. The direction of the edges is from solutions that visit $k-1$ targets to solutions that visit $k$ targets. Specifically, assume that $\mathcal{C}^{*,x_k}_{x_\star ,\zeta}$ contains $n$ configurations, denoted as:
\begin{align}
	\mathcal{C}^{*,x_k}_{x_\star ,\zeta} = \left\{\varsigma^*_{k,1},\varsigma^*_{k,2},\dots,\varsigma^*_{k,n}\right\}.
\end{align}
Let $\gamma_{k-1}$ represent a node in the graph $\mathcal{G}_{TMV}$ corresponding to a solution that visits $k-1$ targets, expressed as:
\begin{align}
	\gamma_{k-1} = \Theta\left(\overline{\varsigma^*_s} * \varsigma^*_1\right) * \Theta\left(\overline{\varsigma^*_1} * \varsigma^*_2\right) * \dots * \Theta\left(\overline{\varsigma^*_{k-1}} * \varsigma^*_s\right).
\end{align}
Then, the node $\gamma_{k-1}$ will connect to $n$ subsequent nodes (i.e., solutions that visit $k$ targets) via $n$ directed edges corresponding to the configurations in $\mathcal{C}^{*,x_k}_{x_\star ,\zeta}$. These subsequent nodes are mathematically expressed as:
\begin{align}
	 & \gamma_{k,1} = \Theta\left(\overline{\varsigma^*_s} * \varsigma^*_1\right) * \dots * \Theta\left(\overline{\varsigma^*_{k-1}} * \varsigma^*_{k,1}\right) * \Theta\left(\overline{\varsigma^*_{k,1}} * \varsigma^*_s\right), \nonumber \\
	 & \gamma_{k,2} = \Theta\left(\overline{\varsigma^*_s} * \varsigma^*_1\right) * \dots * \Theta\left(\overline{\varsigma^*_{k-1}} * \varsigma^*_{k,2}\right) * \Theta\left(\overline{\varsigma^*_{k,2}} * \varsigma^*_s\right), \nonumber \\
	 & \quad \vdots  \nonumber                                                                                                                                                                                                               \\
	 & \gamma_{k,n} = \Theta\left(\overline{\varsigma^*_s} * \varsigma^*_1\right) * \dots * \Theta\left(\overline{\varsigma^*_{k-1}} * \varsigma^*_{k,n}\right) * \Theta\left(\overline{\varsigma^*_{k,n}} * \varsigma^*_s\right).
\end{align}

\begin{remark}
	\label{re_nodeChange}
	For any node $\gamma_{k-1}$ in the directed graph $\mathcal{G}_{TMV}$, the process of transitioning to $\gamma_k$ via the directed edge $\varsigma^*_k \in \mathcal{C}^{*,x_k}_{x_\star ,\zeta}$ can be expressed mathematically as follows:
	\begin{align}
		\label{eq_FSPandRP}
		 & \gamma_{k-1} = \overbrace{\Theta\left(\overline{\varsigma^*_s} * \varsigma^*_1\right) * \dots}^{\text{Fixed subpath of }\gamma_{k-1}} * \underbrace{\Theta\left(\overline{\varsigma^*_{k-1}} * \varsigma^*_s\right)}_{\text{Being Replaced}} , \nonumber                                                                                                                                                                                                                                                                                       \\
		 & \gamma_{k} = \lefteqn{\overbrace{\phantom{\Theta\left(\overline{\varsigma^*_s} * \varsigma^*_1\right) * \dots * \Theta\left(\overline{\varsigma^*_{k-1}} * \varsigma^*_k\right)}}^{\text{Fixed subpath of }\gamma_k}}\Theta\left(\overline{\varsigma^*_s} * \varsigma^*_1\right) * \dots *\underbrace{\Theta\left(\overline{\varsigma^*_{k-1}} * \varsigma^*_k\right) * \Theta\left(\overline{\varsigma^*_k} * \varsigma^*_s\right)}_{\text{Replace }\Theta\left(\overline{\varsigma^*_{k-1}} * \varsigma^*_s\right) \text{ in }\gamma_{k-1}}.
	\end{align}
	That is, the last subpath $\Theta\left(\overline{\varsigma^*_{k-1}} * \varsigma^*_s\right)$ in $\gamma_{k-1}$ is replaced by $\Theta\left(\overline{\varsigma^*_{k-1}} * \varsigma^*_k\right) * \Theta\left(\overline{\varsigma^*_k} * \varsigma^*_s\right)$. In addition, when $k-1=0$, $\gamma_0$ can be written as $\gamma_0 = e_{x_s} = e_{x_s} * \Theta\left(\overline{\varsigma^*_s} * \varsigma^*_s\right)$. Thus, when transitioning to $\gamma_1$ via the directed edge $\varsigma^*_1 \in \mathcal{C}^{*,x_1}_{x_\star ,\zeta}$, the mathematical form still adheres to the above transformation process:
	\begin{align}
		\label{eq_FSPandRP0}
		 & \gamma_0 = e_{x_s} * \underbrace{\Theta\left(\overline{\varsigma^*_s} * \varsigma^*_s\right)}_{\text{Being Replaced}} , \nonumber                                                                                                                                                           \\
		 & \gamma_1 = \underbrace{e_{x_s}}_{\text{Omitted}} * \underbrace{\Theta\left(\overline{\varsigma^*_s} * \varsigma^*_1\right) * \Theta\left(\overline{\varsigma^*_1} * \varsigma^*_s\right)}_{\text{Replace }\Theta\left(\overline{\varsigma^*_s} * \varsigma^*_s\right) \text{ in }\gamma_0}.
	\end{align}
\end{remark}

The cost of a node in the directed graph $\mathcal{G}_{TMV}$ is defined as the path length of its corresponding solution. Consequently, the cost of each edge in $\mathcal{G}_{TMV}$ is defined as the difference between the cost of the subsequent node it connects and the current node. Taking $\gamma_{k-1}$, $\gamma_k$, and $\varsigma^*_k$ as described in \textbf{Remark~\ref{re_nodeChange}} as an example, the cost of $\varsigma^*_k$ in the graph is given by $c(\gamma_k) - c(\gamma_{k-1})$. Based on this, we can derive the following theorem:
\begin{theorem}
	\label{th_TMVGpw}
	The graph $\mathcal{G}_{TMV}$ is a positive weight graph, i.e., the cost of any edge in $\mathcal{G}_{TMV}$ is positive.
\end{theorem}
\begin{proof}
	See Appendix 3.12.\qed
\end{proof}

In summary, the optimal TMV problem can be modeled as a generalized optimal path planning problem in the graph $\mathcal{G}_{TMV}$. Specifically, taking $\gamma_0$ in $\mathcal{G}_{TMV}$ as the starting node and all nodes that visit all targets as the goal set, the task is to find the optimal path from the starting node to the goal set (i.e., identifying the node in the goal set with the lowest cost). Furthermore, according to \textbf{Theorem~\ref{th_TMVGpw}}, $\mathcal{G}_{TMV}$ is a positively weighted graph. Therefore, Dijkstra's algorithm can be employed to solve this generalized optimal path planning problem, thereby achieving the solution to the optimal TMV problem.\footnote{Based on the properties of Dijkstra's algorithm, during each iteration, the cost of the node popped from the priority queue is guaranteed to be no greater than the cost of any node whose cost has not yet been determined. This means that when Dijkstra's algorithm first pops a path visiting all targets from the priority queue, this path is guaranteed to be the optimal solution to the TMV problem.}

Based on the above conclusions, we propose the CDT-TMV algorithm to solve the optimal TMV problem. Its pseudocode is presented in \textbf{Algorithm~\ref{alg_TMV0}} and \textbf{Algorithm~\ref{alg_TMV}}. \textbf{Algorithm~\ref{alg_TMV0}} illustrates the original mathematical form of CDT-TMV (to facilitate the reader's understanding), while \textbf{Algorithm~\ref{alg_TMV}} presents the optimized form using CDT encoding (to assist practitioners in implementation).

In the pseudocode of CDT-TMV, $\mathcal{Q}_\gamma$ represents a priority queue used to store nodes that are yet to be processed. The nodes in this queue are arranged in ascending order based on their path cost as the key value, i.e.,
\begin{align}
	{\tt{Key}}^{\mathcal{Q}_\gamma}(\gamma) = c(\gamma).
\end{align}
${\tt{FixedSubPathDict}}$ is a dictionary structure that stores the fixed subpaths for each node that has been searched by the algorithm. Taking $\gamma_k$ from (\ref{eq_FSPandRP}) as an example:
\begin{align}
	&{\tt{FixedSubPathDict}}[\gamma_k] = \nonumber\\
	&\qquad \Theta\left(\overline{\varsigma^*_s} * \varsigma^*_1\right) * \Theta\left(\overline{\varsigma^*_1} * \varsigma^*_2\right) * \dots * \Theta\left(\overline{\varsigma^*_{k-1}} * \varsigma^*_k\right).
\end{align}
${\tt{NumOfTargetsVisited}}$ is a function that returns the number of targets visited by a searched node. Again, taking $\gamma_k$ from (\ref{eq_FSPandRP}) as an example:
\begin{align}
	{\tt{NumOfTargetsVisited}}[\gamma_k] = k.
\end{align}

At the beginning of CDT-TMV (\textbf{Algorithm~\ref{alg_TMV0}}), the algorithm initializes the node $\gamma_0$ and ${\tt{FixedSubPathDict}}[\gamma_0]$ based on the user-provided initial configuration $\varsigma_s$ and (\ref{eq_FSPandRP0}). The initial node $\gamma_0$ is then pushed into the priority queue $\mathcal{Q}_\gamma$ (lines 2-5). Subsequently, the algorithm enters an iterative process. In each iteration, the algorithm pops a node $\gamma$ from the priority queue and retrieves the number of targets visited by $\gamma$ (lines 12, 13). If $\gamma$ has visited all targets ($k \equiv N$ is true), then $\gamma$ is the solution to the input optimal TMV problem, and the algorithm terminates immediately, returning $\gamma$ (line 9). If $k \equiv N$ is false, the algorithm retrieves the fixed subpath of the current node $\gamma$ using ${\tt{FixedSubPathDict}}$. It then iterates over all configurations in $\mathcal{C}^{*,x_{k+1}}_{x_\star ,\zeta}$ and uses (\ref{eq_FSPandRP}) to compute all subsequent nodes $\gamma_{new}$ of $\gamma$ (lines 12-14). During this process, the algorithm records the fixed subpath of $\gamma_{new}$ in ${\tt{FixedSubPathDict}}$ and pushes $\gamma_{new}$ into the queue $\mathcal{Q}_\gamma$ (lines 15, 16).

Finally, it is worth noting that CDT-TMV does not include a relaxation process as in the traditional Dijkstra's algorithm. This is because $\mathcal{G}_{TMV}$ is essentially a tree-like structure rooted at $\gamma_0$. In other words, for any node $\gamma_k$ in $\mathcal{G}_{TMV}$, there exists only one unique parent node. Additionally, if the input TMV problem has no solution, CDT-TMV will terminate due to $\mathcal{Q}_\gamma \not\equiv \varnothing$ becoming false. This situation arises when there exists a $\mathcal{C}^{*,x_k}_{x_\star ,\zeta}$ is empty. For such cases, we can pre-check whether any $\mathcal{C}^{*,x_k}_{x_\star ,\zeta}$ is empty before entering the iterative process (lines 6-16). If such a case is detected, CDT-TMV directly returns `no solution.'

\begin{algorithm}[H]
	\caption{CDT-TMV.\Comment{Original Mathematical Form}}
	\begin{algorithmic}[1]
		\State {\textbf{Input: }}$\varsigma_s$, $x_1$, $x_2$, $\dots$, $x_N$
		\State $x_s \gets \varsigma_s(1)$
		\State $\gamma_0 \gets e_{x_s}$
		\State ${\tt{FixedSubPathDict}}[\gamma_0] \gets e_{x_s}$
		\State $\mathcal{Q}_\gamma \gets \{\gamma_0\}$
		\State \textbf{while} $\mathcal{Q}_\gamma \not\equiv \varnothing$ \textbf{do}
		\State \hspace{0.25cm}$\gamma \gets {\tt{PopFirst}}(\mathcal{Q}_\gamma)$
		\State \hspace{0.25cm}$k \gets {\tt{NumOfTargetsVisited}}(\gamma)$
		\State \hspace{0.25cm}\textbf{if} $k \equiv N$ \textbf{then return} $\gamma$
		\State \hspace{0.25cm}$\sigma^\gamma_{sub} \gets {\tt{FixedSubPathDict}}[\gamma]$
		\State \hspace{0.25cm}\textbf{for} $\varsigma^*_{k+1} \in \mathcal{C}^{*,x_{k+1}}_{x_\star ,\zeta}$ \textbf{do}
		\State \hspace{0.50cm}$\sigma^{x_{k+1}}_{x_k} \gets \Theta\left(\overline{\varsigma^*_{k}} * \varsigma^*_{k+1}\right)$
		\State \hspace{0.50cm}$\sigma^{x_s}_{x_{k+1}} \gets \Theta\left(\overline{\varsigma^*_{k+1}} * \varsigma^*_s\right)$
		\State \hspace{0.50cm}$\gamma_{new} \gets \sigma^{\gamma_{new}}_{sub} * \sigma^{x_{k+1}}_{x_k} * \sigma^{x_s}_{x_{k+1}}$
		\State \hspace{0.50cm}${\tt{FixedSubPathDict}}[\gamma_{new}] \gets \sigma^{\gamma_{new}}_{sub} * \sigma^{x_{k+1}}_{x_k}$
		\State \hspace{0.50cm}Push $\gamma_{new}$ into $\mathcal{Q}_\gamma$
	\end{algorithmic}
	\label{alg_TMV0}
\end{algorithm}

Regarding the optimized version of CDT-TMV using CDT encoding (\textbf{Algorithm~\ref{alg_TMV}}), its main workflow is identical to \textbf{Algorithm~\ref{alg_TMV0}}. The key difference in (\textbf{Algorithm~\ref{alg_TMV}}) lies in how it obtains the CDT encoding of configurations in $\mathcal{C}^{*,x_{k+1}}_{x_\star ,\zeta}$ by combining $P^*_\zeta(\mathcal{G}_{con}; \mathbf{x}_\star, \mathbf{x}_{k+1})$ with the condition in line 15. Furthermore, \textbf{Algorithm~\ref{alg_TMV}} replaces the function $\Theta$ used in \textbf{Algorithm~\ref{alg_TMV0}} with $\overline{\Gamma^*}$ (i.e., \textbf{Algorithm~\ref{alg_ShortestPath}}).
\begin{algorithm}[H]
	\caption{CDT-TMV.\Comment{Pre-run Algorithm 1}}
	\begin{algorithmic}[1]
		\State {\textbf{Input: }}$\varsigma_s$, $x_1$, $x_2$, $\dots$, $x_N$
		\State $x_s \gets \varsigma_s(1)$
		\State $\varrho_s \gets \Gamma^*_\varrho \circ \varsigma_s$
		\State $\gamma_0 \gets e_{x_s}$
		\State ${\tt{FixedSubPathDict}}[\gamma_0] \gets e_{x_s}$
		\State $\mathcal{Q}_\gamma \gets \{\gamma_0\}$

		\State \textbf{while} $\mathcal{Q}_\gamma \not\equiv \varnothing$ \textbf{do}
		\State \hspace{0.25cm}$\gamma \gets {\tt{PopFirst}}(\mathcal{Q}_\gamma)$
		\State \hspace{0.25cm}$k \gets {\tt{NumberOfTargetsVisited}}(\gamma)$
		\State \hspace{0.25cm}\textbf{if} $k \equiv N$ \textbf{then return} $\gamma$
		\State \hspace{0.25cm}$\sigma^\gamma_{sub} \gets {\tt{FixedSubPathDict}}[\gamma]$
		\State \hspace{0.25cm}$\varrho^\gamma_{sub} \gets \Gamma^*_\varrho \circ \sigma^\gamma_{sub}$
		\State \hspace{0.25cm}$\mathbf{x}_{k+1} \gets {\tt{GetConvexPolygon}}(x_{k+1})$
		\State \hspace{0.25cm}\textbf{for} $\varrho_{next} \in P^*_\zeta(\mathcal{G}_{con};\mathbf{x}_\star,\mathbf{x}_{k+1})$ \textbf{do}
		\State \hspace{0.50cm}\textbf{if} $c\left(\overline{\Gamma^*} \circ \langle x_\star, \varrho_{next}, x_{k+1}\rangle\right) > \zeta$ \textbf{then Continue}
		\State \hspace{0.50cm}$\sigma^{x_{k+1}}_{x_k} \gets \overline{\Gamma^*} \circ \langle x_k, {\tt{RBF}}(\overline{\varrho^\gamma_{sub}}  * \varrho_{next}), x_{k+1}\rangle $
		\State \hspace{0.50cm}$\sigma^{x_s}_{x_{k+1}} \gets \overline{\Gamma^*} \circ \langle x_{k+1}, {\tt{RBF}}(\overline{\varrho_{next}}  * \varrho_{s}), x_s\rangle $

		\State \hspace{0.50cm}$\gamma_{new} \gets \sigma^{\gamma_{new}}_{sub} * \sigma^{x_{k+1}}_{x_k} * \sigma^{x_s}_{x_{k+1}}$
		\State \hspace{0.50cm}${\tt{FixedSubPathDict}}[\gamma_{new}] \gets \sigma^{\gamma_{new}}_{sub} * \sigma^{x_{k+1}}_{x_k}$
		\State \hspace{0.50cm}Push $\gamma_{new}$ into $\mathcal{Q}_\gamma$
	\end{algorithmic}
	\label{alg_TMV}
\end{algorithm}

\subsection{Planning the Optimal Path of the Mobile Robot}
This subsection introduces an improved version of CDT-TPP, termed CDT-UTPP. CDT-UTPP is primarily designed for distance-optimal path planning of mobile robots (untethered) as described in \textbf{Problem~\ref{proUTPP}}. The key feature of CDT-UTPP is that, after a preprocessing phase similar to \textbf{Algorithm~\ref{alg_TCS}}, it can rapidly return the optimal path from $x_s$ to $x_g$ in $X_{free}$ for arbitrary $x_s,x_g \in X_{free}$. Before formally introducing the CDT-UTPP algorithm, we will first analyze in detail the properties of the globally optimal path in \textbf{Problem~\ref{proUTPP}}.

\begin{figure}[!t]
	\centering
	\includegraphics[height=3.0in]{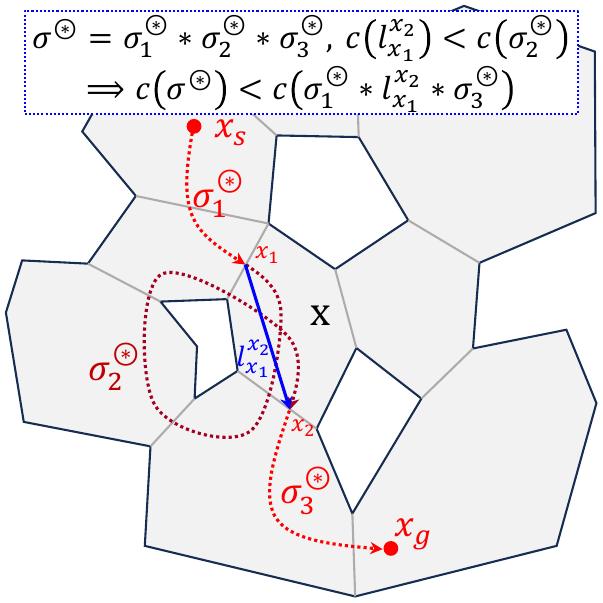}
	\caption{If $\Gamma^*_{\varrho} \circ \sigma^\circledast$ contains duplicate elements, then $\sigma^\circledast$ is necessarily not the globally optimal path.}
	\label{fig_THoptPath}
\end{figure}

\begin{theorem}
	\label{th_optPath}
	If $\sigma^\circledast$ is the optimal path in $P(X_{free};x_s,x_g)$ (as described in (\ref{eqProb4})), then the sequence $\Gamma^*_{\varrho}\circ\sigma^\circledast$ contains no repeated convex polygons.
\end{theorem}
\begin{proof}
	See Appendix 3.13.\qed
\end{proof}


\begin{theorem}
	\label{th_prUTPP}
	Let $\sigma^\circledast \in P(X_{free};x_s,x_g)$ be the optimal path. For any $x_\star \in X_{free}$ that can be connected to with $\sigma^\circledast$,\footnote{This assumption guarantees that $P(X_{free};x_\star,x_s)$, $P(X_{free};x_\star,x_g)$ are not empty.} there exist optimal homotopic paths $\sigma^*_s\in P(X_{free};x_\star,x_s)$ and $\sigma^*_g\in P(X_{free};x_\star,x_g)$ such that
	\begin{equation}
		\label{eq_prUTPP0}
		\sigma^\circledast \simeq_p \overline{\sigma^*_s} * \sigma^*_g,
	\end{equation}
	and neither $\Gamma^*_{\varrho}\circ\sigma^*_s$ nor $\Gamma^*_{\varrho}\circ\sigma^*_g$ contains repeated elements.
\end{theorem}
\begin{proof}
	The proof of this theorem is divided into the following two sub-propositions:\\
	\textbf{Sub-proposition 1}: There exist $\sigma^*_s$ and $\sigma^*_g$ satisfying (\ref{eq_prUTPP0}).\\
	\textbf{Sub-proposition 2}: For any pair of $\sigma^*_s$ and $\sigma^*_g$ satisfying (\ref{eq_prUTPP0}), they can be transformed into a new pair of optimal homotopic paths $\sigma^*_{sk}$ and $\sigma^*_{gk}$, which still satisfy (\ref{eq_prUTPP0}), while neither $\Gamma^*_{\varrho}\circ\sigma^*_{sk}$ nor $\Gamma^*_{\varrho}\circ\sigma^*_{gk}$ contains repeated elements.\\
	See Appendix 3.14.\qed
\end{proof}

Referring to \textbf{Remark~\ref{re_PeqC}}, we can consider $x_\star$ in \textbf{Theorem~\ref{th_prUTPP}} as the anchor point of a tethered robot, while $\sigma^*_s$ and $\sigma^*_g$ correspond to the optimal homotopic configurations for this tethered robot. Therefore, when the tether length $\zeta$ is set sufficiently long, based on \textbf{Sub-proposition 1} of \textbf{Theorem~\ref{th_prUTPP}}, for any $\sigma^\circledast$, there exist $\varsigma^*_s \in \mathcal{C}^{*,x_{s}}_{x_\star ,\zeta}$ and $\varsigma^*_g \in \mathcal{C}^{*,x_{g}}_{x_\star ,\zeta}$, such that
\begin{equation}
	\label{eq_prUTPP2Conf}
	\sigma^\circledast \simeq_p \overline{\varsigma^*_s} * \varsigma^*_g.
\end{equation}
Additionally, based on \textbf{Sub-proposition 2} of \textbf{Theorem~\ref{th_prUTPP}}, we can further restrict the range of $\varsigma^*_s$ and $\varsigma^*_g$ in (\ref{eq_prUTPP2Conf}):
\begin{align}
	\varsigma^*_s \in \mathcal{C}^{\circledast,x_{s}}_{x_\star ,\zeta},\quad \varsigma^*_g \in \mathcal{C}^{\circledast,x_{g}}_{x_\star ,\zeta},
\end{align}
where the configuration set $\mathcal{C}^{\circledast,x}_{x_\star,\zeta}$ is defined as:
\begin{equation}
	\label{eq_UFOCset}
	\mathcal{C}^{\circledast,x}_{x_\star ,\zeta} = \left\{ \varsigma^* \in \mathcal{C}^{*,x}_{x_\star ,\zeta} \middle| \begin{array}{l} \Gamma^*_{\varrho}\circ\varsigma^* \text{ contains no} \\
		\text{repeated elements}
	\end{array}\right\}.
\end{equation}
Furthermore, the feasible range of the globally optimal path $\sigma^\circledast$ is given by:
\begin{equation}
	\label{eq_UOPPset}
	\sigma^\circledast \in \left\{ \Theta\left(\overline{\varsigma^*_s} * \varsigma^*_g\right) \middle| \varsigma^*_s \in \mathcal{C}^{\circledast,x_{s}}_{x_\star ,\zeta}, \varsigma^*_g \in \mathcal{C}^{\circledast,x_{g}}_{x_\star ,\zeta}\right\}.
\end{equation}

Similar to the simplification process from (\ref{eq_FOPPset}) to (\ref{eq_FOPPsetE}), we will further simplify the solution process described in (\ref{eq_UOPPset}) to improve efficiency. First, we define $\Gamma^*_\varrho \circ \left(\bigcup_{x \in \mathbf{x}} \mathcal{C}^{\circledast,x}_{x_\star ,\zeta}\right)$ as $P^\circledast _\zeta(\mathcal{G}_{con};\mathbf{x}_\star,\mathbf{x})$. According to (\ref{eq_UFOCset}) and (\ref{eq_GconPathSet}), we have:
\begin{align}
	 & P^\circledast_\zeta(\mathcal{G}_{con};\mathbf{x}_\star,\mathbf{x}) = \Gamma^*_\varrho \circ \left(\bigcup_{x \in \mathbf{x}} \mathcal{C}^{\circledast,x}_{x_\star ,\zeta}\right) \nonumber      \\
	 & = \Gamma^*_\varrho \circ \left\{ \varsigma^* \in \bigcup_{x \in \mathbf{x}} \mathcal{C}^{*,x}_{x_\star ,\zeta} \middle| \begin{array}{l} \Gamma^*_{\varrho}\circ\varsigma^* \text{ contains no} \\
		                                                                                                                           \text{repeated elements}
	                                                                                                                           \end{array}\right\}  \nonumber \\
	 & = \left\{ \varrho \in \Gamma^*_\varrho \circ\left(\bigcup_{x \in \mathbf{x}} \mathcal{C}^{*,x}_{x_\star ,\zeta}\right)  \middle| \begin{array}{l} \varrho \text{ contains no} \\
		                                                                                                                                    \text{repeated elements}
	                                                                                                                                    \end{array}\right\}  \nonumber                   \\
	 & = \left\{ \varrho \in P^*_\zeta(\mathcal{G}_{con};\mathbf{x}_\star,\mathbf{x}) \middle| \begin{array}{l} \varrho \text{ contains no} \\
		                                                                                           \text{repeated elements}
	                                                                                           \end{array}\right\}.
\end{align}
It is not difficult to find that $P^\circledast_\zeta(\mathcal{G}_{con};\mathbf{x}_\star,\mathbf{x})$ simply imposes additional constraints on $P^*_\zeta(\mathcal{G}_{con}; \mathbf{x}_\star,\mathbf{x})$. Thus, for $P^\circledast_\zeta(\mathcal{G}_{con};\mathbf{x}_\star,\mathbf{x})$, we can derive an iterative relation similar to (\ref{eq_CodingIteration}):
\begin{align}
	\label{eq_UTCodingIteration}
	 & P^\circledast_\zeta(\mathcal{G}_{con};\mathbf{x}_\star,\mathbf{x}) = \nonumber                        \\
	 & \bigcup_{\mathbf{x}' \in {\tt{near}}(\mathbf{x})}  \scalebox{0.95}{$\left\{ \varrho = \varrho' + (\mathbf{x}) \middle|
	\begin{array}{l}
		\varrho' \in P^\circledast(\mathcal{G}_{con};\mathbf{x}_\star,\mathbf{x}') \\
		\land \varrho' \not\ni \mathbf{x}                                \\
		\land \underset{x \in \mathbf{x}}{\mathrm{min}}\ c\left(\overline{\Gamma^*} \circ \langle x_\star, \varrho, x\rangle \right) \leq \zeta
	\end{array} \right\}$},
\end{align}
where the constraint `$\varrho' \not\ni \mathbf{x}$' corresponds to the condition that `$\varrho$ contains no repeated elements.' Moreover, due to this constraint, $P^\circledast_\zeta(\mathcal{G}_{con};\mathbf{x}_\star,\mathbf{x}_\star)$ will only contain the unique sequence $\left(\mathbf{x}_\star\right)$, i.e.,
\begin{equation}
	P^\circledast_\zeta(\mathcal{G}_{con};\mathbf{x}_\star,\mathbf{x}_\star) = \left\{\left(\mathbf{x}_\star\right)\right\}.
\end{equation}

It is worth noting that in this subsection, we attempt to apply the approach used in this paper for solving the path planning problem of tethered robots to address the path planning problem for ordinary mobile robots (\textbf{Problem~\ref{proUTPP}}). This means that we do not need to strictly compute the tether length constraint in (\ref{eq_UTCodingIteration}),\footnote{As mentioned at the beginning of this subsection, $\zeta$ is vaguely set to a sufficiently large number} i.e.,
\begin{equation}
	\label{eq_constraintUTCodingIt}
	\underset{x \in \mathbf{x}}{\mathrm{min}}\ c\left(\overline{\Gamma^*} \circ \langle x_\star, \varrho, x\rangle \right) \leq \zeta.
\end{equation}
To this end, we can construct a relatively relaxed constraint function ${\tt{lower}}C(x_\star, \varrho)$, defined as:
\begin{equation}
	{\tt{lower}}C(x_\star, \varrho) = c\left(\overline{\Gamma^*} \circ \langle x_\star, \varrho, l_c(0.5)\rangle\right) - \frac{c(l_c)}{2},
\end{equation}
where $l_c$ denotes the cutline between $\varrho(T_{\varrho})$ and $\varrho(T_{\varrho }-1)$. According to \textbf{Theorem~\ref{th_finequality}}, it is obvious that:
\begin{equation}
	{\tt{lower}}C(x_\star, \varrho) \leq \underset{x \in \mathbf{x}}{\mathrm{min}}\ c\left(\overline{\Gamma^*} \circ \langle x_\star, \varrho, x\rangle \right).
\end{equation}
By replacing the tether length constraint in (\ref{eq_UTCodingIteration}) with ${\tt{lower}}C(x_\star, \varrho) \leq \zeta$, we obtain:
\begin{align}
	\label{eq_UTCodingIteration2}
	 & P^\circledast_\zeta(\mathcal{G}_{con};\mathbf{x}_\star,\mathbf{x}) = \nonumber                        \\
	 & \bigcup_{\mathbf{x}' \in {\tt{near}}(\mathbf{x})}  \left\{ \varrho = \varrho' + (\mathbf{x}) \middle|
	\begin{array}{l}
		\varrho' \in P^\circledast(\mathcal{G}_{con};\mathbf{x}_\star,\mathbf{x}') \\
		\land \varrho' \not\ni \mathbf{x}                                \\
		\land {\tt{lower}}C(x_\star, \varrho) \leq \zeta
	\end{array} \right\}.
\end{align}
Comparing (\ref{eq_UTCodingIteration}) and (\ref{eq_UTCodingIteration2}), although (\ref{eq_UTCodingIteration2}) increases the number of elements in $P^\circledast_\zeta(\mathcal{G}_{con};\mathbf{x}_\star,\mathbf{x})$ to some extent, it significantly simplifies the computational difficulty of iteratively obtaining $P^\circledast_\zeta(\mathcal{G}_{con};\mathbf{x}_\star,\mathbf{x})$.\footnote{This is because, compared to using \textbf{Algorithm~\ref{alg_EncodingValidity}} to evaluate constraint (\ref{eq_constraintUTCodingIt}), calculating ${\tt{lower}}C(x_\star, \varrho)$ avoids the need to use ternary search to strictly determine the minimum value.}

After obtaining $P^\circledast_\zeta(\mathcal{G}_{con};\mathbf{x}_\star,\mathbf{x})$ for any $\mathbf{x}\in V_{con}$, the feasible range of the globally optimal path can be updated from (\ref{eq_UOPPset}) as follows:
\begin{equation}
	\label{eq_UOPPset2}
	\scalebox{0.93}{$\sigma^\circledast \in \left\{\overline{\Gamma^*} \circ \langle x_s, {\tt{RBF}}(\overline{\varrho_s}*\varrho_g), x_g\rangle \middle| \begin{array}{l}
		\varrho_s \in P^\circledast_\zeta(\mathcal{G}_{con};\mathbf{x}_\star,\mathbf{x}_s)\\
		\varrho_g \in P^\circledast_\zeta(\mathcal{G}_{con};\mathbf{x}_\star,\mathbf{x}_g)
	\end{array}\right\}$}.
\end{equation}

Based on the above discussion, the CDT-UTPP is mainly divided into two stages: \textbf{Preprocessing} and \textbf{Planning}, with their pseudocode shown in \textbf{Algorithm~\ref{alg_UTPP}} and \textbf{Algorithm~\ref{alg_UTPPGet}}, respectively. The Preprocessing Stage of CDT-UTPP (\textbf{Algorithm~\ref{alg_UTPP}}) is essentially similar to that of CDT-TCS (\textbf{Algorithm~\ref{alg_TCS}}), as both are primarily used to obtain $P^\circledast_\zeta(\mathcal{G}_{con};\mathbf{x}_\star,\mathbf{x})$ for all $\mathbf{x}\in V_{con}$. The differences between \textbf{Algorithm~\ref{alg_UTPP}} and \textbf{Algorithm~\ref{alg_TCS}} lie in their line 10 and line 13:
\begin{itemize}
	\item {Line 10 in \textbf{Algorithm~\ref{alg_UTPP}} corresponds to the constraint $\varrho' \not\ni \mathbf{x}$ in (\ref{eq_UTCodingIteration2}).}
	\item {Line 13 in \textbf{Algorithm~\ref{alg_UTPP}} corresponds to the constraint ${\tt{lower}}C(x_\star, \varrho) \leq \zeta$ in (\ref{eq_UTCodingIteration2}).}
\end{itemize}
The Planning Stage of CDT-UTPP (\textbf{Algorithm~\ref{alg_UTPPGet}}) demonstrates how to determine the globally optimal path $\sigma^\circledast$ from $x_s$ to $x_g$ by traversing the set described in (\ref{eq_UOPPset2}).
\begin{algorithm}[H]
	\caption{CDT-UTPP Preprocessing Stage.}
	\begin{algorithmic}[1]
		\State {\textbf{Input: }}$X_{free}$, $x_\star$, $\zeta$
		\State $\mathcal{G}_{con}\gets{\tt{InitializeFreeSpace}}(X_{free})$
		\State $\mathbf{x}_\star \gets {\tt{GetConvexPolygon}}(x_\star)$
		\State Use $\varnothing$ to initialize $\mathcal{Q}_\varrho$ and $P^\circledast_\zeta(\mathcal{G}_{con};\mathbf{x}_\star,\mathbf{x})$, $\mathbf{x}\in V_{con}$
		\State $P^\circledast_\zeta(\mathcal{G}_{con};\mathbf{x}_\star,\mathbf{x}_\star) \overset{+}{\gets} \{(\mathbf{x}_\star)\}$
		\State $\mathcal{Q}_\varrho \overset{+}{\gets} \{(\mathbf{x}_\star)\}$
		\State \textbf{while} $\mathcal{Q}_\varrho \not\equiv \varnothing$ \textbf{do}
		\State \hspace{0.25cm}$\varrho \gets {\tt{PopFirst}}(\mathcal{Q}_\varrho)$
		\State \hspace{0.25cm}\textbf{for} $\mathbf{x}_{near} \in {\tt{Neighbor}}\left(\varrho(T_\varrho)\right) $ \textbf{do}
		\State \hspace{0.50cm}\textbf{if} $\mathbf{x}_{near} \in \varrho$ \textbf{then Continue}

		\State \hspace{0.50cm}$l_c \gets {\tt{GetCutline}}(\mathbf{x}_{near},\varrho(T_{\varrho}))$
		\State \hspace{0.50cm}$c_{mid} \gets c\left(\overline{\Gamma^*} \circ \langle x_\star, \varrho + (\mathbf{x}_{near}), l_c(0.5)\rangle\right)$
		\State \hspace{0.50cm}\textbf{if} $c_{mid} - \frac{c(l_c)}{2} <= \zeta$ \textbf{then}
		\State \hspace{0.75cm}$P^\circledast_\zeta(\mathcal{G}_{con};\mathbf{x}_\star,\mathbf{x}_{near}) \overset{+}{\gets} \{\varrho_{new}\}$
		\State \hspace{0.75cm}$\mathcal{Q}_\varrho \overset{+}{\gets} \{\varrho_{new}\}$
	\end{algorithmic}
	\label{alg_UTPP}
\end{algorithm}

\begin{algorithm}[H]
	\caption{CDT-UTPP Planning Stage.}
	\begin{algorithmic}[1]
		\State {\textbf{Input: }}$x_{s}$, $x_{g}$
		\State $\mathbf{x}_s \gets {\tt{GetConvexPolygon}}(x_s)$
		\State $\mathbf{x}_g \gets {\tt{GetConvexPolygon}}(x_g)$
		\State $\sigma^\circledast \gets {\tt{null}}$
		\State \textbf{for} $\varrho_s \in P^\circledast_\zeta(\mathcal{G}_{con};\mathbf{x}_\star,\mathbf{x}_s) $ \textbf{do}
		\State \hspace{0.25cm}\textbf{for} $\varrho_g \in P^\circledast_\zeta(\mathcal{G}_{con};\mathbf{x}_\star,\mathbf{x}_g) $ \textbf{do}
		\State \hspace{0.50cm}$\sigma \gets \overline{\Gamma^*} \circ \langle x_s, {\tt{RBF}}(\overline{\varrho_s}  * \varrho_g), x_g\rangle$
		\State \hspace{0.50cm}\textbf{if} $c(\sigma) < c(\sigma^\circledast)$ \textbf{then}
		\State \hspace{0.75cm}$\sigma^\circledast \gets \sigma$
		\State \textbf{return} $\sigma^\circledast$
	\end{algorithmic}
	\label{alg_UTPPGet}
\end{algorithm}

\section{Simulation Experiment}
In the simulation experiments, we adopt four different simulation environments, as shown in Fig.~\ref{fig_simMap}, to test and analyze the performance of the proposed algorithm across various planning tasks. Each simulation environment uses a square binary image as the state space $X$, where the black regions represent the obstacle space $X_{obs}$ and the white regions represent the free space $X_{free}$. In each environment, two different initial configurations are set (red and purple curves, with the blue `$\star$' indicating the anchor point of the tethered robot), along with four different target points, which are used to design the planning tasks for the experiments. The specific parameters of the experimental environments are provided in \textbf{Table~\ref{tab:simMapData}}.

In terms of hardware, the simulations are conducted on a personal computer equipped with an Intel Core i9-13900k processor (5.40 GHz) and 32 GB of RAM (4800 MT/s). On the software side, the proposed path homotopy representation and related planning algorithms are implemented in C++ on Ubuntu 20.04. The code for the comparative methods in Subsections 6.1 and 6.2 is also implemented by us in C++ and optimized to the best of our ability. The comparative algorithms in Subsection 6.3 utilize the Open Motion Planning Library (OMPL) \citep{sucan2012open} version.

\begin{figure*}[!t]
	\centering
	\subfloat[]{\includegraphics[width=1.65in]{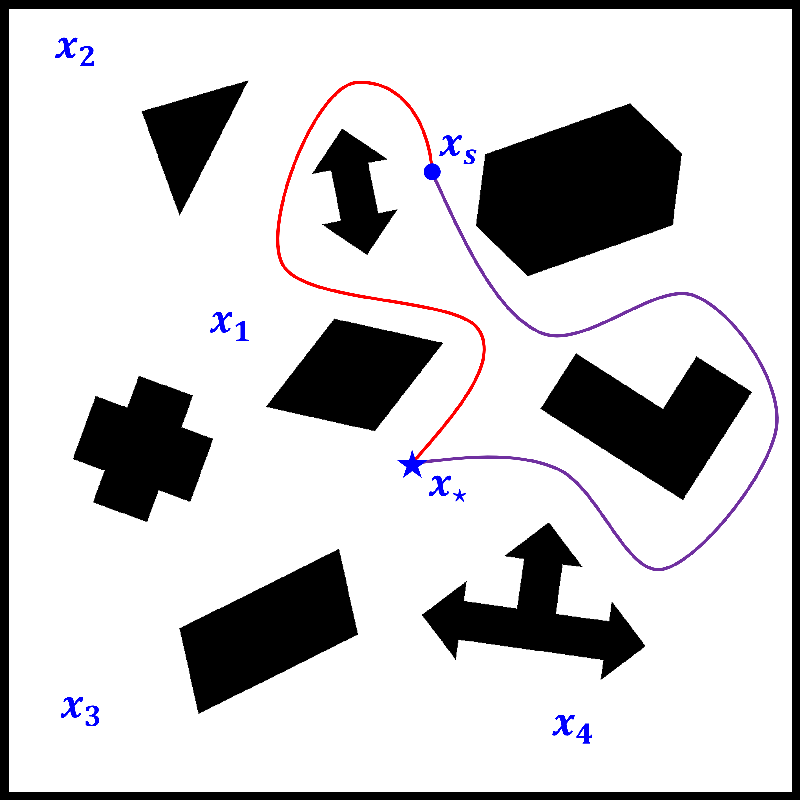}}
	\hfil
	\subfloat[]{\includegraphics[width=1.65in]{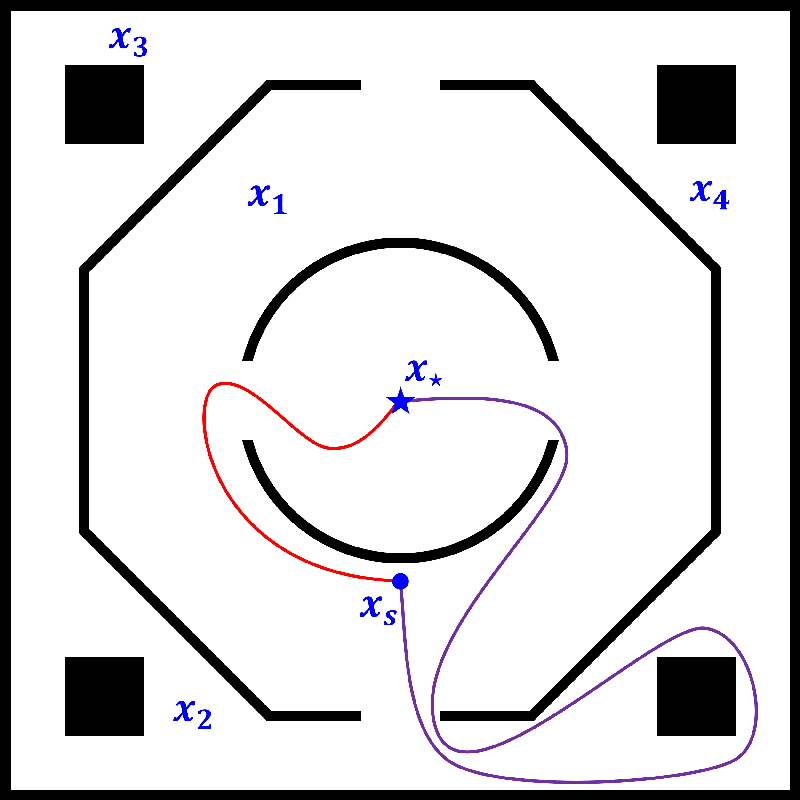}}
	\hfil
	\subfloat[]{\includegraphics[width=1.65in]{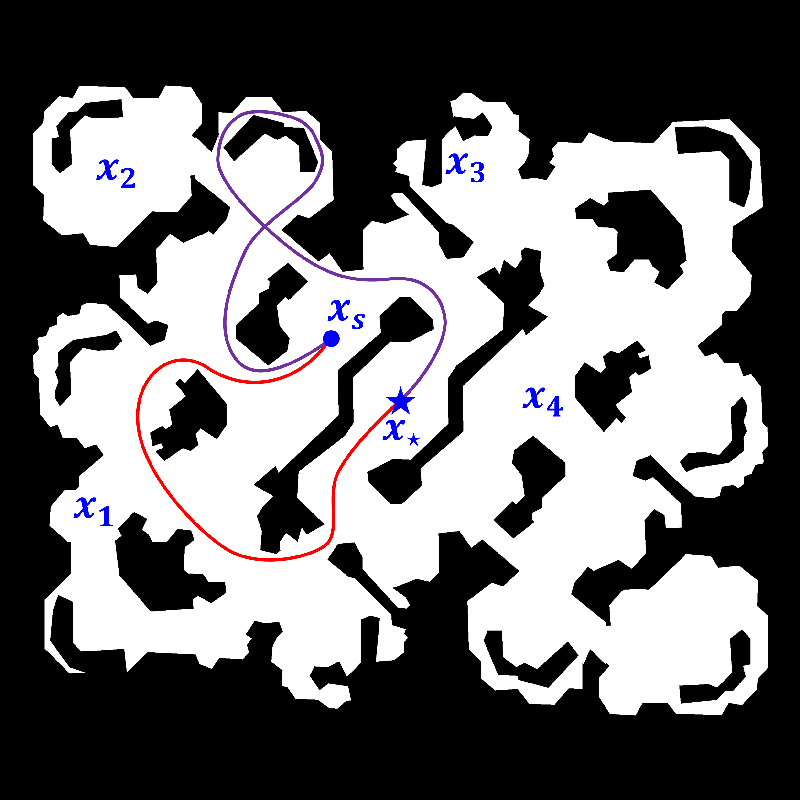}}
	\hfil
	\subfloat[]{\includegraphics[width=1.65in]{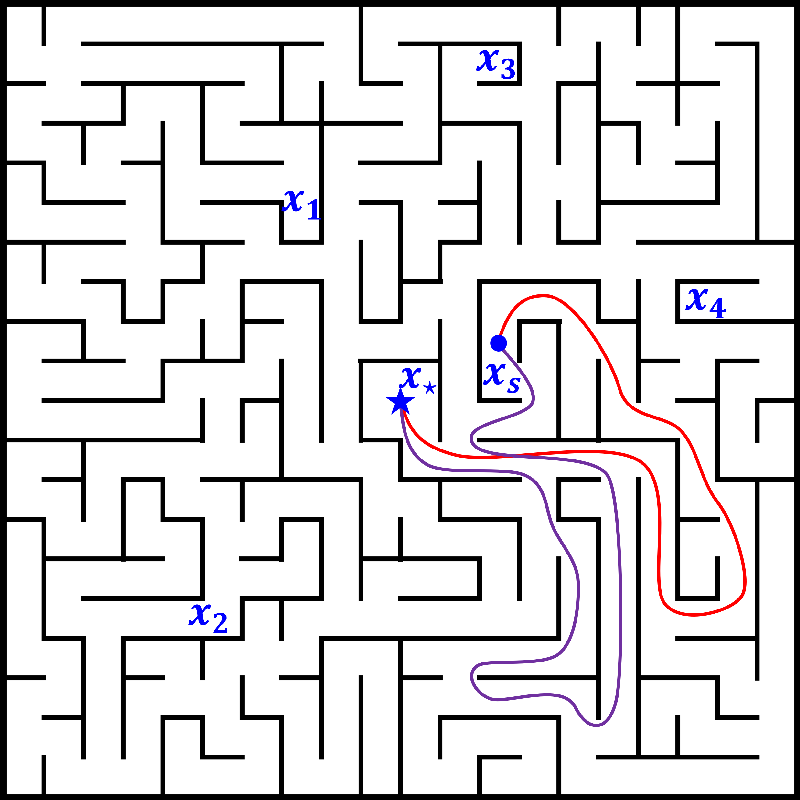}}
	\caption{Illustration of the simulation environments. (a) Cluttered environment, (b) Trap, (c) StarCraft II map, (d) Maze.  Each environment is configured with two distinct initial setups (i.e., red and purple curves) and four different goal points $x_1$-$x_4$. The blue `$\star$' represents the anchor point $x_\star$, and the blue circle indicates the endpoint $x_s$ of the initial configuration.}
	\label{fig_simMap}
\end{figure*}
\begin{figure*}[!t]
	\centering
	\subfloat[]{\includegraphics[width=1.65in]{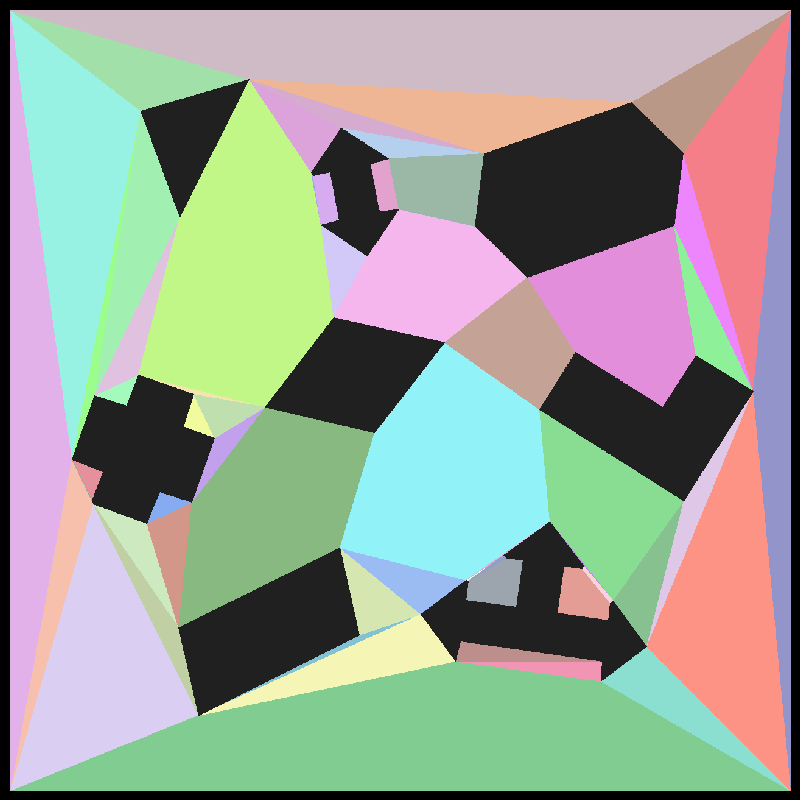}}
	\hfil
	\subfloat[]{\includegraphics[width=1.65in]{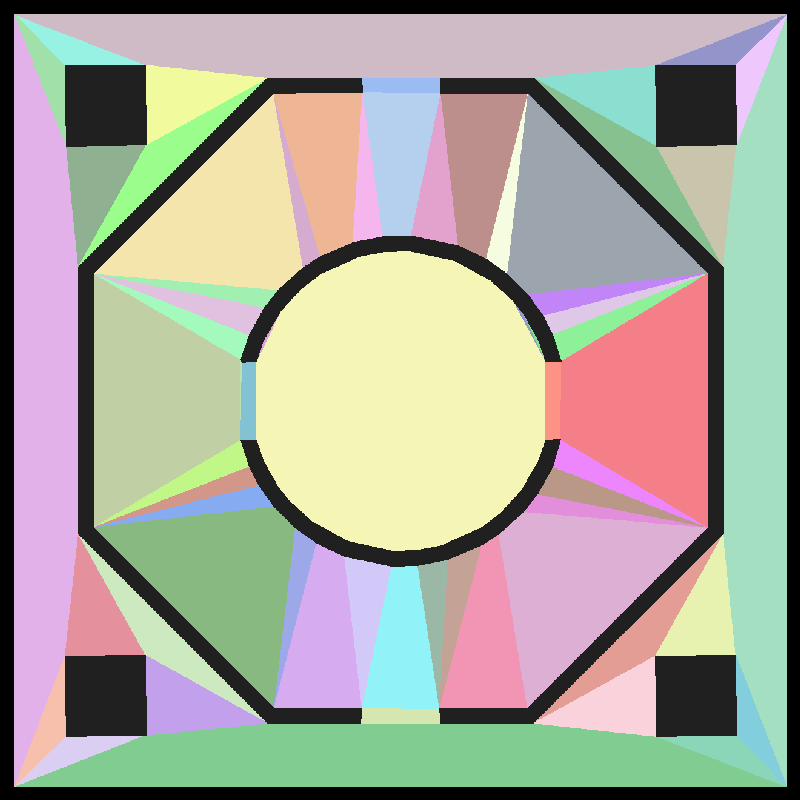}}
	\hfil
	\subfloat[]{\includegraphics[width=1.65in]{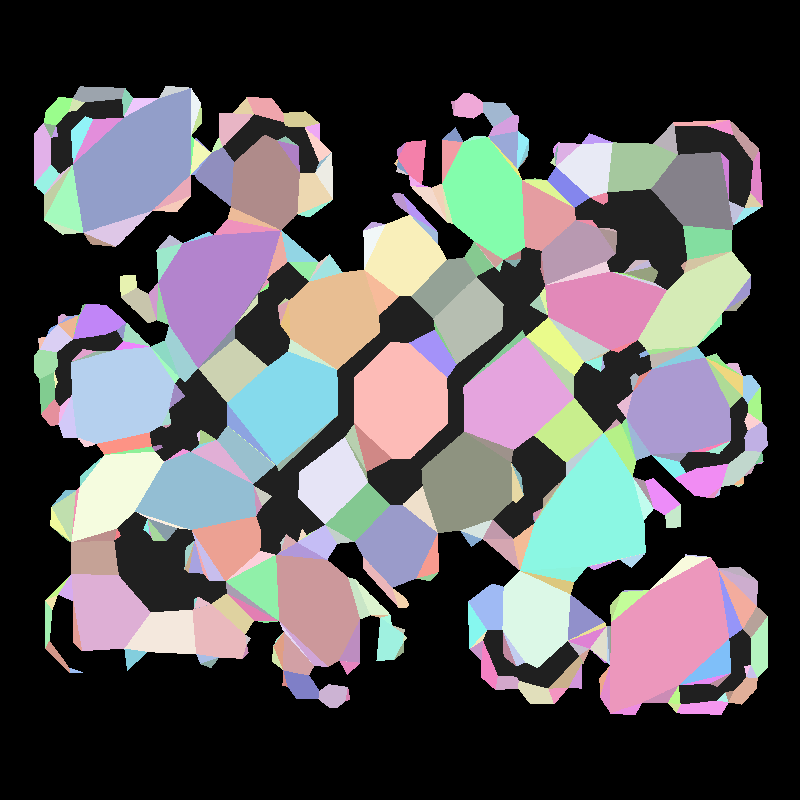}}
	\hfil
	\subfloat[]{\includegraphics[width=1.65in]{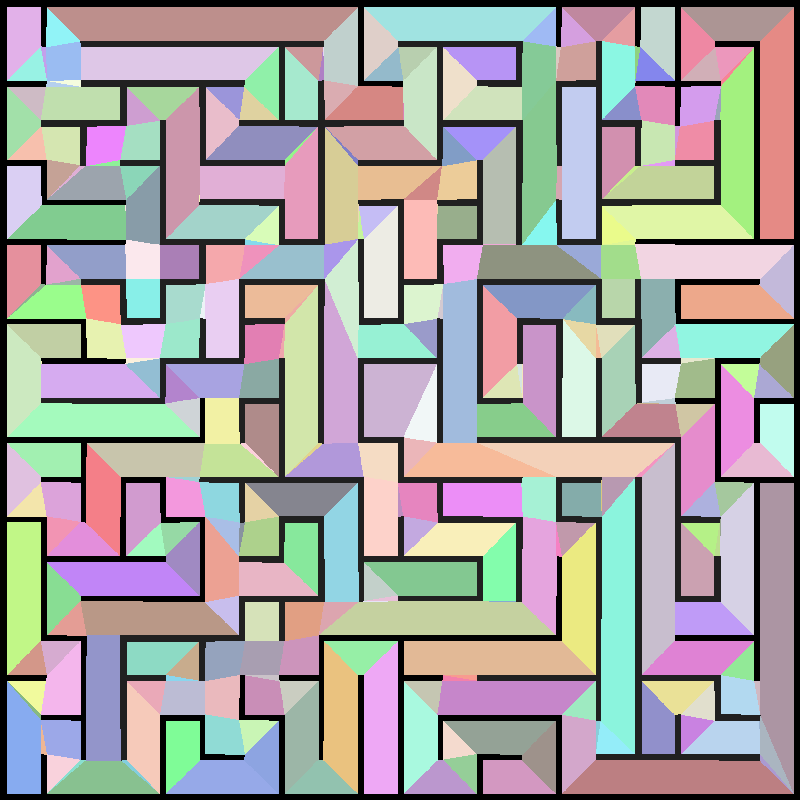}}
	\caption{Illustrates the results of the convex division of each simulation environment by the CDT-TCS algorithm.}
	\label{fig_debugMap}
\end{figure*}

\begin{table*}
	\begin{center}
		\caption{Specific Parameters of the Simulation Experiment Environment \label{tab:simMapData}}
		\scriptsize
		\begin{tabular}{|c|c|c|c|c|c|c|c|c|c|}
			\hline
			\multirow{2.5}*{\makecell[c]{Env.}} &
			\multirow{2.5}*{\makecell[c]{Map Image\\Resolution}} &
			\multirow{2.5}*{\makecell[c]{Number of\\Obstacles}} &
			\multirow{2.5}*{\makecell[c]{Tether\\Length}} &
			\multicolumn{6}{c|}{\makecell*[c]{Point Coordinates}}\\
			\cline{5-10}
			& & & & \makecell*[c]{$x_{\star}$} & \makecell*[c]{$x_s$} & \makecell*[c]{$x_1$} & \makecell*[c]{$x_2$} & \makecell*[c]{$x_3$} & \makecell*[c]{$x_4$}\\
			\hline
			\makecell*[c]{Cluttered} & $1320\times1320$ & $8$ & $1500$ & $(676, 761)$ & $(708, 281)$ & $(368, 535)$ & $(116, 82)$ & $(129, 1161)$ & $(933, 1193)$\\
			\hline
			\makecell*[c]{Trap} & $1320\times1320$ & $8$ & $2300$ & $(660, 660)$ & $(660, 958)$ & $(435, 327)$ & $(306, 1177)$ & $(203, 58)$ & $(1168, 316)$\\
			\hline
			\makecell*[c]{SC2}  & $3120\times3120$ & $15$ & $3200$ & $(1574, 1574)$ & $(1236, 1346)$ & $(365, 1990)$ & $(443, 670)$  & $(1809, 655)$ & $(2108, 1562)$\\
			\hline
			\makecell*[c]{Maze} & $2480\times2480$ & $14$ & $4000$ & $(1239, 1239)$ & $(1540, 1080)$ & $(932, 631)$ & $(641, 1915)$  & $(1538, 190)$ & $(2186, 942)$\\
			\hline
		\end{tabular}
	\end{center}
\end{table*}

\subsection{Optimal Tethered Configuration Search Problems}
In this subsection, we first evaluate the performance of the CDT-TCS algorithm proposed in Section 4 for solving the Optimal Tethered Configuration Search problem. We compare our method with that of Kim \citep{kim2014path}, who, in their original paper, used a grid map to represent the environment. They constructed a Homotopy-Augmented Graph (HAG) by incorporating $h$-signatures and then utilized Dijkstra's algorithm to determine the reachable workspace under tether constraints (preprocessing step). The preprocessing step in Kim's approach can be regarded as planning all feasible Optimal Tethered Configurations at each grid point. However, due to the low precision of grid-based representations of the environment, we replace the grid map in Kim's original method with a Probabilistic Roadmap (PRM) for comparative experiments.  Additionally, considering the trade-off between environmental representation accuracy and computational efficiency, the number of sampling points in the PRM was set to 20,000 (the target points $x_1$ to $x_4$ are not included in the PRM).

In this subsection experiments, we first execute CDT-TCS (\textbf{Algorithm~\ref{alg_TCS}}) and Kim's preprocessing step in each environment with $x_\star$ as the anchor point, recording the time taken for each.  After the preprocessing step, we further utilize CDT-TCS (\textbf{Algorithm~\ref{alg_TCSAFOC}}) and Kim's method to compute the feasible Optimal Tethered Configuration sets $\mathcal{C}^{*,x_k}_{x_\star, \zeta}$ for the four goal points $x_1$-$x_4$, where $k \in \{1, 2, 3, 4\}$. We record the number of configurations in $\mathcal{C}^{*,x_k}_{x_\star, \zeta}$, denoted as $\left\lvert\mathcal{C}^{*,x_k}_{x_\star, \zeta}\right\rvert $, as well as the time required to obtain $\mathcal{C}^{*,x_k}_{x_\star, \zeta}$. For the aforementioned experiments, we run each 100 times and collect the relevant statistics.

The convex decomposition result of our algorithm on the simulated environment is shown in Fig.~\ref{fig_debugMap}. The number of resulting convex polygons (i.e., the number of vertices in graph $\mathcal{G}_{con}$), the number of cutlines (i.e., the number of edges in graph $\mathcal{G}_{con}$), and the number of Tethered Configuration CDT Encodings obtained by CDT-TCS are summarized in \textbf{Table~\ref{tab:CDTMapData}}.  

\begin{table}
	\begin{center}
		\caption{Results of the Preprocessing of the CDT-TCS Algorithm for Each Simulation Environment. \label{tab:CDTMapData}}
		\footnotesize
		\begin{tabular}{|c|c|c|c|}
			\hline
			\makecell*[c]{Env.} &
			\makecell*[c]{Number of\\Polygons\\$\left\lvert V_{con}\right\rvert$} &
			\makecell*[c]{Number of\\Cutlines\\$\left\lvert E_{con}\right\rvert$} &
			\makecell*[c]{$\displaystyle \sum_{\mathbf{x} \in X_{con}} \left\lvert P^*_\zeta(\mathcal{G}_{con};\mathbf{x}_\star,\mathbf{x})\right\rvert$}\\
			\hline
			\makecell*[c]{Cluttered} & $59$ & $66$ & $1068$ \\
			\hline
			\makecell*[c]{Trap} & $65$ & $72$ & $1045$ \\
			\hline
			\makecell*[c]{SC2}  & $410$ & $424$ & $10790$ \\
			\hline
			\makecell*[c]{Maze} & $342$ & $355$ & $7473$ \\
			\hline
		\end{tabular}
	\end{center}
\end{table}

Fig.~\ref{fig_SimTCS} presents a comparison of the average preprocessing time between CDT-TCS and Kim's method, while Fig.~\ref{fig_SimTCSget} shows the average time required to retrieve the feasible set of optimal tethered configurations for a given target point. The experimental results demonstrate that CDT-TCS maintains a high level of real-time performance in the preprocessing phase, with average processing times of $3.16$ ms, $5.22$ ms, $55.19$ ms, and $35.35$ ms across different environments. Compared to the traditional method proposed by \cite{kim2014path}, CDT-TCS achieves a computational speedup of approximately two to three orders of magnitude across various environments. This significant improvement in performance is primarily attributed to the following key factors:
\begin{enumerate}
		\item[1)] {CDT-TCS adopts a topological representation of the environment based on convex dissection (i.e., $\mathcal{G}_{con}$). Compared to the PRM-based scheme used in Kim's algorithm, this allows CDT-TCS to represent the connectivity relationships between regions of the free space more completely while using fewer nodes. This reduces the number of nodes that need to be considered during the preprocessing phase, significantly improving efficiency;}
		\item[2)] {CDT-TCS uses CDT encoding as the homotopy invariant for paths. Since CDT encoding is constructed based on $\mathcal{G}_{con}$ (where its second element, the Tethered Configuration CDT Encoding, corresponds to a path in $\mathcal{G}_{con}$), it provides directional guidance for tethered configuration searches, as demonstrated in processes (\ref{eq_CodingIteration}) and (\ref{eq_CodingIteration2}). In contrast, Kim's method employs an $h$-signature that merely checks for homotopy during the search, without providing directional guidance.}
		\item[3)] {Based on \textbf{Theorem~\ref{th_finequality}} and \textbf{Theorem~\ref{th_fconvex}}, for the vast majority of Tethered Configuration CDT Encodings, CDT-TCS can conservatively determine their validity. For encodings where conservative determination is not possible, CDT-TCS can still quickly and accurately validate their effectiveness using ternary search. This ensures both computational efficiency and precision in determining valid configurations.}
\end{enumerate}
Additionally, combined with \textbf{Table~\ref{tab:CDTMapData}} analysis, it can be observed that the preprocessing time of CDT-TCS maintains a strong positive linear correlation with the number of Tethered Configuration CDT Encodings. This implies that the preprocessing time of CDT-TCS increases linearly only as the number of valid solutions to the TCS problem grows.\footnote{It is important to note that the number of solutions to the TCS problem may experience explosive growth as $\zeta$ increases. This is primarily due to excessively large $\zeta$ values leading to an increased number of potential ways the tether can wrap around obstacles in the environment. The feasible permutations and combinations of these wrapping patterns result in a significant number of non-homotopic valid configurations for the tethered robot.}
\begin{figure}[!t]
	\centering
	\includegraphics[width=3.0in]{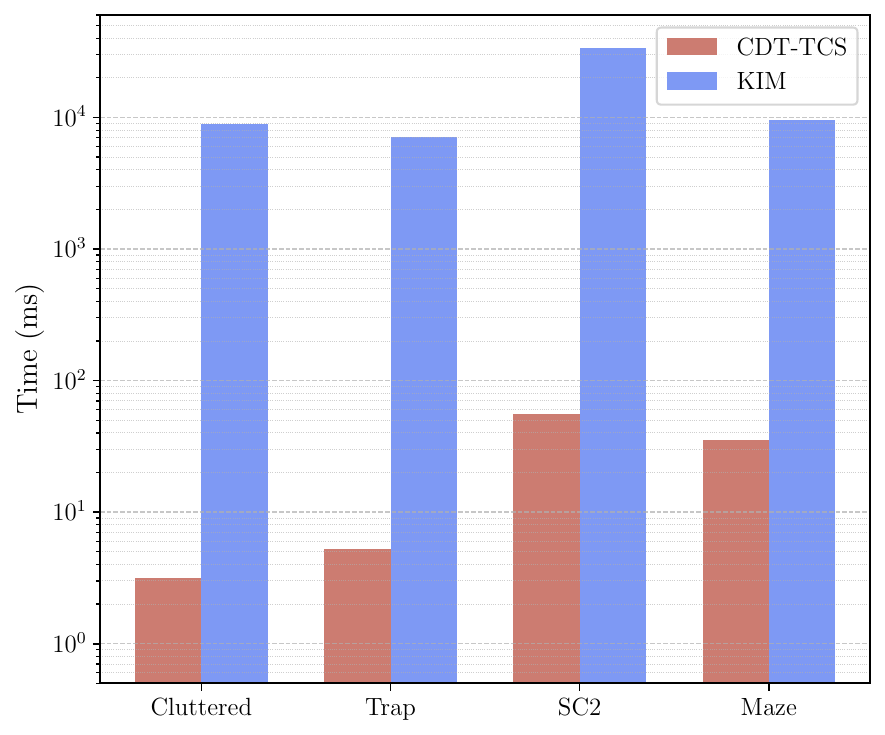}
	\caption{Illustration of the average preprocessing time for CDT-TCS and Kim's algorithm in each environment.}
	\label{fig_SimTCS}
\end{figure}

\begin{figure}[!t]
	\centering
	\subfloat[]{\includegraphics[width=1.65in]{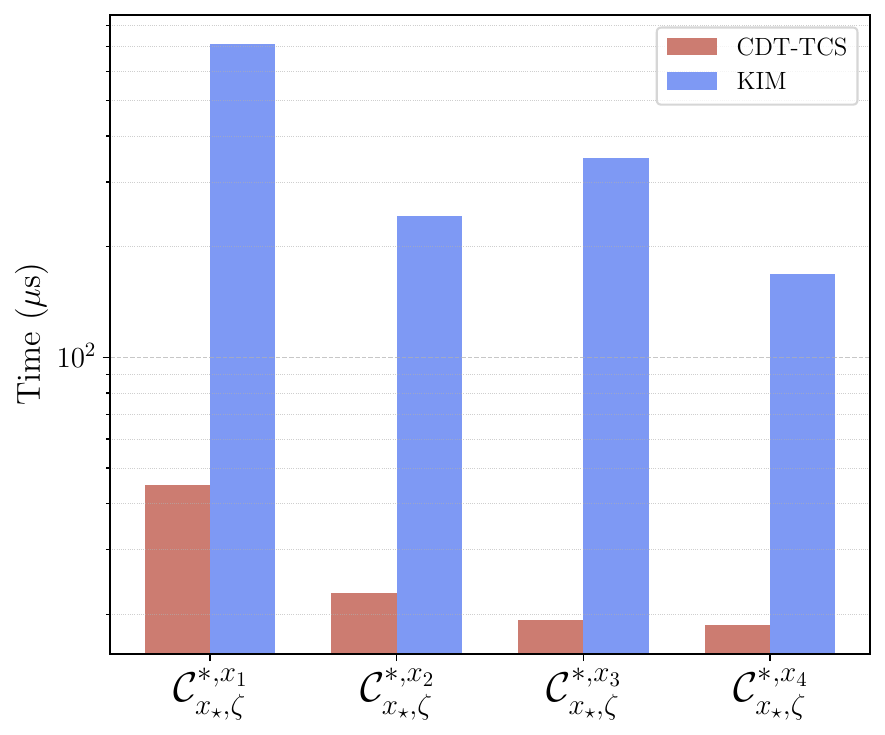}}
	\hfil
	\subfloat[]{\includegraphics[width=1.65in]{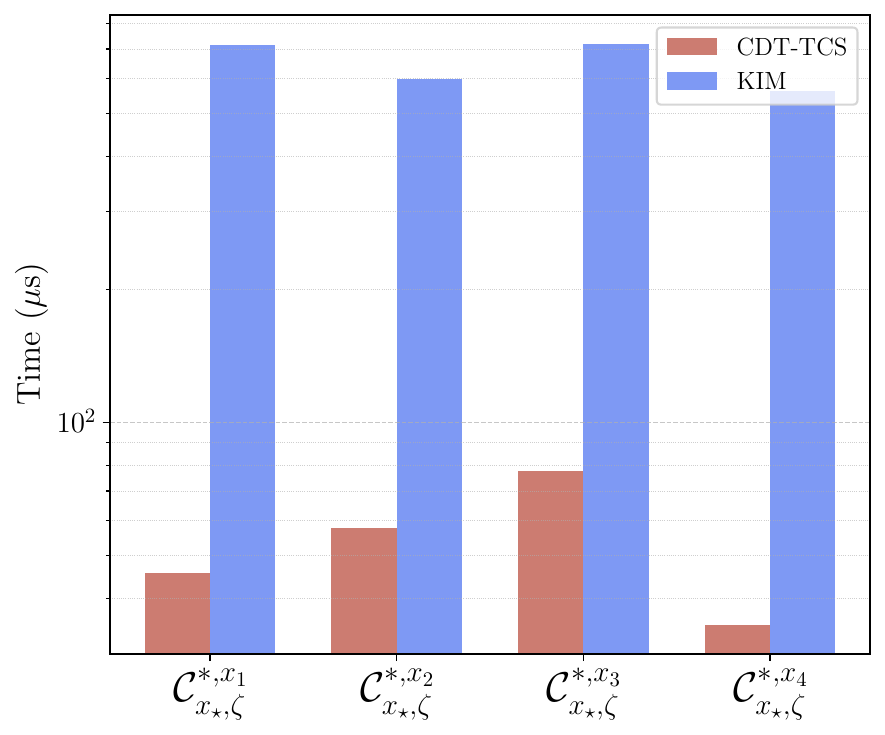}}\\
	\subfloat[]{\includegraphics[width=1.65in]{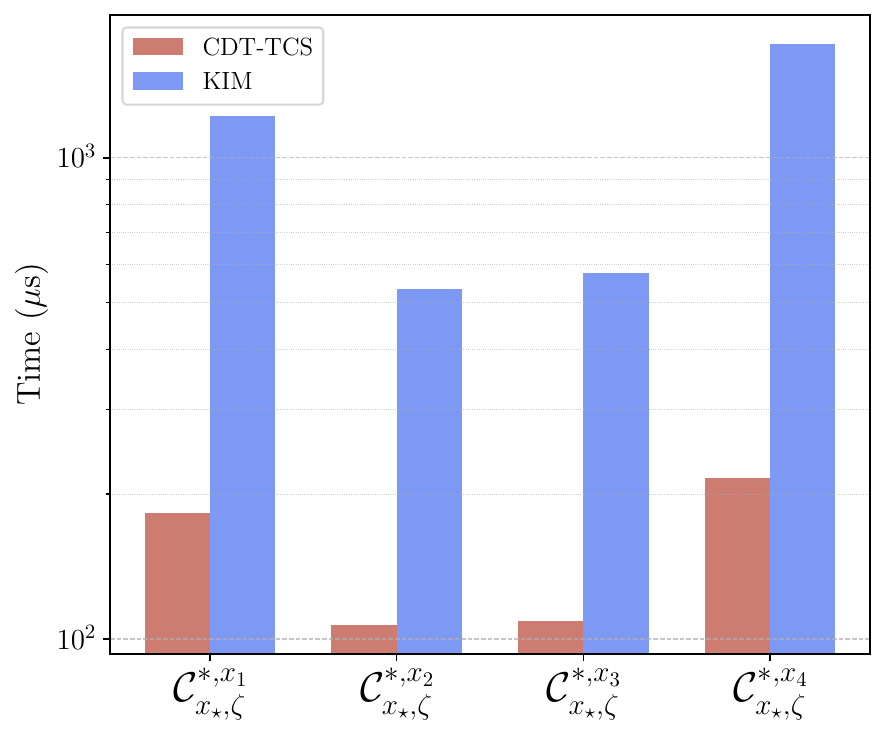}}
	\hfil
	\subfloat[]{\includegraphics[width=1.65in]{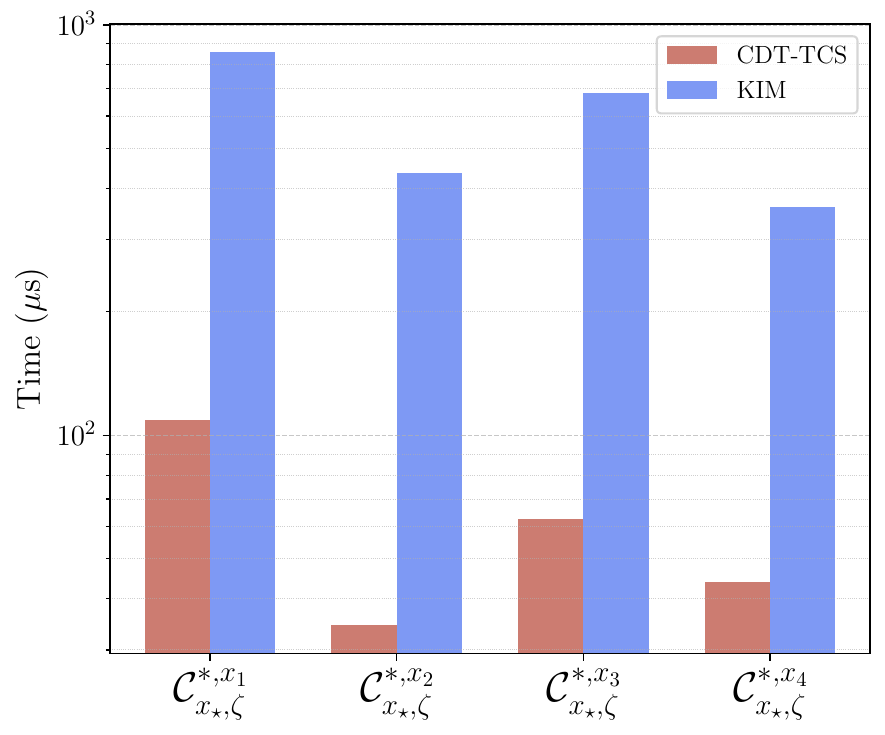}}
	\caption{Illustration of the average time required to retrieve the set of feasible optimal tethered configurations for a given goal point. (a) Cluttered environment, (b) Trap, (c) StarCraft II map, (d) Maze.}
	\label{fig_SimTCSget}
\end{figure}

Fig.~\ref{fig_SimTCSget} demonstrates that CDT-TCS exhibits extremely high efficiency in obtaining the feasible Optimal Tethered Configuration set for target points, with computation times consistently remaining at the microsecond level. These times are nearly negligible, aligning with the description in \textbf{Remark~\ref{re_proTCS}}. Therefore, we conclude that CDT-TCS can rapidly determine the Optimal Tethered Configuration set for all points in the space for tethered robots. In contrast, the traditional Kim method performs slightly worse than CDT-TCS. The primary reason is that during each search, Kim's method needs to traverse all neighboring points $x_{\text{near}}$ in the vicinity of the target point $x_k$ within the HAG and identify all shortest non-homotopic paths at $x_k$. This process involves performing relatively time-consuming collision detection operations for each $l_{x_{\text{near}}}^{x_k}$.

\begin{figure}[!t]
	\centering
	\includegraphics[width=3.2in]{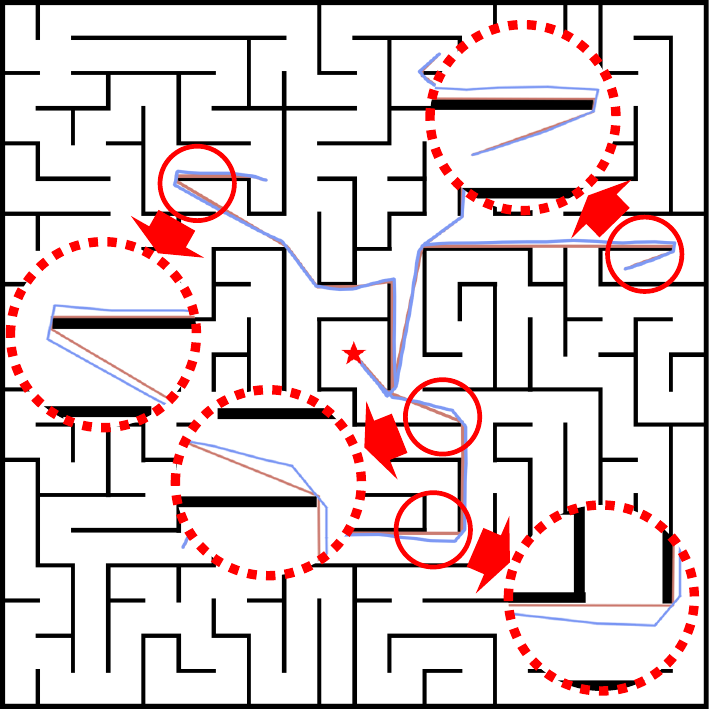}
	\caption{Illustrates a comparison of several pairs of optimal configurations computed by the CDT-TCS and Kim algorithms. In the figure, the configurations calculated by Kim's algorithm are depicted in Lavender Blue, while those computed by the CDT-TCS algorithm are shown in Dusty Rose.}
	\label{fig_TCSConfigExp}
\end{figure}

\begin{table*}
	\begin{center}
		\caption{The Number of True and Found Optimal Configurations for Four Goal Points \label{tab:TCSNumConfig}}
		\small
		\begin{tabular}{|c|ccc|ccc|ccc|ccc|}
			\hline
			\multirow{2.5}*{\makecell[c]{Env.}} &
			\multicolumn{3}{c|}{\makecell*[c]{$\mathcal{C}^{*,x_1}_{x_\star, \zeta}$}} &
			\multicolumn{3}{c|}{\makecell*[c]{$\mathcal{C}^{*,x_2}_{x_\star, \zeta}$}} &
			\multicolumn{3}{c|}{\makecell*[c]{$\mathcal{C}^{*,x_3}_{x_\star, \zeta}$}} &
			\multicolumn{3}{c|}{\makecell*[c]{$\mathcal{C}^{*,x_4}_{x_\star, \zeta}$}}\\
			\cline{2-13}
			& \makecell*[c]{True} & CDT-TCS & Kim & True & CDT-TCS & Kim & True & CDT-TCS & Kim & True & CDT-TCS & Kim \\
			\hline
			\makecell*[c]{Cluttered} & $20$ & $20$ & $20$ & $9$ & $9$ & $9$ & $9$ & $9$ & $9$ & $7$ & $7$ & $7$ \\
			\hline
			\makecell*[c]{Trap} & $15$ & $15$ & $15$ & $14$ & $14$ & $14$ & $13$ & $13$ & $13$ & $14$ & $14$ & $14$ \\
			\hline
			\makecell*[c]{SC2}  & $24$ & $24$ & $24$ & $13$ & $13$ & $13$ & $17$ & $17$ & \textbf{15} & $32$ & $32$ & \textbf{31} \\
			\hline
			\makecell*[c]{Maze} & $20$ & $20$ & \textbf{16} & $11$ & $11$ & \textbf{9} & $17$ & $17$ & $17$ & $13$ & $13$ & \textbf{10} \\
			\hline
		\end{tabular}
	\end{center}
\end{table*}

Finally, it is worth noting that due to the adoption of a convex dissection-based geometric representation of the environment, CDT-TCS achieves higher computational accuracy for optimal configurations (or paths) compared to Kim's algorithm, which uses PRM to represent the environment. Specifically, the lengths of the optimal configurations (or paths) computed by CDT-TCS are shorter. Fig.~\ref{fig_TCSConfigExp} illustrates several pairs of optimal configurations calculated by CDT-TCS and Kim's algorithm. The errors introduced by PRM may prevent Kim's algorithm from successfully finding certain optimal configurations with lengths close to $\zeta$, particularly in complex environments such as SC2 and Maze used in the experiments. \textbf{Table~\ref{tab:TCSNumConfig}} records the true counts of $\|\mathcal{C}^{*,x_k}_{x_\star ,\zeta}\|$ as well as the numbers of configurations found by CDT-TCS and Kim's algorithm.

\subsection{Optimal Tethered Path Planning Problems}
In this subsection, we will continue to evaluate the performance of the CDT-TPP algorithm proposed in Subsection 5.1 and Kim's algorithm in solving the Optimal Tethered Path Planning problem. In the Kim algorithm, the planning phase primarily relies on using the A* algorithm on the preprocessed HAG to determine the optimal path that satisfies the tether constraints. Since the preprocessing steps for both CDT-TPP (via \textbf{Algorithm~\ref{alg_TCS}}, CDT-TCS) and Kim's algorithm need to be executed only once for the same environment, the time spent on preprocessing is excluded from the statistics in this subsection.

In the experiments, we will use $\varsigma_{s1}$ and $\varsigma_{s2}$ from the simulation environment as the initial configurations and sequentially plan the optimal paths for moving from these initial configurations to the target points $x_1$, $x_2$, $x_3$, and $x_4$. During this process, we will record the time consumed by the CDT-TPP algorithm and the Kim algorithm in planning these paths, as well as the costs of the resulting paths. Each experiment will be repeated 100 times, and the relevant data will be statistically analyzed.

\begin{figure}[!t]
	\centering
	\includegraphics[width=2.4in]{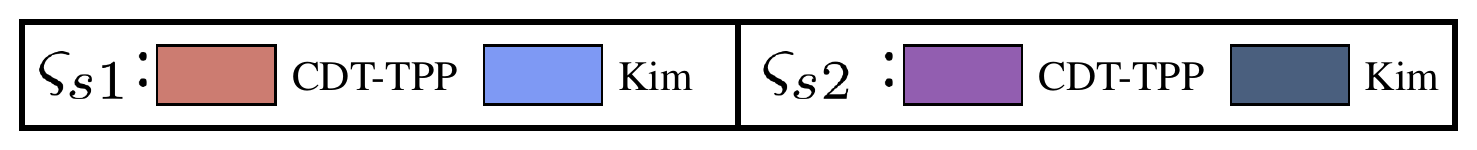}\\
	\subfloat[]{\includegraphics[width=1.65in]{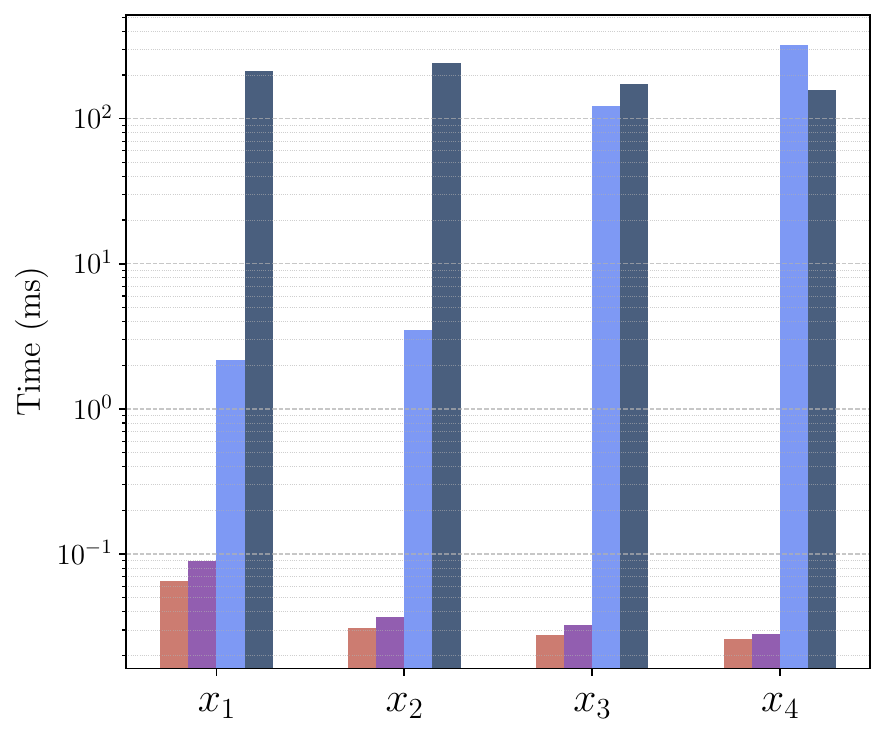}}
	\hfil
	\subfloat[]{\includegraphics[width=1.65in]{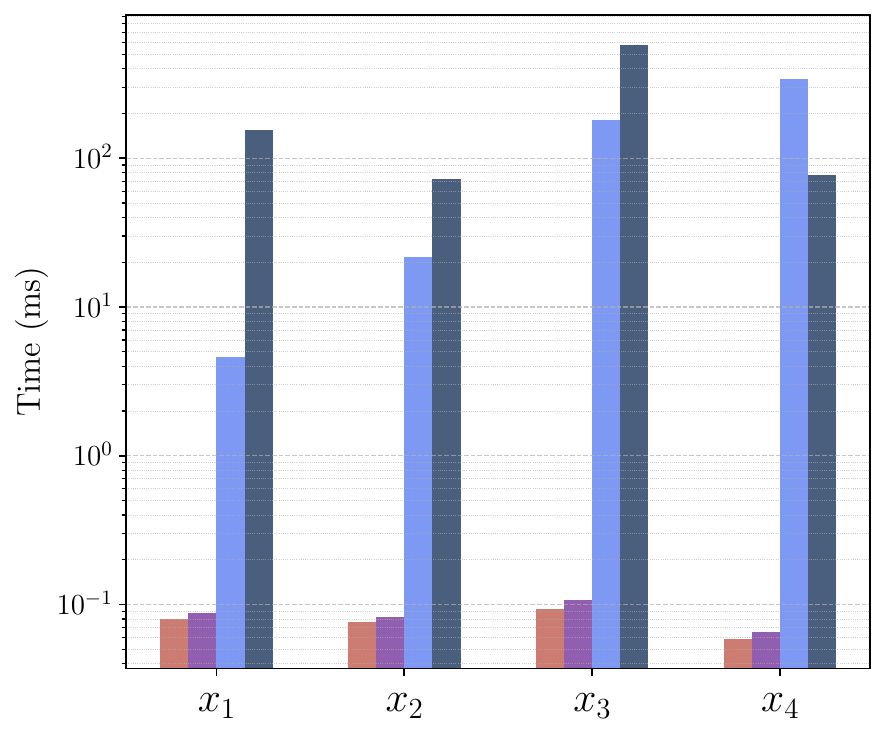}}\\
	\subfloat[]{\includegraphics[width=1.65in]{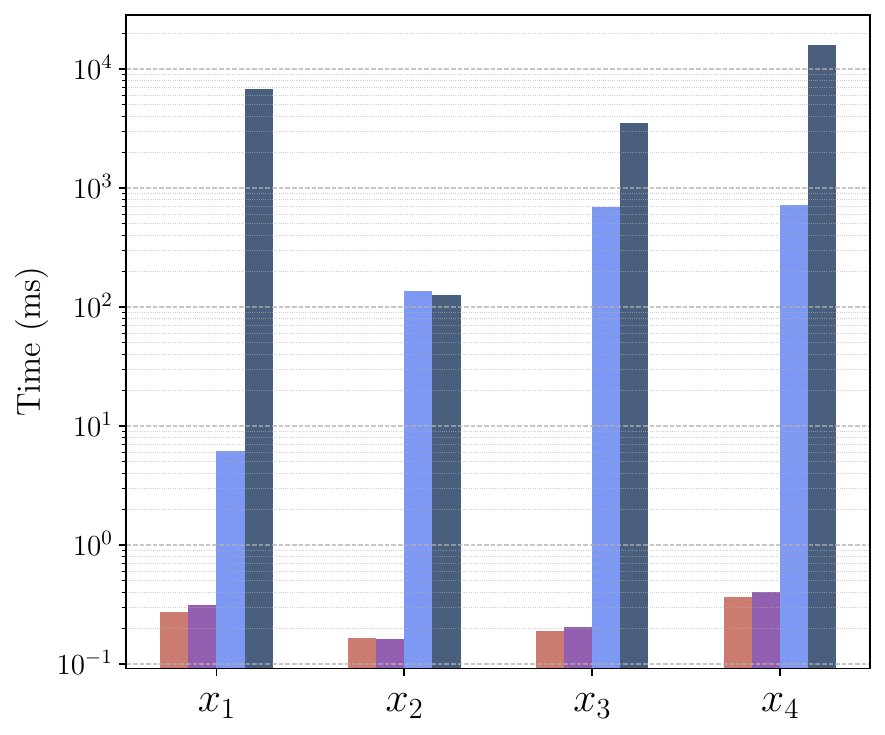}}
	\hfil
	\subfloat[]{\includegraphics[width=1.65in]{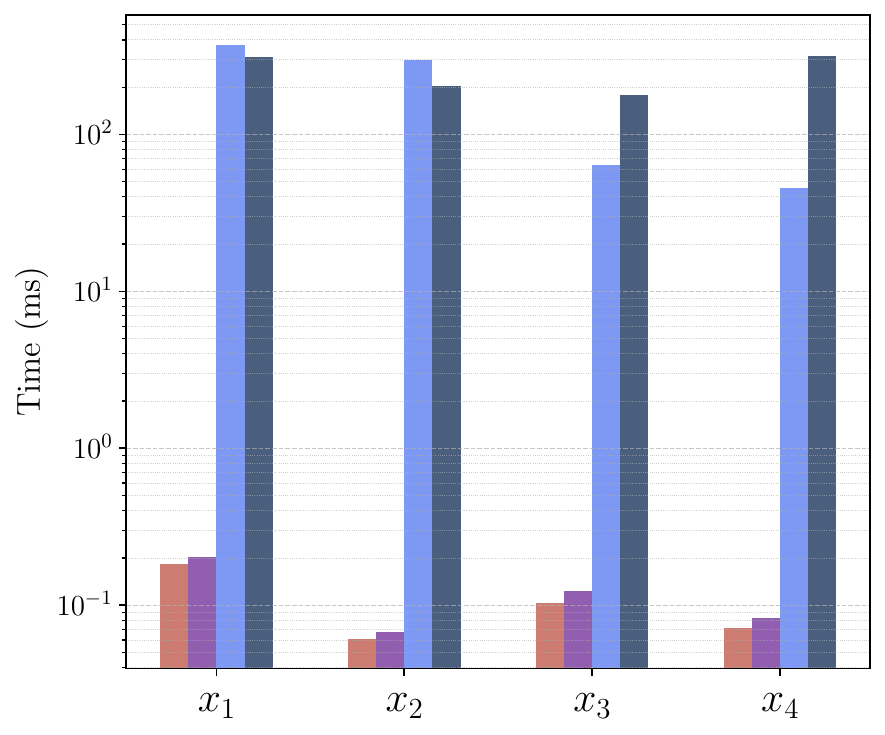}}
	\caption{Illustration of the average path planning times for CDT-TPP and KIM algorithms. (a) Cluttered environment, (b) Trap, (c) StarCraft II map, (d) Maze.}
	\label{fig_SimTPP}
\end{figure}

{
\setlength{\tabcolsep}{3pt}
\begin{table*}
	\begin{center}
		\caption{Path Costs Planned by CDT-TPP and Kim's Algorithm, and the Number of Nodes Visited by Kim's Algorithm \label{tab:TPPcost}}
		\footnotesize
		\begin{tabular}{|c|c|cccc|cccc|cccc|cccc|}
			\hline
			\multicolumn{2}{|c|}{\multirow{2.5}*{\makecell[c]{Goals}}} &
			\multicolumn{4}{c|}{\makecell*[c]{Cluttered}} &
			\multicolumn{4}{c|}{\makecell*[c]{Trap}} &
			\multicolumn{4}{c|}{\makecell*[c]{SC2}} &
			\multicolumn{4}{c|}{\makecell*[c]{Maze}}\\
			\cline{3-18}
			\multicolumn{2}{|c|}{} & \makecell*[c]{$x_1$} & $x_2$ & $x_3$ & $x_4$ & $x_1$ & $x_2$ & $x_3$ & $x_4$ & $x_1$ & $x_2$ & $x_3$ & $x_4$ & $x_1$ & $x_2$ & $x_3$ & $x_4$ \\
			\hline
			\multicolumn{18}{|c|}{\makecell*[c]{Path Cost}}\\
			\hline
			\multirow{2.5}*{\makecell[c]{$\varsigma_{s1}$}} & \makecell*[c]{CDT-TPP} & $435$ & $632$ & $1214$ & $1380$ & $770$ & $532$ & $1409$ & $1404$ & $1090$ & $1392$ & $2718$ & $2716$ & $4001$ & $3846$ & $2047$ & $1518$ \\
			& \makecell*[c]{Kim} & $437$ & $631$ & $1213$ & $1383$ & $772$ & $535$ & $1414$ & $1413$ & $1090$ & $1404$ & $2738$ & $2728$ & $4093$ & $3937$ & $2104$ & $1602$ \\
			\hline
			\multirow{2.5}*{\makecell[c]{$\varsigma_{s2}$}} & \makecell*[c]{CDT-TPP} & $1701$ & $2142$ & $1939$ & $1289$ & $2355$ & $1800$ & $2993$ & $1630$ & $3203$ & $1392$ & $3194$ & $3729$ & $4585$ & $4029$ & $4311$ & $4547$ \\
			& \makecell*[c]{Kim} & $1708$ & $2149$ & $1946$ & $1294$ & $2363$ & $1808$ & $3005$ & $1636$ & $3239$ & $1404$ & $3239$ & $3770$ & $4708$ & $4137$ & $4422$ & $4680$ \\
			\hline
			\multicolumn{18}{|c|}{\makecell*[c]{The Number of Vertices Visited by the Kim Algorithm During the Planning Phase ($xek = x \times 10^k$)}}\\
			\hline
			\multicolumn{2}{|c|}{\makecell*[c]{$\varsigma_{s1}$}} & $127$ & $240$ & $4971$ & $123e2$ & $320$ & $1225$ & $7849$ & $134e2$ & $206$ & $3335$ & $151e2$ & $152e2$ & $186e2$ & $154e2$ & $3883$ & $2845$\\
			\hline
			\multicolumn{2}{|c|}{\makecell*[c]{$\varsigma_{s2}$}} & $9953$ & $111e2$ & $8132$ & $6550$ & $8777$ & $4610$ & $233e2$ & $4535$ & $192e3$ & $3516$ & $912e2$ & $387e3$ & $169e2$ & $118e2$ & $108e2$ & $174e2$ \\
			\hline
		\end{tabular}
	\end{center}
\end{table*}
}

\textbf{Table~\ref{tab:TPPcost}} presents the path costs planned by CDT-TPP and the Kim algorithm, as well as the number of nodes visited by Kim. The comparison of average path planning times between the CDT-TPP and Kim algorithms is shown in Fig.~\ref{fig_SimTPP}. In terms of planning time, the CDT-TPP algorithm consistently outperforms the Kim algorithm. In this round of experiments, the performance of CDT-TPP was highly stable, with its consumed time consistently remaining within the range of 0.02-0.4 ms. Additionally, the planning efficiency of CDT-TPP is not sensitive to the choice of initial configuration. As can be observed from the experiments, for the same target point, the difference in time consumption when using $\varsigma_{s1}$ or $\varsigma_{s2}$ as the initial configuration is minimal. This is because, according to the pseudocode shown in \textbf{Algorithm~\ref{alg_TPP}}, the planning time consumed by CDT-TPP is primarily linearly correlated with the number of elements in the encoding set $P^*_\zeta(\mathcal{G}_{con}; \mathbf{x}_\star, \mathbf{x}_g)$, where $\mathbf{x}_g$ represents the convex polygon containing the target point.

In contrast, the planning efficiency of the Kim algorithm is highly unstable. For different initial configurations, the time consumption can vary by as much as 2-3 orders of magnitude. For instance, in the case of target point $x_1$ in the SC2 environment, planning takes only 6.2 ms when the initial configuration is $\varsigma_{s1}$, but nearly 13 seconds when the initial configuration is $\varsigma_{s2}$. Such variability is unacceptable for tethered robot platforms that require high real-time performance but have limited computational resources. We believe the primary reasons for this drawback in the Kim algorithm are as follows:
\begin{enumerate}
		\item[1)] {The Kim algorithm uses PRM for environment representation and constructs the HAG by combining it with $h$-signatures. This results in the HAG containing a vast number of vertices (Much more than PRM), which increases the complexity of the graph structure and slows down the search process;}
		\item[2)] {During planning, the Kim algorithm employs the A* algorithm as the base planner. However, under tether constraints, certain special initial configurations may cause the heuristic function (Euclidean distance) used by A* to fail, nearly degrading it to Dijkstra's algorithm. In such cases, the Kim algorithm ends up traversing a large number of invalid vertices in the HAG, significantly increasing computation time.}
\end{enumerate}
Based on this, it is not difficult to identify one of the main reasons why Kim's algorithm is significantly slower than CDT-TPP during the planning phase. Specifically, the preprocessing stage of Kim's algorithm primarily focuses on labeling the vertices in the HAG that satisfy the tether constraints. Beyond this, the preprocessing stage provides no additional guidance for subsequent path planning. In contrast, CDT-TPP leverages the theoretical foundation provided by (\ref{eq_FOPPsetE}) and \textbf{Corollary~\ref{Cor_genf}} to directly utilize the results of the CDT-TCS algorithm to obtain the potential set of optimal paths. It then quickly determines the optimal path through a simple traversal of this set.

\subsection{Tethered Multi-Goal Visitin Problems}
In this subsection, we evaluate the performance of the CDT-TMV algorithm proposed in Subsection 5.2 for solving the Optimal TMV problem. It is important to note that, due to we have not identified any other open-source algorithms specifically designed for addressing the Optimal TMV Problem, we compare CDT-TMV with an exhaustive search method to demonstrate the advantages of CDT-TMV.

In the experiments, we use $\varsigma_{s1}$ as the initial configuration and set $(x_1, x_2, x_3, x_4)$ as the sequence of target points that need to be visited sequentially. Both the CDT-TMV algorithm and the exhaustive search method are employed to plan the optimal path for this TMV problem. During this process, we record the time consumed by the CDT-TMV algorithm and the exhaustive search method, as well as the number of calls made to \textbf{Algorithm~\ref{alg_ShortestPath}}. Each experiment is repeated 100 times, and the relevant data are statistically analyzed.

\begin{figure}[!t]
	\centering
	\subfloat[]{\includegraphics[width=1.65in]{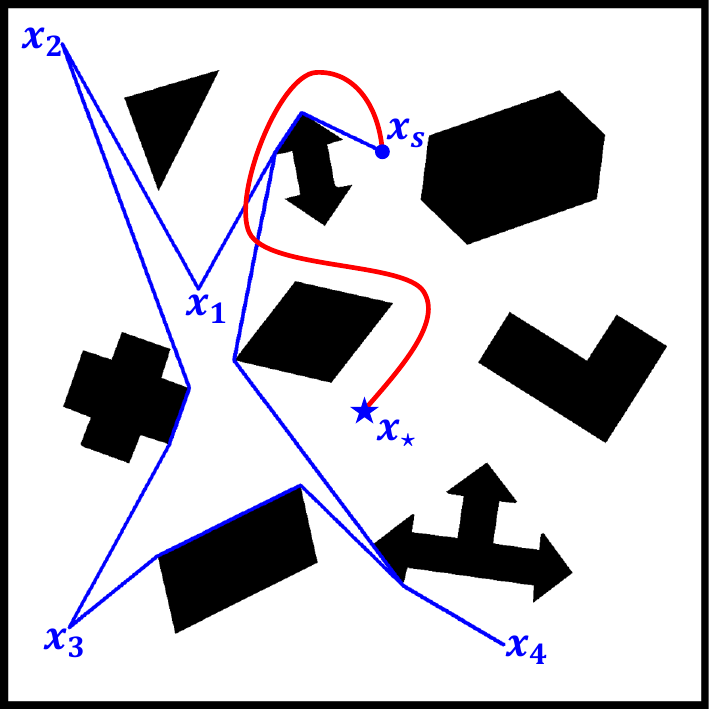}}
	\hfil
	\subfloat[]{\includegraphics[width=1.65in]{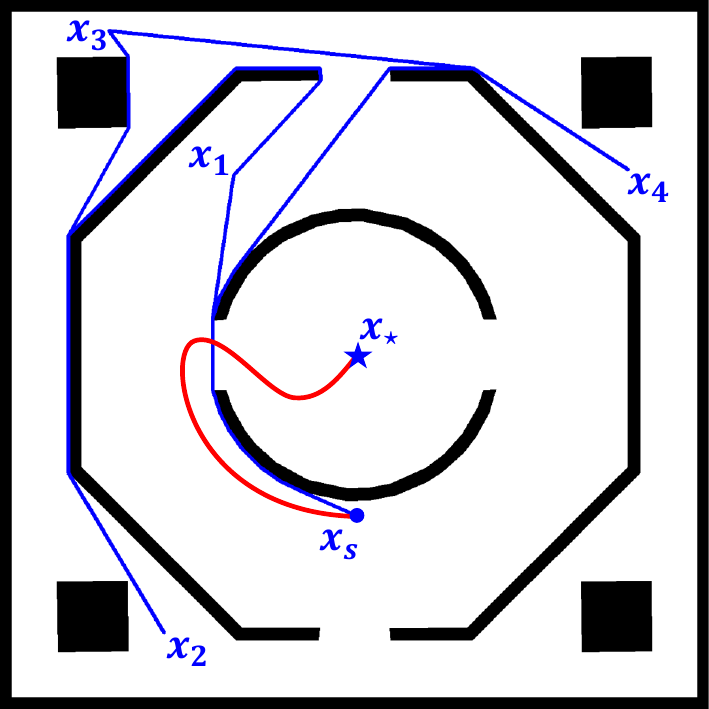}}
	\\
	\subfloat[]{\includegraphics[width=1.65in]{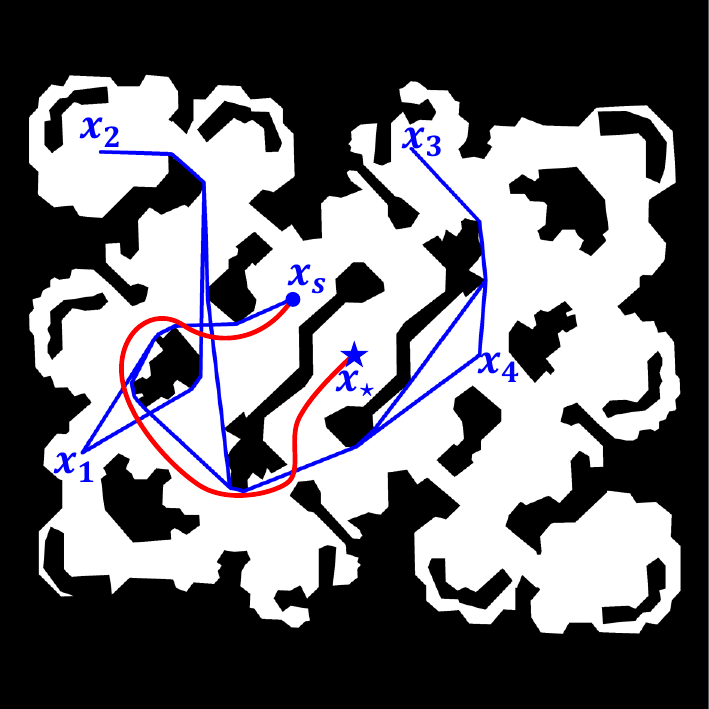}}
	\hfil
	\subfloat[]{\includegraphics[width=1.65in]{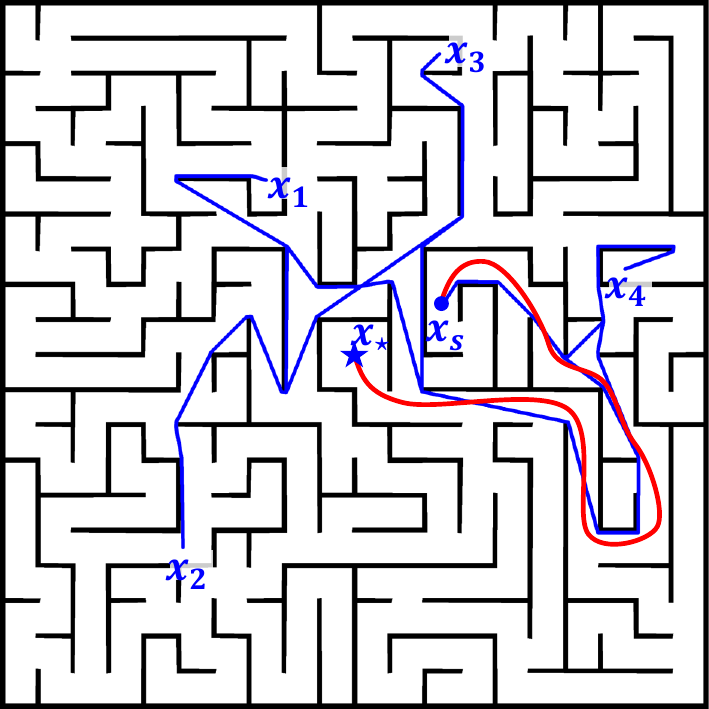}}
	\caption{Illustration of the optimal paths for the four TMV tasks planned by CDT-TMV.}
	\label{fig_simTMVpath}
\end{figure}

\begin{figure}[!t]
	\centering
	\subfloat[]{\includegraphics[width=1.65in]{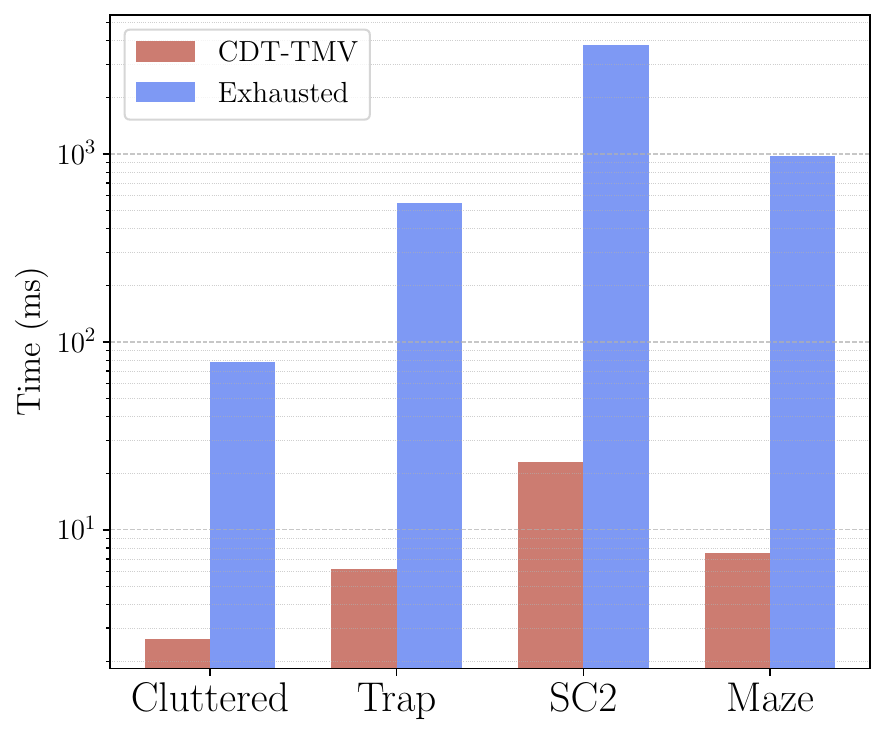}}
	\hfil
	\subfloat[]{\includegraphics[width=1.65in]{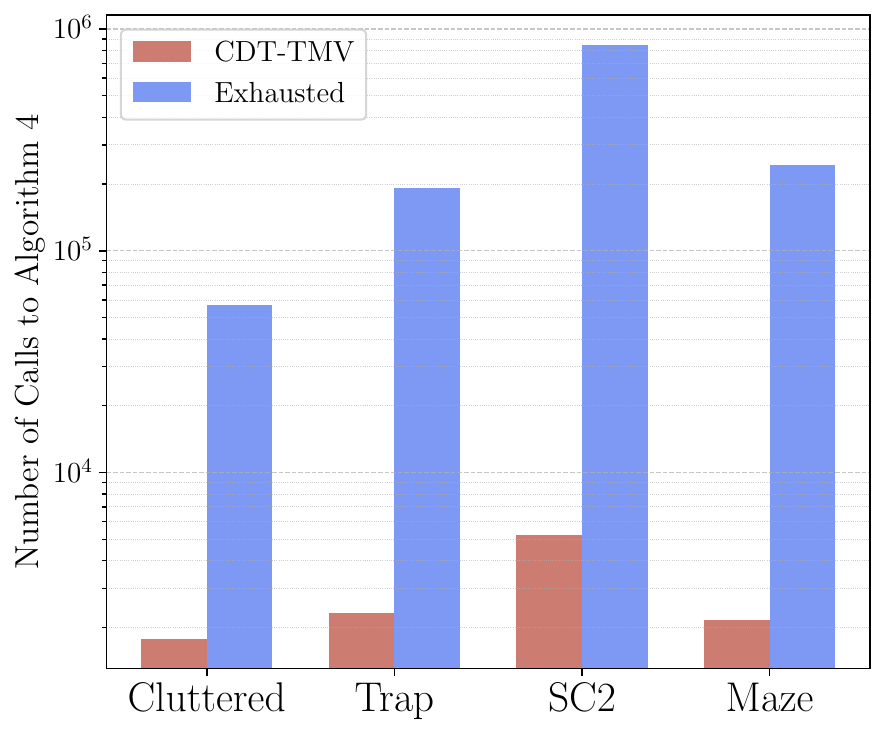}}
	\caption{Comparison of (a) average time consumed and (b) the number of calls to \textbf{Algorithm~\ref{alg_ShortestPath}} between CDT-TMV and the exhaustive search method.}
	\label{fig_SimTMV}
\end{figure}

The Optimal TMV paths obtained from the planning process are shown in Fig.~\ref{fig_simTMVpath}. The average time consumed and the number of calls to \textbf{Algorithm~\ref{alg_ShortestPath}} for both CDT-TMV and the exhaustive search method are presented in Fig.~\ref{fig_SimTMV}. From the experimental results, it is evident that CDT-TMV transforms the Optimal TMV problem into a generalized path planning process on the graph $\mathcal{G}_{TMV}$, utilizing Dijkstra's algorithm for rapid solution. This approach avoids unnecessary attempts of many non-essential configuration sequences by CDT-TMV. Additionally, during each calculation of a new node, CDT-TMV reuses the computational results of its parent node according to \textbf{Remark~\ref{re_nodeChange}}. These two measures significantly reduce the number of calls to \textbf{Algorithm~\ref{alg_ShortestPath}} by CDT-TMV. In the four TMV tasks, the number of calls to \textbf{Algorithm~\ref{alg_ShortestPath}} by the CDT-TMV algorithm was reduced by factors of $31.9$, $81.9$, $162.6$, and $111.8$ compared to the exhaustive search method.

In fact, the TMV problem designed for this experiment, which only involves visiting four target points, is relatively simple. When the number of target points increases, the space that the exhaustive search method needs to traverse will grow exponentially, making it impractical to solve the Optimal TMV problem using the exhaustive search method. In contrast, CDT-TMV can still achieve rapid solutions under such conditions.

\subsection{Fast Optimal Untethered Path Planning Problems}

\begin{figure*}[!t]
	\centering
	\includegraphics[width=4.0in]{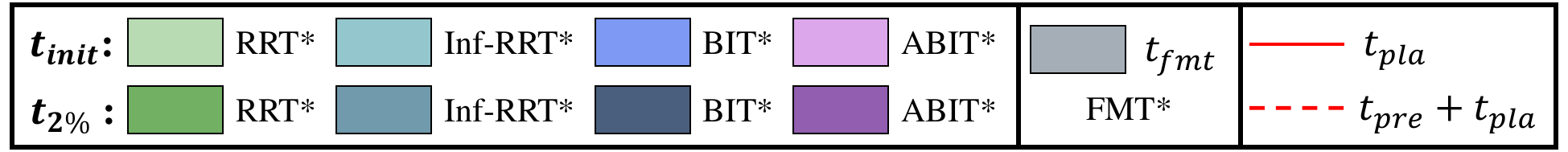}
	\\
	\subfloat[]{\includegraphics[width=1.65in]{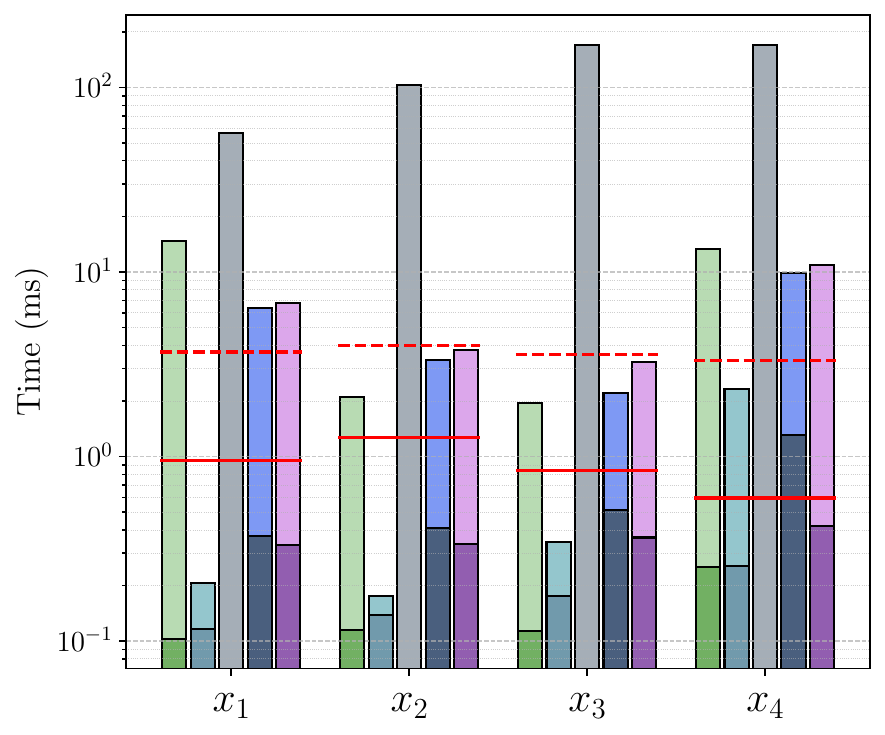}}
	\hfil
	\subfloat[]{\includegraphics[width=1.65in]{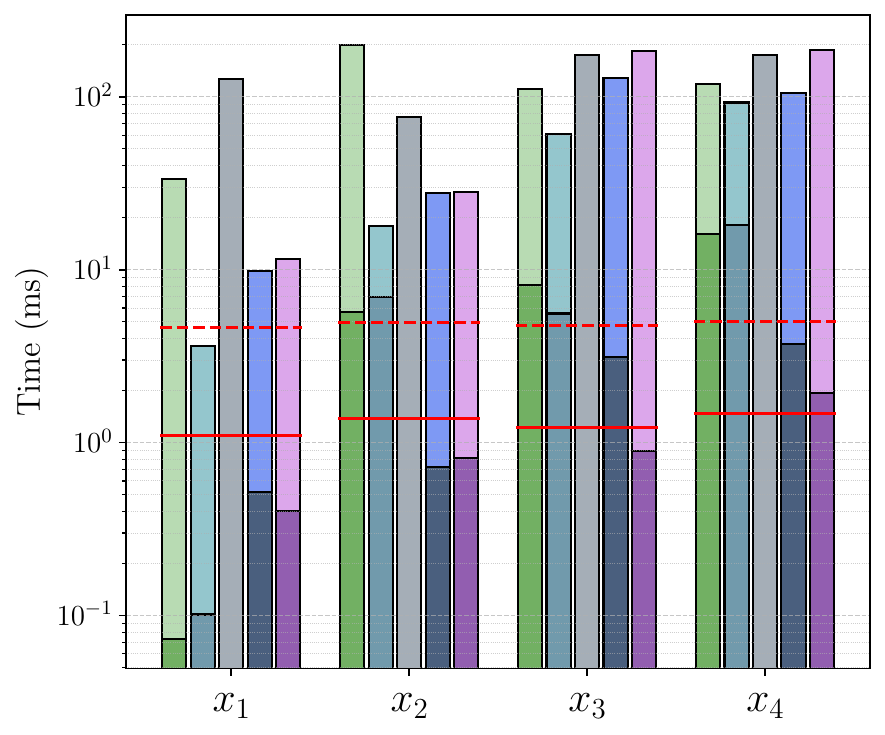}}
	\hfil
	\subfloat[]{\includegraphics[width=1.65in]{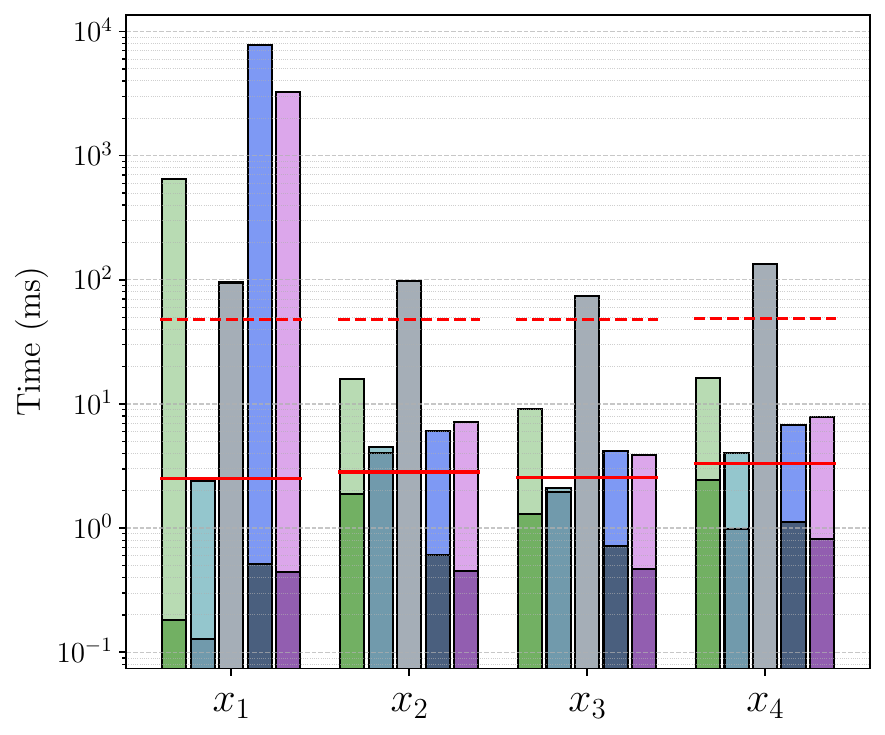}}
	\hfil
	\subfloat[]{\includegraphics[width=1.65in]{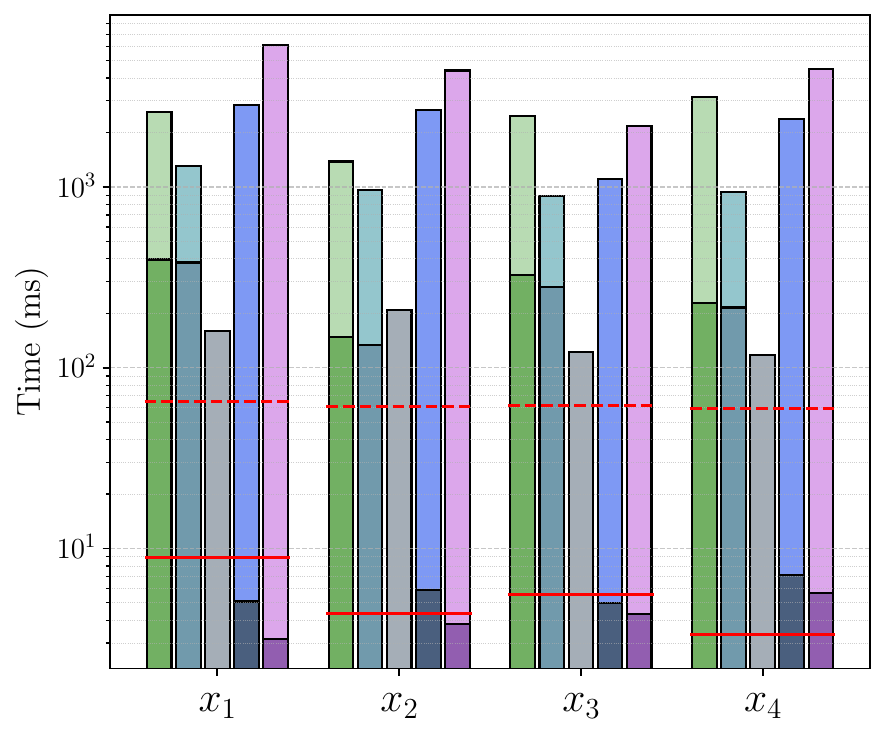}}
	\caption{(a)-(d) correspond to the illustrations of the average planning times for CDT-UTPP and the comparison algorithms in four simulated environments. The solid red line represents the planning time $t_{pla}$ for each task using CDT-UTPP, while the dashed red line indicates the total time $t_{pre} + t_{pla}$, including the preprocessing time for CDT-UTPP.}
	\label{fig_simUTPP}
\end{figure*}

In this subsection, to evaluate the efficiency of the CDT-UTPP algorithm proposed in Subsection 5.3 for solving distance-optimal path planning tasks in 2D environments, we compare CDT-UTPP with six state-of-the-art optimal path planning algorithms: RRT*, PRM* \citep{karaman2011sampling}, Informed-RRT* \citep{gammell2018informed}, FMT* \citep{janson2015fast}, BIT* \citep{gammell2020batch}, and Advanced BIT* (ABIT*) \citep{strub2020advanced}. As shown in Fig.~\ref{fig_simMap}, for each simulation environment, we set up four path planning tasks. Specifically, starting from $x_S$, we plan the optimal paths to $x_1$, $x_2$, $x_3$, and $x_4$.

For the CDT-UTPP algorithm, we conservatively set the parameter $\zeta$ to $2000$, $3000$, $4000$, and $4000$ pixels for the four simulation environments, respectively. This ensures that CDT-UTPP can obtain the optimal path between any two points in the environment.

In the experiments, for the CDT-UTPP algorithm, we record its preprocessing time $t_{pre}$ (\textbf{Algorithm~\ref{alg_UTPP}}) and the planning time $t_{pla}$ (\textbf{Algorithm~\ref{alg_UTPPGet}}) for each task. For RRT*, PRM*, Informed-RRT*, BIT*, and ABIT*, we record the time $t_{init}$ required to find an initial solution and the time $t_{5\%}$ required to find a solution with a cost less than $1.05 \times$ the optimal cost. For the FMT* algorithm, we set the number of sampling points to 20,000 and record the time $t_{fmt}$ required to find a solution. Each experiment is repeated 100 times, and the relevant data are statistically analyzed.

The average planning times for CDT-UTPP and the comparison algorithms are shown in Fig.~\ref{fig_simUTPP}. The experimental results demonstrate that the time $t_{pla}$ required by CDT-UTPP to plan the optimal path remains at the millisecond level. Even when accounting for the preprocessing time $t_{pre} + t_{pla}$, it fully satisfies the requirements for real-time performance. It is worth noting that CDT-UTPP is not an anytime algorithm; instead, it directly provides the planned optimal path after execution. Therefore, CDT-UTPP does not have an advantage over other anytime algorithms (such as RRT*, PRM*, Informed-RRT*, BIT*, and ABIT*) in terms of determining the initial path. However, this is unlikely to have a significant impact in practical applications, as the quality of the initial path is often poor and unsuitable for direct use in robot navigation.

In terms of efficiency in determining the optimal path, the experimental results show that, in most cases, CDT-UTPP performs comparably to or even outperforms the current state-of-the-art path planning algorithms. Particularly in relatively complex environments (e.g., Trap and Maze), the intricate connectivity structures of obstacles significantly diminish the performance advantages of informed sampling strategies used by Informed-RRT*, BIT*, and ABIT*, as well as the heuristic search strategies of BIT* and ABIT*. By contrast, CDT-UTPP exhibits remarkable stability in such environments, consistently completing the planning process in under 10 ms. This is because CDT-UTPP primarily achieves optimal path planning by identifying the homotopy class to which the path belongs. This planning approach is more influenced by the number of independent obstacles in the space and less affected by the specific spatial structure of each individual obstacle. The results shown in Fig.~\ref{fig_simUTPP}(a) and (b) further support this point: for environments with the same number of independent obstacles, the planning times for CDT-UTPP are nearly identical. However, due to changes in obstacle structures, the efficiency of other planners shows noticeable variations.

Finally, it is important to note that CDT-UTPP, as a novel planner, adopts a technical approach entirely different from traditional planners. While it outperforms existing planners in most cases, it still has some limitations that need to be addressed. For instance, for a new environment, we cannot directly determine the appropriate values for the parameters $x_\star$ and $\zeta$. In such cases, we must rely on experience and repeated trials to select a suitable $x_\star$ and $\zeta$. Additionally, CDT-UTPP is only applicable to static 2D environments and cannot be easily adapted, like other algorithms, to handle dynamic obstacles or real-time global planning through simple modifications. Nevertheless, we believe that CDT-UTPP is highly suitable for robotic platforms with relatively fixed working environments and limited computational resources. In such scenarios, engineers can carefully select an optimal $x_\star$ and $\zeta$ for the environment, which can then be embedded into the robot for long-term operation.

\begin{figure*}[!t]
	\centering
	\includegraphics[width=6.8in]{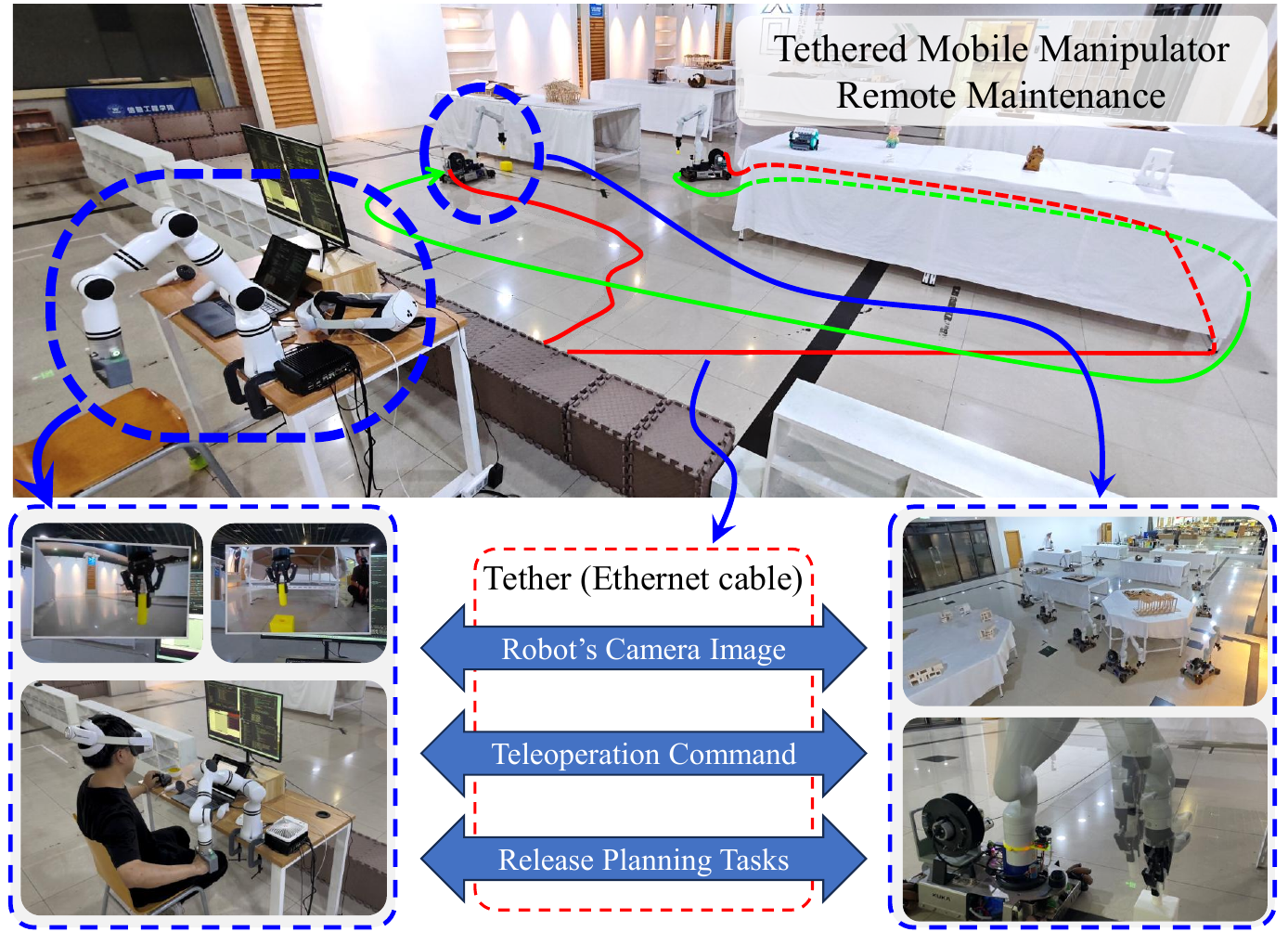}
	\caption{The tethered mobile manipulator automatically moves to the target location based on task inputs, after which the experimentalist remotely operates the robotic arm for assembly.}
	\label{fig_real}
\end{figure*}

\begin{figure*}[!t]
	\centering
	\subfloat[]{\includegraphics[height=1.45in]{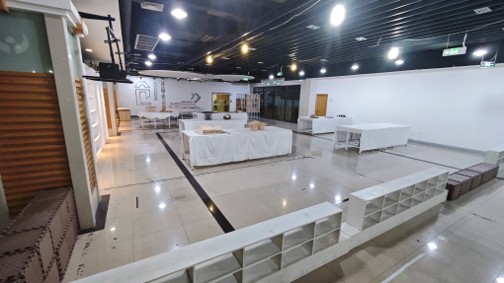}}
	\hfil
	\subfloat[]{\includegraphics[height=1.45in]{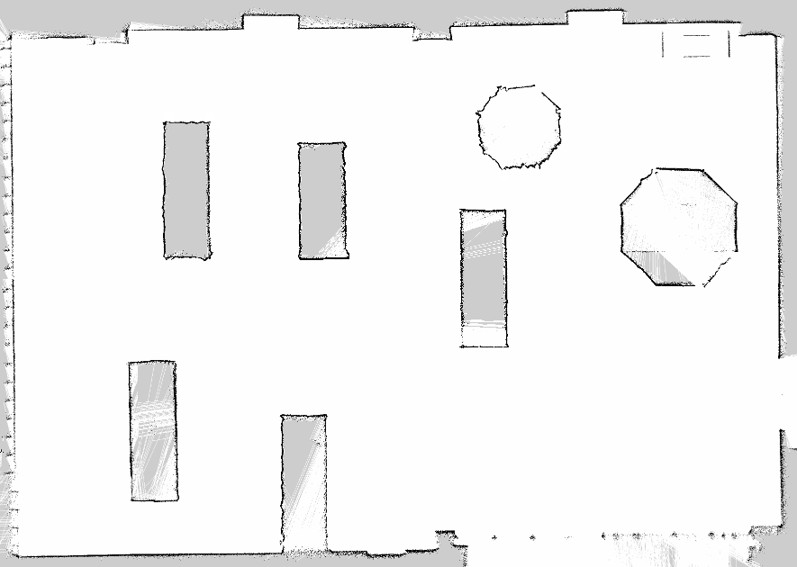}}
	\hfil
	\subfloat[]{\includegraphics[height=1.45in]{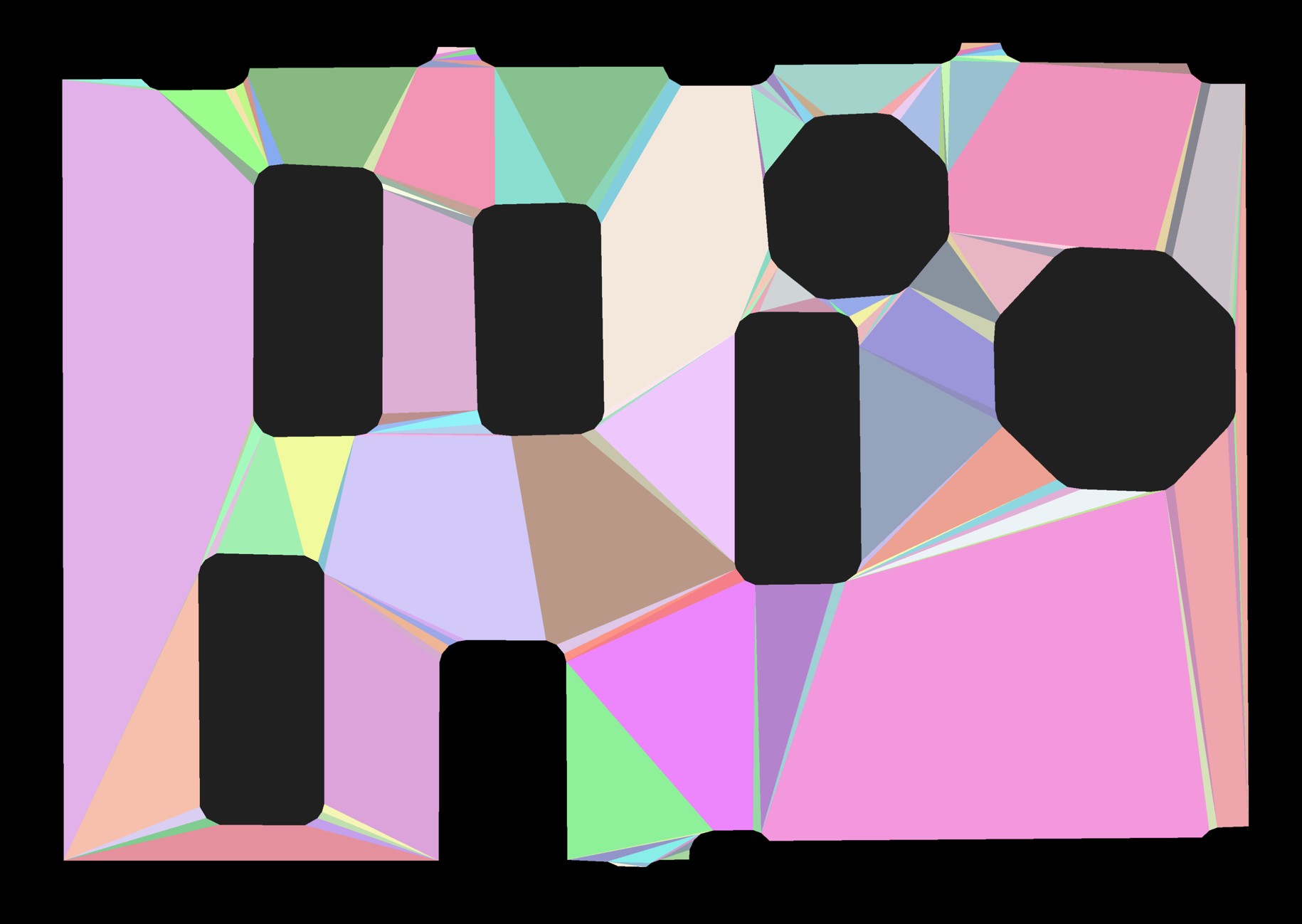}}
	\caption{Illustration of a real-world experimental site. (a) Photos of the exhibition hall. (b) 2D map of the exhibition hall created with Cartographer. (c) The results of the map expansion and convex division using the robot radius.}
	\label{fig_rEnv}
\end{figure*}

\begin{figure*}[!t]
	\centering
	\subfloat[]{\includegraphics[height=1.35in]{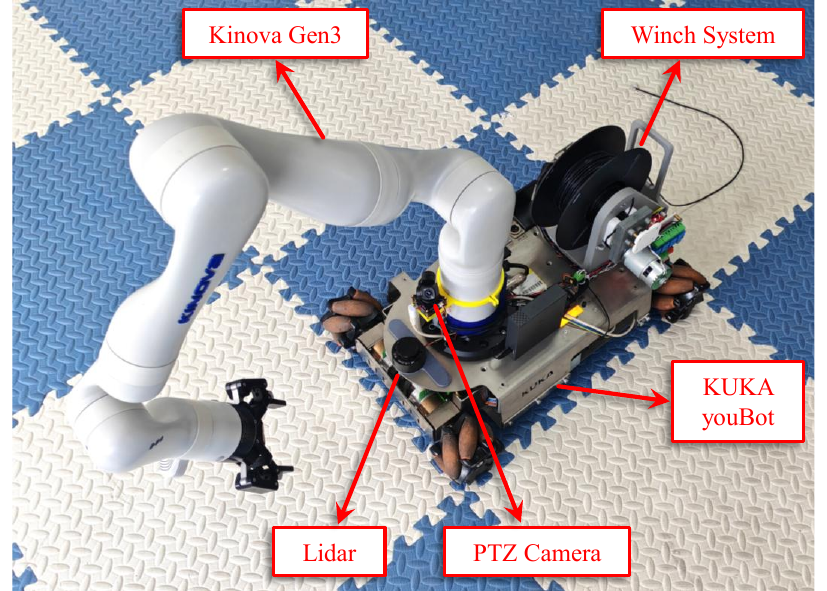}}
	\hfil
	\subfloat[]{\includegraphics[height=1.35in]{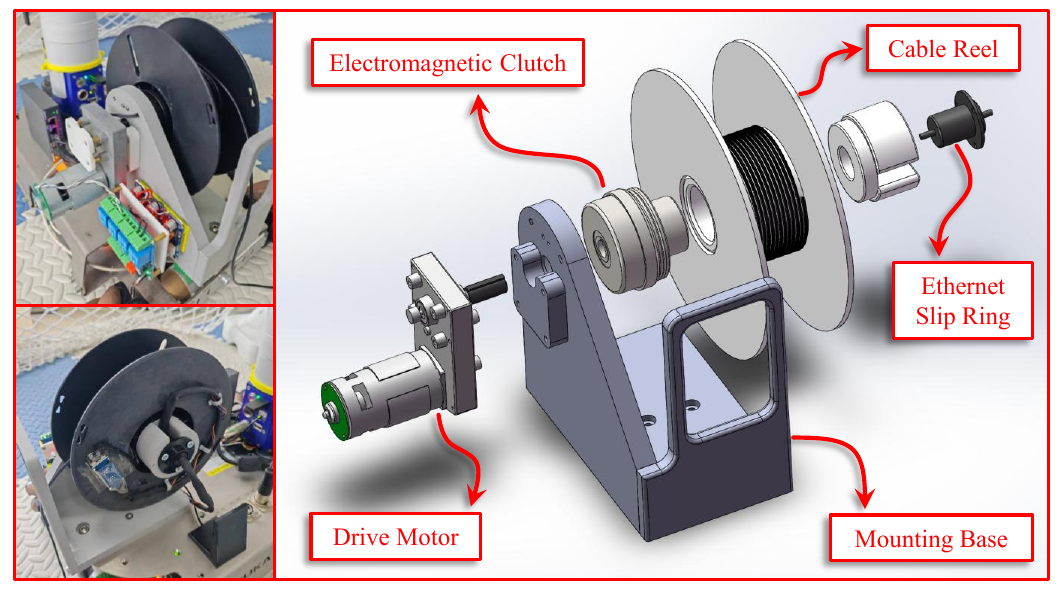}}
	\hfil
	\subfloat[]{\includegraphics[height=1.35in]{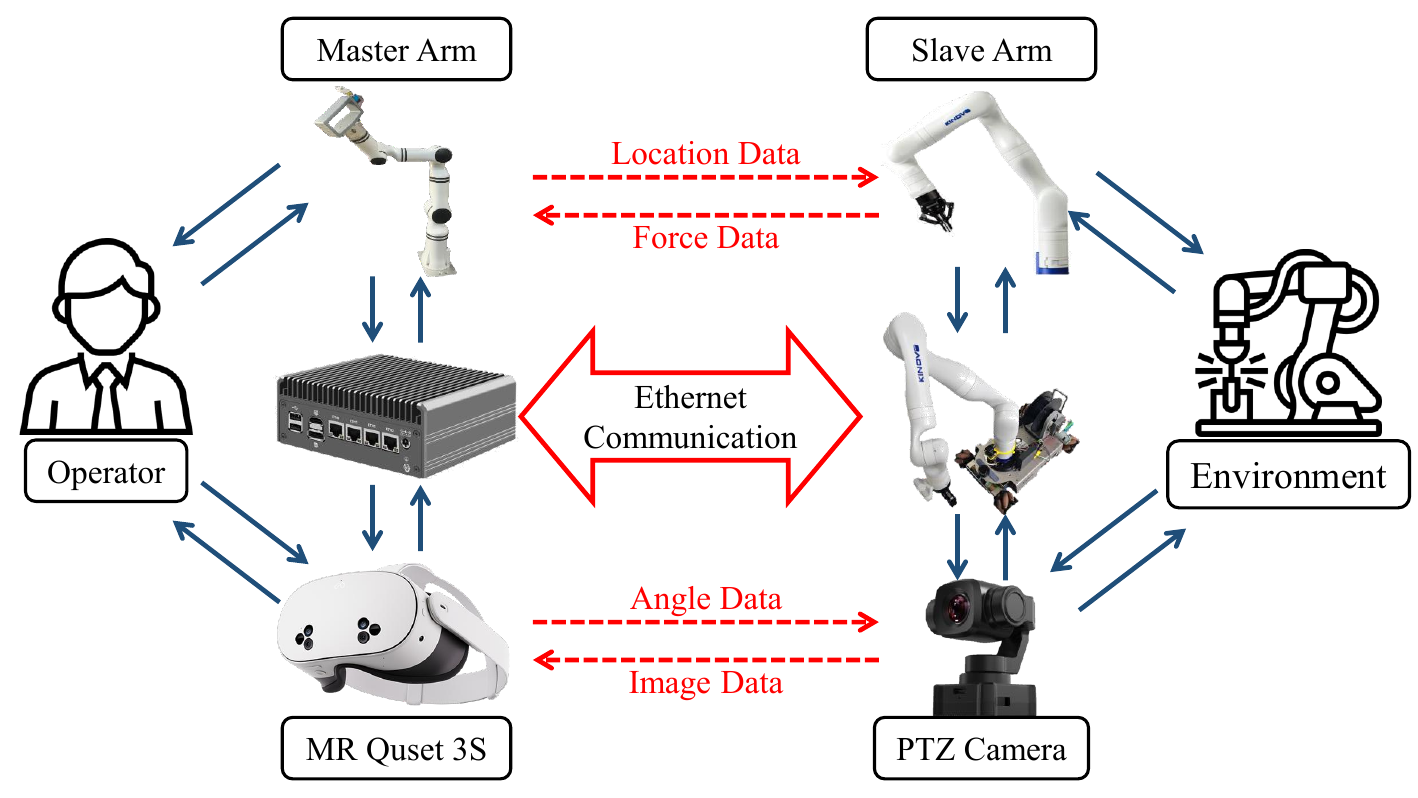}}
	\caption{(a) The tethered mobile manipulator used in the real-world experiment. (b) The robot's tethered rope winch system. (c) Illustration of teleoperation system used in the real-world experiment.}
	\label{fig_rRobot}
\end{figure*}

\section{Real-world Experiments}

To demonstrate the effectiveness and value of the proposed algorithms in practical applications, we present the results of experiments conducted in real-world scenarios. As shown in Fig.~\ref{fig_real}, we designed a task in which a tethered mobile manipulator operates in an unmanned environment without signal coverage to reach a designated area for maintenance. The maintenance operation is primarily performed by the experimenter through teleoperation of the robotic arm to achieve precise assembly of objects. During this process, the robot needs to maintain high-bandwidth real-time video and command transmission with the base station (anchor point) via an Ethernet cable (tether).

The test site, as shown in Fig.~\ref{fig_rEnv}, is a $12 \times 17$ square meter exhibition hall. The map of the environment was pre-constructed using Cartographer, with a resolution of 1 cm. The site contains 6 independent obstacles. The robot used in the experiment is shown in Fig.~\ref{fig_rRobot}(a). On the hardware side, the robot is based on the KUKA youBot platform, equipped with a Kinova Gen3 7-DOF robotic arm. A Robotiq 2F-85 gripper is mounted at the end of the robotic arm for object manipulation. Its main controller is an industrial PC (IPC) equipped with an Intel N100 processor (2.9 GHz) and 8 GB of RAM (4800 MT/s), which communicates with other components via Ethernet. The structure of the robot's tethered rope winch system is shown in Fig.~\ref{fig_rRobot}(b), in which a 16 m long Ethernet cable is stored. On the software side, the robot runs Ubuntu 20.04. The motion control of the robot is divided into two parts: chassis movement and teleoperated robotic arm assembly. The global path for the chassis movement is generated by the proposed algorithm based on task input, and trajectory tracking is performed using a PID algorithm. Localization is achieved using the Lama algorithm \citep{pedrosa2017efficient} for laser-based positioning. The teleoperation system for the robotic arm is shown in Fig.~\ref{fig_rRobot}(c). During the experiment, the operator uses MR glasses to obtain real-time visual feedback from the robot's perspective and controls the Kinova robotic arm through an RM75-6F master arm, achieving force-position synchronization to enable remote assembly of objects \citep{shi2025exploring}.

\subsection{Tethered Configuration Search and Path Planning}
This subsection presents experiments in which a tethered mobile manipulator carries repair accessories from an initial configuration to a target location to repair a damaged device. As shown in Fig.~\ref{fig_realTPPtask}, the above experiment is conducted in four groups. The initial configurations are represented by red curves and rectangles, the target positions are indicated by blue areas. During the experiments, we recorded the time taken by the CDT-TCS algorithm for tethered configuration search $t_{pre}$, the time taken by the CDT-TPP algorithm to plan the optimal path $t_{pla}$, as well as the robot's travel distance $C_{travel}$ and time $t_{travel}$. These data are summarized in \textbf{Table~\ref{tab:RealExpTPP}}. Fig.~\ref{fig_realTPPres} shows the trajectory of the tethered mobile manipulator and images captured during the assembly process.

\begin{figure*}[!t]
	\centering
	\subfloat[]{\includegraphics[width=1.6in]{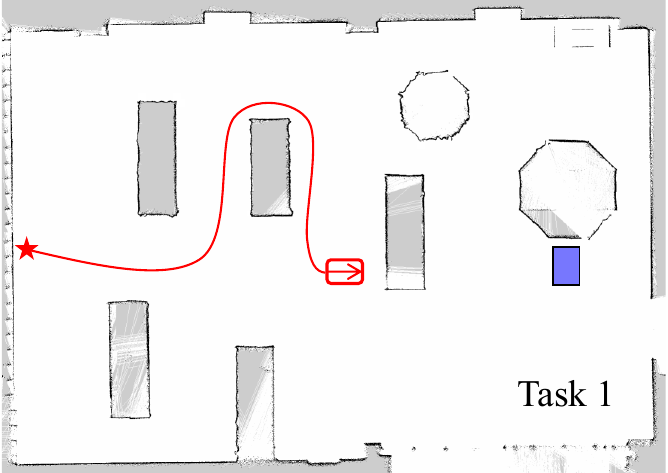}}
	\hfil
	\subfloat[]{\includegraphics[width=1.6in]{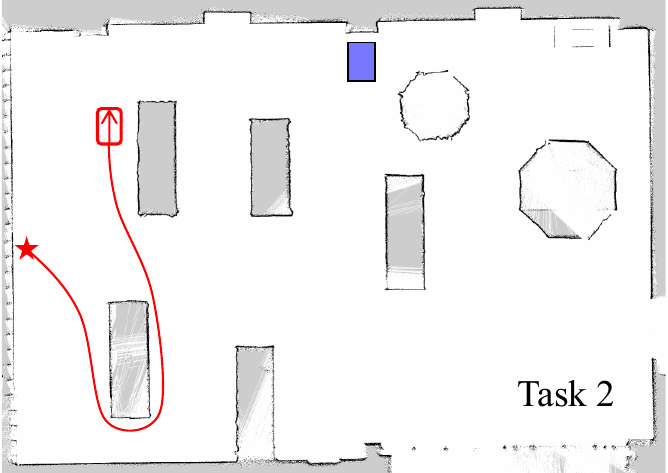}}
	\hfil
	\subfloat[]{\includegraphics[width=1.6in]{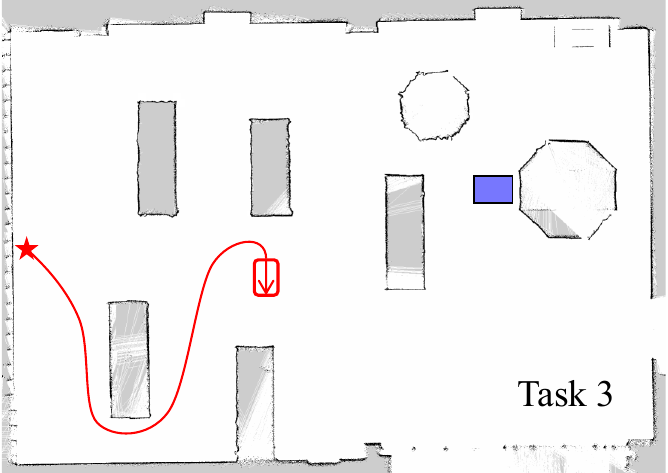}}
	\hfil
	\subfloat[]{\includegraphics[width=1.6in]{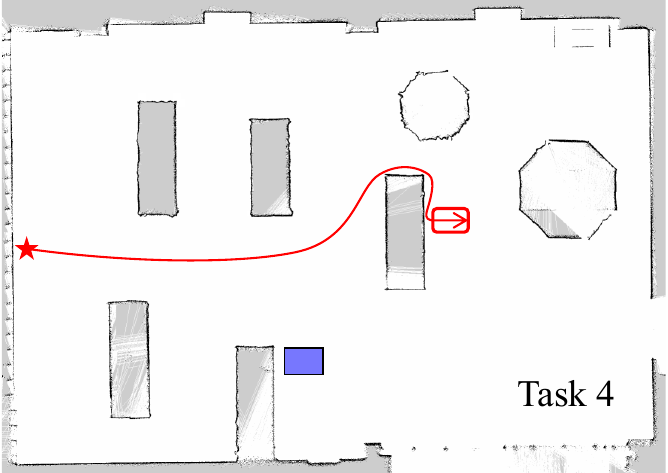}}
	\caption{Illustration of the Tethered Path Planning task in the real-world experiment. The initial configuration is represented by the red curve and rectangle, while the blue area indicates the target position.}
	\label{fig_realTPPtask}
\end{figure*}

\begin{figure*}[!t]
	\centering
	\subfloat[]{\includegraphics[width=1.6in]{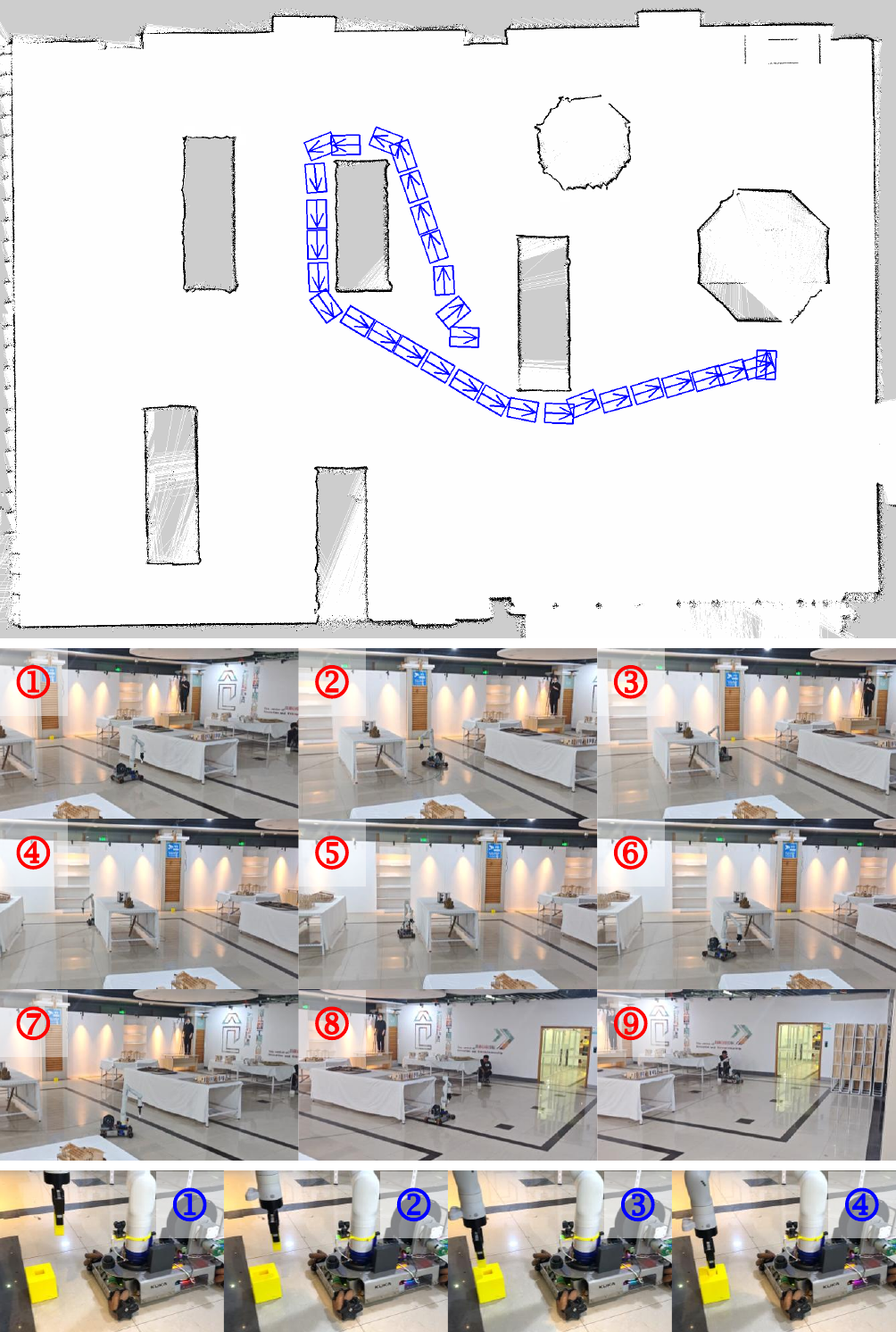}}
	\hfil
	\subfloat[]{\includegraphics[width=1.6in]{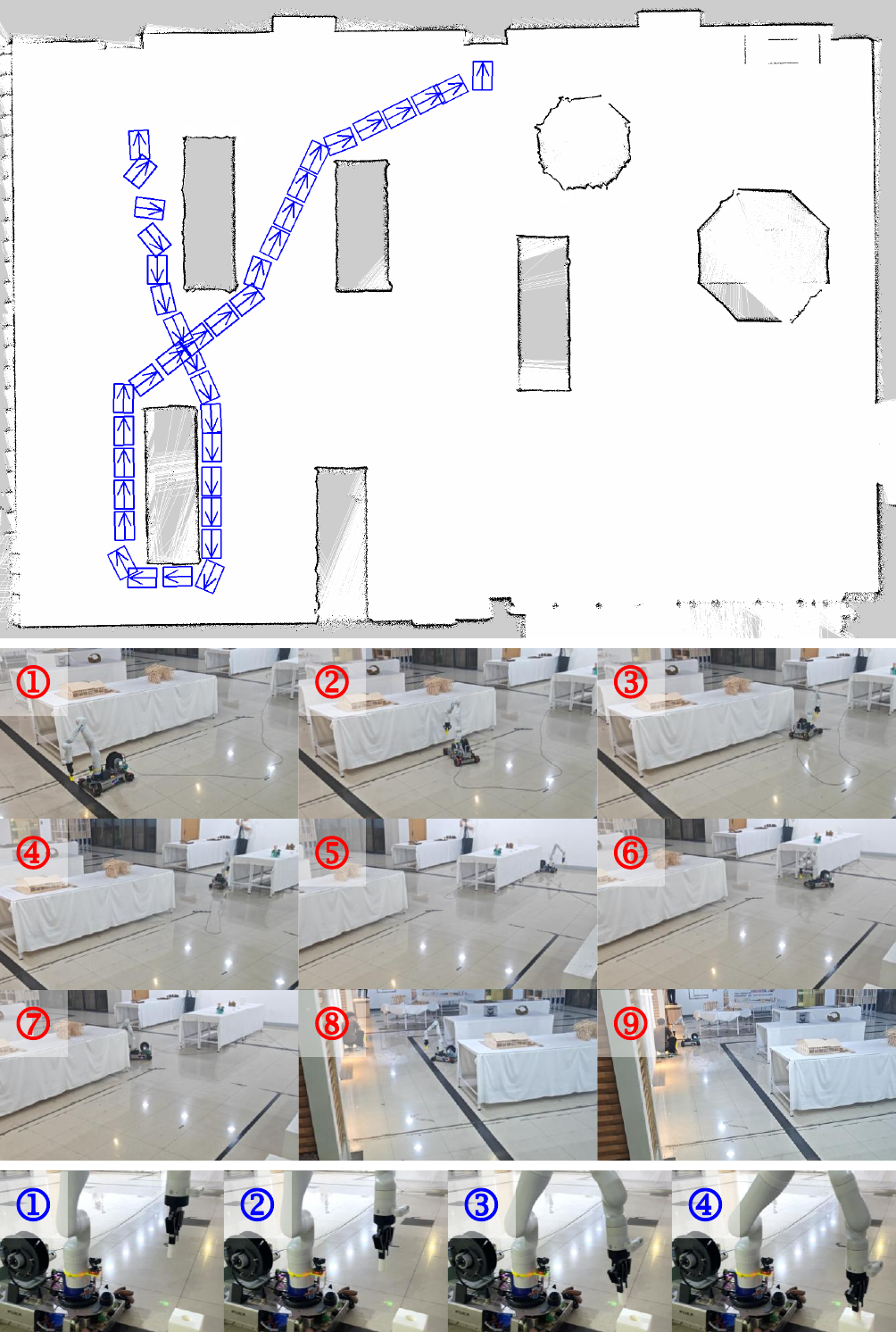}}
	\hfil
	\subfloat[]{\includegraphics[width=1.6in]{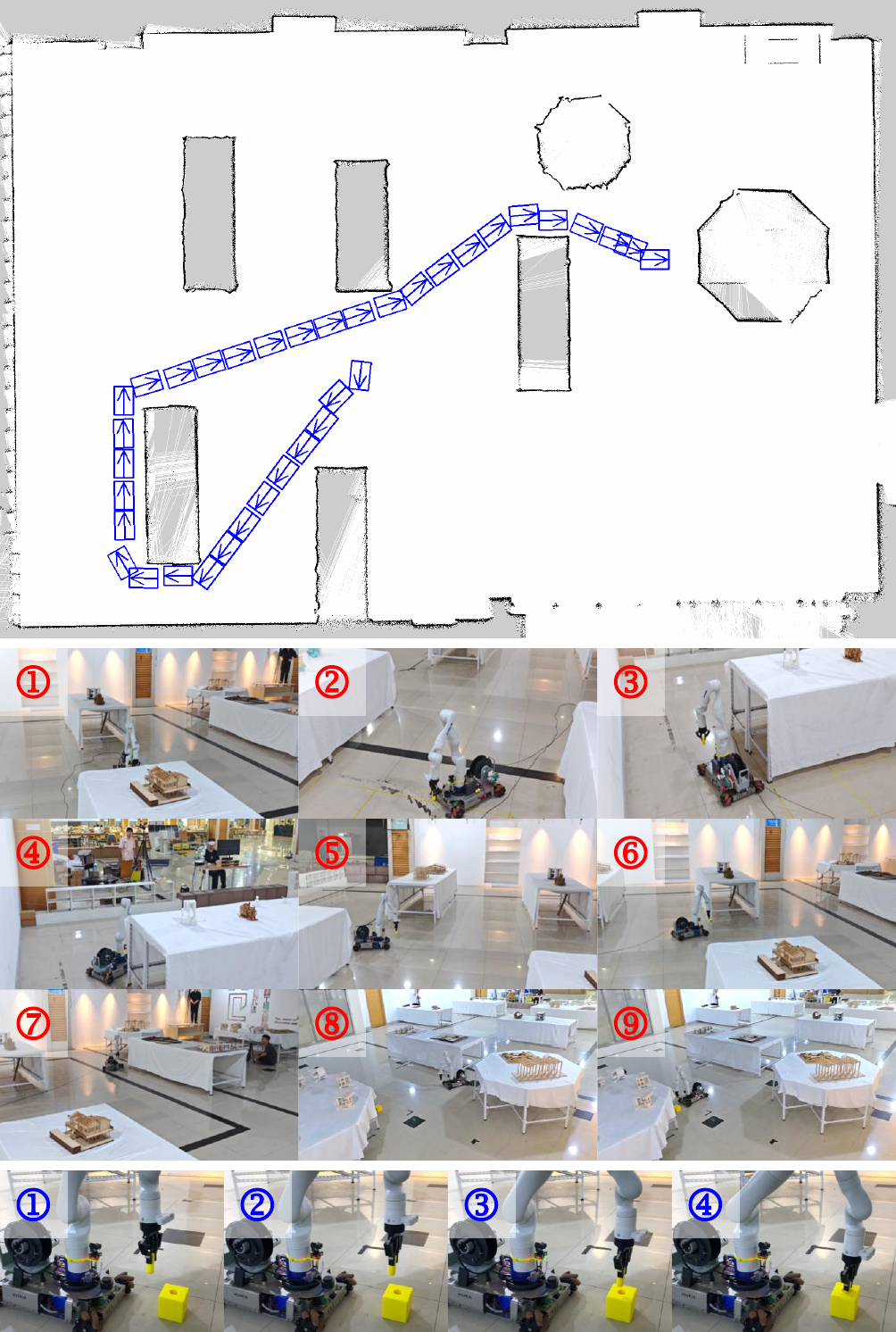}}
	\hfil
	\subfloat[]{\includegraphics[width=1.6in]{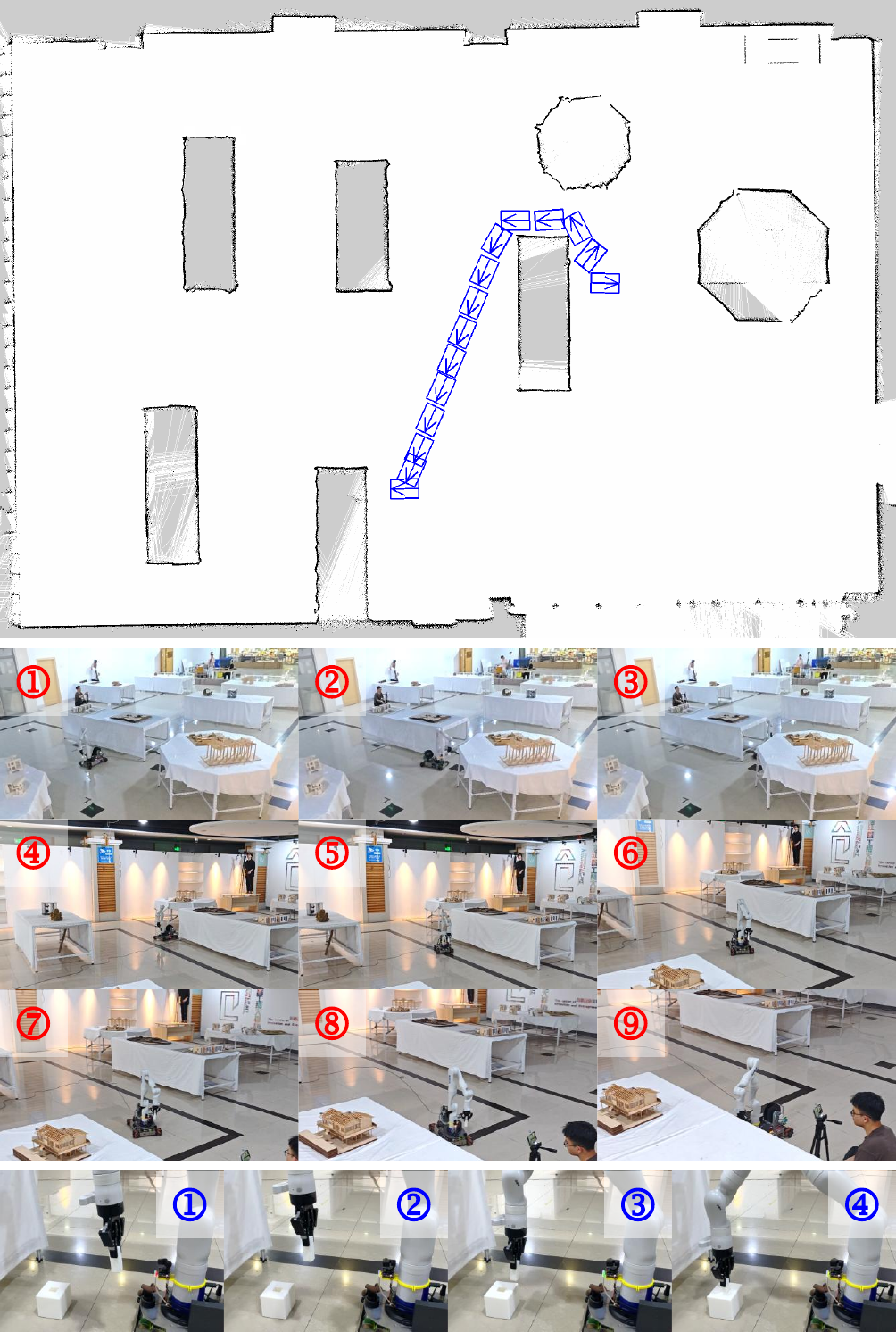}}
	\hfil
	\caption{(a)-(d) show the experimental processes for tasks 1 to 4 in the Tethered Path Planning real-world experiments. The upper part displays the trajectory of the tethered mobile manipulator, while the lower part presents snapshots of the mobile manipulator during movement and assembly.}
	\label{fig_realTPPres}
\end{figure*}

\begin{table}
	\begin{center}
		\caption{Experimental Results of Configuration Search and Path Planning for Tethered Mobile Manipulator \label{tab:RealExpTPP}}
		\small
		\begin{tabular}{|c|c|c|c|c|}
			\hline
			\makecell*[c]{Tasks} &
			\makecell*[c]{$t_{pre}$ (ms)} &
			\makecell*[c]{$t_{pla}$ ($\mu$s)} &
			\makecell*[c]{$t_{travel}$ (s)} &
			\makecell*[c]{$C_{travel}$ (m)}\\
			\hline
			\makecell*[c]{Task1} & \multirow{6.25}*{\makecell[c]{$6.63$}} & $58.33$ & $51.2$ & $18.08$ \\
			\cline{1-1} \cline{3-5}
			\makecell*[c]{Task2} &  & $79.18$ & $64.6$ & $23.76$ \\
			\cline{1-1} \cline{3-5}
			\makecell*[c]{Task3} &  & $81.64$ & $54.3$ & $21.59$ \\
			\cline{1-1} \cline{3-5}
			\makecell*[c]{Task4} &  & $77.92$ & $22.5$ & $8.40$ \\
			\hline
		\end{tabular}
	\end{center}
\end{table}

\subsection{Tethered Multi-Goal Visiting}
This subsection presents an experiment where a tethered mobile manipulator sequentially visits multiple target locations to pick up parts, perform assembly tasks, and then return to the initial configuration. The initial configuration and target locations are shown in Fig.~\ref{fig_realTMVtask}. In this experiment, the CDT-TMV algorithm took 3.52 ms to plan the optimal path, the robot traveled a total distance of 47.58 meters, and the entire experiment lasted 476 seconds, with the robot's movement accounting for 150 seconds of the total time.

\begin{figure}[!t]
	\centering
	\includegraphics[width=3.2in]{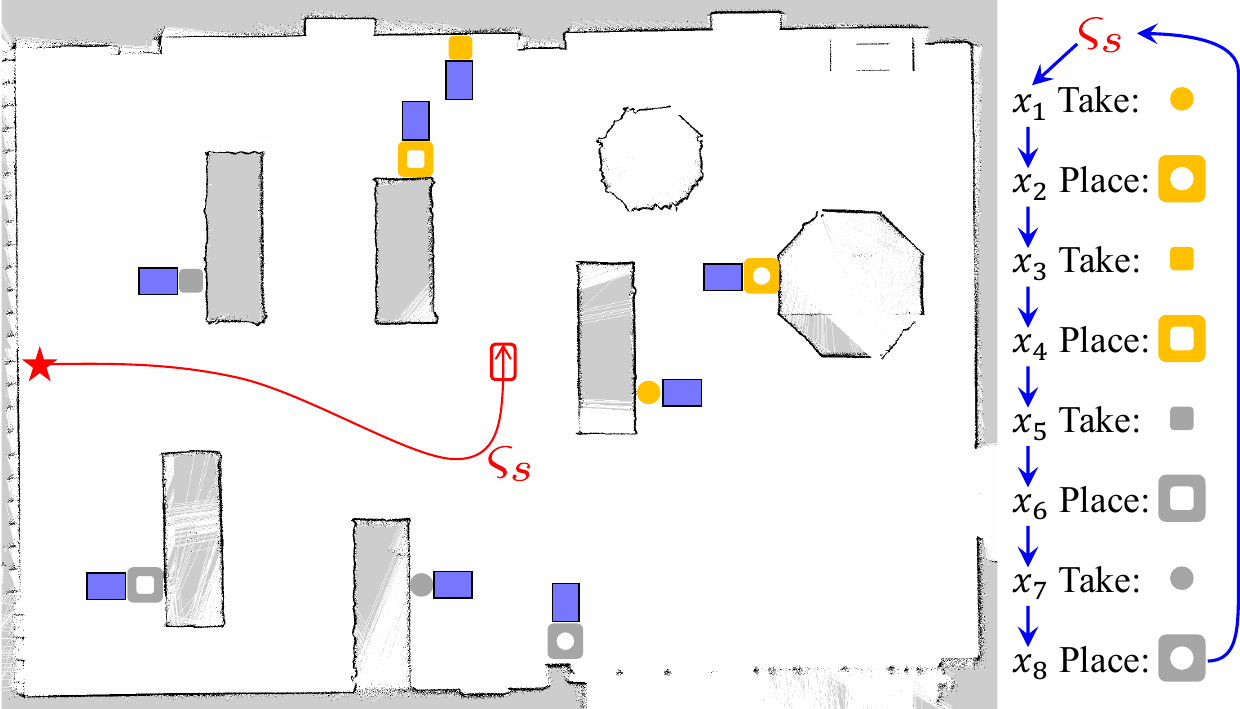}
	\caption{Illustration of the Tethered Multi-Goal Visiting task in the real-world experiment. The tethered mobile manipulator sequentially visits $x_1$ to $x_8$ to pick up components and perform assembly according to the task instructions on the right, returning to the initial configuration after completion.}
	\label{fig_realTMVtask}
\end{figure}

Fig.~\ref{fig_realTMVres} shows the trajectory of the tethered mobile manipulator during experiment, along with images captured during the assembly process. Fig.~\ref{fig_realTMVfun} illustrates the length of the Ethernet cable (tether) when it was fully tightened during the robot's movement. As can be observed from Fig.~\ref{fig_realTMVfun}, the graph consists of multiple upward-opening convex curves. The positions at the connection between these convex curves correspond to the target points. This observation strongly validates the correctness of \textbf{Theorem~\ref{th_genfconvex}} presented in this paper.

\begin{figure}[!t]
	\centering
	\includegraphics[width=3.2in]{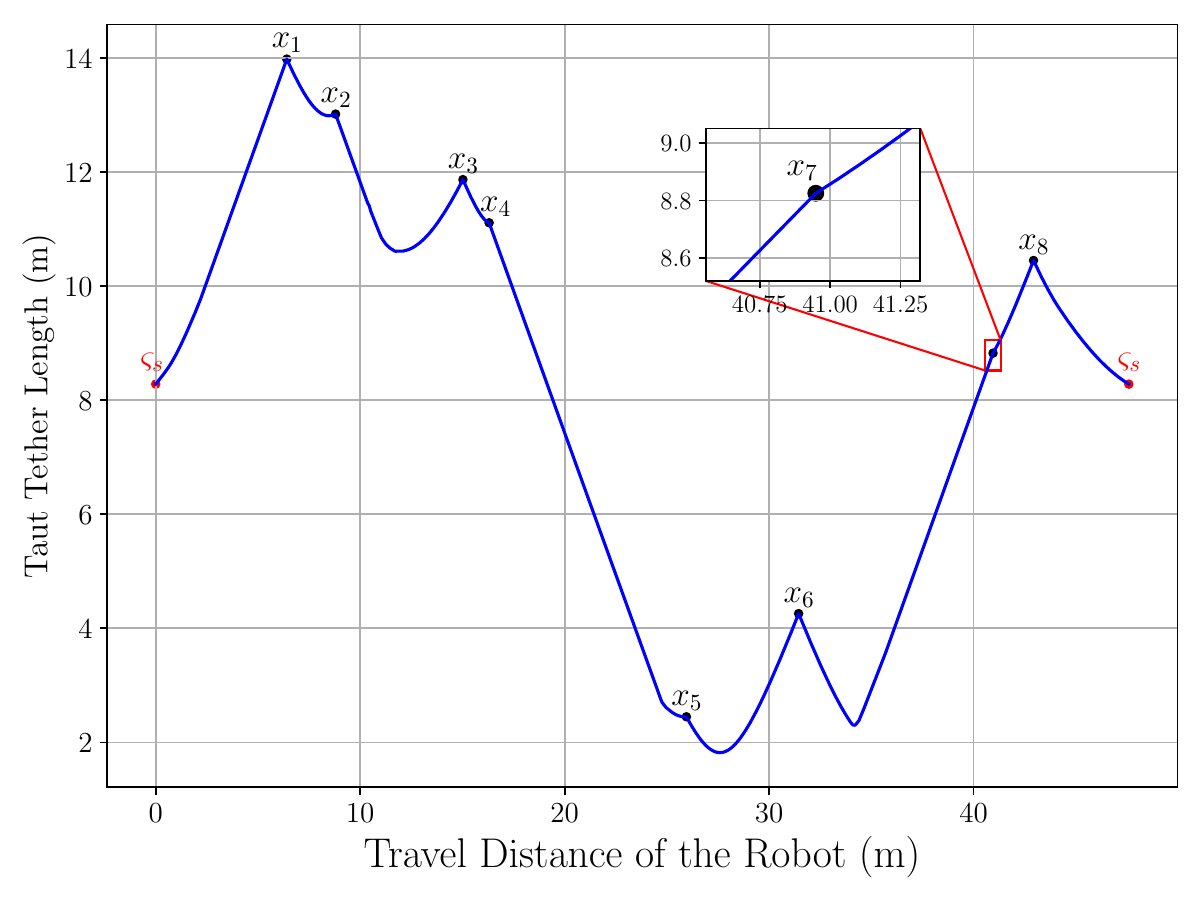}
	\caption{Relationship between the distance traveled by the robot along $\gamma^*$ and the length of the tightened tether during the Tethered Multi-Goal Visiting task, i.e., the graph of $c\left(\Theta\left(\varsigma_s * \gamma^*_{[0,t]}\right)\right)$ vs. $c(\gamma^*_{[0,t]})$.}
	\label{fig_realTMVfun}
\end{figure}

\begin{figure*}[!t]
	\centering
	\includegraphics[width=6.8in]{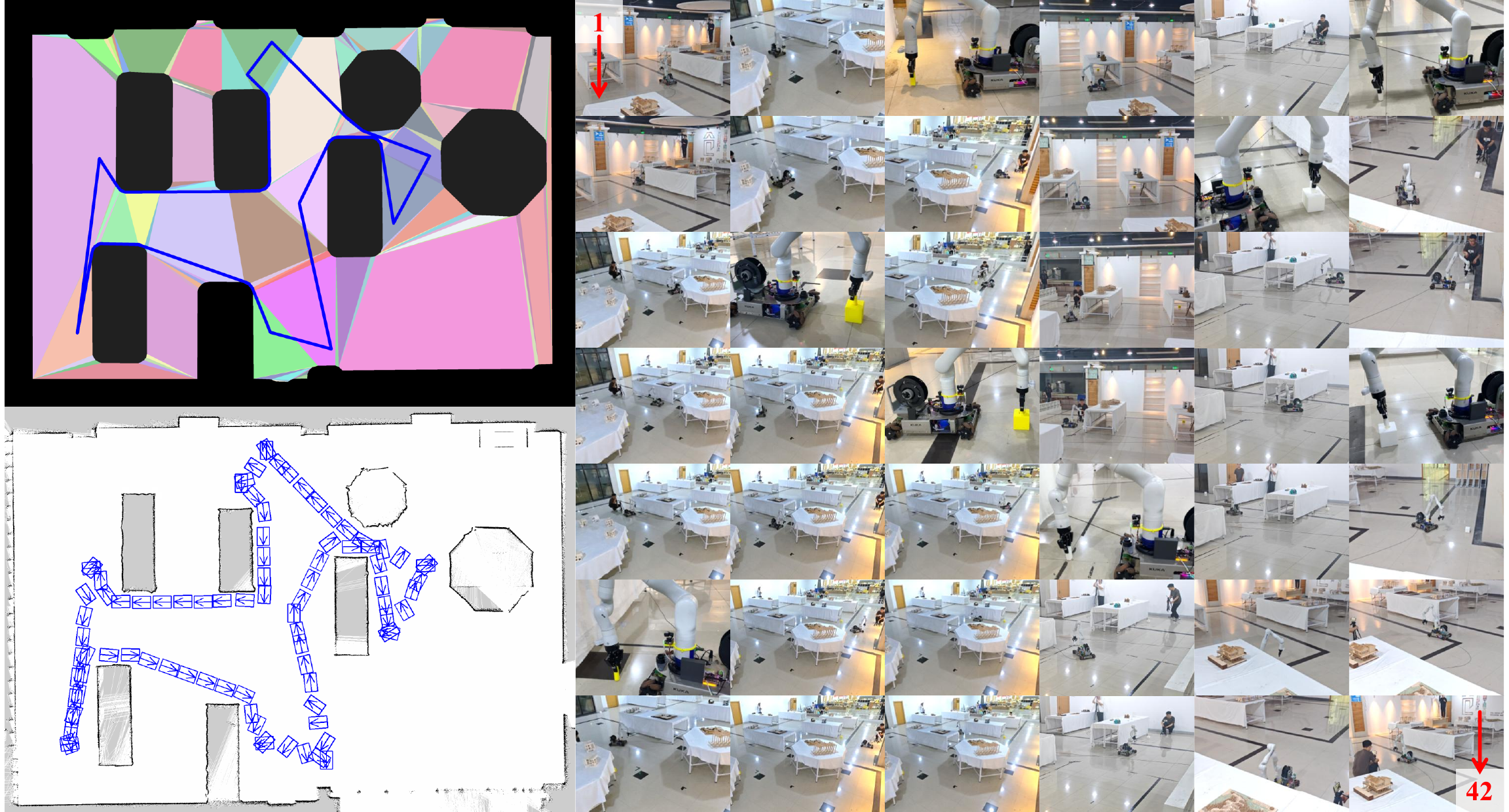}
	\caption{Illustration of the Tethered Multi-Goal Visiting real-world experiment process. The top left shows the robot's planned trajectory using CDT-TMV, the bottom left displays the actual trajectory of the tethered mobile manipulator, and the right side shows snapshots of the experimental process.}
	\label{fig_realTMVres}
\end{figure*}

\subsection{Untethered Path Planning}
In this subsection, we briefly validate the performance of the CDT-UTPP algorithm in real-world scenarios. To achieve this, we removed the tethered rope from the mobile manipulator and set the parameter $\zeta$ of CDT-UTPP to 18 meters. The start and end points of the planning tasks were transmitted to the robot via Wi-Fi. Subsequently, the robot autonomously moved along the optimal path planned by CDT-UTPP to reach the target position. The above experiment was conducted four times, with the start and end points for each task shown in Fig.~\ref{fig_realUTPPtask}(a). \textbf{Table~\ref{tab:RealExpUTPP}} records experimental results, including the initialization time of CDT-UTPP, the time required to plan the optimal path, the distance traveled by the robot, and the movement time. Fig.~\ref{fig_realUTPPtask}(b) illustrates the robot's trajectories during these experiments.

\begin{figure*}[!t]
	\centering
	\subfloat[]{\includegraphics[width=3.2in]{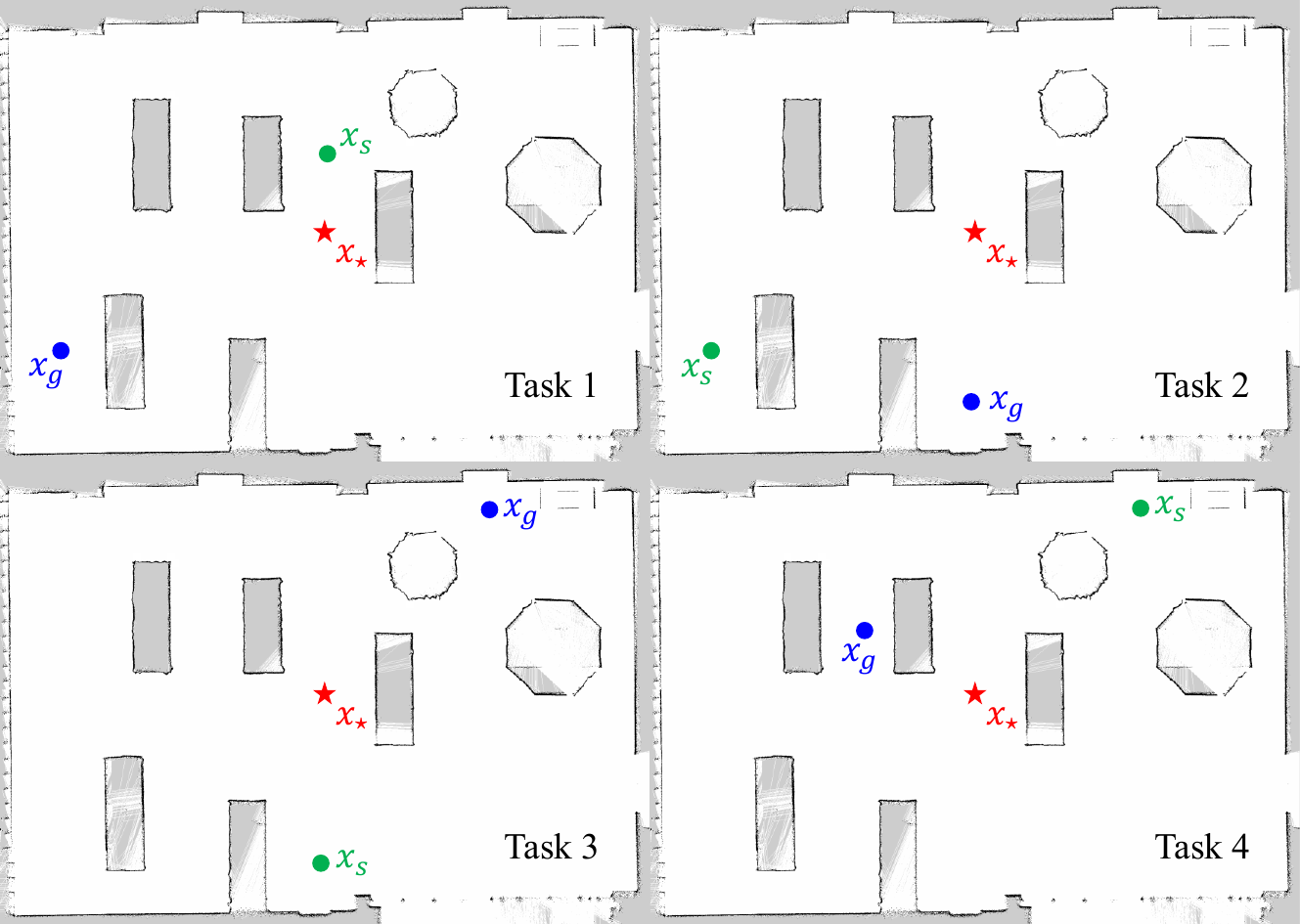}}
	\hfil
	\subfloat[]{\includegraphics[width=3.2in]{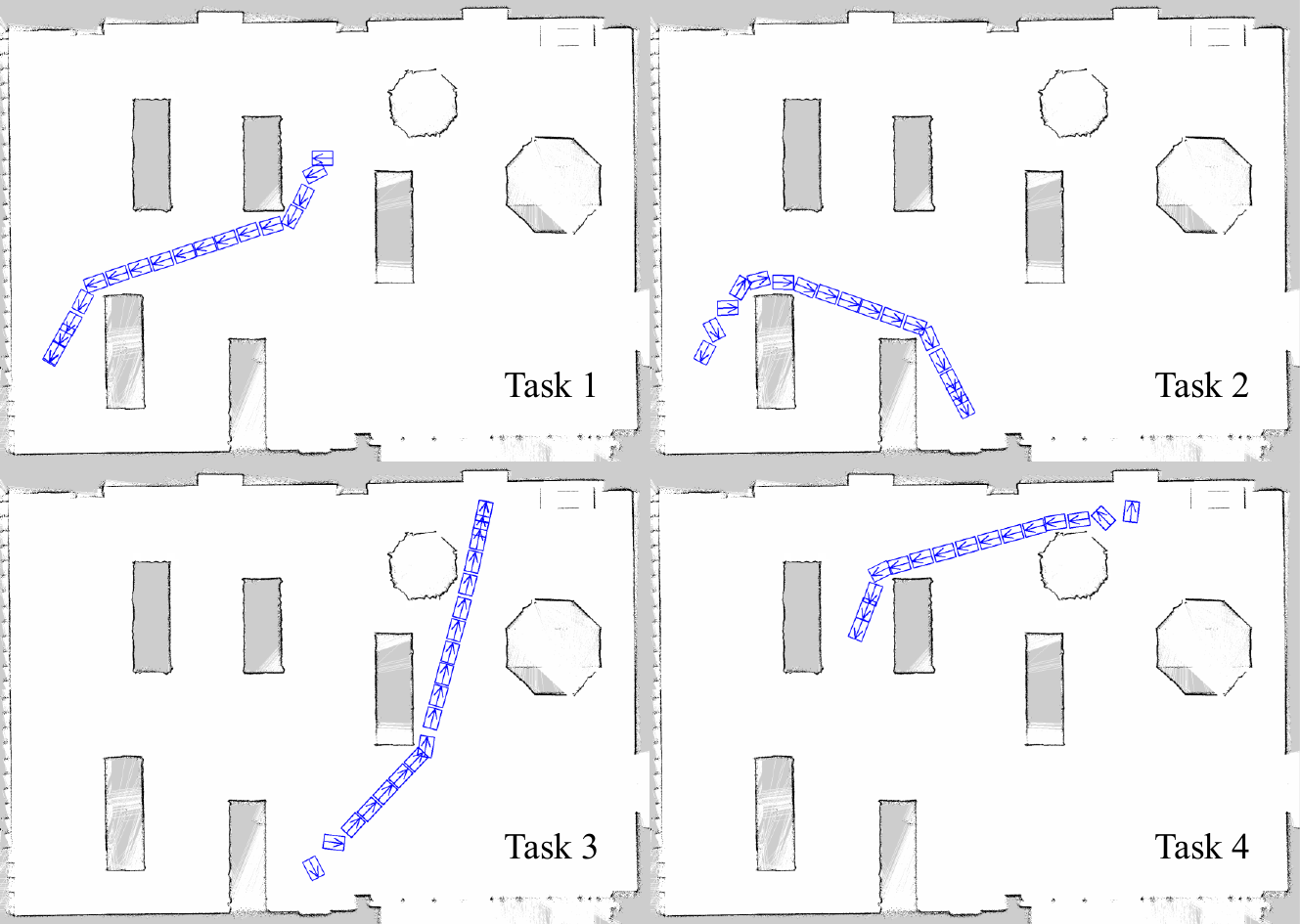}}
	\caption{(a) Illustration of the Untethered Path Planning task in the real-world experiment. $x_\star$ represents the placement position of the CDT-UTPP anchor point, while $x_s$ and $x_g$ denote the start and end points, respectively. (b) Trajectory of the robot during the experiment.}
	\label{fig_realUTPPtask}
\end{figure*}

\begin{table}
	\begin{center}
		\caption{Experimental Results of Initialization and Path Planning for Untethered Mobile Manipulator \label{tab:RealExpUTPP}}
		\small
		\begin{tabular}{|c|c|c|c|c|}
			\hline
			\makecell*[c]{Tasks} &
			\makecell*[c]{$t_{pre}$ (ms)} &
			\makecell*[c]{$t_{pla}$ ($\mu$s)} &
			\makecell*[c]{$t_{travel}$ (s)} &
			\makecell*[c]{$C_{travel}$ (m)}\\
			\hline
			\makecell*[c]{Task1} & \multirow{6.25}*{\makecell[c]{$11.3$}} & $600.36$ & $25.6$ & $9.66$ \\
			\cline{1-1} \cline{3-5}
			\makecell*[c]{Task2} &  & $479.06$ & $24.4$ & $9.67$ \\
			\cline{1-1} \cline{3-5}
			\makecell*[c]{Task3} &  & $534.67$ & $28.8$ & $11.12$ \\
			\cline{1-1} \cline{3-5}
			\makecell*[c]{Task4} &  & $821.44$ & $22.4$ & $8.89$ \\
			\hline
		\end{tabular}
	\end{center}
\end{table}

\subsection{Discussions}
While the proposed algorithms demonstrate strong performance in both simulated and real-world experiments, it is important to acknowledge their potential limitations when applied to practical scenarios.

First, our approach assumes that the environment can be represented as a polygonal workspace, which is a reasonable assumption for most bounded 2D environments. However, similar to many existing studies, we model the robot as a point mass, neglecting its physical dimensions (e.g., radius). In practice, this requires inflating the obstacle regions in the map to account for the robot's size. This inflation may result in underutilization of the tether length, as some positions that are physically reachable by the robot might be deemed infeasible due to the inflated constraints. Specifically, paths planned by our algorithm could be overly conservative, potentially preventing the robot from reaching certain locations near the tether length limit. Nevertheless, we argue that this limitation has minimal negative impact in most cases, as the tether length is typically much larger than the robot's radius. Moreover, this conservative planning approach enhances safety by avoiding scenarios where the tether is fully stretched, thereby reducing the risk of tether breakage.

Second, our method relies on the assumption that the tether can be freely dragged within the $X_{free}$. If the environment contains numerous obstacles with sharp edges or narrow passages that could entangle or trap the tether, the algorithm may face challenges in directly applying its current framework. For instance, such environments could lead to inaccurate feasibility assessments or suboptimal path planning. To address this issue, one potential solution is to modify the tether's physical structure, such as using flexible materials or incorporating active tension control mechanisms, to mitigate the risk of entanglement. Alternatively, future work could explore integrating additional sensory feedback or adaptive planning strategies to handle highly cluttered environments more effectively.

In summary, while these limitations highlight areas for further improvement, the robustness and efficiency of our algorithms remain advantageous for a wide range of tethered robot applications.

\section{Conclusion}

This paper focuses on the problem of optimal configuration search and path planning for tethered robots in 2D environments. By integrating a homotopy invariant framework based on Convex Dissection Topology (CDT) with concepts from algebraic topology and geometric optimization, we propose a simple yet efficient iterative process (\textbf{Remark~\ref{re_TCS}}) to rapidly determine the \textbf{Tethered Configuration CDT Encoding Set} for tethered robots. Furthermore, we derive the range of the cost function for optimal configurations during linear motion (\textbf{Theorem~\ref{th_finequality}}) and rigorously prove its convexity (\textbf{Theorem~\ref{th_fconvex}}). This theoretical result significantly simplifies the computational complexity of the iterative process. Building on this foundation, we introduce a novel framework, CDT-TCS, which efficiently computes the complete set of optimal feasible configurations for tethered robots at all positions in the environment through a single computation.

In extending CDT-TCS to practical applications, we first analyze the solution space for the optimal Tethered Path Planning (TPP) problem and then derive and prove a generalized form of \textbf{Theorem~\ref{th_fconvex}} (\textbf{Theorem~\ref{th_genfconvex}}). Leveraging this result, we propose CDT-TPP, an algorithm designed for fast resolution of the optimal TPP problem. For the optimal Tethered Multi-Goal Visiting (TMV) problem, we demonstrate that it can be transformed into a generalized path planning problem on a positively weighted directed graph $\mathcal{G}_{\text{TMV}}$. By combining Dijkstra's algorithm with CDT encoding, we develop CDT-TMV, which achieves rapid solutions for the optimal TMV problem. Additionally, we investigate the properties of globally optimal paths for untethered robots (\textbf{Theorem~\ref{th_optPath}} and \textbf{\ref{th_prUTPP}}) and propose CDT-UTPP, an algorithm capable of efficiently computing the globally optimal path between any two points in the environment after a preprocessing phase.

Finally, we extensively evaluate the proposed algorithms in both simulated and real-world environments. Experimental results demonstrate that these algorithms achieve satisfactory performance, exhibiting practicality and engineering value. Future work will focus on extending this framework to multi tethered robot path planning and addressing real-time path planning challenges for tethered robots in dynamic environments.

\begin{dci}
The author(s) declared no potential conﬂicts of interest with respect to the research, authorship, and/or publication of this article.
\end{dci}

\begin{funding}
The author(s) disclosed receipt of the following ﬁnancial support for the research, authorship, and/or publication of this article: This work was supported by the National Natural Science Foundation of China (62173305).
\end{funding}

\bibliographystyle{SageH}
\bibliography{uBIB}

\begin{thebibliography}{73}
\providecommand{\natexlab}[1]{#1}
\providecommand{\url}[1]{\texttt{#1}}
\providecommand{\urlprefix}{URL }
\expandafter\ifx\csname urlstyle\endcsname\relax
  \providecommand{\doi}[1]{DOI:\discretionary{}{}{}#1}\else
  \providecommand{\doi}{DOI:\discretionary{}{}{}\begingroup
  \urlstyle{rm}\Url}\fi

\bibitem[{An et~al.(2024)An, Wang, Wang, Wang, Huang, He and
  Wang}]{an2024etpnav}
An D, Wang H, Wang W, Wang Z, Huang Y, He K and Wang L (2024) Etpnav: Evolving
  topological planning for vision-language navigation in continuous
  environments.
\newblock \emph{IEEE Transactions on Pattern Analysis and Machine Intelligence}
  .

\bibitem[{Bhattacharya and Ghrist(2018)}]{bhattacharya2018path}
Bhattacharya S and Ghrist R (2018) Path homotopy invariants and their
  application to optimal trajectory planning.
\newblock \emph{Annals of Mathematics and Artificial Intelligence} 84(3):
  139--160.

\bibitem[{Bhattacharya et~al.(2015)Bhattacharya, Kim, Heidarsson, Sukhatme and
  Kumar}]{bhattacharya2015topological}
Bhattacharya S, Kim S, Heidarsson H, Sukhatme GS and Kumar V (2015) A
  topological approach to using cables to separate and manipulate sets of
  objects.
\newblock \emph{The International Journal of Robotics Research} 34(6):
  799--815.

\bibitem[{Bhattacharya et~al.(2012)Bhattacharya, Likhachev and
  Kumar}]{bhattacharya2012topological}
Bhattacharya S, Likhachev M and Kumar V (2012) Topological constraints in
  search-based robot path planning.
\newblock \emph{Autonomous Robots} 33: 273--290.

\bibitem[{Cao et~al.(2023)Cao, Cao, Yuan, Nguyen and Xie}]{cao2023neptune}
Cao M, Cao K, Yuan S, Nguyen TM and Xie L (2023) Neptune: Nonentangling
  trajectory planning for multiple tethered unmanned vehicles.
\newblock \emph{IEEE Transactions on Robotics} 39(4): 2786--2804.

\bibitem[{Chen et~al.(2021)Chen, Chen, Chuang, V{\'a}zquez and
  Savarese}]{chen2021topological}
Chen K, Chen JK, Chuang J, V{\'a}zquez M and Savarese S (2021) Topological
  planning with transformers for vision-and-language navigation.
\newblock In: \emph{Proceedings of the IEEE/CVF Conference on Computer Vision
  and Pattern Recognition}. pp. 11276--11286.

\bibitem[{Chen et~al.(2022)Chen, Pei, Lu and Li}]{chen2022deep}
Chen P, Pei J, Lu W and Li M (2022) A deep reinforcement learning based method
  for real-time path planning and dynamic obstacle avoidance.
\newblock \emph{Neurocomputing} 497: 64--75.

\bibitem[{Chen et~al.(2013)Chen, Huang, Ren, He and He}]{chen2013history}
Chen Y, Huang R, Ren X, He L and He Y (2013) History of the tether concept and
  tether missions: a review.
\newblock \emph{International Scholarly Research Notices} 2013(1): 502973.

\bibitem[{Chi et~al.(2022)Chi, Ding, Wang, Chen and Sun}]{chi2021generalized}
Chi W, Ding Z, Wang J, Chen G and Sun L (2022) A generalized voronoi
  diagram-based efficient heuristic path planning method for rrts in mobile
  robots.
\newblock \emph{IEEE Transactions on Industrial Electronics} 69(5): 4926--4937.

\bibitem[{Chintam et~al.(2024)Chintam, Lei, Osmanoglu, Wang and
  Luo}]{chintam2024informed}
Chintam P, Lei T, Osmanoglu B, Wang Y and Luo C (2024) Informed sampling space
  driven robot informative path planning.
\newblock \emph{Robotics and Autonomous Systems} 175: 104656.

\bibitem[{Chipade et~al.(2024)Chipade, Kumar and Yong}]{chipade2024withy}
Chipade VS, Kumar R and Yong SZ (2024) Withy a*: Winding-constrained motion
  planning for tethered robot using hybrid a.
\newblock In: \emph{2024 IEEE International Conference on Robotics and
  Automation (ICRA)}. IEEE, pp. 8771--8777.

\bibitem[{Choudhury et~al.(2016)Choudhury, Gammell, Barfoot, Srinivasa and
  Scherer}]{choudhury2016regionally}
Choudhury S, Gammell JD, Barfoot TD, Srinivasa SS and Scherer S (2016)
  Regionally accelerated batch informed trees (rabit*): A framework to
  integrate local information into optimal path planning.
\newblock In: \emph{2016 IEEE International Conference on Robotics and
  Automation (ICRA)}. IEEE, pp. 4207--4214.

\bibitem[{Dijkstra(1959)}]{dijkstra1959note}
Dijkstra EW (1959) A note on two problems in connexion with graphs:(numerische
  mathematik, 1 (1959), p 269-271) .

\bibitem[{Du et~al.(2023)Du, He, Cao, Garg, Kaddoum and Hassan}]{du2023ai}
Du P, He X, Cao H, Garg S, Kaddoum G and Hassan MM (2023) Ai-based
  energy-efficient path planning of multiple logistics uavs in intelligent
  transportation systems.
\newblock \emph{Computer Communications} 207: 46--55.

\bibitem[{Filliung et~al.(2024)Filliung, Drupt, Peraud, Dune, Boizot, Comport,
  Anthierens and Hugel}]{filliung2024augmented}
Filliung M, Drupt J, Peraud C, Dune C, Boizot N, Comport A, Anthierens C and
  Hugel V (2024) An augmented catenary model for underwater tethered robots.
\newblock In: \emph{2024 IEEE International Conference on Robotics and
  Automation (ICRA)}. IEEE, pp. 6051--6057.

\bibitem[{Friedrich et~al.(2017)Friedrich, Csiszar, Lechler and
  Verl}]{friedrich2017efficient}
Friedrich C, Csiszar A, Lechler A and Verl A (2017) Efficient task and path
  planning for maintenance automation using a robot system.
\newblock \emph{IEEE Transactions on Automation Science and Engineering} 15(3):
  1205--1215.

\bibitem[{Gammell et~al.(2018)Gammell, Barfoot and
  Srinivasa}]{gammell2018informed}
Gammell JD, Barfoot TD and Srinivasa SS (2018) Informed sampling for
  asymptotically optimal path planning.
\newblock \emph{IEEE Transactions on Robotics} 34(4): 966--984.

\bibitem[{Gammell et~al.(2020)Gammell, Barfoot and
  Srinivasa}]{gammell2020batch}
Gammell JD, Barfoot TD and Srinivasa SS (2020) Batch informed trees (bit*):
  Informed asymptotically optimal anytime search.
\newblock \emph{The International Journal of Robotics Research} 39(5):
  543--567.

\bibitem[{Gammell and Strub(2021)}]{gammell2021asymptotically}
Gammell JD and Strub MP (2021) Asymptotically optimal sampling-based motion
  planning methods.
\newblock \emph{Annual Review of Control, Robotics, and Autonomous Systems} 4:
  295--318.

\bibitem[{Ge et~al.(2011)Ge, Zhang, Abraham and Rebsamen}]{ge2011simultaneous}
Ge SS, Zhang Q, Abraham AT and Rebsamen B (2011) Simultaneous path planning and
  topological mapping (sp2atm) for environment exploration and goal oriented
  navigation.
\newblock \emph{Robotics and Autonomous Systems} 59(3-4): 228--242.

\bibitem[{Gross and Tucker(2001)}]{gross2001topological}
Gross JL and Tucker TW (2001) \emph{Topological graph theory}.
\newblock Courier Corporation.

\bibitem[{Han et~al.(2022)Han, Qu, Tong, Jiang, Zlatanova, Wang and
  Cheng}]{han2022grid}
Han B, Qu T, Tong X, Jiang J, Zlatanova S, Wang H and Cheng C (2022)
  Grid-optimized uav indoor path planning algorithms in a complex environment.
\newblock \emph{International Journal of Applied Earth Observation and
  Geoinformation} 111: 102857.

\bibitem[{Hart et~al.(1968)Hart, Nilsson and Raphael}]{hart1968formal}
Hart PE, Nilsson NJ and Raphael B (1968) A formal basis for the heuristic
  determination of minimum cost paths.
\newblock \emph{IEEE transactions on Systems Science and Cybernetics} 4(2):
  100--107.

\bibitem[{Hert and Lumelsky(1995)}]{hert1995moving}
Hert S and Lumelsky V (1995) Moving multiple tethered robots between arbitrary
  configurations.
\newblock In: \emph{Proceedings 1995 IEEE/RSJ International Conference on
  Intelligent Robots and Systems. Human Robot Interaction and Cooperative
  Robots}, volume~2. IEEE, pp. 280--285.

\bibitem[{Hert and Lumelsky(1996)}]{hert1996ties}
Hert S and Lumelsky V (1996) The ties that bind: Motion planning for multiple
  tethered robots.
\newblock \emph{Robotics and autonomous systems} 17(3): 187--215.

\bibitem[{Hert and Lumelsky(2002)}]{hert2002motion}
Hert S and Lumelsky V (2002) Motion planning in r/sup 3/for multiple tethered
  robots.
\newblock \emph{IEEE transactions on robotics and automation} 15(4): 623--639.

\bibitem[{Hitz et~al.(2017)Hitz, Galceran, Garneau, Pomerleau and
  Siegwart}]{hitz2017adaptive}
Hitz G, Galceran E, Garneau M{\`E}, Pomerleau F and Siegwart R (2017) Adaptive
  continuous-space informative path planning for online environmental
  monitoring.
\newblock \emph{Journal of Field Robotics} 34(8): 1427--1449.

\bibitem[{Janson et~al.(2015)Janson, Schmerling, Clark and
  Pavone}]{janson2015fast}
Janson L, Schmerling E, Clark A and Pavone M (2015) Fast marching tree: A fast
  marching sampling-based method for optimal motion planning in many
  dimensions.
\newblock \emph{The International journal of Robotics Research} 34(7):
  883--921.

\bibitem[{Jeong et~al.(2019)Jeong, Lee and Kim}]{jeong2019quick}
Jeong IB, Lee SJ and Kim JH (2019) Quick-rrt*: Triangular inequality-based
  implementation of rrt* with improved initial solution and convergence rate.
\newblock \emph{Expert Systems with Applications} 123: 82--90.

\bibitem[{Karaman and Frazzoli(2011)}]{karaman2011sampling}
Karaman S and Frazzoli E (2011) Sampling-based algorithms for optimal motion
  planning.
\newblock \emph{The International Journal of Robotics Research} 30(7):
  846--894.

\bibitem[{Kavraki et~al.(1996)Kavraki, Svestka, Latombe and
  Overmars}]{kavraki1996probabilistic}
Kavraki LE, Svestka P, Latombe JC and Overmars MH (1996) Probabilistic roadmaps
  for path planning in high-dimensional configuration spaces.
\newblock \emph{IEEE transactions on Robotics and Automation} 12(4): 566--580.

\bibitem[{Kim et~al.(2014)Kim, Bhattacharya and Kumar}]{kim2014path}
Kim S, Bhattacharya S and Kumar V (2014) Path planning for a tethered mobile
  robot.
\newblock In: \emph{2014 IEEE International Conference on Robotics and
  Automation (ICRA)}. IEEE, pp. 1132--1139.

\bibitem[{Kim and Likhachev(2015)}]{kim2015path}
Kim S and Likhachev M (2015) Path planning for a tethered robot using
  multi-heuristic a* with topology-based heuristics.
\newblock In: \emph{2015 IEEE/RSJ International Conference on Intelligent
  Robots and Systems (IROS)}. IEEE, pp. 4656--4663.

\bibitem[{LaValle and Kuffner~Jr(2001)}]{lavalle2001randomized}
LaValle SM and Kuffner~Jr JJ (2001) Randomized kinodynamic planning.
\newblock \emph{The International Journal of Robotics Research} 20(5):
  378--400.

\bibitem[{Liu et~al.(2023{\natexlab{a}})Liu, Fu, Liu, Zhang and
  Chen}]{liu2023homotopy}
Liu J, Fu M, Liu A, Zhang W and Chen B (2023{\natexlab{a}}) A homotopy
  invariant based on convex dissection topology and a distance optimal path
  planning algorithm.
\newblock \emph{IEEE Robotics and Automation Letters} .

\bibitem[{Liu et~al.(2023{\natexlab{b}})Liu, Fu, Liu, Zhang and
  Chen}]{liu2023homotopyarXiv}
Liu J, Fu M, Liu A, Zhang W and Chen B (2023{\natexlab{b}}) A homotopy
  invariant based on convex dissection topology and a distance optimal path
  planning algorithm.
\newblock \urlprefix\url{https://arxiv.org/abs/2302.13026}.

\bibitem[{Liu et~al.(2023{\natexlab{c}})Liu, Fu, Zhang, Chen, Prakapovich and
  Sychou}]{liu2023cdt}
Liu J, Fu M, Zhang W, Chen B, Prakapovich R and Sychou U (2023{\natexlab{c}})
  Cdt-dijkstra: Fast planning of globally optimal paths for all points in 2d
  continuous space.
\newblock In: \emph{2023 IEEE/RSJ International Conference on Intelligent
  Robots and Systems (IROS)}. IEEE, pp. 2224--2231.

\bibitem[{Liu et~al.(2025)Liu, Fu, Zhang, Chen, Sychou and
  Belotserkovsky}]{liu2025cdrt}
Liu J, Fu M, Zhang W, Chen B, Sychou U and Belotserkovsky A (2025) Cdrt-rrt*:
  Real-time rapidly exploring random tree star based on convex dissection.
\newblock \emph{Expert Systems with Applications} 268: 126291.

\bibitem[{Liu et~al.(2023{\natexlab{d}})Liu, Wang, Yang, Liu, Li and
  Wang}]{liu2023path}
Liu L, Wang X, Yang X, Liu H, Li J and Wang P (2023{\natexlab{d}}) Path
  planning techniques for mobile robots: Review and prospect.
\newblock \emph{Expert Systems with Applications} : 120254.

\bibitem[{Lozano-P{\'e}rez and Wesley(1979)}]{lozano1979algorithm}
Lozano-P{\'e}rez T and Wesley MA (1979) An algorithm for planning
  collision-free paths among polyhedral obstacles.
\newblock \emph{Communications of the ACM} 22(10): 560--570.

\bibitem[{McCammon and Hollinger(2017)}]{mccammon2017planning}
McCammon S and Hollinger GA (2017) Planning and executing optimal
  non-entangling paths for tethered underwater vehicles.
\newblock In: \emph{2017 IEEE International Conference on Robotics and
  Automation (ICRA)}. IEEE, pp. 3040--3046.

\bibitem[{McGarey et~al.(2017)McGarey, MacTavish, Pomerleau and
  Barfoot}]{mcgarey2017tslam}
McGarey P, MacTavish K, Pomerleau F and Barfoot TD (2017) Tslam: Tethered
  simultaneous localization and mapping for mobile robots.
\newblock \emph{The International Journal of Robotics Research} 36(12):
  1363--1386.

\bibitem[{McGarey et~al.(2018)McGarey, Yoon, Tang, Pomerleau and
  Barfoot}]{mcgarey2018developing}
McGarey P, Yoon D, Tang T, Pomerleau F and Barfoot TD (2018) Developing and
  deploying a tethered robot to map extremely steep terrain.
\newblock \emph{Journal of Field Robotics} 35(8): 1327--1341.

\bibitem[{Mechsy et~al.(2017)Mechsy, Dias, Pragithmukar and
  Kulasekera}]{mechsy2017novel}
Mechsy L, Dias M, Pragithmukar W and Kulasekera A (2017) A novel offline
  coverage path planning algorithm for a tethered robot.
\newblock In: \emph{2017 17th International Conference on Control, Automation
  and Systems (ICCAS)}. IEEE, pp. 218--223.

\bibitem[{Miao et~al.(2021)Miao, Chen, Yan and Wu}]{miao2021path}
Miao C, Chen G, Yan C and Wu Y (2021) Path planning optimization of indoor
  mobile robot based on adaptive ant colony algorithm.
\newblock \emph{Computers \& Industrial Engineering} 156: 107230.

\bibitem[{Munkres(2018)}]{munkres2018elements}
Munkres JR (2018) \emph{Elements of algebraic topology}.
\newblock CRC press.

\bibitem[{Pedrosa et~al.(2017)Pedrosa, Pereira and Lau}]{pedrosa2017efficient}
Pedrosa E, Pereira A and Lau N (2017) Efficient localization based on scan
  matching with a continuous likelihood field.
\newblock In: \emph{2017 IEEE International Conference on Autonomous Robot
  Systems and Competitions (ICARSC)}. IEEE, pp. 61--66.

\bibitem[{Pokorny et~al.(2016)Pokorny, Kragic, Kavraki and
  Goldberg}]{pokorny2016high}
Pokorny FT, Kragic D, Kavraki LE and Goldberg K (2016) High-dimensional
  winding-augmented motion planning with 2d topological task projections and
  persistent homology.
\newblock In: \emph{2016 IEEE International Conference on Robotics and
  Automation (ICRA)}. IEEE, pp. 24--31.

\bibitem[{Polzin and Hughes(2024)}]{polzin2024into}
Polzin M and Hughes J (2024) Into the ice: Exploration and data capturing in
  glacial moulins by a tethered robot.
\newblock \emph{Journal of Field Robotics} 41(3): 654--668.

\bibitem[{Quinlan and Khatib(1993)}]{quinlan1993elastic}
Quinlan S and Khatib O (1993) Elastic bands: Connecting path planning and
  control.
\newblock In: \emph{[1993] Proceedings IEEE International Conference on
  Robotics and Automation}. IEEE, pp. 802--807.

\bibitem[{Ratliff et~al.(2009)Ratliff, Zucker, Bagnell and
  Srinivasa}]{ratliff2009chomp}
Ratliff N, Zucker M, Bagnell JA and Srinivasa S (2009) Chomp: Gradient
  optimization techniques for efficient motion planning.
\newblock In: \emph{2009 IEEE international conference on robotics and
  automation}. IEEE, pp. 489--494.

\bibitem[{Sacerdoti(1974)}]{sacerdoti1974planning}
Sacerdoti ED (1974) Planning in a hierarchy of abstraction spaces.
\newblock \emph{Artificial intelligence} 5(2): 115--135.

\bibitem[{Sahin and Bhattacharya(2024)}]{sahin2024topo}
Sahin A and Bhattacharya S (2024) Topo-geometrically distinct path computation
  using neighborhood-augmented graph, and its application to path planning for
  a tethered robot in 3d.
\newblock \emph{IEEE Transactions on Robotics} .

\bibitem[{Shapovalov and Pereira(2020)}]{shapovalov2020exploration}
Shapovalov D and Pereira GA (2020) Exploration of unknown environments with a
  tethered mobile robot.
\newblock In: \emph{2020 IEEE/RSJ International Conference on Intelligent
  Robots and Systems (IROS)}. IEEE, pp. 6826--6831.

\bibitem[{Shi et~al.(2025)Shi, Jin, Yang, Lu and Li}]{shi2025exploring}
Shi D, Jin S, Yang C, Lu Z and Li Q (2025) Exploring the synergistic effects of
  teleoperation scaling ratio and learning from demonstration.
\newblock \emph{IEEE Transactions on Automation Science and Engineering} .

\bibitem[{Sinden(1990)}]{sinden1990tethered}
Sinden FW (1990) The tethered robot problem.
\newblock \emph{The international journal of robotics research} 9(1): 122--133.

\bibitem[{Strub and Gammell(2020)}]{strub2020advanced}
Strub MP and Gammell JD (2020) Advanced bit*(abit*): Sampling-based planning
  with advanced graph-search techniques.
\newblock In: \emph{2020 IEEE International Conference on Robotics and
  Automation (ICRA)}. IEEE, pp. 130--136.

\bibitem[{Strub and Gammell(2022)}]{strub2022adaptively}
Strub MP and Gammell JD (2022) Adaptively informed trees (ait*) and effort
  informed trees (eit*): Asymmetric bidirectional sampling-based path planning.
\newblock \emph{The International Journal of Robotics Research} 41(4):
  390--417.

\bibitem[{Sucan et~al.(2012)Sucan, Moll and Kavraki}]{sucan2012open}
Sucan IA, Moll M and Kavraki LE (2012) The open motion planning library.
\newblock \emph{IEEE Robotics \& Automation Magazine} 19(4): 72--82.

\bibitem[{Sun et~al.(2024)Sun, Sun and Shao}]{sun2024accelerated}
Sun D, Sun Z and Shao P (2024) Accelerated path planning for large-scale grid
  maps.
\newblock \emph{IEEE Access} .

\bibitem[{Sun et~al.(2022)Sun, Wang and Meng}]{sun2022multi}
Sun Z, Wang J and Meng MQH (2022) Multi-tree guided efficient robot motion
  planning.
\newblock \emph{Procedia Computer Science} 209: 31--39.

\bibitem[{Tan and Jin(2025)}]{tan2025manipulability}
Tan J and Jin L (2025) Manipulability optimization for redundant manipulators
  using singular value analysis: A convex approach.
\newblock \emph{IEEE/ASME Transactions on Mechatronics} .

\bibitem[{Tan et~al.(2024)Tan, Shang and Jin}]{tan2024metaheuristic}
Tan J, Shang M and Jin L (2024) Metaheuristic-based rnn for manipulability
  optimization of redundant manipulators.
\newblock \emph{IEEE Transactions on Industrial Informatics} 20(4): 6489--6498.

\bibitem[{Vashisth et~al.(2024)Vashisth, Ruckin, Magistri, Stachniss and
  Popovic}]{vashisth2024deep}
Vashisth A, Ruckin J, Magistri F, Stachniss C and Popovic M (2024) Deep
  reinforcement learning with dynamic graphs for adaptive informative path
  planning.
\newblock \emph{IEEE Robotics and Automation Letters} .

\bibitem[{V{\"o}r{\"o}s(2001)}]{voros2001low}
V{\"o}r{\"o}s J (2001) Low-cost implementation of distance maps for path
  planning using matrix quadtrees and octrees.
\newblock \emph{Robotics and Computer-Integrated Manufacturing} 17(6):
  447--459.

\bibitem[{Wan et~al.(2022)Wan, Zhong, Ma and Zhang}]{wan2022accurate}
Wan Y, Zhong Y, Ma A and Zhang L (2022) An accurate uav 3-d path planning
  method for disaster emergency response based on an improved multiobjective
  swarm intelligence algorithm.
\newblock \emph{IEEE Transactions on Cybernetics} 53(4): 2658--2671.

\bibitem[{Wang and Bhattacharya(2018)}]{wang2018topological}
Wang X and Bhattacharya S (2018) A topological approach to workspace and motion
  planning for a cable-controlled robot in cluttered environments.
\newblock \emph{IEEE Robotics and Automation Letters} 3(3): 2600--2607.

\bibitem[{Wen et~al.(2024)Wen, Zhang, Bi, Liu, Yuan and Fang}]{wen2024g}
Wen J, Zhang X, Bi Q, Liu H, Yuan J and Fang Y (2024) G$^2$ vd planner:
  Efficient motion planning with grid-based generalized voronoi diagrams.
\newblock \emph{IEEE Transactions on Automation Science and Engineering} .

\bibitem[{Yang et~al.(2025{\natexlab{a}})Yang, Lu and
  Wang}]{yang2025constrained}
Yang C, Lu Z and Wang N (2025{\natexlab{a}}) A constrained dmp framework for
  robot skills learning and generalization from human demonstrations.
\newblock In: \emph{Advanced Teleoperation and Robot Learning for Dexterous
  Manipulation}. Springer, pp. 127--151.

\bibitem[{Yang et~al.(2025{\natexlab{b}})Yang, Lu and Wang}]{yang2025motor}
Yang C, Lu Z and Wang N (2025{\natexlab{b}}) Motor learning and generalization
  using broad learning adaptive neural control.
\newblock In: \emph{Advanced Teleoperation and Robot Learning for Dexterous
  Manipulation}. Springer, pp. 75--95.

\bibitem[{Yang et~al.(2024{\natexlab{a}})Yang, Huang, Wang and
  Xiong}]{yang2024tree}
Yang T, Huang L, Wang Y and Xiong R (2024{\natexlab{a}}) Tree-based
  representation of locally shortest paths for 2d k-shortest non-homotopic path
  planning.
\newblock In: \emph{2024 IEEE International Conference on Robotics and
  Automation (ICRA)}. IEEE, pp. 16553--16559.

\bibitem[{Yang et~al.(2022)Yang, Xiong and Wang}]{yang2022efficient}
Yang T, Xiong R and Wang Y (2022) Efficient distance-optimal tethered path
  planning in planar environments: The workspace convexity.
\newblock \emph{arXiv preprint arXiv:2208.03969} .

\bibitem[{Yang et~al.(2024{\natexlab{b}})Yang, Meng, Meng and
  Yang}]{yang2024rampage}
Yang Y, Meng F, Meng Z and Yang C (2024{\natexlab{b}}) Rampage: toward
  whole-body, real-time, and agile motion planning in unknown cluttered
  environments for mobile manipulators.
\newblock \emph{IEEE Transactions on Industrial Electronics} .

\end{thebibliography}





\appendix
\section*{Appendix 1. Source Code}
To facilitate reproducibility and foster further research in the area of tethered robot path planning, we have open-sourced the C++ implementations of the proposed algorithm and the baseline methods used for comparison. The code is open-source and accessible at: \href{https://github.com/TZY-H/THPP.git}{https://github.com/TZY-H/THPP.git}.

\section*{Appendix 2. Multimedia Extensions}
A supplementary video demonstration of this work is publicly accessible online at: \href{https://youtu.be/d8Jhry0jZZM}{https://youtu.be/d8Jhry0jZZM}.

\section*{Appendix 3. Proofs of Theorems}
\subsection*{3.1 Proof of \textbf{Property~\ref{th_eqRelation}}}
\begin{proof}\mbox{}\newline
	\textit{Proof of \textbf{Reflexivity}.} For any $\sigma \in P(X)$, let $\varphi(t) = t$, where $t \in I$. Then, $\sigma(t) = \sigma(\varphi(t))$. Since $\dot{\varphi} = 1 \geq 0$, it follows that $\sigma = \sigma$. \newline
	\textit{Proof of \textbf{Symmetry}.} For any $\sigma_1, \sigma_2 \in P(X)$, if $\sigma_1 = \sigma_2$, then there exists a continuous mapping $\varphi$ such that $\dot{\varphi} \geq 0$ and $\sigma_1(t) = \sigma_2(\varphi(t))$. Let $\tau = \varphi(t)$, which implies $\sigma_2(\tau) = \sigma_1(\varphi^{-1}(\tau))$, where $\varphi^{-1}$ is the inverse mapping of $\varphi$. To show that $\dot{\varphi^{-1}} \geq 0$, we differentiate both sides of the equation $\varphi^{-1}(\tau) = t$ with respect to $t$:
	\begin{equation}
		\frac{d}{d \tau}\left[\varphi^{-1}(\tau)\right]\cdot \dot{\varphi}(t) = 1.
	\end{equation}
	Thus, we obtain
	\begin{equation}
		\dot{\varphi^{-1}}(\tau) = \frac{1}{\dot{\varphi}(t)} \geq 0.
	\end{equation}
	Therefore, $\sigma_2 = \sigma_1$. \newline
	\textit{Proof of \textbf{Transitivity}.} For any $\sigma_1, \sigma_2, \sigma_3 \in P(X)$, if $\sigma_1 = \sigma_2$ and $\sigma_2 = \sigma_3$, then there exist mappings $\varphi_1$ and $\varphi_2$ such that $\dot{\varphi_1} \geq 0$, $\dot{\varphi_2} \geq 0$, and  
	\begin{equation}
	\sigma_1(t) = \sigma_2(\varphi_1(t)), \quad \sigma_2(t) = \sigma_3(\varphi_2(t)).
	\end{equation}
	Consequently, we have  
	\begin{equation}
		\sigma_1(t) = \sigma_2(\varphi_1(t)) = \sigma_2(\varphi_2 \circ \varphi_1(t)),
	\end{equation}
	where $\varphi_2 \circ \varphi_1$ denotes the composition of $\varphi_2$ and $\varphi_1$. Let $\varphi = \varphi_2 \circ \varphi_1$. Then,  
    \begin{equation}
        \sigma_1(t) = \sigma_3(\varphi(t)),
    \end{equation}
	and  
    \begin{equation}
        \dot{\varphi}(t) = \dot{\varphi_2}\left(\varphi_1(t)\right) \cdot \dot{\varphi_1}(t) \geq 0.
    \end{equation}
	Thus, $\sigma_1 = \sigma_3$.
\end{proof}

\subsection*{3.2 Proof of \textbf{Lemma~\ref{NSCLOP}}}
\begin{proof}
	\ \newline \indent
	\textbf{Necessary:} The law of proof by contradiction is used herein. Suppose $\sigma^*$ is the shortest path in $[\sigma^*]_{\simeq_p}$, and $\exists t_1, t_2 \in I$ such that $\sigma^*_{[t_1,t_2]}$ is not the shortest path in $\left[\sigma^*_{[t_1,t_2]}\right]_{\simeq_p}$; that is, there exists $\sigma' \simeq_p \sigma^*_{[t_1,t_2]}$ such that $c(\sigma') < c\left(\sigma^*_{[t_1,t_2]}\right)$. The path $\sigma^*$ can be expressed as
	\begin{equation}
		\label{eq_NSCLOP1}
		\sigma^* = \sigma^*_{[0,t_1]}*\sigma^*_{[t_1,t_2]}*\sigma^*_{[t_2,1]}
	\end{equation}
	Hence,
	\begin{align}
		\label{eq_NSCLOP2}
		c\left(\sigma^*\right) & = c\left(\sigma^*_{[0,t_1]}\right) + c\left(\sigma^*_{[t_1,t_2]}\right) + c\left(\sigma^*_{[t_2,1]}\right) \nonumber \\
		                       & < c\left(\sigma^*_{[0,t_1]}\right) + c\left(\sigma'\right) + c\left(\sigma^*_{[t_2,1]}\right) \nonumber              \\
		                       & < c\left(\sigma^*_{[0,t_1]} * \sigma' * \sigma^*_{[t_2,1]}\right).
	\end{align}
	Therefore, there exists a path $\sigma^*_{[0,t_1]} * \sigma' * \sigma^*_{[t_2,1]}$ in $[\sigma^*]_{\simeq_p}$ that is shorter than $\sigma^*$. This contradicts our original hypothesis. Therefore, the sufficiency condition is true.

	\textbf{Sufficient:} When the latter of the proposition holds, let $t_1=0$ and $t_2=1$. At this moment $\sigma^*_{[t_1,t_2]}(t)=\sigma^*(t)$. Therefore, $\sigma^*$ is the shortest path in $[\sigma^*]_{\simeq_p}$. Therefore, the necessary condition is true.
\end{proof}

\subsection*{3.3 Proof of \textbf{Theorem~\ref{SSDR}}}
\begin{proof}
	For the first assertion, since the length of $\Gamma^*_\varrho \circ \sigma^*$ is $1$, $x_s$ and $x_e$ belong to the same convex polygon. Based on the properties of convex polygons, the straight line $l^{x_e}_{x_s}$ also lies within the same convex polygon. Therefore, we have:
	\begin{equation}
		\label{eq_SSDR31}
		\Gamma^* \circ \sigma^* = \Gamma^* \circ l^{x_e}_{x_s} \Longleftrightarrow \sigma^* \simeq_p l^{x_e}_{x_s}.
	\end{equation}
	Furthermore, since $l^{x_e}_{x_s}$ is the shortest path between $x_s$ and $x_e$, the assertion is valid.

	For the second assertion, $\sigma^*$ can be expressed as follows:
	\begin{equation}
		\label{eq_SSDR3}
		\sigma^* = \sigma_1 * \sigma_2 * \dots * \sigma_n ,
	\end{equation}
	where $\sigma_k \in P(X_{free},x_{k-1},x_k)$. According to the description of the assertion, $x_{k-1}$ and $x_k$ are points on adjacent cutlines, meaning they belong to the same convex polygon. Moreover, based on \textbf{Lemma~\ref{NSCLOP}}, the optimal path $\sigma^*$ must guarantee that each of its segments is optimal; hence, for any $\sigma_k$, $\sigma_k = l^{x_k}_{x_{k-1}}$, the assertion is valid.
\end{proof}

\subsection*{3.4 Proof of \textbf{Lemma~\ref{th_phomotopy}}}
\begin{proof}
	According to the definition of $p^*_{\sigma, l}$ in (\ref{eq_pfun}), it follows that
	\begin{equation}
		p^*_{\sigma, l}(t_1) \simeq_p \sigma * l_{[0,t_1]},
	\end{equation}
	\begin{equation}
		p^*_{\sigma, l}(t_2) \simeq_p \sigma * l_{[0,t_2]}.
	\end{equation}
	Furthermore, based on \textbf{Definition~\ref{defLine}} of the straight-line path, it is straightforward to show that:
	\begin{equation}
		l_{[0,t_2]} \simeq_p l_{[0,t_1]} * l_{[t_1,t_2]}.
	\end{equation}
	Thus, we have:
	\begin{align}
		p^*_{\sigma, l}(t_2) & \simeq_p \sigma * l_{[0,t_2]} \nonumber                 \\
		                     & \simeq_p \sigma * l_{[0,t_1]} * l_{[t_1,t_2]} \nonumber \\
		                     & \simeq_p p^*_{\sigma, l}(t_1)*l_{[t_1,t_2]}.
	\end{align}
\end{proof}

\subsection*{3.5 Proof of \textbf{Theorem~\ref{th_finequality}}}
\begin{proof}
	As illustrated in Fig.~\ref{fig_FPfun}(b), for the right-hand side of inequality (\ref{eq_finequality0P}), $f(t+dt)+c^l_{dt}$ can be interpreted as the cost of the path $p^*_{\sigma,l}(t+dt)$ combined with the segment $l_{[t+dt,t]}$. Thus, we have:
	\begin{align}
		\label{eq_finequality1}
		f(t) & = c\left(p^*_{\sigma,l}(t)\right) \nonumber                                                                                                            \\
		     & \leq c\left(p^*_{\sigma,l}(t+dt)*l_{[t+dt,t]}\right) \tag{\underline{According to (\ref{eq_pfun}) and \textbf{Lemma~\ref{th_phomotopy}}.}} \nonumber \\
		     & = \left(c \circ p^*_{\sigma,l}\right)(t+dt) + c(l_{[t+dt,t]}) \nonumber                                                                                \\
		     & = f(t+dt) + c^l_{dt}.
	\end{align}
	Similarly, for the left-hand side of inequality (\ref{eq_finequality0P}), we can use a similar approach to show:
	\begin{align}
		\label{eq_finequality2}
		 & \begin{aligned}
			   f(t+dt) & = c\left(p^*_{\sigma,l}(t+dt)\right)                       \\
			           & \leq c\left(p^*_{\sigma,l}(t)*l_{[t,t+dt]}\right)          \\
			           & = \left(c \circ p^*_{\sigma,l}\right)(t) + c(l_{[t,t+dt]}) \\
			           & = f(t) + c^l_{dt}
		   \end{aligned} \nonumber \\
		 & \Longleftrightarrow  f(t+dt) - c^l_{dt} \leq f(t).
	\end{align}
	In summary, the inequality (\ref{eq_finequality0P}) is proved. To address inequality (\ref{eq_finequality0M}), let us substitute $-dt$ for $dt$ in (\ref{eq_finequality0P}), yielding:
	\begin{equation}
		\label{eq_finequality0PM}
		f(t-dt) - c^l_{-dt} \leq f(t) \leq f(t-dt) + c^l_{-dt}.
	\end{equation}
	Using equation (\ref{eq_finequalityC}), it is straightforward to show that $c^l_{dt} = c^l_{-dt}$. Thus, we can rewrite the inequality (\ref{eq_finequality0PM}) as:
	\begin{equation}
		f(t-dt) - c^l_{dt} \leq f(t) \leq f(t-dt) + c^l_{dt}.
	\end{equation}
	This completes the proof.
\end{proof}

\subsection*{3.6 Proof of \textbf{Lemma~\ref{th_fcontinuity}}}
\begin{proof}
	To prove the proposition, we use the $\varepsilon - \delta$ definition of continuity. Specifically, we need to show that for every $\varepsilon >0$, $t_c \in I$, there exists a $\delta > 0$ such that if $\left\lvert t-t_c\right\rvert < \delta$ , then $\left\lvert f(t) - f(t_c)\right\rvert < \varepsilon$. From \textbf{Theorem~\ref{th_finequality}}, we have:
	\begin{equation}
		\left\lvert f(t) - f(t_c)\right\rvert \leq c^l_{t-t_c} = \left\lvert t-t_c\right\rvert \cdot c(l).
	\end{equation}
	Thus, if we choose $\delta=\frac{\varepsilon }{c(l)}$, and if $\left\lvert t-t_c\right\rvert < \delta$, we obtain:
	\begin{equation}
		\left\lvert f(t) - f(t_c)\right\rvert \leq \left\lvert t-t_c\right\rvert \cdot c(l) < \frac{\varepsilon }{c(l)} \cdot c(l) = \varepsilon.
	\end{equation}
	Therefore, $f$ satisfies the $\varepsilon - \delta$ definition of continuity, and the theorem is proven.
\end{proof}

\subsection*{3.7 Proof of \textbf{Theorem~\ref{th_fconvex}}}
\begin{figure}[!t]
	\centering
	\subfloat[]{\includegraphics[height=1.65in]{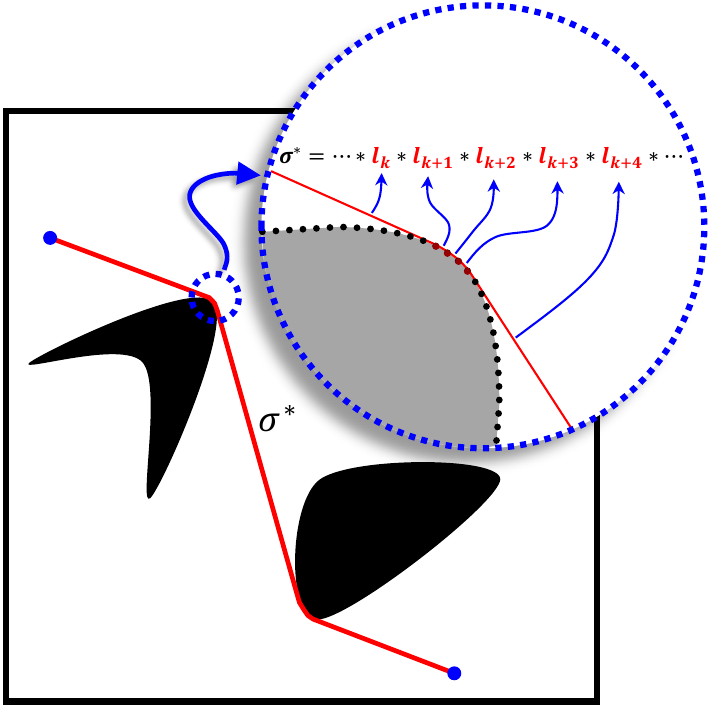}}
	\hfil
	\subfloat[]{\includegraphics[height=1.65in]{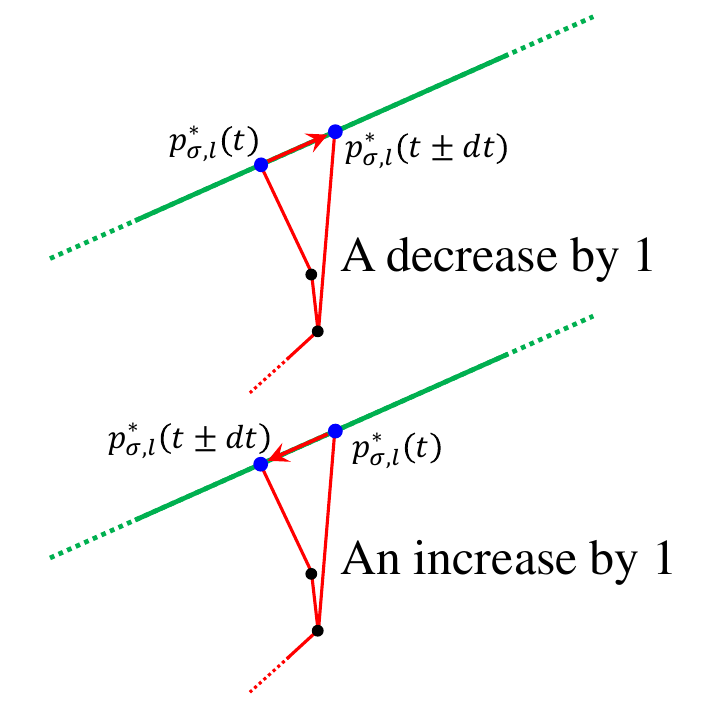}}
	\caption{(a) Any optimal homotopic path in the space can be viewed as a connection of an infinite (or finite) number of line segments, where each segment bends at obstacle points caused by obstacles. (b) Illustration of the change in the number of bends when $dt$ is sufficiently small: the path $p^*_{\sigma, l}(t \pm dt)$ increases (or decreases) the bend count by 1 compared to $p^*_{\sigma, l}(t)$.}
	\label{fig_fconvex1}
\end{figure}
\begin{proof}
	Building on \textbf{Lemma~\ref{th_fcontinuity}}, to prove that $f$ is strictly convex, it suffices to show that its second derivative, $\ddot{f}$, is strictly greater than zero (with the possibility of being $+\infty$). To this end, the inequality $\ddot{f} \geq 0$ can be expressed in terms of the limit form of the second derivative as follows:
	\begin{equation}
		\label{eq_fconvexDDot}
		\ddot{f}(t) = \lim_{dt\to 0} \frac{f(t + dt) - 2f(t) + f(t - dt)}{{dt}^2} \geq 0.
	\end{equation}
	Since ${dt}^2 > 0$, it is sufficient to prove that:
	\begin{align}
		\label{eq_fconvexDDot1}
		                    & \lim_{dt\to 0} f(t + dt) + f(t - dt) \geq 2f(t), \nonumber \\
		\Longleftrightarrow & \lim_{dt\to 0} \scalebox{0.95}{$c\left(p^*_{\sigma, l}(t + dt)\right)  + c\left(p^*_{\sigma, l}(t - dt)\right) \geq 2c\left(p^*_{\sigma, l}(t)\right)$}.
	\end{align}
	Referring to Fig.~\ref{fig_fconvex1}, before proving inequality (\ref{eq_fconvexDDot1}), we first highlight the following two key concepts related to limits:
	\begin{enumerate}
		\item[1)] {According to \textbf{Lemma~\ref{NSCLOP}}, any optimal homotopic path in space can be viewed as the concatenation of an infinite (or finite) number of straight line segments, each bent at obstruction points due to obstacles. Specifically, the paths $p^*_{\sigma, l}(t)$, $p^*_{\sigma, l}(t+dt)$, and $p^*_{\sigma, l}(t-dt)$ can be expressed as:
		      \begin{align}
			       & p^*_{\sigma, l}(t) = l_1 * l_2 * \dots * l_{-2} * l_{-1}, \nonumber                    \\
			       & p^*_{\sigma, l}(t+dt) = l^{+}_1 * l^{+}_2 * \dots * l^{+}_{-2} * l^{+}_{-1}, \nonumber \\
			       & p^*_{\sigma, l}(t-dt) = l^{-}_1 * l^{-}_2 * \dots * l^{-}_{-2} * l^{-}_{-1}.
		      \end{align}
		      where $l_{-k}$, $l^+_{-k}$ and $l^-_{-k}$ represent the last $k$-th line segment in each of the three paths, respectively.
		      }
		\item[2)] {The value of $dt$ is infinitesimally small, such that the path $p^*_{\sigma, l}(t\pm dt)$ differs from $p^*_{\sigma, l}(t)$ in terms of the bending induced by obstruction points, with three possible variations: an increase by 1, no change, or a decrease by 1.}
	\end{enumerate}
\begin{figure*}[!t]
	\centering
	\subfloat[]{\includegraphics[width=2.25in]{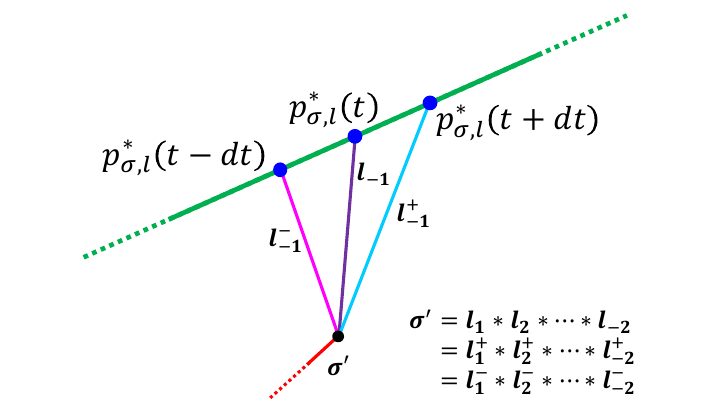}}
	\hfil
	\subfloat[]{\includegraphics[width=2.25in]{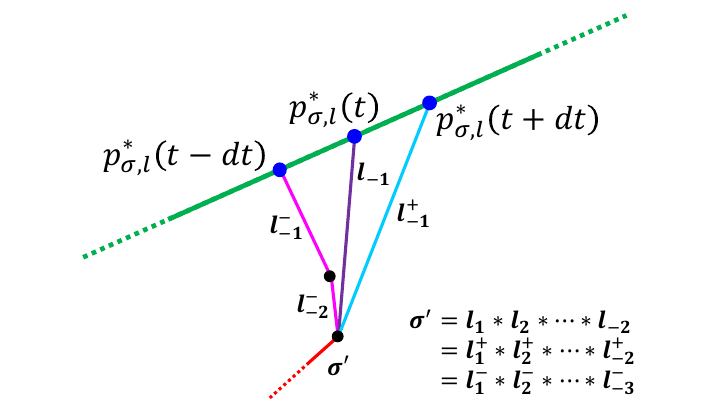}}
	\hfil
	\subfloat[]{\includegraphics[width=2.25in]{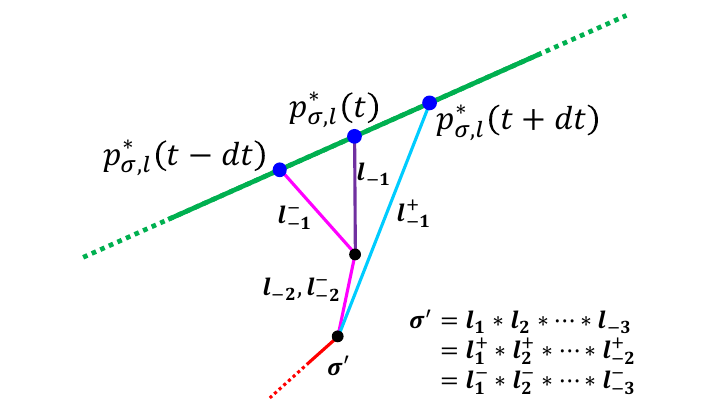}}
	\\
	\subfloat[]{\includegraphics[width=2.25in]{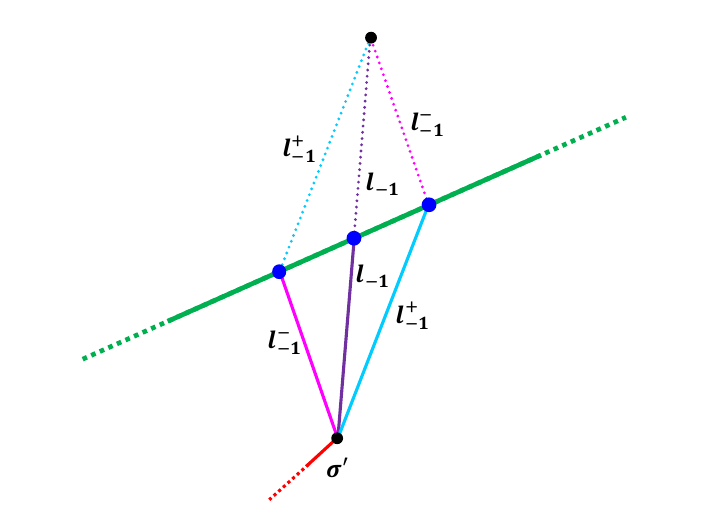}}
	\hfil
	\subfloat[]{\includegraphics[width=2.25in]{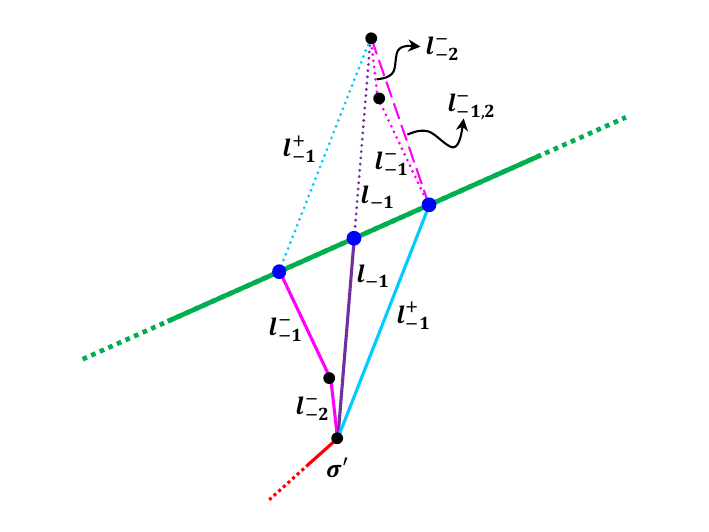}}
	\hfil
	\subfloat[]{\includegraphics[width=2.25in]{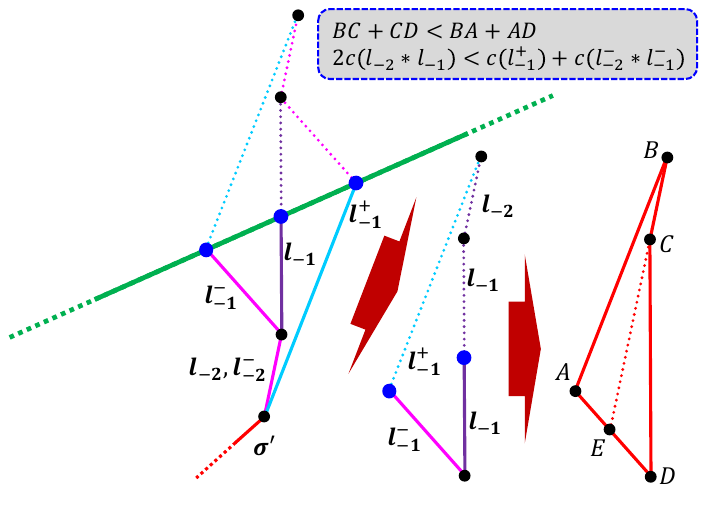}}
	\caption{(a)-(c) correspond to the three possible cases of changes in the number of bends for $\left(p^*_{\sigma, l}(t+dt), p^*_{\sigma, l}(t-dt)\right)$ compared to $p^*_{\sigma, l}(t)$, namely, (0, 0), (+1, 0), and (-1, 0). (d)-(f) show the auxiliary lines required to prove the three cases in \textbf{Theorem~\ref{th_fconvex}}.}
	\label{fig_BendProof}
\end{figure*}

	To illustrate, as shown in \textbf{Table~\ref{tab:Bending}}, the bending variation of $\left(p^*_{\sigma, l}(t+dt), p^*_{\sigma, l}(t-dt)\right)$ compared to $p^*_{\sigma, l}(t)$ will exhibit five possible cases. However, since $dt$ can be either positive or negative, four of these cases will be pairwise equivalent. Consequently, the number of distinct possibilities reduces to three, i.e. $(0,0)$, $(+1,0)$ and $(-1,0)$, as illustrated in Fig.~\ref{fig_BendProof}(a)-(c).
	\begin{table}
		\begin{center}
			\caption{Possible Variations in the Bending of optimal homotopic Paths due to Obstruction Points for $p^*_{\sigma, l}(t\pm dt)$\label{tab:Bending}}
			\centering
			\begin{tabular}{|c|c|c|c|c|c|}
				\hline
				\multicolumn{6}{|c|}{\makecell*[c]{ Variations in the Bending Number of  \\ $p^*_{\sigma, l}(t\pm dt)$ Relative to $p^*_{\sigma, l}(t)$}}\\
				\hline
				\makecell*[c]{$p^*_{\sigma, l}(t+dt)$} & $0$ & $+1$ & $0$  & $-1$ & $0$  \\
				\hline
				\makecell*[c]{$p^*_{\sigma, l}(t-dt)$} & $0$ & $0$  & $+1$ & $0$  & $-1$ \\
				\hline
			\end{tabular}
		\end{center}
	\end{table}

	Referring to Fig.~\ref{fig_BendProof}(d), when the number of bends changes to $(0, 0)$, the proof of Inequality (\ref{eq_fconvexDDot1}) is as follows:
	\begin{align}
		\label{eq_BendProof1}
		\lim_{dt\to 0} c(l^+_{-1}) + c(l^-_{-1})            & \geq 2c(l_{-1}) \tag{\underline{Triangle Inequality.}} \nonumber \\
		c(l^+_{-1}) + c(l^-_{-1}) + 2c(\sigma')             & \geq 2(c(l_{-1})+c(\sigma')) \nonumber                           \\
		c(\sigma' * l^+_{-1}) + c(\sigma' * l^-_{-1})       & \geq 2c(\sigma' * l_{-1}) \nonumber                              \\
		c(p^*_{\sigma, l}(t+dt)) + c(p^*_{\sigma, l}(t-dt)) & \geq 2c(p^*_{\sigma, l}(t)).
	\end{align}
	Referring to Fig.~\ref{fig_BendProof}(e), when the number of bends changes to $(+1, 0)$, the proof of Inequality (\ref{eq_fconvexDDot1}) is as follows:
	\begin{align}
		\label{eq_BendProof2}
		 & \lim_{dt\to 0} \scalebox{0.95}{$c(l^+_{-1}) + c(l^-_{-2}) + c(l^-_{-1}) > c(l^+_{-1}) + c(l^-_{-1,2}) \geq 2c(l_{-1})$} \nonumber      \\
		 & \Longleftrightarrow \lim_{dt\to 0} c(p^*_{\sigma, l}(t+dt)) + c(p^*_{\sigma, l}(t-dt)) > 2c(p^*_{\sigma, l}(t)).
	\end{align}
	Referring to Fig.~\ref{fig_BendProof}(f), when the number of bends changes to $(-1, 0)$, the proof of Inequality (\ref{eq_fconvexDDot1}) is as follows:
	\begin{align}
		\label{eq_BendProof3_1}
		\lVert\overrightarrow{BC}\rVert + \lVert\overrightarrow{CD}\rVert & < \lVert\overrightarrow{BC}\rVert + \lVert\overrightarrow{CE}\rVert + \lVert\overrightarrow{ED}\rVert \nonumber \\
		                                                                  & = \lVert\overrightarrow{BE}\rVert + \lVert\overrightarrow{ED}\rVert \nonumber                                   \\
		                                                                  & < \lVert\overrightarrow{BA}\rVert + \lVert\overrightarrow{AE}\rVert + \lVert\overrightarrow{ED}\rVert \nonumber \\
		                                                                  & = \lVert\overrightarrow{BA}\rVert + \lVert\overrightarrow{AD}\rVert,
	\end{align}
	where $\overrightarrow{CE}$ is the extension of $\overrightarrow{BC}$. Based on the equation (\ref{eq_BendProof3_1}), further we have:
	\begin{align}
		\label{eq_BendProof3_2}
		\lim_{dt\to 0} c(l^+_{-1}) + c(l^-_{-1})            & > 2c(l_{-1}) + c(l_{-2}) \nonumber    \\
		c(l^+_{-1}) + c(l^-_{-1}) + c(l^-_{-2})             & > 2c(l_{-1}) + 2c(l_{-2}) \nonumber   \\
		c(l^+_{-1}) + c(l^-_{-2}*l^-_{-1})                  & > 2c(l_{-2}*l_{-1}) \nonumber         \\
		c(\sigma'*l^+_{-1}) + c(\sigma'*l^-_{-2}*l^-_{-1})  & > 2c(\sigma'*l_{-2}*l_{-1}) \nonumber \\
		c(p^*_{\sigma, l}(t+dt)) + c(p^*_{\sigma, l}(t-dt)) & > 2c(p^*_{\sigma, l}(t)).
	\end{align}

	In conclusion, the inequality (\ref{eq_fconvexDDot1}) is proved, that is, the theorem is established.
\end{proof}

\subsection*{3.8 Proof of \textbf{Theorem~\ref{th_genfconvex}}}

\begin{figure*}[!t]
	\centering
	\subfloat[]{\includegraphics[width=2.25in]{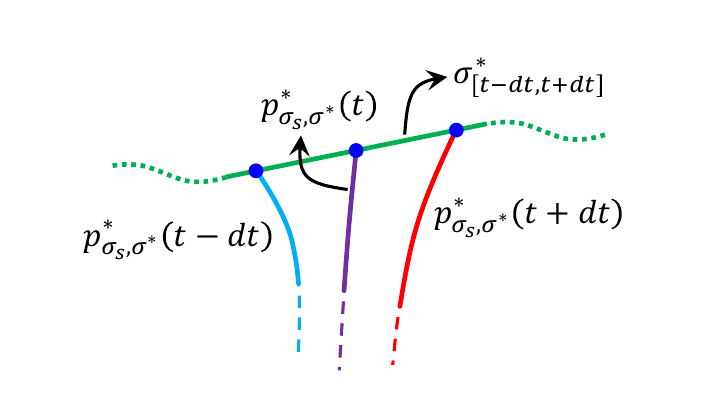}}
	\hfil
	\subfloat[]{\includegraphics[width=2.25in]{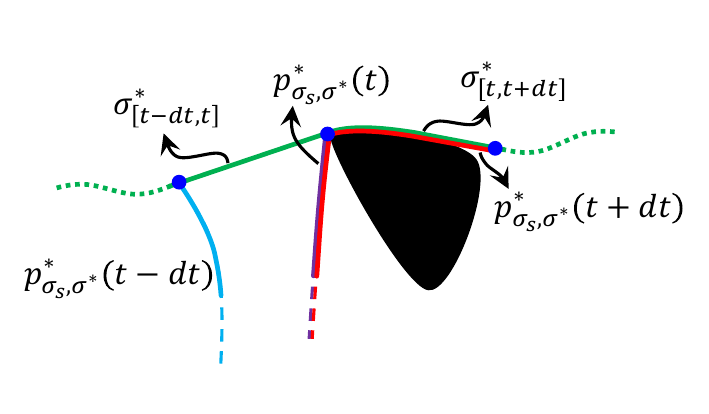}}
	\hfil
	\subfloat[]{\includegraphics[width=2.25in]{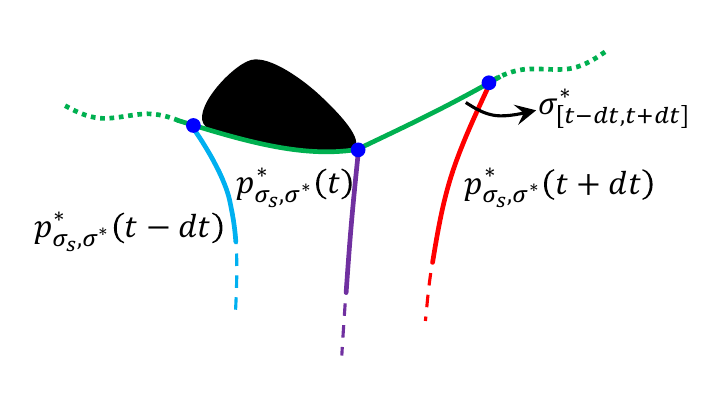}}
	\caption{Illustration of the three possible cases for the subpath $\sigma^*_{[t - dt, t + dt]}$ in inequality (\ref{eq_genfconvexDDot1}) as $dt$ approaches 0. (a) $\sigma^*_{[t - dt, t + dt]}$ is a straight line segment. (b) $\sigma^*_{[t - dt, t + dt]}$ is bent due to obstacles, and $p^*_{\sigma_s, \sigma^*}(t)$ is a subpath of $p^*_{\sigma_s, \sigma^*}(t - dt)$ or/and $p^*_{\sigma_s, \sigma^*}(t + dt)$. (c) $\sigma^*_{[t - dt, t + dt]}$ is bent due to obstacles, and $p^*_{\sigma_s, \sigma^*}(t)$ is not a subpath of $p^*_{\sigma_s, \sigma^*}(t - dt)$ or $p^*_{\sigma_s, \sigma^*}(t + dt)$.}
	\label{fig_genfpr}
\end{figure*}
\begin{proof}
	The proof of the continuity of the function $g$ is similar to that of $f$ and will not be elaborated here. Therefore, similar to $f$, the key to proving that $g$ is strictly convex lies in demonstrating the following inequality for any $t$:
	\begin{align}
		\label{eq_genfconvexDDot1}
		& \lim_{dt\to 0} g(t + dt) + g(t - dt) \geq 2g(t) \nonumber\\
		\Longleftrightarrow & \lim_{dt\to 0} c\left(p^*_{\sigma_s, \sigma^*}(t + dt)\right)  + c\left(p^*_{\sigma_s, \sigma^*}(t - dt)\right) \nonumber\\
		&\geq 2c\left(p^*_{\sigma_s, \sigma^*}(t)\right).
	\end{align}
	As shown in the Fig.~\ref{fig_genfpr}, the proof of inequality (\ref{eq_genfconvexDDot1}) can be divided into three cases for consideration:
	\begin{enumerate}
		\item[1)] {The subpath $\sigma^*_{[t - dt, t + dt]}$ is a straight line segment;}
		\item[2)] {The subpath $\sigma^*_{[t - dt, t + dt]}$ is bent due to obstacles, and $p^*_{\sigma_s, \sigma^*}(t)$ is a subpath of $p^*_{\sigma_s, \sigma^*}(t - dt)$ or/and $p^*_{\sigma_s, \sigma^*}(t + dt)$;}
		\item[3)] {The subpath $\sigma^*_{[t - dt, t + dt]}$ is bent due to obstacles, and $p^*_{\sigma_s, \sigma^*}(t)$ is not a subpath of $p^*_{\sigma_s, \sigma^*}(t - dt)$ or $p^*_{\sigma_s, \sigma^*}(t + dt)$.}
	\end{enumerate}

	For the case where $\sigma^*_{[t - dt, t + dt]}$ is a straight line segment, the function $g$ degenerates into the function $f$ on this segment. According to \textbf{Theorem~\ref{th_fconvex}}, inequality (\ref{eq_genfconvexDDot1}) holds in this case, and $g$ is a convex function.

	As shown in the Fig.~\ref{fig_genfpr}(b), for the second case of $\sigma^*_{[t - dt, t + dt]}$. Taking $p^*_{\sigma_s, \sigma^*}(t)$ as a subpath of $p^*_{\sigma_s, \sigma^*}(t + dt)$ as an example, we have:
	\begin{align}
		p^*_{\sigma_s, \sigma^*}(t + dt) = p^*_{\sigma_s, \sigma^*}(t) * \sigma^*_{[t,t+dt]}.
	\end{align}
	Let $c^{\sigma^*}_{dt}$ denote the length of $\sigma^*_{[t, t + dt]}$, i.e.,
	\begin{align}
		c^{\sigma^*}_{dt} = c\left(\sigma^*_{[t,t+dt]}\right) = \left\lvert dt\right\rvert \cdot c\left(\sigma^*\right).
	\end{align}
	Thus, we have:
	\begin{align}
		\label{eq_genf_S1}
		c\left(p^*_{\sigma_s, \sigma^*}(t + dt)\right) = c\left(p^*_{\sigma_s, \sigma^*}(t)\right) + c^{\sigma^*}_{dt}.
	\end{align}
	From (\ref{eq_genphomo}), it follows that:
	\begin{align}
		p^*_{\sigma_s, \sigma^*}(t) \simeq_p p^*_{\sigma_s, \sigma^*}(t-dt) * \sigma^*_{[t-dt,t]}.
	\end{align}
	Furthermore, based on the definition of $p^*_{\sigma_s, \sigma^*}$, $p^*_{\sigma_s, \sigma^*}(t)$ is the optimal homotopic path within its homotopy class. Therefore:
	\begin{align}
		\label{eq_genf_S2}
		c\left(p^*_{\sigma_s, \sigma^*}(t)\right) \leq c\left(p^*_{\sigma_s, \sigma^*}(t-dt)\right) + c^{\sigma^*}_{dt}.
	\end{align}
	Combining (\ref{eq_genf_S1}) and (\ref{eq_genf_S2}), we obtain:
	\begin{align}
		c\left(p^*_{\sigma_s, \sigma^*}(t+dt)\right) +c\left(p^*_{\sigma_s, \sigma^*}(t-dt)\right) \geq 2c\left(p^*_{\sigma_s, \sigma^*}(t)\right).
	\end{align}
	In this way, we have proven that inequality (\ref{eq_genfconvexDDot1}) holds when $\sigma^*_{[t - dt, t + dt]}$ is the second case.

	For the third case of $\sigma^*_{[t - dt, t + dt]}$, we draw an analogy to \textbf{Lemma~\ref{th_phomotopy}} and refer to its proof process. It is not difficult to arrive at the following conclusion:\\
	\textit{For any $t_1, t_2 \in I$,}
	\begin{equation}
		\label{eq_genphomo}
		p^*_{\sigma_s, \sigma^*}(t_2) \simeq_p p^*_{\sigma_s, \sigma^*}(t_1)*\sigma^*_{[t_1,t_2]}.
	\end{equation}
	\begin{figure}[!t]
		\centering
		\subfloat[]{\includegraphics[width=1.65in]{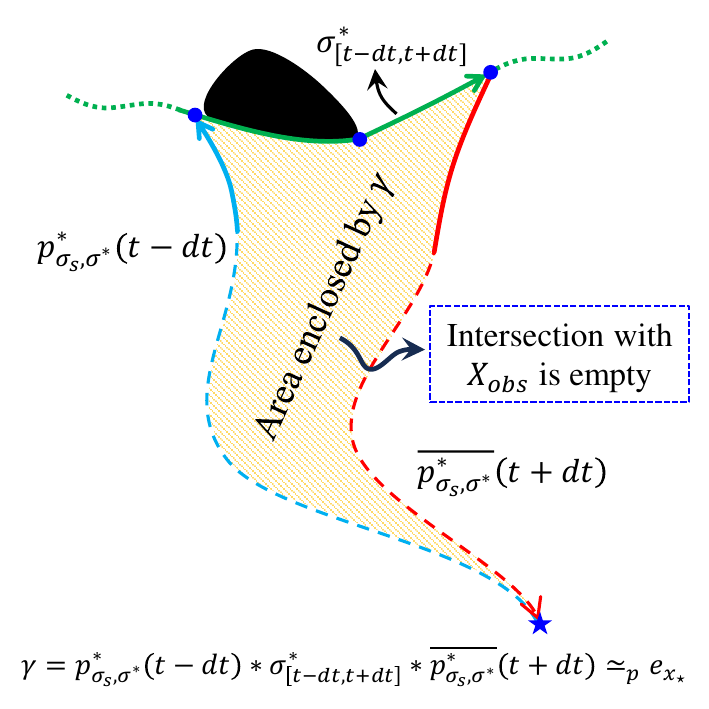}}
		\hfil
		\subfloat[]{\includegraphics[width=1.65in]{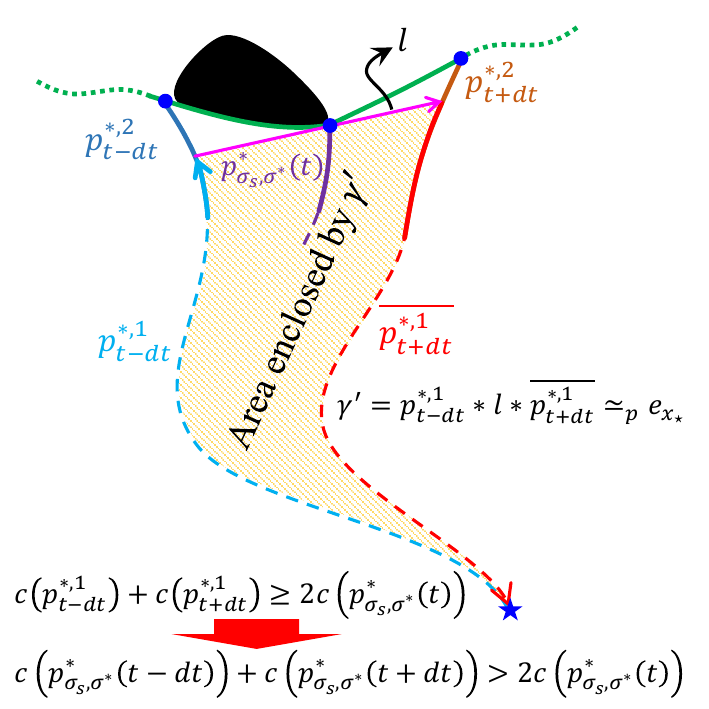}}
		\caption{(a) Illustration of the null-homotopic loop $\gamma$ formed by $p^*_{\sigma_s, \sigma^*}(t - dt)$, $\sigma^*_{[t - dt, t + dt]}$, and $\overline{p^*_{\sigma_s, \sigma^*}}(t + dt)$. (b) Illustration of the line segment $l$ and the null-homotopic loop $\gamma'$ constructed based on $l$ and $\gamma$.}
		\label{fig_genfpr3}
	\end{figure}
	As shown in the Fig.~\ref{fig_genfpr3}(a), we can construct a special loop $\gamma$, defined as  
	\begin{equation}
		\gamma = p^*_{\sigma_s, \sigma^*}(t-dt) * \sigma^*_{[t-dt,t+dt]} * \overline{p^*_{\sigma_s, \sigma^*}}(t+dt).
	\end{equation}
	From (\ref{eq_genphomo}), it is evident that $\gamma$ satisfies:  
	\begin{equation}
		\gamma \simeq_p e_{x_\star}.
	\end{equation}
	That is, $\gamma$ is a null-homotopic loop (a contractible loop). Furthermore, according to the winding number theory, any point enclosed by $\gamma$ (with a non-zero winding number) belongs to the free space $X_{free}$.

	As shown in the Fig.~\ref{fig_genfpr3}(b), in this case, there exists a line segment $l$ enclosed by $\gamma$, which satisfies $\sigma^*(t) \in l$, $l(0) \in p^*_{\sigma_s, \sigma^*}(t - dt)$, and $l(1) \in p^*_{\sigma_s, \sigma^*}(t + dt)$. Based on this, $p^*_{\sigma_s, \sigma^*}(t - dt)$ and $p^*_{\sigma_s, \sigma^*}(t + dt)$ can be expressed in the following sub-segment forms:
	\begin{align}
		p^*_{\sigma_s, \sigma^*}(t-dt) = p^{*,1}_{t-dt} * p^{*,2}_{t-dt},\\
		p^*_{\sigma_s, \sigma^*}(t+dt) = p^{*,1}_{t+dt} * p^{*,2}_{t+dt},
	\end{align}
	where $p^{*,1}_{t - dt}(1) = p^{*,2}_{t - dt}(0) = l(0)$ and $p^{*,1}_{t + dt}(1) = p^{*,2}_{t + dt}(0) = l(1)$. On this basis, we can construct a new loop $\gamma'$ as follows:
	\begin{align}
		\gamma' = p^{*,1}_{t-dt} * l * \overline{p^{*,1}_{t+dt}}.
	\end{align}
	Since all points enclosed by $\gamma$ belong to $X_{free}$, the points enclosed by the loop $\gamma'$ also belong to $X_{free}$. This implies that $\gamma'$ is also a null-homotopic loop. Therefore, we have:
	\begin{align}
		\gamma' = p^{*,1}_{t-dt} * l * \overline{p^{*,1}_{t+dt}} \simeq_p e_{x_\star}
		\Longleftrightarrow p^{*,1}_{t+dt} \simeq_p p^{*,1}_{t-dt} * l.
	\end{align}
	Thus, according to \textbf{Theorem~\ref{th_fconvex}}, we can derive the following relationship:
	\begin{align}
		c\left(p^{*,1}_{t-dt}\right) + c\left(p^{*,1}_{t+dt}\right) \geq 2c\left(p^*_{\sigma_s, \sigma^*}(t)\right),
	\end{align}
	and since $p^{*,1}_{t - dt}$ and $p^{*,1}_{t + dt}$ are subpaths of $p^*_{\sigma_s, \sigma^*}(t - dt)$ and $p^*_{\sigma_s, \sigma^*}(t + dt)$, respectively, it follows that:
	\begin{align}
		&c\left(p^*_{\sigma_s, \sigma^*}(t + dt)\right)  + c\left(p^*_{\sigma_s, \sigma^*}(t - dt)\right)\nonumber\\
		&\qquad\qquad > c\left(p^{*,1}_{t-dt}\right) + c\left(p^{*,1}_{t+dt}\right) \geq 2c\left(p^*_{\sigma_s, \sigma^*}(t)\right).
	\end{align}
	In summary, we have proven that inequality (\ref{eq_genfconvexDDot1}) holds for all cases. Thus, the proposition is established.
\end{proof}

\subsection*{3.9 Proof of \textbf{Theorem~\ref{th_TMVPath}}}
\begin{proof}
	To prove this proposition, we need to show that $\gamma$ satisfies all three conditions of the solution to the TMV problem, namely (\ref{eqProbTMV1}), (\ref{eqProbTMV2}), and (\ref{eqProbTMV3}).

	First, the terms $\Theta\left(\overline{\varsigma^*_s} * \varsigma^*_1\right)$, $\Theta\left(\overline{\varsigma^*_1} * \varsigma^*_2\right)$, $\dots$, $\Theta\left(\overline{\varsigma^*_N} * \varsigma^*_s\right)$ in (\ref{eq_TMVPath}) correspond one-to-one with $\sigma_{s,1}$, $\sigma_{1,2}$, $\sigma_{2,3}$, $\dots$, $\sigma_{N,s}$ in (\ref{eqProbTMV1}). Based on the definitions of the configurations $\varsigma^*_1, \varsigma^*_2, \dots, \varsigma^*_N$ and the operation $\Theta$, it is clear that $\gamma$ satisfies condition (\ref{eqProbTMV1}). Furthermore, according to \textbf{Corollary~\ref{Cor_genf}}, $\gamma$ also satisfies condition (\ref{eqProbTMV2}). From (\ref{eq_TMVPath}), we have the following relationship:
	\begin{align}
		\label{eq_TMVPath1}
		\varsigma^*_s * \gamma & = \varsigma^*_s * \Theta\left(\overline{\varsigma^*_s} * \varsigma^*_1\right) * \Theta\left(\overline{\varsigma^*_1} * \varsigma^*_2\right) * \dots * \Theta\left(\overline{\varsigma^*_N} * \varsigma^*_s\right) \nonumber \\
		                       & \simeq_p \varsigma^*_s * \overline{\varsigma^*_s} * \varsigma^*_1 * \overline{\varsigma^*_1} * \varsigma^*_2 * \dots * \overline{\varsigma^*_N} * \varsigma^*_s \nonumber                                                   \\
		                       & \simeq_p \varsigma^*_s.
	\end{align}
	Thus, $\gamma$ satisfies condition (\ref{eqProbTMV3}). In summary, the proposition is proven.
\end{proof}

\subsection*{3.10 Proof of \textbf{Lemma~\ref{th_optTMVsubPath}}}
\begin{proof}
	We prove this by contradiction. Assume that for the optimal solution $\gamma^*$, there exists a subpath $\sigma_{i,j}\neq \Theta\left(\sigma_{i,j}\right)$, i.e., there exists $\sigma'_{i,j}\in[\sigma_{i,j}]_{\simeq_p}$ such that $c(\sigma'_{i,j})<c(\sigma_{i,j})$. Then, we can construct a new solution $\gamma'$, defined as:
	\begin{align}
		\gamma' = \sigma_{s,1} * \sigma_{1,2} * \dots * \sigma'_{i,j} * \dots + \sigma_{N,s}.
	\end{align}
	The following inequality can be obtained:
	\begin{align}
		c(\gamma^*) & = c(\sigma_{s,1}) + \dots + c(\sigma_{i,j}) + \dots + c(\sigma_{N,s}) \nonumber  \\
		            & > c(\sigma_{s,1}) + \dots + c(\sigma'_{i,j}) + \dots + c(\sigma_{N,s}) \nonumber \\
		            & = c(\gamma').
	\end{align}
	This indicates that the cost of $\gamma'$ is less than that of $\gamma^*$, contradicting the assumption that $\gamma^*$ is the optimal solution. Therefore, the original proposition holds.
\end{proof}

\subsection*{3.11 Proof of \textbf{Theorem~\ref{th_optTMVPath}}}
\begin{proof}
	Consider the dynamic process where the tethered robot starts from the initial configuration $\varsigma_s$ and sequentially visits $N$ targets along the subpaths of $\gamma^*$ as shown in (\ref{eqProbTMV1}). According to (\ref{eqProbTMV2}), the configurations of the tethered robot upon reaching $x_1,x_2,\dots,x_N$ exist. Let $\varsigma^*_{x_1} \in \mathcal{C}^{*,x_1}_{x_\star ,\zeta}$, $\varsigma^*_{x_2} \in \mathcal{C}^{*,x_2}_{x_\star ,\zeta}$,$\dots$, $\varsigma^*_{x_N} \in \mathcal{C}^{*,x_N}_{x_\star ,\zeta}$ denote the configurations of the tethered robot when it reaches these targets, respectively. These configurations satisfy the following relationships:
	\begin{align}
		\label{eq_optTMVPath1}
		 & \varsigma^*_{x_1} \simeq_p \varsigma_s * \sigma_{s,1}, \nonumber                \\
		 & \varsigma^*_{x_2} \simeq_p \varsigma_s * \sigma_{s,1} * \sigma_{1,2}, \nonumber \\
		 & \ \vdots  \nonumber                                                             \\
		 & \varsigma^*_{x_N} \simeq_p \varsigma_s * \sigma_{s,1} * \dots * \sigma_{N-1,N}.
	\end{align}
	Thus, there are:
	\begin{align}
		\label{eq_optTMVPath2}
		 & \sigma_{s,1} \simeq_p \overline{\varsigma^*_s} * \varsigma^*_{x_1}, \nonumber     \\
		 & \sigma_{1,2} \simeq_p \overline{\varsigma^*_{x_1}} * \varsigma^*_{x_2}, \nonumber \\
		 & \ \vdots  \nonumber                                                               \\
		 & \sigma_{N-1,N} \simeq_p \overline{\varsigma^*_{x_{N-1}}} * \varsigma^*_{x_N}.
	\end{align}
	Furthermore:
	\begin{align}
		\label{eq_optTMVPath3}
		\varsigma_s * \gamma^* & = \varsigma_s * \sigma_{s,1} * \sigma_{1,2} * \dots * \sigma_{N-1,N} * \sigma_{N,s} \nonumber                                                                                         \\
		                       & \simeq_p \varsigma_s * \left(\overline{\varsigma^*_s} * \varsigma^*_{x_1}\right) * \dots * \left(\overline{\varsigma^*_{x_{N-1}}} * \varsigma^*_{x_N}\right) * \sigma_{N,s} \nonumber \\
		                       & \simeq_p \varsigma^*_{x_N} * \sigma_{N,s}.
	\end{align}
	According to (\ref{eqProbTMV3}), after the tethered robot completes its motion along $\gamma^*$, the final configuration should be homotopic to the initial configuration, i.e., $\varsigma_s * \gamma^* \simeq_p \varsigma_s$. Therefore:
	\begin{align}
		\label{eq_optTMVPath4}
		\varsigma_s * \gamma^* \simeq_p \varsigma^*_{x_N} * \sigma_{N,s} \simeq_p \varsigma_s \Longleftrightarrow \sigma_{N,s} \simeq_p \overline{\varsigma^*_{x_N}} * \varsigma_s.
	\end{align}
	From \textbf{Lemma~\ref{th_optTMVsubPath}}, (\ref{eq_optTMVPath2}), and (\ref{eq_optTMVPath4}), we obtain:
	\begin{align}
		\label{eq_optTMVPath5}
		 & \sigma_{s,1} = \Theta\left(\overline{\varsigma^*_s} * \varsigma^*_{x_1}\right), \nonumber           \\
		 & \sigma_{1,2} = \Theta\left(\overline{\varsigma^*_{x_1}} * \varsigma^*_{x_2}\right), \nonumber       \\
		 & \ \vdots  \nonumber                                                                                 \\
		 & \sigma_{N-1,N} = \Theta\left(\overline{\varsigma^*_{x_{N-1}}} * \varsigma^*_{x_N}\right), \nonumber \\
		 & \sigma_{N,s} = \Theta\left(\overline{\varsigma^*_{x_N}} * \varsigma_s\right).
	\end{align}
	Thus:
	\begin{align}
		\gamma^* = \Theta\left(\overline{\varsigma^*_s} * \varsigma^*_{x_1}\right) * \Theta\left(\overline{\varsigma^*_{x_1}} * \varsigma^*_{x_2}\right) * \dots * \Theta\left(\overline{\varsigma^*_{x_N}} * \varsigma^*_s\right).
	\end{align}
	In summary, $\varsigma^*_{x_1},\varsigma^*_{x_2},\dots,\varsigma^*_{x_N}$ correspond to the configurations $\varsigma^*_1,\varsigma^*_2,\dots,\varsigma^*_N$ in the proposition that satisfy (\ref{eq_optTMVPath}).
\end{proof}

\subsection*{3.12 Proof of \textbf{Theorem~\ref{th_TMVGpw}}}
\begin{proof}
	According to \textbf{Remark~\ref{re_nodeChange}}, for any directed edge $\varsigma^*_k$ in $\mathcal{G}_{TMV}$ and its connected nodes $\gamma_{k-1}$ and $\gamma_k$, the cost of the directed edge $\varsigma^*_k$ is given by:
	\begin{align}
		{\tt{EdgeCost}}(\varsigma^*_k)
		 & = c(\gamma_{k}) - c(\gamma_{k-1}) \nonumber                                                                                                             \\
		 & = c\left(\Theta\left(\overline{\varsigma^*_{k-1}} * \varsigma^*_k\right) * \Theta\left(\overline{\varsigma^*_k} * \varsigma^*_s\right)\right) \nonumber \\
		 & \quad - c\left(\Theta\left(\overline{\varsigma^*_{k-1}} * \varsigma^*_s\right)\right).
	\end{align}
	Additionally, we have:
	\begin{align}
		\Theta\left(\overline{\varsigma^*_{k-1}} * \varsigma^*_k\right) * \Theta\left(\overline{\varsigma^*_k} * \varsigma^*_s\right) & \simeq_p \overline{\varsigma^*_{k-1}} * \varsigma^*_k * \overline{\varsigma^*_k} * \varsigma^*_s \nonumber \\
		                                                                                                                              & \simeq_p \overline{\varsigma^*_{k-1}} * \varsigma^*_s \nonumber                                            \\
		                                                                                                                              & \simeq_p \Theta\left(\overline{\varsigma^*_{k-1}} * \varsigma^*_s\right).
	\end{align}
	By the definition of $\Theta$, $\Theta\left(\overline{\varsigma^*_{k-1}} * \varsigma^*_s\right)$ is the optimal path in its homotopy class. Therefore, we have:
	\begin{align}
		 & c\left(\Theta\left(\overline{\varsigma^*_{k-1}} * \varsigma^*_k\right) * \Theta\left(\overline{\varsigma^*_k} * \varsigma^*_s\right)\right) \geq c\left(\Theta\left(\overline{\varsigma^*_{k-1}} * \varsigma^*_s\right)\right) \nonumber \\
		 & \Longleftrightarrow \ {\tt{EdgeCost}}(\varsigma^*_k) \geq 0 .
	\end{align}
	Since $\varsigma^*_k$ is an arbitrary directed edge in $\mathcal{G}_{TMV}$, it follows that all edge costs in $\mathcal{G}_{TMV}$ are non-negative. Hence, $\mathcal{G}_{TMV}$ is a positive weight graph.
\end{proof}

\subsection*{3.13 Proof of \textbf{Theorem~\ref{th_optPath}}}
\begin{proof}
	\textbf{Theorem~\ref{th_optPath}} will be proved by contradiction.	Assume there exists an optimal path $\sigma^\circledast$ such that the sequence $\Gamma^*_{\varrho}\circ\sigma^\circledast$ contains repeated convex polygons.

	As illustrated in Fig.~\ref{fig_THoptPath}, let $\mathbf{x}$ denote a repeated convex polygon in $\Gamma^*_{\varrho}\circ\sigma^\circledast$. According to the definition of $\Gamma^*_{\varrho}$, the path $\sigma^\circledast$ must pass through the convex polygon $\mathbf{x}$ repeatedly. Let $x_1$ and $x_2$ denote the intersection points of $\sigma^\circledast$ with the cutlines when it enters $\mathbf{x}$ for the first and second time, respectively. Thus, $\sigma^\circledast$ can be expressed as:
	\begin{equation}
		\sigma^\circledast = \sigma^\circledast_1 * \sigma^\circledast_2 * \sigma^\circledast_3,
	\end{equation}
	where $\sigma^\circledast_1$ represents the subpath of $\sigma^\circledast$ from $x_s$ to $x_1$, $\sigma^\circledast_2$ represents the subpath of $\sigma^\circledast$ from $x_1$ to $x_2$, $\sigma^\circledast_3$ represents the subpath of $\sigma^\circledast$ from $x_2$ to $x_g$. Since $x_1$ and $x_2$ belong to the same convex polygon $\mathbf{x}$, the straight-line path $l_{x_1}^{x_2}$ exists. Therefore, a new path $\sigma_{new} \in P(X_{free};x_s,x_g)$ can be constructed as:
	\begin{equation}
		\sigma_{new} = \sigma^\circledast_1 * l_{x_1}^{x_2} * \sigma^\circledast_3.
	\end{equation}
	Since $l_{x_1}^{x_2}$ is the optimal path from $x_1$ to $x_2$, it follows that:
	\begin{align}
		c(\sigma_{new}) & = c(\sigma^\circledast_1) + c(l_{x_1}^{x_2}) + c(\sigma^\circledast_3) \nonumber        \\
		                & < c(\sigma^\circledast_1) + c(\sigma^\circledast_2) + c(\sigma^\circledast_3) \nonumber \\
		                & = c(\sigma^\circledast).
	\end{align}
	Thus, a path shorter than $\sigma^\circledast$ has been found, which contradicts the assumption that $\sigma^\circledast$ is optimal. Therefore, the assumption is false, and the original proposition holds.
\end{proof}

\subsection*{3.14 Proof of \textbf{Theorem~\ref{th_prUTPP}}}
\begin{figure*}[!t]
	\centering
	\subfloat[]{\includegraphics[height=1.3in]{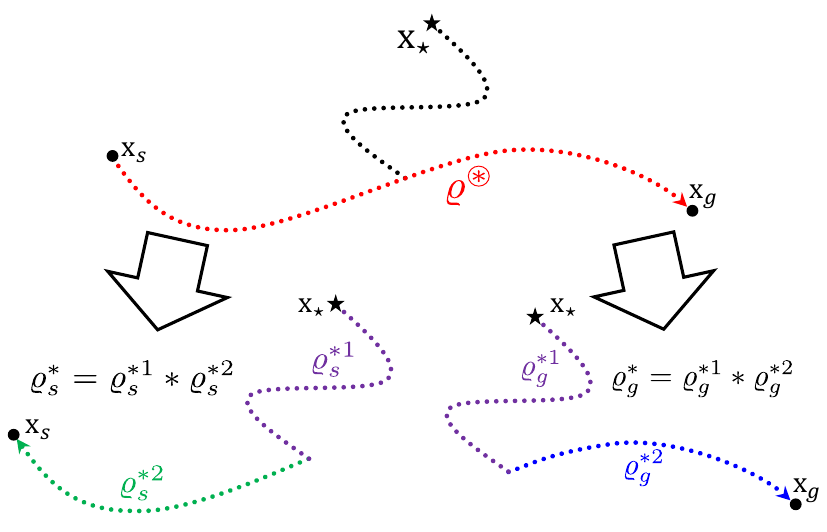}}
	\hfil
	\subfloat[]{\includegraphics[height=1.3in]{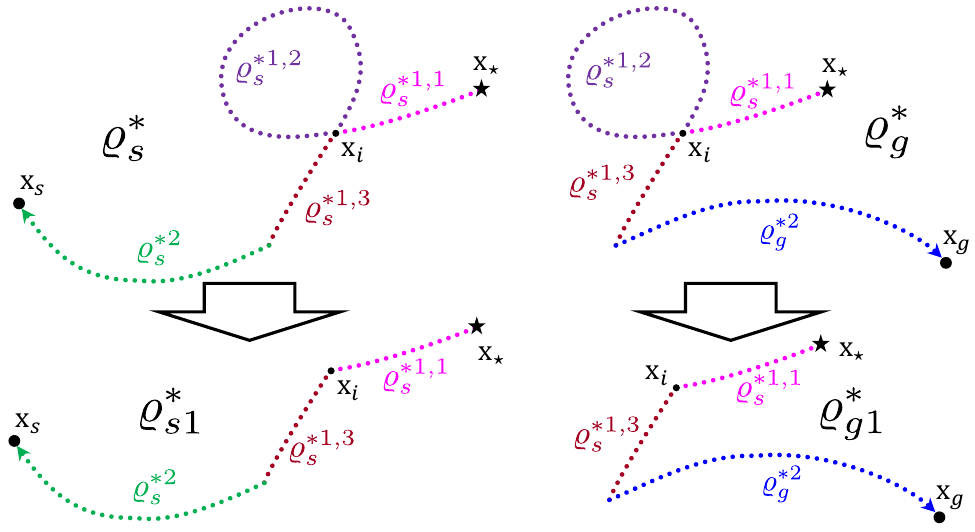}}
	\hfil
	\subfloat[]{\includegraphics[height=1.3in]{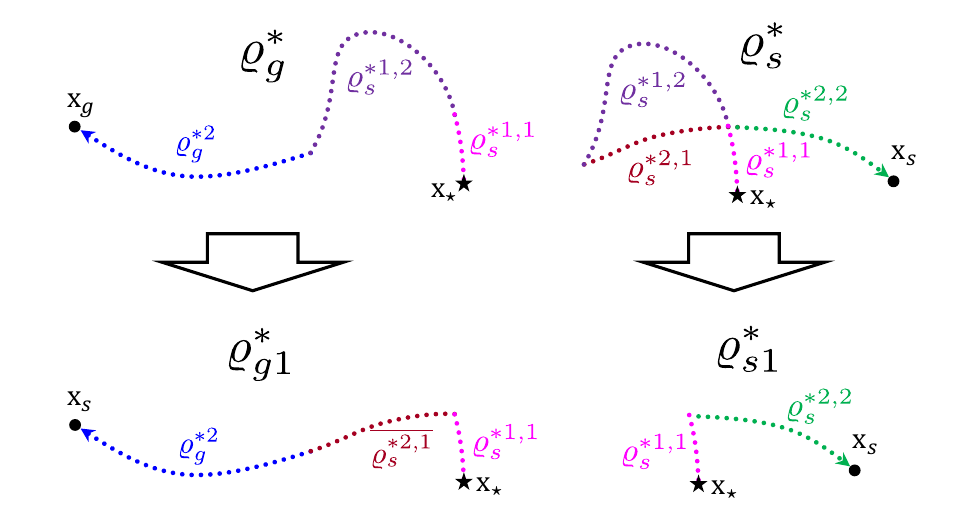}}
	\caption{Illustration involved in the proof of \textbf{Theorem~\ref{th_prUTPP}}.}
	\label{fig_OptNRCP}
\end{figure*}
\begin{proof}\mbox{}\newline
	\textbf{Sub-proposition 1}: There exist $\sigma^*_s$ and $\sigma^*_g$ satisfying (\ref{eq_prUTPP0}).\\
	\textbf{Sub-proposition 2}: For any pair of $\sigma^*_s$ and $\sigma^*_g$ satisfying (\ref{eq_prUTPP0}), they can be transformed into a new pair of optimal homotopic paths $\sigma^*_{sk}$ and $\sigma^*_{gk}$, which still satisfy (\ref{eq_prUTPP0}), while neither $\Gamma^*_{\varrho}\circ\sigma^*_{sk}$ nor $\Gamma^*_{\varrho}\circ\sigma^*_{gk}$ contains repeated elements.

	\textit{Proof of Sub-proposition 1.} Since $\sigma^\circledast$ and $x_\star$ can be connected in $X_{free}$, $P(X_{free};x_\star,x_s)$ is non-empty. Let $\sigma^*_s$ be an arbitrary optimal homotopic path in $P(X_{free};x_\star,x_s)$. Then, a corresponding $\sigma^*_g\in P(X_{free};x_\star,x_g)$ can be constructed such that (\ref{eq_prUTPP0}) holds, i.e.,
	\begin{equation}
		\sigma^*_g = \Theta\left(\sigma^*_s * \sigma^\circledast\right).
	\end{equation}

	\textit{Proof of Sub-proposition 2.} Since $\sigma^\circledast \simeq_p \overline{\sigma^*_s} * \sigma^*_g$, based on (\ref{eqCDTencoding}), (\ref{eqMAPTstarProd2}), and (\ref{eqMAPTstarProd}), we have:
	\begin{align}
		                    & \Gamma^*\circ \sigma^\circledast = (\Gamma^*\circ \overline{\sigma^*_s} ) * (\Gamma^*\circ \sigma^*_g) \nonumber                                          \\
		\Longleftrightarrow & \left\langle x_s, \varrho^\circledast, x_g\right\rangle = \left\langle x_s, {\tt{RBF}}\left(\varrho^*_s * \varrho^*_g\right), x_g \right\rangle \nonumber \\
		\Longrightarrow     & \varrho^\circledast = {\tt{RBF}}\left(\varrho^*_s * \varrho^*_g\right),
	\end{align}
	where $\varrho^\circledast = \Gamma^*_{\varrho}\circ \sigma^\circledast$, $\varrho^*_s = \Gamma^*_{\varrho}\circ \sigma^*_s$, and $\varrho^*_g = \Gamma^*_{\varrho}\circ \sigma^*_g$. As illustrated in Fig.~\ref{fig_OptNRCP}(a), and referring to the definition of ${\tt{RBF}}$, the sequences $\varrho^*_s$ and $\varrho^*_g$ can be expressed in the following segmented forms:
	\begin{align}
		 & \varrho^*_s = \varrho^{*1}_s * \varrho^{*2}_s, \\
		 & \varrho^*_g = \varrho^{*1}_g * \varrho^{*2}_g,
	\end{align}
	where $\varrho^{*1}_s$ and $\varrho^{*1}_g$ are the subsequences that form the rollback paths in $\varrho^*_s * \varrho^*_g$. Specifically, the subpaths satisfy the following relationships:
	\begin{align}
		\varrho^\circledast & = {\tt{RBF}}\left(\overline{\varrho^*_s}*\varrho^*_g\right) \nonumber                                                        \\
		                    & = {\tt{RBF}}\left(\overline{(\varrho^{*1}_s * \varrho^{*2}_s)} * (\varrho^{*1}_g * \varrho^{*2}_g)\right)  \nonumber         \\
		                    & = {\tt{RBF}}\left(\overline{\varrho^{*2}_s} * (\overline{\varrho^{*1}_s} * \varrho^{*1}_g) * \varrho^{*2}_g\right) \nonumber \\
		                    & = \overline{\varrho^{*2}_s} * \varrho^{*2}_g.
	\end{align}
	Additionally, it is important to note that, by the definition of $\Gamma^*_{\varrho}$, the sequences $\varrho^*_s$ and $\varrho^*_g$ contain no rollback paths. Therefore, their subsequences $\varrho^{*1}_s$ and $\varrho^{*1}_g$ also contain no rollback paths. Consequently, in the rollback path $\overline{\varrho^{*1}_s}*\varrho^{*1}_g$, $\overline{\varrho^{*1}_s}$ and $\varrho^{*1}_g$ are mutually inverse, which implies:
	\begin{equation}
		\varrho^{*1}_s = \varrho^{*1}_g.
	\end{equation}
	Assume that $\varrho^*_s$ and $\varrho^*_g$ each contain several repeated elements. According to \textbf{Theorem~\ref{th_optPath}}, since $\sigma^\circledast$ is the globally optimal path, $\overline{\varrho^{*2}_s} * \varrho^{*2}_g$ contains no repeated elements. Therefore, the subsequences $\varrho^{*2}_s$ and $\varrho^{*2}_g$ also contain no repeated elements. Consequently, the repeated elements in $\varrho^*_s$ (or $\varrho^*_g$) can only occur in the following two scenarios:
	\begin{enumerate}
		\item[1)] {The subsequence $\varrho^{*1}_s$ (or $\varrho^{*1}_g$) contains repeated elements within itself.}
		\item[2)] {There are repeated elements between the subsequences $\varrho^{*1}_s$ and $\varrho^{*2}_s$ (or between $\varrho^{*1}_g$ and $\varrho^{*2}_g$).}
	\end{enumerate}

	We now describe a process to transform $\varrho^*_s$ into a sequence without repeated elements.

	First, as illustrated in Fig.~\ref{fig_OptNRCP}(b), consider the case where the subsequence $\varrho^{*1}_s$ contains repeated elements within itself (this can be interpreted as the sequence $\varrho^{*1}_s$ intersecting with itself). Suppose $\varrho^{*1}_s$ intersects with itself at $\mathbf{x}_i$. Then, $\varrho^{*1}_s$ and $\varrho^{*1}_g$ can be divided into three subsegments:
	\begin{equation}
		\varrho^{*1}_g = \varrho^{*1}_s = \varrho^{*1,1}_s * \varrho^{*1,2}_s * \varrho^{*1,3}_s,
	\end{equation}
	where 
	\begin{align}
	\varrho^{*1,1}_s(T_{\varrho^{*1,1}_s}) = \varrho^{*1,2}_s(0) = \varrho^{*1,2}_s(T_{\varrho^{*1,2}_s }) = \varrho^{*1,3}_s(0) = \mathbf{x}_i.\nonumber
	\end{align}
	Furthermore, two new sequences $\varrho^*_{s1}\in P(\mathcal{G}_{con};\mathbf{x}_\star,\mathbf{x}_s)$ and $\varrho^*_{g1}\in P(\mathcal{G}_{con};\mathbf{x}_\star,\mathbf{x}_g)$ can be constructed as follows:
	\begin{align}
		 & \varrho^*_{s1} = \varrho^{*1,1}_s * \varrho^{*1,3}_s * \varrho^{*2}_s, \\
		 & \varrho^*_{g1} = \varrho^{*1,1}_s * \varrho^{*1,3}_s * \varrho^{*2}_g.
	\end{align}
	Clearly, since $\varrho^*_{s1}$ removes the subsequence $\varrho^{*1,2}_s$ compared to $\varrho^*_s$, it eliminates the repeated element $\mathbf{x}_i$ without introducing new repeated elements. (The same applies to $\varrho^*_{g1}$ compared to $\varrho^*_g$.) Further, the verification of ${\tt{RBF}}\left(\overline{\varrho^*_{s1}}*\varrho^*_{g1}\right) = \varrho^\circledast$ is as follow:
	\begin{align}
		  & {\tt{RBF}}\left(\overline{\varrho^*_{s1}}*\varrho^*_{g1}\right) \nonumber                                                                                                                   \\
		= & {\tt{RBF}}\left(\overline{\left(\varrho^{*1,1}_s * \varrho^{*1,3}_s * \varrho^{*2}_s\right) } * \varrho^{*1,1}_s * \varrho^{*1,3}_s * \varrho^{*2}_g\right) \nonumber                       \\
		= & {\tt{RBF}}\left(\overline{\varrho^{*2}_s} *\left( \overline{\varrho^{*1,3}_s} * \overline{\varrho^{*1,1}_s} * \varrho^{*1,1}_s * \varrho^{*1,3}_s\right)  * \varrho^{*2}_g\right) \nonumber \\
		= & \overline{\varrho^{*2}_s} * \varrho^{*2}_g = \varrho^\circledast.
	\end{align}

	As illustrated in Fig.~\ref{fig_OptNRCP}(c). For the case where repeated elements exist between the subsequences $\varrho^{*1}_s$ and $\varrho^{*2}_s$, assume that $\varrho^{*1}_s$ and $\varrho^{*2}_s$ intersect at $\mathbf{x}_i$. Therefore, $\varrho^{*1}_s$, $\varrho^{*1}_g$, and $\varrho^{*2}_s$ can be divided into two subsegments:
	\begin{align}
		 & \varrho^{*1}_g = \varrho^{*1}_s = \varrho^{*1,1}_s * \varrho^{*1,2}_s, \\
		 & \varrho^{*2}_s = \varrho^{*2,1}_s * \varrho^{*2,2}_s.
	\end{align}
	Furthermore, two new sequences $\varrho^*_{s1}\in P(\mathcal{G}_{con};\mathbf{x}_\star,\mathbf{x}_s)$ and $\varrho^*_{g1}\in P(\mathcal{G}_{con};\mathbf{x}_\star,\mathbf{x}_g)$ can be constructed as follows:
	\begin{align}
		 & \varrho^*_{s1} = \varrho^{*1,1}_s * \varrho^{*2,2}_s,                              \\
		 & \varrho^*_{g1} = \varrho^{*1,1}_s * \overline{\varrho^{*2,1}_s}  * \varrho^{*2}_g.
	\end{align}
	Clearly, this construction ensures that $\varrho^*_{s1}$ and $\varrho^*_{g1}$ eliminate the repeated element $\mathbf{x}_i$ from the original sequences $\varrho^*_s$ and $\varrho^*_g$ without introducing new repeated elements. Further, the verification of ${\tt{RBF}}\left(\overline{\varrho^*_{s1}}*\varrho^*_{g1}\right) = \varrho^\circledast$ is as follow:
	\begin{align}
		  & {\tt{RBF}}\left(\overline{\varrho^*_{s1}}*\varrho^*_{g1}\right) \nonumber                                                                                                  \\
		= & {\tt{RBF}}\left(\overline{\left(\varrho^{*1,1}_s * \varrho^{*2,2}_s\right)} * \varrho^{*1,1}_s * \overline{\varrho^{*2,1}_s}  * \varrho^{*2}_g\right) \nonumber            \\
		= & {\tt{RBF}}\left(\overline{\varrho^{*2,2}_s} * \left(\overline{\varrho^{*1,1}_s} * \varrho^{*1,1}_s\right)  * \overline{\varrho^{*2,1}_s} * \varrho^{*2}_g\right) \nonumber \\
		= & \overline{\varrho^{*2,2}_s} * \overline{\varrho^{*2,1}_s} * \varrho^{*2}_g = \overline{\varrho^{*2,1}_s * \varrho^{*2,2}_s} * \varrho^{*2}_g \nonumber                     \\
		= & \overline{\varrho^{*2}_s} * \varrho^{*2}_g = \varrho^\circledast.
	\end{align}

	Using the two methods described above, we can construct a recursive process to eventually obtain two sequences $\varrho^*_{sk}$ and $\varrho^*_{gk}$ that contain no repeated elements and satisfy ${\tt{RBF}}\left(\overline{\varrho^*_{sk}}*\varrho^*_{gk}\right) = \varrho^\circledast$. Furthermore, based on $\varrho^*_{sk}$ and $\varrho^*_{gk}$, we can construct $\sigma^*_{sk}$ and $\sigma^*_{gk}$ that meet the requirements of \textbf{Sub-proposition 2}, i.e.,
	\begin{align}
		\sigma^*_{sk} = \overline{\Gamma^*} \circ \langle x_\star, \varrho^*_{sk}, x_s\rangle, \\
		\sigma^*_{gk} = \overline{\Gamma^*} \circ \langle x_\star, \varrho^*_{sk}, x_g\rangle.
	\end{align}
	In conclusion, \textbf{Sub-proposition 2} holds, and therefore, \textbf{Theorem~\ref{th_prUTPP}} is proven.
\end{proof}

\end{document}